\documentclass[english]{article}
\usepackage[OT1]{fontenc}
\usepackage[latin9]{inputenc}
\usepackage{babel}
\usepackage{color}
\usepackage{float}
\usepackage{placeins}
\usepackage{booktabs}
\usepackage{mathtools}
\usepackage{enumitem}
\usepackage{multirow}
\usepackage{dsfont}
\usepackage{amsmath}
\usepackage{amsthm}
\usepackage{amssymb}
\usepackage{geometry}
\usepackage{setspace}
\usepackage{makecell}
\usepackage[authoryear,round]{natbib}
\usepackage[pdfusetitle,
 bookmarks=true,bookmarksnumbered=false,bookmarksopen=false,
 breaklinks=false,pdfborder={0 0 1},backref=false,colorlinks=true]
 {hyperref} 
\hypersetup{
 linkcolor=red,citecolor=blue,urlcolor=magenta,filecolor=cyan}

\makeatletter

\providecommand{\tabularnewline}{\\}
\floatstyle{ruled}
\newfloat{algorithm}{tbp}{loa}
\providecommand{\algorithmname}{Algorithm}
\floatname{algorithm}{\protect\algorithmname}

\@ifundefined{date}{}{\date{}}
\usepackage{babel}

\usepackage{array}
\usepackage{tabularx}
\usepackage{bm}
\usepackage{multirow}
\usepackage{graphicx}
\usepackage{tikz}
\usetikzlibrary{calc}

\providecommand{\tabularnewline}{\\}

\allowdisplaybreaks

\usepackage{algpseudocode}

\usepackage{amsthm}

\theoremstyle{plain}

\newtheorem{lemma}{\textbf{Lemma}}\newtheorem{theorem}{\textbf{Theorem}}
\newtheorem{corollary}{\textbf{Corollary}}\newtheorem{assumption}{\textbf{Assumption}}\newtheorem{proposition}{\textbf{Proposition}}

\theoremstyle{definition}
\newtheorem{remark}{\textbf{Remark}}

\usepackage{dsfont}

\def\cA{\mathcal{A}}

\def\cC{\mathcal{C}}

\def\cE{\mathcal{E}}

\def\cM{\mathcal{M}}

\def\cP{\mathcal{P}}

\def\cS{\mathcal{S}}
\def\cT{\mathcal{T}}
\def\cU{\mathcal{U}}

\def\Var{\mathrm{Var}}

\usepackage{pgfplots}
\pgfplotsset{compat=1.18}

\DeclareRobustCommand{\suppref}[2]{%
    \ref{#1}%
}

\DeclareRobustCommand{\suppplainref}[2]{%
    \ref{#1}%
}

\DeclareRobustCommand{\mainref}[2]{%
    \ref{#1}%
}

\DeclareRobustCommand{\mainequationref}[2]{%
    \eqref{#1}%
}

\DeclareRobustCommand{\suppExperimentDetails}{%
    \suppref{app:experiment-details}{F}}
\DeclareRobustCommand{\suppMinimumSpanComplementary}{%
    \suppref{app:minimum-span-complementary}{F.1}}
\DeclareRobustCommand{\suppReductionUpperBounds}{%
    \suppref{app:reduction-upper-bounds}{D}}
\DeclareRobustCommand{\suppSpanAgnosticHorizonCalibration}{%
    \suppref{app:span-agnostic-horizon-calibration}{E}}
\DeclareRobustCommand{\suppRectangularDiscountedBellmanRepresentation}{%
    \suppref{lem:rectangular-discounted-bellman-representation}{2}}
\DeclareRobustCommand{\suppPerturbationBoundProof}{%
    \suppref{proof:perturbation-bound-average-reward}{B.2}}
\DeclareRobustCommand{\suppMinimumSpanAttainment}{%
    \suppplainref{prop:minimum-span-attainment}{4}}
\DeclareRobustCommand{\suppRobustBellmanExistence}{%
    \suppplainref{prop:robust-bellman-existence}{3}}
\DeclareRobustCommand{\suppLowerBoundProofs}{%
    \suppref{sec:proof_lower_bound}{C}}
\DeclareRobustCommand{\suppAnchoredCalibrationSpecification}{%
    \suppref{subsec:anchored-calibration-computable-quantities}{E.1}}
\DeclareRobustCommand{\suppRochSpanComparison}{%
    \suppref{subsec:comparison-roch-span}{A.4}}
\DeclareRobustCommand{\suppSpanIndependenceProof}{%
    \suppref{subsec:proof-span-independence}{B.1}}
\DeclareRobustCommand{\suppRobustBellmanExistenceAttainment}{%
    \suppref{subsec:robust-bellman-existence-attainment}{A.2}}

\DeclareRobustCommand{\mainAnchoredSpanAgnosticAlgorithm}{%
    \mainref{alg:anchored-span-agnostic-robust-amdp}{2}}
\DeclareRobustCommand{\mainSpanInformedAlgorithm}{%
    \mainref{alg:robust-amdp}{1}}
\DeclareRobustCommand{\mainUnichainAssumption}{%
    \mainref{assump:unichain}{1}}
\DeclareRobustCommand{\mainAnchoredNominalCandidateScore}{%
    \mainequationref{eq:anchored-nominal-candidate-score}{7}}
\DeclareRobustCommand{\mainAnchoredRobustCandidateScore}{%
    \mainequationref{eq:anchored-robust-candidate-score}{8}}
\DeclareRobustCommand{\mainNominalSpanDefinition}{%
    \mainequationref{eq:nominal_span}{3a}}
\DeclareRobustCommand{\mainNominalAnchorSupersolution}{%
    \mainequationref{eq:nominal-anchor-supersolution}{6}}

\DeclareRobustCommand{\mainRobustSpanDefinition}{%
    \mainequationref{eq:robust_span}{3b}}
\DeclareRobustCommand{\mainRobustAverageRewardBellmanEquation}{%
    \mainequationref{eq:robust-average-reward-bellman}{2b}}
\DeclareRobustCommand{\mainRobustDMDPSolverTolerance}{%
    \mainequationref{eq:robust-dmdp-solver-tolerance}{5}}
\DeclareRobustCommand{\mainComponentRateFigure}{%
    \mainref{fig:component-rate-checks}{1}}
\DeclareRobustCommand{\mainSelectorAdaptivityFigure}{%
    \mainref{fig:selector-adaptivity}{2}}
\DeclareRobustCommand{\mainComponentRatePanel}[1]{%
    \hyperref[fig:component-rate-checks]{\ref*{fig:component-rate-checks}#1}%
}
\DeclareRobustCommand{\mainSelectorAdaptivityPanel}[1]{%
    \hyperref[fig:selector-adaptivity]{\ref*{fig:selector-adaptivity}#1}%
}
\DeclareRobustCommand{\mainPerturbationBoundProposition}{%
    \mainref{prop:perturbation-bound-average-reward}{2}}
\DeclareRobustCommand{\mainSpanIndependenceProposition}{%
    \mainref{prop:span-independence}{1}}
\DeclareRobustCommand{\mainExperimentsSection}{%
    \mainref{sec:experiments}{5}}
\DeclareRobustCommand{\mainProblemSetupSection}{%
    \mainref{sec:problem-setup}{2}}
\DeclareRobustCommand{\mainReductionUpperBoundSection}{%
    \mainref{sec:reduction-upper-bound}{4}}
\DeclareRobustCommand{\mainSpanAgnosticAdaptationSection}{%
    \mainref{sec:span-agnostic-adaptation}{5.2}}
\DeclareRobustCommand{\mainSpanAgnosticUpperBoundTheorem}{%
    \mainref{thm:anchored-calibrated-policy-upper-bound}{3}}
\DeclareRobustCommand{\mainMinimaxLowerBoundTheorem}{%
    \mainref{thm:minimax-lower}{1}}
\DeclareRobustCommand{\mainSpanInformedUpperBoundTheorem}{%
    \mainref{thm:TV-upper-bound-average}{2}}

\makeatother

\title{Robust Average-Reward Markov Decision Processes: \\ Minimax-Optimal Learning via Plug-in Reductions}
\author{
    Yuepeng Yang\thanks{Department of Statistics and Data Science, Yale University.}
    \thanks{Department of Statistics and Data Science, the Wharton School, University of Pennsylvania.} \\
    Yale \& Penn
    \and
    Yuxin Chen\footnotemark[2] \\
    Penn
    \and
    Yuejie Chi\footnotemark[1] \\
    Yale
    }
\date{August 5, 2026}

\begin{document}

\maketitle

\begin{abstract}

Distributionally robust Markov decision processes provide a principled
framework for sequential decision making under model uncertainty. We study how
many samples are necessary and sufficient to learn an
$\varepsilon$-optimal robust policy under the average-reward criterion. A
generative model provides samples from the nominal transition kernel, whereas
policy performance is evaluated over $(s,a)$-rectangular total-variation
uncertainty sets of radius at most $\sigma$.

Let $H_0$ and $H_\sigma$ denote the nominal and robust optimal bias spans,
respectively. We identify $\sigma H_0$ as the perturbation scale separating
high- and low-tolerance regimes. Our matching upper and lower bounds show that,
up to logarithmic factors, the minimax total sample complexity is
\[ 
    NSA\asymp\frac{SA}{\varepsilon^2}\cdot
    \begin{cases}
        \min\{H_0,H_\sigma\},
        & \varepsilon\gtrsim\sigma H_0\\[1mm]
        \min\{H_0,H_\sigma\}+\sigma H_\sigma^2,
        & \varepsilon\lesssim\sigma H_0
    \end{cases},
\]
where $S$ and $A$ are the numbers of states and actions, and $N$ is the number of samples per state-action pair. The sample complexity consists of a linear-span term that resembles the nominal AMDP results, and a robustness-specific term that appears only in the low-tolerance regime. We
attain these rates using reduction-based plug-in procedures that
select the reduction---nominal or robust---and its discount factor: a
span-informed procedure that makes these choices using known span parameters,
and a span-agnostic procedure that calibrates both choices from data.

\end{abstract}

\section{Introduction}

Reinforcement learning (RL) \citep{sutton2018reinforcement}, as a paradigm for sequential decision making under uncertainty, enables agents to learn
optimal behavior through interactions with an environment. RL has found successful
applications in domains such as robotics \citep{mnih2015human,kober2013reinforcement}, game playing \citep{silver2016mastering}, and generative AI \citep{guo2025deepseek}.
One popular model underpinning RL is a Markov decision process (MDP), where the goal of the agent is to learn a policy that maximizes 
some form of aggregated expected reward in the environment. Common aggregations
include the total reward over a finite horizon 
and the sum of discounted rewards over an infinite horizon. Despite their popularity, they may be less suitable for continuous learning tasks \citep{naik2019discounted}.
In this work, we focus on the long-term average reward:
\[
    \rho^{\pi}(s):=\lim_{T\rightarrow\infty}\mathbb{E}_{P^{0}}^{\pi}\left[\frac{1}{T}\sum_{t=0}^{T-1}r(s_{t},a_{t})\mid s_{0}=s\right],
\]
which characterizes the steady-state performance of a policy by evaluating
the expected reward gained per time step as the number of steps approaches
infinity. Here, $r(s_t, a_t)$ is the instantaneous reward received at time step $t$ when the agent selects action $a_t$ in state $s_t$ according to policy $\pi$, and the expectation is taken over the randomness of the trajectory according to the transition kernel $P^0$ and policy $\pi$, given the initial state $s_0 =s$. Unlike discounted settings that prioritize earlier rewards,
this metric seeks a policy that maximizes the consistent, long-term
gain across all states.

A significant challenge in standard RL is the reliance on a fixed
probability kernel. The policy learned in one environment may not be effective in another environment even if the change is modest \citep{ramesh2024distributionally,sinha2020formulazero}.
A popular approach to address this issue is to consider the distributionally
robust optimization (DRO) framework, where the probability distribution
of the environment is allowed to vary within a prescribed uncertainty
set instead of being fixed \citep{MohajerinEsfahani2018Data,Wiesemann2014Distributionally,Goh2010Distributionally,Duchi2021Learning}.
In the context of Markov decision processes, transition uncertainty has
long been studied through robust and distributionally robust formulations
\citep{iyengar2005robust,nilim2005robust,xu2012distributionally,wiesemann2013robust}.
More directly relevant to our setting, distributionally robust average-reward
MDPs (AMDPs) model the transition kernel as belonging to a designated
uncertainty set $\mathcal{P}$ \citep{Wang2023Robust,Wang2023Model}.
Under this framework, we seek a policy that is effective under the
worst-case scenario within that set, defined as the robust average
reward:
\[
    \rho_{\mathcal{P}}^{\pi}(s):=\min_{P\in\mathcal{P}}\lim_{T\rightarrow\infty}\mathbb{E}_{P}^{\pi}\left[\frac{1}{T}\sum_{t=0}^{T-1}r(s_{t},a_{t})\mid s_{0}=s\right].
\]
By optimizing against this most pessimistic model, we help ensure that
the agent's performance remains reliable even when the environment
dynamics are uncertain.

This paper studies the statistical cost of distributional robustness in a generative-model setting, in which an algorithm has access to
a simulator that produces $N$ independent samples from the nominal
transition kernel $P^{0}$ for every state-action pair. Meanwhile, the performance is evaluated through
the robust average reward over a rectangular total-variation uncertainty set
around $P^0$. Within this
framework, a key statistical question is:
\begin{center}
    \emph{How many samples are necessary and sufficient to}\\
    \emph{obtain a policy that is $\varepsilon$-optimal in
        robust average reward?}
    \par\end{center}

A key problem parameter for studying sample complexity in average-reward MDPs is the optimal bias span. It quantifies the dynamic range of an average-reward MDP by measuring
how much the transient reward relative
to the long-run average varies across initial states. 
It is known that nominal average-reward MDPs have minimax-optimal sample complexity $\widetilde{O}(SAH_0\varepsilon^{-2})$ \citep{wang2022near,zurek2024span}, where $H_0$ is the nominal optimal bias span.

In the robust setting, recent work has established upper bounds for this
problem. \citet{roch2025reduction} developed a reduction from robust AMDPs to
robust discounted MDPs, while \citet{roch2025provably} proposed a
parameter-free variant of robust Halpern iteration. The guarantees depend
quadratically on the respective robust span parameters $H_{\mathrm{Roch}}$ and $H_{\mathrm{RHI}}$ defined in the corresponding papers,
with sample complexities
$\widetilde{O}(SAH_{\mathrm{Roch}}^2\varepsilon^{-2})$ and
$\widetilde{O}(SAH_{\mathrm{RHI}}^2\varepsilon^{-2})$, respectively.
The contrast between the minimax theory for nominal AMDPs and the available robust upper bounds leaves our main statistical question largely unresolved. 

\subsection{Our contributions}

We characterize the minimax sample complexity and develop reduction-based procedures that attain it. For \((s,a)\)-rectangular total-variation uncertainty sets of radius at most \(\sigma\), our matching upper and lower bounds characterize
the minimax sample complexity, up to logarithmic factors, as
\begin{equation}\label{eq:sample-complexity}
    NSA\asymp\frac{SA}{\varepsilon^2}
    \begin{cases}
        \min\{H_0,H_\sigma\},
        & \varepsilon\gtrsim\sigma H_0,\\[1mm]
        \min\{H_0,H_\sigma\}+\sigma H_\sigma^2,
        & \varepsilon\lesssim\sigma H_0.
    \end{cases}    
\end{equation}
Here, $H_0$ and $H_\sigma$ are the nominal and robust span parameters, respectively. Both parameters are at least 1, and neither controls the other in general:
Proposition~\ref{prop:span-independence} shows that, for any fixed
$\sigma>0$, any prescribed pair $(H_0,H_\sigma)$ can be realized.

\paragraph*{High- and low-tolerance regimes.}
In Proposition~\ref{prop:perturbation-bound-average-reward}, we show that the optimal robust average reward is at
most $\sigma H_0$ smaller than the nominal optimal average reward. 
Comparing this perturbation scale with the target tolerance $\varepsilon$
separates two regimes of robust learning.
In the
\emph{high-tolerance regime} $\sigma H_0\lesssim\varepsilon$, solving the nominal AMDP achieves robust
$\varepsilon$-optimality at the standard rate
$\widetilde O(SAH_0\varepsilon^{-2})$. Moreover, using the robust reduction when $H_\sigma \le H_0$ improves it to $\widetilde O(SAH_\sigma\varepsilon^{-2})$. In the
\emph{low-tolerance regime}
$\varepsilon\lesssim\sigma H_0$, however, the perturbation must be accounted for. Table~\ref{tab:four-regime-summary}
summarizes these regimes under the two orderings of $H_0$ and $H_\sigma$.

\begin{table}
    \caption{Four-regime summary of the
    minimax-optimal
    sample complexity, up to logarithmic factors.
    The dashed outline marks the regimes where the
    robust reduction is used.}
    \label{tab:four-regime-summary}
    \centering
    \small
    \renewcommand{\arraystretch}{1.25}
    \begin{tabularx}{\linewidth}{
        @{}
        >{\centering\arraybackslash}p{0.18\linewidth}
        >{\centering\arraybackslash}X
        @{\hspace{0.8em}}
        >{\centering\arraybackslash}X
        @{}
    }
        \toprule
        \textbf{Span ordering}
        &
        \textbf{High tolerance}
        $\sigma H_0\lesssim\varepsilon$
        &
        \textbf{Low tolerance}
        $\varepsilon\lesssim\sigma H_0$
        \\
        \midrule
        \addlinespace[0.8em]
        $H_\sigma\le H_0$
        &
        \tikz[remember picture,baseline=(robust-top-left.base)]{
            \node[inner sep=0pt] (robust-top-left)
            {$\dfrac{SAH_\sigma}{\varepsilon^2}$};
        }
        &
        \tikz[remember picture,baseline=(robust-top-right.base)]{
            \node[inner sep=0pt] (robust-top-right)
            {$\dfrac{SA(H_\sigma+\sigma H_\sigma^2)}{\varepsilon^2}$};
        }
        \\
        \addlinespace[0.8em]
        $H_0<H_\sigma$
        &
        \tikz[remember picture,baseline=(nominal-bottom-left.base)]{
            \node[inner sep=0pt] (nominal-bottom-left)
            {$\dfrac{SAH_0}{\varepsilon^2}$};
        }
        &
        \tikz[remember picture,baseline=(robust-bottom-right.base)]{
            \node[inner sep=0pt] (robust-bottom-right)
            {$\dfrac{SA(H_0+\sigma H_\sigma^2)}{\varepsilon^2}$};
        }
        \\
        \addlinespace[0.8em]
        \bottomrule
    \end{tabularx}
    \begin{tikzpicture}[remember picture,overlay]
        \coordinate (robust-left) at
            ($(robust-top-left)!-0.5!(robust-top-right)$);
        \coordinate (robust-right) at
            ($(robust-top-right)!-0.5!(robust-top-left)$);
        \coordinate (robust-top) at
            ($(robust-top-left)!-0.5!(nominal-bottom-left)$);
        \coordinate (robust-bottom) at
            ($(nominal-bottom-left)!-0.5!(robust-top-left)$);
        \coordinate (robust-mid-x) at
            ($(robust-top-left)!0.5!(robust-top-right)$);
        \coordinate (robust-mid-y) at
            ($(robust-top-left)!0.5!(nominal-bottom-left)$);
        \draw[draw=red,dashed,rounded corners=2pt,line width=0.7pt]
            ($(robust-left |- robust-top)$)
            -- ($(robust-right |- robust-top)$)
            -- ($(robust-right |- robust-bottom)$)
            -- ($(robust-mid-x |- robust-bottom)$)
            -- ($(robust-mid-x |- robust-mid-y)$)
            -- ($(robust-left |- robust-mid-y)$)
            -- cycle;
        \node[anchor=center,fill=white,inner xsep=3pt,inner ysep=0.5pt,
            font=\scriptsize\itshape,text=red]
            at (robust-mid-x |- robust-top)
            {robust reduction};
    \end{tikzpicture}
\end{table}

\paragraph*{Minimax sample complexity lower bound.}
We prove a minimax lower bound matching the rate \eqref{eq:sample-complexity} in Theorem~\ref{thm:minimax-lower}. To our knowledge, this is the first lower bound that reveals the fundamental statistical limit of distributionally robust AMDPs. 
It consists of a linear term 
$\min\{H_0,H_\sigma\}$ that is reminiscent of the standard AMDP lower bound,
and a robust-specific term $\sigma H_\sigma^2$ that captures the additional cost of robustness in the low-tolerance regime $\varepsilon\lesssim \sigma H_0$.

\paragraph*{Span-informed and span-agnostic upper bounds.}
We develop two model-based procedures based on 
reductions to discounted problems. When $H_0$ and $H_\sigma$ are known, Algorithm~\ref{alg:robust-amdp}
uses the nominal reduction when $H_0<H_\sigma$ and
$\sigma H_0\lesssim\varepsilon$, and the robust reduction otherwise. When the
spans are unknown, Algorithm~\ref{alg:anchored-span-agnostic-robust-amdp} uses the data to adaptively
select a policy based on the preferable reduction and discount factor. Both procedures attain the
minimax rates \eqref{eq:sample-complexity} up to logarithmic
factors. This uniformly improves upon the previous span-informed \citep{roch2025reduction} and span-agnostic bounds \citep{roch2025provably}.
Table~\ref{tab:sample_complexity_comparison} compares our guarantees
with existing results for standard and robust AMDPs.

The upper bounds rely on three complementary ideas: comparing the robust
discounted problem with a nominal reference solution, controlling discounted
value functions through their spans, and using concentration bounds governed
by these spans rather than by the full effective horizon. Together, these ideas
yield the full minimax rate in \eqref{eq:sample-complexity}, including its sharp
dependence on the span parameters, robustness radius, and target accuracy
across both tolerance regimes. For unknown spans, we convert the same bounds
into data-dependent certificates that guide the selection of the reduction and
discount factor.

\paragraph*{Consequence under an $H_\sigma$-only assumption.}
Our results also cover the setting where only the robust optimal bias span
$H_\sigma$ is assumed bounded, with no analogous assumption on $H_0$.  The
minimax sample complexity in this setting is
$\widetilde\Theta\!\left(SA(H_\sigma+\sigma H_\sigma^2)\varepsilon^{-2}\right)$.
We attain this rate using the robust reduction when $H_\sigma$ is known and a
robust-only variant of the span-agnostic procedure when it is unknown.

\begin{table}[htbp]
    \begin{onehalfspace}
    \caption{Comparison of robust and standard AMDP sample-complexity bounds under a generative model.
    Here, $H_{\mathrm{Roch}}$ and $H_{\mathrm{RHI}}$ denote the
    respective span parameters used in the two cited works; both are at least
    $H_\sigma$.
    Span knowledge refers to whether the algorithm uses the relevant span
    parameter. Logarithmic factors are omitted.}
    \label{tab:sample_complexity_comparison}
        \centering
        \setlength{\tabcolsep}{2.5pt}

        \newcolumntype{L}[1]{>{\raggedright\arraybackslash}m{#1}}
        \newcolumntype{C}[1]{>{\centering\arraybackslash}m{#1}}
        \begin{tabular}{L{0.175\textwidth} C{0.13\textwidth} C{0.41\textwidth} C{0.225\textwidth}}
            \toprule
            Setting
             & Span Knowledge
             & Sample Complexity
             & Reference\tabularnewline
            \midrule
            \multirow{2}{*}{Standard AMDP}
             & Yes
             & $SAH_{0}\varepsilon^{-2}$
             & \mbox{\citet{zurek2024span}} \tabularnewline
             & No
             & $SAH_{0}\varepsilon^{-2}$
             & \mbox{\citet{zurek2025spanagnostic}} \tabularnewline
            \midrule
            \multirow{5}{*}{Robust AMDP}
             & Yes
             & $SAH_{\mathrm{Roch}}^{2}\varepsilon^{-2}$
             & \citet{roch2025reduction} \tabularnewline
             & No
             & $SAH_{\mathrm{RHI}}^{2}\varepsilon^{-2}$
             & \citet{roch2025provably} \tabularnewline
            \addlinespace[0.4em]
             & Yes
             & \multirow{3}{0.41\textwidth}{\centering
             \(
             \displaystyle \frac{SA}{\varepsilon^2}\cdot
             \begin{cases}
             \min\{H_0,H_\sigma\},
             & \varepsilon\gtrsim\sigma H_0\\[1mm]
             \min\{H_0,H_\sigma\}+\sigma H_\sigma^2,
             & \varepsilon\lesssim\sigma H_0
             \end{cases}
             \)}
             & \textbf{Theorem~\ref{thm:TV-upper-bound-average}} \tabularnewline
             & No
             &
             & \textbf{Theorem~\ref{thm:anchored-calibrated-policy-upper-bound}}\tabularnewline 
             & (lower bound)
             & 
             & \textbf{Theorem~\ref{thm:minimax-lower}}\tabularnewline
            \bottomrule
        \end{tabular}
    \end{onehalfspace}
\end{table}

\subsection{Related work}

\paragraph*{Average-reward MDPs.}
Average-reward MDPs provide a classical framework for sequential decision
making under long-run performance criteria \citep{puterman2014markov}. This
criterion is well-suited for continuing tasks \citep{kumar2025continual} in which performance is measured by the steady-state reward per period rather than by a finite horizon or a discounted
sum. The literature has studied dynamic programming and planning
\citep{puterman2014markov}, regret and exploration
\citep{jaksch2010near,bartlett2009regal,fruit2018efficient}, structural
complexity measures such as diameter, mixing time, and bias span
\citep{jaksch2010near,Jin2020Efficiently,jin2021towards,wang2022near,zurek2024span},
and model-free learning and function approximation
\citep{wei2020model,wei2021learning,zhang2023sharper,jin2024feasible,lee2025near,jiao2026sample}.

\paragraph*{Distributionally robust MDPs.}
Distributionally robust MDPs build on the broader robust optimization principle
of optimizing against a worst-case model in an ambiguity set. Some foundational references on distributionally robust MDPs include \citet{iyengar2005robust,nilim2005robust,xu2012distributionally}. In the discounted setting, a growing
literature has established finite-sample guarantees for robust and
distributionally robust RL\@. Under a generative model, such guarantees have been established for model-based approaches \citep{Yang2022THEORETICAL,panaganti2022sample,shi2023curious,Clavier2024Optimal} and model-free Q-learning \citep{Yang2022THEORETICAL,Wang2024Sample}. In the offline setting, robust policy learning from pre-collected datasets has been studied \citep{panaganti2022robust,shi2024distributionally,Wang2024Samplea}.

\paragraph*{Robust average-reward MDPs.} For the average-reward criterion,
\citet{Wang2023Robust} and \citet{Wang2023Model} developed robust Bellman
equations, robust relative value iteration, and model-free robust AMDP
algorithms. \citet{Chen2025Sample} established mixing-time-based sample-complexity guarantees
for robust AMDPs under a uniform ergodicity condition over all transition
kernels in the uncertainty set. Recent work has also
considered efficient model-free robust average-reward methods and
non-rectangular robust AMDPs
\citep{xu2025efficient,Wang2026Non}.

\paragraph*{Distributionally robust optimization.}
Our formulation is also connected to the broader DRO literature, which studies
decision-making rules with uniform performance over ambiguity sets.
Representative foundational references include \citet{Goh2010Distributionally},
\citet{Wiesemann2014Distributionally}, \citet{MohajerinEsfahani2018Data}, and
\citet{Duchi2021Learning}. The DRO viewpoint provides a principled way
to trade nominal performance for reliability under sampling error, model
misspecification, and distribution shift, while retaining an optimization
problem whose conservatism is controlled explicitly by the ambiguity set.

\subsection{Paper organization and notation}

The rest of the paper is organized as follows. Section~\ref{sec:problem-setup}
formulates distributionally robust average-reward MDPs
and defines the structural assumptions and span parameters used throughout the
paper. Section~\ref{sec:lower-bound} proves the minimax lower bound and explains
why the robust optimal bias span $H_{\sigma}$ is the relevant complexity
parameter. Section~\ref{sec:reduction-upper-bound}
develops the reduction from robust average-reward MDPs to robust discounted
MDPs, presents both the span-informed and span-agnostic algorithms, and proves
the corresponding sample complexity upper bounds. Section~\ref{sec:experiments} presents experiments illustrating our theoretical predictions. Section~\ref{sec:discussion} concludes with
a discussion of future directions.
The appendices collect the proofs and supplementary experimental results.
\paragraph*{Notation.}
For any finite set $\mathcal{X}$, $\Delta(\mathcal{X})$ denotes the
probability simplex over $\mathcal{X}$. For any function
$h:\mathcal{S}\rightarrow\mathbb{R}$, its span seminorm is defined as
$\|h\|_{\mathrm{span}}:=\max_{s}h(s)-\min_{s}h(s)$. For a stationary policy
$\pi$ and a transition kernel $P$, we write $P_\pi$ for the induced
state-to-state transition matrix, so that
$(P_\pi h)(s)=\mathbb{E}_{a\sim\pi(\cdot\mid s)}
    \sum_{s'\in\cS}P(s'\mid s,a)h(s')$.
The total variation distance between two probability distributions $P$ and $Q$
over a finite state space $\mathcal{S}$ is defined as
$\|P-Q\|_{\mathrm{TV}}=\frac{1}{2}\sum_{s\in\mathcal{S}}|P(s)-Q(s)|$.
We use standard asymptotic notation such as $\widetilde{O}(\cdot)$,
$\widetilde{\Theta}(\cdot)$, and $\Omega(\cdot)$, where the tilde indicates
the suppression of logarithmic factors. For nonnegative $x$ and $y$, we
write $x\lesssim y$ if $x\le Cy$ for a universal constant $C>0$, and 
$x\gtrsim y$ if $x\ge Cy$ for a universal constant $C>0$. We also write $x\ll y$ and $x\gg y$ to indicate a
separation of scales.
In particular, we use $1-O(\delta)$ to
mean that an event occurs with probability at least $1-C\delta$ for some
constant $C$. We use $\mathbf 1$ to denote the all-ones vector, with its
dimension clear from context. For any vectors $\bm{x}, \bm{y} \in \mathbb{R}^d$, we use
$\bm{x} \le \bm{y}$ to denote $x_i \le y_i$ for all
$i \in \{1, \ldots, d\}$. For a scalar $x$, let
$[x]_+\coloneqq\max\{x,0\}$; for a vector, $[\cdot]_+$ is applied
coordinatewise. Let $[N]$ be $\{1,\ldots, N\}$.

\section{Problem formulation}\label{sec:problem-setup}

This section sets up the model of robust average-reward MDPs and the span parameters that determine the sample complexity. We begin with the basic definitions of the nominal and robust
average-reward MDPs. We then state the
structural assumptions and explain the split into high- and low-tolerance regimes.

\subsection{Robust average-reward MDPs}

\paragraph*{Standard average-reward MDP.} We start by introducing the standard average-reward Markov decision process (AMDP), which is specified by
$\cM^0 = (\mathcal{S},\mathcal{A},P^{0},r)$. Here,
$\mathcal{S}=\{1,\ldots,S\}$ is the state space,
$\mathcal{A}=\{1,\ldots,A\}$ is the action space, $P^{0}=\{P_{s,a}^{0}\}_{(s,a)\in\mathcal{S}\times\mathcal{A}}$ is the transition kernel, where $P_{s,a}^{0}$ is the next-state distribution given the state-action pair $(s,a)$, and $r:\mathcal{S}\times\mathcal{A}\to[0,1]$ is the reward function. A stationary policy
$\pi:\mathcal{S}\to\Delta(\mathcal{A})$ specifies an action selection rule for a given state $s\in \cS$, where $\pi(s)$ is a probability distribution over the action space.

Average reward measures the long-run steady-state value of a policy. For a
transition kernel $P^0$ and a policy $\pi$, the average reward from initial
state $s$ is
\[
    \rho_{P^0}^{\pi}(s)
    \coloneqq
    \lim_{T\rightarrow\infty}
    \mathbb{E}_{P^0}^{\pi}\left[
        \frac{1}{T}\sum_{t=0}^{T-1}r(s_t,a_t)
        \mid s_0=s
        \right],
\]
whenever the limit exists. The expectation is taken over the
action $a_t\sim\pi(s_t)$ and the next state
$s_{t+1}\sim P^0_{s_t,a_t}$. Under kernel $P^0$, denote the optimal average
reward from initial state $s$ by
\[
    \rho_{P^0}^{\star}(s)
    \coloneqq
    \sup_{\pi}\rho_{P^0}^{\pi}(s).
\]

\paragraph*{Distributionally robust AMDP.}
Since the performance of a policy can be sensitive to perturbations of the transition kernel, the distributionally robust formulation evaluates policies by their worst-case performance against all plausible transition
kernels near a nominal one $P^0$. Specifically, a distributionally robust AMDP is written as
$\mathcal{M}=(\mathcal{S},\mathcal{A},P^0,\mathcal{U},r)$, where
$(\mathcal{S},\mathcal{A},P^0,r)$ is the nominal AMDP described above and
$\mathcal{U}$ describes the admissible transition perturbations by mapping a transition kernel to a set of transition kernels. We focus on
$(s,a)$-rectangular total-variation (TV) uncertainty sets, meaning that the uncertainty for each state-action pair is decoupled. For each state-action pair
$(s,a)$, we are given a local radius $\sigma_{s,a}\in[0,\sigma]$. The uncertainty set near $P^0$, denoted by $\mathcal{P}\coloneqq\mathcal{U}(P^0)$, is defined as
\[
    \mathcal{U}(Q)
    \coloneqq
    \prod_{(s,a)\in\mathcal{S}\times\mathcal{A}}\mathcal{U}_{s,a}(Q_{s,a}),
    \quad
    \mathcal{U}_{s,a}(Q_{s,a})
    =
    \left\{
    P_{s,a}\in\Delta(\mathcal{S}):
    \|P_{s,a}-Q_{s,a}\|_{\mathrm{TV}}\le \sigma_{s,a}
    \right\},
\]
for any transition kernel $Q =\{Q_{s,a}\}$. Furthermore, let $\mathcal{P}_{s,a}\coloneqq\mathcal{U}_{s,a}(P_{s,a}^0)$ be the uncertainty set for the transition vector $P_{s,a}^0$ at state-action pair $(s,a)$. Here we use a broader class of uncertainty sets that allow local radii $\sigma_{s,a}$ to be smaller than $\sigma$, while \citet{shi2023curious} assumes $\sigma_{s,a} = \sigma$ for all $(s,a)$. 

For a policy $\pi$, the robust
average reward is the worst-case average reward over $\mathcal{P}$, and the
robust optimal average reward is the best such worst-case value:
\[
    \rho^{\pi,\sigma}(s)
    \coloneqq
    \inf_{P\in\mathcal{P}}\rho_P^\pi(s).
\]
The optimal robust average reward from initial state $s$ is
\[
    \rho^{\star,\sigma}(s)
    \coloneqq
    \sup_{\pi}\rho^{\pi,\sigma}(s).
\]

\paragraph*{Sampling model and the goal.}
We assume that we have access to a generative model that samples from the nominal transition kernel
$P^0$. For each $(s,a)\in\mathcal{S}\times\mathcal{A}$, we generate $N$ independent samples
\[
    s_{i,s,a}'\sim P^0_{s,a},
    \qquad i=1,\ldots,N.
\]
We also assume that the reward function $r$ is known.

Our learning goal is to use as few samples as possible to compute a policy $\widehat{\pi}$ such that, for every
initial state $s$,
\[
    \rho^{\star,\sigma}(s)
    -
    \rho^{\widehat{\pi},\sigma}(s)
    \le
    \varepsilon,
\]
for a target accuracy $\varepsilon$.

\subsection{Assumptions and key parameters}
We first state a structural assumption used throughout our analysis and then introduce the nominal and robust optimal bias spans that govern the sample complexity.

\begin{assumption}[Unichain]\label{assump:unichain}
    For every stationary policy $\pi$ and transition kernel $P\in\mathcal{P}$,
    the induced Markov chain contains exactly one recurrent class.
\end{assumption}
This assumption ensures that
$\rho_{P^0}^\pi(s)$ and $\rho^{\pi,\sigma}(s)$ are well defined and independent of
the initial state $s$ for every stationary policy $\pi$; see
\citet{Wang2023Robust,Wang2023Model}. We therefore omit the state argument from now on and
write $\rho^\pi$, $\rho^\star$, $\rho^{\pi,\sigma}$, and
$\rho^{\star,\sigma}$ as scalars. We also let
$\pi^\star_\sigma\in\arg\max_\pi\rho^{\pi,\sigma}$ denote a robust optimal
policy.

\paragraph*{Optimal bias functions.} The average reward captures only the
long-run reward rate. A bias function complements it by measuring differences
in the transient reward accumulated from different initial states relative to
this rate. In this paper, we use nominal and robust optimal bias functions,
characterized through their respective average-reward Bellman optimality
equations. For $h\in\mathbb R^{\mathcal S}$, define the corresponding Bellman
operators by
\begin{align*}
    (\mathcal T_0h)(s)
    &\coloneqq
    \max_{a\in\mathcal A}
    \left\{
        r(s,a)+(P^0_{s,a})^\top h
    \right\}, \\
    (\mathcal T_\sigma h)(s)
    &\coloneqq
    \max_{a\in\mathcal A}
    \left\{
        r(s,a)
        +
        \min_{P_{s,a}\in\mathcal P_{s,a}}P_{s,a}^\top h
    \right\}.
\end{align*}
A nominal optimal bias is any vector $h\in\mathbb R^{\mathcal S}$ satisfying
the first equation below, whereas a robust optimal bias is any vector
$h\in\mathbb R^{\mathcal S}$ satisfying the second:
\begin{subequations}
    \begin{align}
        \rho^\star\mathbf 1+h
        &=
        \mathcal T_0h,
        \label{eq:nominal-average-reward-bellman} \\
        \rho^{\star,\sigma}\mathbf 1+h
        &=
        \mathcal T_\sigma h.
        \label{eq:robust-average-reward-bellman}
    \end{align}
\end{subequations}
Both Bellman equations are invariant under adding a constant to $h$.

\paragraph*{Optimal bias spans.} For any vector
$h\in\mathbb R^{\mathcal S}$, define its span seminorm by
\[
    \|h\|_{\mathrm{span}}
    \coloneqq
    \max_{s\in\mathcal{S}}h(s)-\min_{s\in\mathcal{S}}h(s).
\]
Define the \emph{nominal optimal bias span} and the
\emph{robust optimal bias span} by
\begin{subequations}
    \begin{align}
        H_0
        &\coloneqq
        \max\left\{
            1,
            \inf_{h:\,\rho^\star\mathbf 1+h=\mathcal T_0h}
            \|h\|_{\mathrm{span}}
        \right\}, \label{eq:nominal_span} \\
        H_{\sigma}
        &\coloneqq
        \max\left\{
            1,
            \inf_{h:\,\rho^{\star,\sigma}\mathbf 1+h=\mathcal T_\sigma h}
            \|h\|_{\mathrm{span}}
        \right\}. \label{eq:robust_span}
    \end{align}
\end{subequations}
The inner infima in \eqref{eq:nominal_span} and \eqref{eq:robust_span} are
attained; see Propositions~\suppRobustBellmanExistence{} and~\suppMinimumSpanAttainment{} in
Appendix~\suppRobustBellmanExistenceAttainment{}. We denote the
corresponding minimum-span solutions by $h_{P^0}^\star$ and
$h^{\star,\sigma}$, respectively.

\begin{remark}[Comparison with prior robust bias spans]
We define the robust optimal bias span by minimizing over robust Bellman
solutions. This quantity is no larger than either the uniform-over-kernels span
$H_{\mathrm{Roch}}$ used by \citet{roch2025reduction} or the
robust-optimal-pair span $H_{\mathrm{RHI}}$ used by
\citet{roch2025provably}. The precise comparisons are given in
Appendix~\suppRochSpanComparison{}, which also explains why the smaller
quantity is sufficient for our discounted-to-average reduction.

More generally, neither $H_0$ nor $H_\sigma$ controls the other. Proposition~\ref{prop:span-independence}
shows that even for a
fixed uncertainty budget, any prescribed pair $(H_0, H_\sigma)$ can be realized. The proof is given in
Appendix~\suppSpanIndependenceProof{}.
\end{remark}

\begin{proposition}[Independence of the nominal and robust optimal bias spans]
\label{prop:span-independence}
Fix $\sigma>0$ and any $H_0,H_\sigma\ge1$. There exists a robust AMDP satisfying Assumption~\ref{assump:unichain}, with local
radii at most $\sigma$, whose nominal optimal bias span is exactly $H_0$ and
whose robust optimal bias span is exactly $H_\sigma$.
\end{proposition}

\subsection{High- and low-tolerance regimes}
\label{subsec:free-lunch}

The perturbation scale $\sigma H_0$ separates the
\emph{high-tolerance regime} $\sigma H_0\lesssim\varepsilon$, where a suitable
nominal optimal policy is guaranteed to be $\varepsilon$-optimal for the robust
problem, from the \emph{low-tolerance regime}
$\varepsilon\lesssim\sigma H_0$, where nominal optimization is no longer
guaranteed to suffice. The following proposition formalizes this by providing a perturbation bound on the optimal average reward. Its proof is deferred to
Appendix~\suppPerturbationBoundProof{}.

\begin{proposition}[Nominal-to-robust perturbation bound]\label{prop:perturbation-bound-average-reward}
    There exists a nominal optimal policy $\pi$ such that
    \[
        \rho^\star-\sigma H_0
        \le
        \rho^{\pi,\sigma}
        \le
        \rho^{\star,\sigma}
        \le
        \rho^\star.
    \]
\end{proposition}

Proposition~\ref{prop:perturbation-bound-average-reward} shows that the difference between the robust optimal average reward $\rho^{\star,\sigma}$ and the nominal optimal average reward $\rho^\star$ is at most $\sigma H_0$. Moreover, there exists a nominal optimal policy that
is $\sigma H_0$-optimal for the robust problem. This observation suggests that in the high-tolerance regime, we should be able to learn a robust $\varepsilon$-optimal policy by solving the nominal problem.

\section{Minimax sample complexity lower bound}\label{sec:lower-bound}

In this section, we establish a minimax lower bound for robust AMDPs.
Consider a bounded class of robust AMDPs---denoted by $\mathfrak{M}(H_{0},H_{\sigma},\sigma)$---that satisfy
Assumption~\ref{assump:unichain} and have local
uncertainty radii at most $\sigma$, $|\cS|=S$, $|\cA|=A$, and span constraints
\[
    \|h_{P^0}^{\star}\|_{\mathrm{span}}\le H_0,
    \qquad
    \|h^{\star,\sigma}\|_{\mathrm{span}}\le H_\sigma.
\]
The following theorem characterizes the
minimum sample complexity required to obtain a policy that is
$\varepsilon$-optimal in the robust average reward. The proof is deferred to
Appendix~\suppLowerBoundProofs{}.

\begin{theorem}\label{thm:minimax-lower}
    Let $\varepsilon\in(0,0.01]$. Assume that $S\ge5$, $A\ge3$,
    $\min\{H_0,H_\sigma\}\ge4$, and $0<\sigma\le1/2$. For some sufficiently
    small universal constant $C>0$, suppose the number of samples $NSA$ satisfies
    (a)
    \[
        NSA
        \le
        \frac{
            CSA\min\{H_0,H_\sigma\}
        }{\varepsilon^{2}},
    \]
    or (b) $\varepsilon\le0.01\sigma H_0$, and
    \[
        NSA
        \le
        \frac{
            CSA\left(
                \min\{H_0,H_\sigma\}+\sigma H_{\sigma}^2
            \right)
        }{\varepsilon^{2}}.
    \]
    Then
    \[
        \inf_{\widehat{\pi}}
        \sup_{\cM\in\mathfrak{M}(H_{0},H_{\sigma},\sigma)}
        \mathbb{P}_{\cM}\!\left\{
            \rho^{\star,\sigma}
            -
            \rho^{\widehat{\pi},\sigma}
            >
            \varepsilon
        \right\}
        \ge
        \frac14.
    \]
\end{theorem}

Theorem~\ref{thm:minimax-lower} shows that to achieve an
$\varepsilon$-optimal policy in the robust average reward, every algorithm
requires at least
\[
    \Omega\Bigg(\frac{SA}{\varepsilon^2}
            \Big(
                \underbrace{\min\{H_0,H_\sigma\}}_{\text{linear min-span term}}
                +
                \underbrace{\sigma H_{\sigma}^2}_{\text{robust-specific term}}
            \Big)
    \Bigg)
\]
samples, where the robust-specific term is active when $\varepsilon\le0.01\sigma H_0$. This gives the lower-bound side of
the four-regime summary in Table~\ref{tab:four-regime-summary}. To the best of
our knowledge, this is the first characterization of the minimax sample
complexity of robust AMDPs.

The lower bound consists of two components.
The linear min-span term $\min\{H_0,H_\sigma\}$ is analogous to the $H_0$ term
in the standard AMDP lower bound, but depends on the smaller of the two span
parameters. This reflects the fact that the MDP class
$\mathfrak{M}(H_0,H_\sigma,\sigma)$ is constrained by both span parameters.

On the other hand, the robust-specific term $\sigma H_{\sigma}^2$ captures the
additional cost of distributional robustness. For comparison,
\citet{shi2023curious} establish a robust discounted-MDP lower bound when every
local uncertainty radius is fixed at $\sigma$ and find that robustness can
require fewer samples than the standard problem. This contrast emphasizes
that the specification of the uncertainty sets affects sample complexity.

\paragraph*{Consequence under an $H_\sigma$-only assumption.}
If only the robust optimal bias span $H_\sigma$ is assumed bounded, Proposition~\ref{prop:span-independence} implies that $H_0$ can be arbitrarily large. Then Theorem~\ref{thm:minimax-lower} gives the
minimax lower bound
$\Omega\!\left(SA(H_\sigma+\sigma H_\sigma^2)\varepsilon^{-2}\right)$.

\section{The plug-in approach for robust AMDPs}\label{sec:reduction-upper-bound}

\subsection{Motivation for a reduction to DMDPs}

Directly solving the robust average-reward problem is difficult because the
objective is defined through the long-run robust gain, and its Bellman operator is not contractive. The discounted MDP, on the other hand, offers a more tractable surrogate: if the discount factor $\gamma$ is close
enough to $1$, then the robust discounted value approximates the robust
average reward after multiplying by $1-\gamma$. Formally, for a fixed policy
$\pi$ and discount factor $\gamma\in(0,1)$, let
\[
    V_{\gamma}^{\pi,\sigma}(s)
    =
    \inf_{P\in\mathcal{P}}
    \mathbb{E}_{P}^{\pi}\left[
        \sum_{t=0}^{\infty}\gamma^t r(s_t,a_t)
        \mid s_0=s
        \right]
\]
be the robust discounted value function. The robust optimal discounted value
is defined componentwise by
\[
    V_\gamma^{\star,\sigma}(s)
    \coloneqq
    \sup_\pi V_\gamma^{\pi,\sigma}(s),
    \qquad s\in\mathcal S.
\]
Lemma~\suppRectangularDiscountedBellmanRepresentation{} shows that
compactness and $(s,a)$-rectangularity identify these values with the unique
fixed points of their respective robust discounted Bellman operators. The
fixed-policy discounted value approximates the robust average reward
$\rho^{\pi,\sigma}$, especially when $\gamma$ is close to $1$.

This motivates a reduction-based route:
choose a large effective horizon $(1-\gamma)^{-1}$ and return the optimal policy for the corresponding robust discounted MDP as the solution to the robust AMDP problem. The choice of $\gamma$ is subtle as it
balances two competing requirements: when it is too far from 1, 
the discounted objective does not approximate the average-reward objective well; when it is too close to 1, 
the statistical error of estimating the discounted problem becomes
 too large, since the sample complexity of DMDPs grows polynomially with the effective horizon $(1-\gamma)^{-1}$ \citep{li2024breaking}. In the next two sections, we consider two algorithms based on this reduction principle: a span-informed procedure for known $H_0$ and $H_\sigma$, and a span-agnostic procedure for unknown spans.

\subsection{The span-informed reduction}
\label{subsec:span-informed-reduction}
\label{subsec:glimpse-upper-bound-analysis}

For each state-action pair $(s,a)\in\cS\times\cA$, we observe $N$ independent
next-state samples from the nominal transition kernel $P_{s,a}^0$ and form the
empirical nominal kernel
\begin{equation}
    \widehat{P}^0_{s,a}(s')
    =
    \frac{1}{N}\sum_{i=1}^N \mathbf{1}_{s_i'=s'},
    \qquad
    (s,a,s')\in\cS\times\cA\times\cS.
    \label{eq:empirical-nominal-kernel}
\end{equation}
Given $\widehat P^0$, we assume access to a distributionally robust DMDP
solver that returns a
policy $\widehat\pi$ satisfying
\begin{equation}
\label{eq:robust-dmdp-solver-tolerance}    
    \left\|
    \widehat V_\gamma^{\star,\sigma}
    -\widehat V_\gamma^{\widehat\pi,\sigma}
    \right\|_\infty
    \le
    \varepsilon_{\mathrm{opt}},
\end{equation}
where
\[
    \widehat V_\gamma^{\pi,\sigma}(s)
    =
    \inf_{P\in\cU(\widehat P^0)}
    \mathbb E_P^\pi\left[
    \sum_{t=0}^\infty\gamma^t r(s_t,a_t)
    \mid s_0=s
    \right],
\]
and $\widehat V_\gamma^{\star,\sigma}$ is the corresponding optimal value. We
refer to $\varepsilon_{\mathrm{opt}}$ as the solver tolerance. This solver also serves as a nominal DMDP solver when $\cU(\widehat P^0)$ is the singleton $\{\widehat P^0\}$.
Under the $(s,a)$-rectangular TV uncertainty sets in
Section~\ref{sec:problem-setup}, robust value iteration or policy iteration can
solve this discounted problem to arbitrary accuracy
\citep{iyengar2005robust,nilim2005robust,ho2018fast}.

As discussed in Section~\ref{subsec:free-lunch}, in the high-tolerance regime, it is possible
to find a nominally optimal policy that is robustly $\varepsilon$-optimal.
Moreover, intuitively, the nominal problem could be statistically easier to
solve than the robust problem. This suggests that we should consider both the
nominal and robust discounted reductions and select the one that is more
sample-efficient.

The span-informed procedure treats $H_0$ and $H_\sigma$ as known. It uses them
to select between nominal and robust discounted reductions and pick a discount factor $\gamma$ of the right scale.
\begin{algorithm}
    \caption{Span-informed procedure for robust AMDPs}
    \label{alg:robust-amdp}
    \noindent\textbf{Input.}
    For every $(s,a)$, $N$ nominal transition samples
    $\{(s,a,s_i')\}_{i=1}^N$; the uncertainty set rule $\cU$; the nominal and
    robust optimal bias spans $H_0$ and $H_\sigma$; the desired accuracy
    $\varepsilon$; and the solver tolerance $\varepsilon_{\mathrm{opt}}$.
    \begin{enumerate}[leftmargin=*,label=\textbf{\arabic*.},itemsep=0.1em,topsep=0.5em]
        \item \textbf{Estimate the nominal model.}
        Construct $\widehat P^0$ from the $N$ samples according to
        \eqref{eq:empirical-nominal-kernel}.

        \item \textbf{Nominal or robust reduction.}
        \begin{itemize}[leftmargin=1.5em,itemsep=0.2em,topsep=0.2em]
            \item If $H_0<H_\sigma$ and $7\sigma H_0\le\varepsilon$, set
            $\gamma=1-\varepsilon/(20H_0)$ and run a nominal DMDP solver on
            $(\cS,\cA,r,\widehat P^0,\gamma)$.
            \item Otherwise, set $\gamma=1-\varepsilon/(3H_\sigma)$ and run a
            robust DMDP solver on
            $(\cS,\cA,r,\cU(\widehat P^0),\gamma)$.
        \end{itemize}

        \item Return the resulting policy $\widehat\pi$.
    \end{enumerate}
\end{algorithm}
This span-informed procedure satisfies the following sample-complexity
guarantee.
\begin{theorem}[Span-informed upper bound]\label{thm:TV-upper-bound-average}
    There exist a sufficiently large universal constant $C>0$ and a
    sufficiently small universal constant $c_{\mathrm{opt}}>0$ such that the
    following holds.
    Suppose that $\cM$ is a robust AMDP
    satisfying Assumption~\ref{assump:unichain}. Let $\varepsilon\in(0,1]$ be
    the target accuracy and let $\delta\in(0,1/2]$. 
    Assume that $\varepsilon_{\mathrm{opt}}\le c_{\mathrm{opt}}\varepsilon$
    and that $\widehat \pi$ is the output of
    Algorithm~\ref{alg:robust-amdp}.
    Suppose that either of the following conditions holds:
    \begin{enumerate}[label=\textup{(\alph*)},leftmargin=2.5em]
        \item When $7\sigma H_0\le\varepsilon$,
        \[
            NSA
            \ge
            CSA\log\!\left(
            \frac{H_\sigma SAN}{\varepsilon\delta}
            \right)\cdot
            \frac{\min\{H_0,H_\sigma\}}{\varepsilon^2};
        \]
        \item when $7\sigma H_0>\varepsilon$,
        \[
            NSA
            \ge
            CSA\log\!\left(
            \frac{H_\sigma SAN}{\varepsilon\delta}
            \right)\cdot
            \frac{\min\{H_0,H_\sigma\}+\sigma H_\sigma^2}
            {\varepsilon^2}.
        \]
    \end{enumerate}
    Then, with probability at least $1-O(\delta)$, the returned policy
    $\widehat{\pi}$ satisfies
    $\rho^{\star,\sigma}-\rho^{\widehat{\pi},\sigma}\le\varepsilon$.
\end{theorem}

The two displayed conditions give the high- and low-tolerance rates,
respectively. Comparing Theorem~\ref{thm:TV-upper-bound-average} with the lower
bound in Theorem~\ref{thm:minimax-lower} shows that
Algorithm~\ref{alg:robust-amdp} is minimax-optimal, up to logarithmic factors. Notably, the high-tolerance regime does not necessarily favor the nominal reduction: when $H_\sigma\le H_0$, the robust reduction yields the smaller span dependence and sample complexity.

\paragraph*{Variant when only $H_\sigma$ is known.}
If $H_\sigma$ is known but $H_0$ is not, always taking the robust reduction
gives the sample-complexity upper bound
$\widetilde O\!\left(SA(H_\sigma+\sigma H_\sigma^2)\varepsilon^{-2}\right)$.
Comparing with the $H_\sigma$-only lower bound, this rate is minimax-optimal
up to logarithmic factors.

\paragraph*{Comparison with a prior reduction framework.}
\citet{roch2025reduction} gave a robust-only reduction-based procedure with sample
complexity
$\widetilde O(SAH_{\mathrm{Roch}}^2\varepsilon^{-2})$, where
$H_{\mathrm{Roch}}$ is a robust span parameter different from $H_\sigma$. We compare them in Appendix~\suppRochSpanComparison{} and prove
that $H_\sigma\le H_{\mathrm{Roch}}$. Consequently, since $\sigma \le 1$ and all span parameters are at least $1$,
our span-informed rate uniformly matches or improves upon their result. 

\paragraph*{Analysis overview.}
We first focus on the robust branch of Algorithm~\ref{alg:robust-amdp}. The robust reduction sets
\[
    \gamma
    =
    1-\frac{\varepsilon}{3H_{\sigma}}.
\]
At this discount factor, it is enough to find a policy satisfying
\[
    \left\|
    V_{\gamma}^{\star,\sigma}
    -V_{\gamma}^{\widehat\pi,\sigma}
    \right\|_{\infty}
    \le
    H_{\sigma}.
\]
A generic robust discounted-MDP bound at this horizon and accuracy would give
the suboptimal sample complexity
$\widetilde O(SAH_{\sigma}\varepsilon^{-3})$. 
Theorem~\ref{thm:TV-upper-bound-average} obtains the sharper rate through three
refinements.

The first refinement compares the discounted problem with a nominal Bellman
supersolution. We call a pair $(\bar\rho,\bar h)$ a \emph{nominal anchor} if
\begin{equation}
    \bar\rho+\bar h(s)
    \ge
    \max_{a\in\cA}
    \left\{
    r(s,a)+P_{s,a}^0\bar h
    \right\},
    \qquad s\in\cS.
    \label{eq:nominal-anchor-supersolution}
\end{equation}
We choose the nominal optimal pair $(\rho^\star,h_{P^0}^\star)$ as the anchor
when $H_0<H_\sigma$ and the trivial pair $(1,0)$ otherwise. Using this anchor
as a reference in the variance analysis yields $\min\{H_0,H_\sigma\}$, rather
than $H_\sigma$ alone, in the leading statistical term.

The second refinement proves that the true and empirical discounted value
functions used in the analysis have spans of order $H_\sigma$. Their
fluctuations can therefore be controlled by $H_\sigma$, instead of the
worst-case discounted scale $(1-\gamma)^{-1}$. 

The third refinement replaces the blanket sample-size requirement
$N\gtrsim(1-\gamma)^{-2}$ from generic robust discounted analyses with refined
localized requirements. Together, these refinements show that, up to
logarithmic factors, achieving discounted error $H_\sigma$ requires
\[
    N
    \gtrsim
    \frac{\min\{H_0,H_\sigma\}+\sigma H_{\sigma}^2}
    {(1-\gamma)^2H_{\sigma}^2}
    +\frac{1}{1-\gamma}.
\]
Substituting $1-\gamma\asymp\varepsilon/H_\sigma$ gives
\[
    N
    \gtrsim
    \frac{\min\{H_0,H_\sigma\}+\sigma H_{\sigma}^2}{\varepsilon^2}
    +\frac{H_{\sigma}}{\varepsilon},
\]
where the last term is absorbed when the robust reduction is used. 

The nominal branch of Algorithm~\ref{alg:robust-amdp} applies the same discounted plug-in reduction with degenerate uncertainty sets. We use an additional argument to transfer its nominal performance guarantee to a robust performance guarantee. 
The complete discounted theorem and the analysis of Theorem~\ref{thm:TV-upper-bound-average} are given in
Appendix~\suppReductionUpperBounds{}. 

\subsection{The span-agnostic reduction}
\label{subsec:span-agnostic-reduction}

When $H_0$ and $H_\sigma$ are unknown, neither the reduction nor the
discount factor in Algorithm~\ref{alg:robust-amdp} can be selected directly.
This motivates
a span-agnostic procedure that adaptively selects the appropriate discount
factor and chooses between the nominal and robust discounted reductions.

Here we describe the framework of the span-agnostic reduction and present the main theoretical result. For conciseness, we defer the full details of the algorithm to Appendix~\suppAnchoredCalibrationSpecification{}. The key idea of this procedure is to construct a set of candidate policies from both nominal and robust discounted reductions and then select the best one based on a lower-confidence bound on its robust average reward.

\paragraph*{Two independent data batches.}

We use two batches so that the anchor certificate and the nominal policy
candidates are independent of the data used to assess robust performance. For
each state-action pair, split the $N$ transition samples into
a \emph{nominal batch} $\mathcal D_{\mathrm{nom}}$ of size
$N_{\mathrm{nom}}$ and a \emph{robust batch}
$\mathcal D_{\mathrm{rob}}$ of size $N_{\mathrm{rob}}$, where both sizes are
constant fractions of $N$. Let
$\widehat P_{\mathrm{nom}}^0$ and $\widehat P_{\mathrm{rob}}^0$ be the two
empirical nominal kernels, constructed as in
\eqref{eq:empirical-nominal-kernel}.
The nominal batch is used to compute the anchor certificate and the nominal
policy candidates. The robust
batch is then used to compute the robust policy candidates and to evaluate the
nominal candidates in the empirical robust MDP.

\paragraph*{Anchor certificate.}
Recall from the span-informed analysis in
Section~\ref{subsec:span-informed-reduction} that, when $H_0<H_\sigma$, the
nominal optimal pair $(\rho^\star,h_{P^0}^\star)$ provides an analytical anchor
satisfying \eqref{eq:nominal-anchor-supersolution}. Its bias span $H_0$ yields
the $H_0$-dependent part of the upper bound in
Theorem~\ref{thm:TV-upper-bound-average}. In the span-agnostic setting,
however, $H_0$ is unknown and therefore cannot be used directly to calibrate
the confidence penalty. We instead use the nominal batch to compute an anchor
certificate
$(\widehat\rho_{\mathrm{anc}}^+,\widehat H_{\mathrm{anc}}^+)$.
With probability at least $1-O(\delta)$, this certificate is associated with a
nominal anchor $(\bar\rho,\bar h)$ satisfying
\eqref{eq:nominal-anchor-supersolution} and
\[
    \bar\rho
    \le
    \widehat\rho_{\mathrm{anc}}^+,
    \qquad
    \max\{1,\|\bar h\|_{\mathrm{span}}\}
    \le
    \widehat H_{\mathrm{anc}}^+.
\]
Moreover, the certificate satisfies
$\widehat H_{\mathrm{anc}}^+=O(H_0)$, so it permits the confidence penalty to
retain the $H_0$ scale without requiring $H_0$ as an input.
For conciseness, we defer the exact algorithm and the lemma confirming these properties to
Appendix~\suppAnchoredCalibrationSpecification{}.

\paragraph*{Candidate policies.}
To adapt to the unknown effective horizon, we consider a dyadic grid
$\Gamma_N$ consisting of discount factors $\gamma = 1-2^{-k}$, $k=1,2,\ldots$, whose exact cutoff is sample-size-dependent and given in
Appendix~\suppAnchoredCalibrationSpecification{}. For each $\gamma\in\Gamma_N$, we compute two candidate policies. 
First, we solve the empirical nominal discounted MDP
with kernel $\widehat P_{\mathrm{nom}}^0$ to obtain
$\widehat\pi_\gamma^0$. Then, in the empirical robust discounted MDP centered at
$\widehat P_{\mathrm{rob}}^0$, we evaluate $\widehat\pi_\gamma^0$ to obtain
$\widehat V_\gamma^{\widehat\pi_\gamma^0,\sigma}$ and solve the robust
discounted problem to obtain $\widehat\pi_\gamma$ and
$\widehat V_\gamma^{\widehat\pi_\gamma,\sigma}$.

\paragraph*{Policy selection by lower-confidence bounds.}
For each $\gamma\in\Gamma_N$, we assign each of the two candidate policies a lower-confidence bound on its robust
average reward. The superscripts
$\mathrm{nom}$ and $\mathrm{rob}$ distinguish the nominal and robust candidate
families, respectively. Let
$\operatorname{pen}_\gamma^{\mathrm{nom}}(\widehat\pi_\gamma^0)$ denote the
confidence penalty for the nominal candidate. Its lower-confidence bound is
\begin{equation}
    \operatorname{LCB}_{\gamma}^{\mathrm{nom}}(\widehat\pi_\gamma^0)
    \coloneqq
    (1-\gamma)\min_s
    \widehat V_\gamma^{\widehat\pi_\gamma^0,\sigma}(s)
    -\operatorname{pen}_\gamma^{\mathrm{nom}}(\widehat\pi_\gamma^0).
    \label{eq:anchored-nominal-candidate-score}
\end{equation}
For the robust candidate, the anchor certificate and the trivial anchor
$(\rho,h)=(1,0)$ give
two valid penalties,
$\operatorname{pen}_\gamma^{\mathrm{anc}}(\widehat\pi_\gamma)$ and
$\operatorname{pen}_\gamma^{\mathrm{triv}}(\widehat\pi_\gamma)$,
respectively. Since both induce valid lower-confidence bounds, we
subtract the smaller penalty and define
\begin{equation}
    \operatorname{LCB}_{\gamma}^{\mathrm{rob}}(\widehat\pi_\gamma)
    \coloneqq
    (1-\gamma)\min_s
    \widehat V_\gamma^{\widehat\pi_\gamma,\sigma}(s)
    -
    \min\left\{
    \operatorname{pen}_\gamma^{\mathrm{anc}}(\widehat\pi_\gamma),
    \operatorname{pen}_\gamma^{\mathrm{triv}}(\widehat\pi_\gamma)
    \right\}.
    \label{eq:anchored-robust-candidate-score}
\end{equation}
The complete data-dependent definitions of the penalties
$\operatorname{pen}_\gamma^{\mathrm{nom}}$,
$\operatorname{pen}_\gamma^{\mathrm{anc}}$, and
$\operatorname{pen}_\gamma^{\mathrm{triv}}$ are given in
Appendix~\suppAnchoredCalibrationSpecification{}. With
probability at least $1-O(\delta)$, simultaneously for every
$\gamma\in\Gamma_N$,
\[
    \operatorname{LCB}_{\gamma}^{\mathrm{nom}}(\widehat\pi_\gamma^0)
    \le
    \rho^{\widehat\pi_\gamma^0,\sigma},
    \qquad
    \operatorname{LCB}_{\gamma}^{\mathrm{rob}}(\widehat\pi_\gamma)
    \le
    \rho^{\widehat\pi_\gamma,\sigma}.
\]
We therefore return the candidate policy with the largest lower-confidence
bound.

Algorithm~\ref{alg:anchored-span-agnostic-robust-amdp} summarizes the procedure; the full implementable specification is deferred to Appendix~\suppAnchoredCalibrationSpecification{}.
\begin{algorithm}
    \caption{Span-agnostic reduction with nominal and robust candidates}
    \label{alg:anchored-span-agnostic-robust-amdp}
    \noindent\textbf{Input.}
    For every $(s,a)$, $N$ independent nominal transition samples;
    $\varepsilon$, $\delta$, $\sigma$, and $\cU$.
    \begin{enumerate}[leftmargin=*,label=\textbf{\arabic*.},itemsep=0.1em,topsep=0.5em]
        \item \textbf{Prepare the empirical models.}
        Split the samples into $\mathcal D_{\mathrm{nom}}$ and
        $\mathcal D_{\mathrm{rob}}$, form
        $\widehat P_{\mathrm{nom}}^0$ and
        $\widehat P_{\mathrm{rob}}^0$, construct $\Gamma_N$, and compute the
        anchor certificate from $\mathcal D_{\mathrm{nom}}$.
        \item \textbf{Construct the candidate policies.}
        For every $\gamma\in\Gamma_N$, compute the nominal candidate
        $\widehat\pi_\gamma^0$ from $\widehat P_{\mathrm{nom}}^0$, evaluate it
        in the empirical robust MDP centered at
        $\widehat P_{\mathrm{rob}}^0$, and compute the robust candidate
        $\widehat\pi_\gamma$ in the same empirical robust MDP.
        \item \textbf{Select a policy.}
        Compute the lower-confidence bounds in
        \eqref{eq:anchored-nominal-candidate-score} and
        \eqref{eq:anchored-robust-candidate-score}, and return a policy with
        the largest lower-confidence bound.
    \end{enumerate}
\end{algorithm}

The following theorem gives the sample-complexity guarantee for the resulting
span-agnostic procedure.
\begin{theorem}[Span-agnostic robust policy learning]
    \label{thm:anchored-calibrated-policy-upper-bound}
    There exist a sufficiently large universal constant $C>0$ and
    sufficiently small universal constants
    $c,c_{\mathrm{opt}}>0$ such that the following holds.
    Let $\cM$ be a robust AMDP satisfying
    Assumption~\ref{assump:unichain}, and fix
    $\varepsilon\in(0,1]$, $\delta\in(0,1)$, and $N\ge 16$. Let
    $\widehat{\pi}$ be the output of
    Algorithm~\ref{alg:anchored-span-agnostic-robust-amdp}, and assume
    $\varepsilon_{\mathrm{opt}}\le c_{\mathrm{opt}}\varepsilon$.
    Suppose that either of the following conditions holds:
    \begin{enumerate}[label=\textup{(\alph*)},leftmargin=2.5em]
        \item When $\sigma H_0\le c\varepsilon$,
        \[
            NSA
            \ge
            CSA\log\!\left(\frac{SAN}{\delta}\right)\cdot
            \frac{\min\{H_0,H_\sigma\}}{\varepsilon^2};
        \]
        \item when $\sigma H_0>c\varepsilon$,
        \[
            NSA
            \ge
            CSA\log\!\left(\frac{SAN}{\delta}\right)\cdot
            \frac{\min\{H_0,H_\sigma\}+\sigma H_\sigma^2}
            {\varepsilon^2}.
        \]
    \end{enumerate}
    Then, with probability at least $1-O(\delta)$, the returned policy
    $\widehat\pi$ satisfies
    \[
        \rho^{\star,\sigma}-\rho^{\widehat\pi,\sigma}
        \le
        \varepsilon.
    \]
\end{theorem}

This guarantee is adaptive: Algorithm~\ref{alg:anchored-span-agnostic-robust-amdp}
uses neither $H_0$ nor $H_\sigma$ and is not told which regime holds. By
maximizing its lower-confidence bounds over the discount-factor grid and both
candidate families, it automatically uses the data to select both the discount
factor and whether to use the nominal or robust reduction.
Comparing with Theorem~\ref{thm:minimax-lower}, Theorem~\ref{thm:anchored-calibrated-policy-upper-bound} shows that
Algorithm~\ref{alg:anchored-span-agnostic-robust-amdp} is minimax-optimal
across all regimes up to logarithmic factors:
\begin{itemize}
    \item In the high-tolerance regime, the $H_0$ rate is certified by a
    nominal candidate when $H_0<H_\sigma$, whereas the $H_\sigma$ rate is
    certified by a robust candidate when $H_\sigma\le H_0$.
    \item In the low-tolerance regime, the robust candidate gives sample
    complexity
    \[
        \widetilde O\!\left(
        SA\frac{\min\{H_0,H_\sigma\}+\sigma H_\sigma^2}{\varepsilon^2}
        \right).
    \]
    The minimum in the leading span term comes from the two anchor choices:
    the calibrated nominal anchor supplies $H_0$ when $H_0<H_\sigma$, while
    the trivial anchor supplies $H_\sigma$ when $H_\sigma\le H_0$.
\end{itemize}
Both rates match the minimax lower bound in
Theorem~\ref{thm:minimax-lower} up to logarithmic factors.

\paragraph*{Robust-only variant.}
A simpler variant uses only the robust candidates and the trivial-anchor
penalty.  Its uniform sample complexity is
$\widetilde O\!\left(SA(H_\sigma+\sigma H_\sigma^2)\varepsilon^{-2}\right)$.
This robust-only variant attains the $H_\sigma$-only minimax rate, while the
full span-agnostic procedure can additionally exploit a smaller $H_0$ when
$H_0<H_\sigma$.

\paragraph*{Comparison with Roch et al.}
\citet{roch2025provably} propose a model-free, span-agnostic approach with
sample complexity
$\widetilde O(SAH_{\mathrm{RHI}}^2\varepsilon^{-2})$, where
$H_{\mathrm{RHI}}$ is a robust span parameter satisfying
$H_{\mathrm{RHI}}\ge H_\sigma$. See Appendix~A.4 for a detailed
comparison. Our sample-complexity guarantee uniformly matches or improves upon theirs.

\paragraph*{Analysis overview.}
If we knew $H_0$, $H_\sigma$, and which reduction is more
sample-efficient, we could choose a single discount factor and use that
reduction to obtain an $O(\varepsilon)$-optimal policy.
Algorithm~\ref{alg:anchored-span-agnostic-robust-amdp} does not know these
quantities, but its dyadic grid contains a discount factor 
whose effective horizon is within a factor of two of the ideal horizon. At this
grid point, one of the two candidate policies is $O(\varepsilon)$-optimal, and
its lower-confidence bound is at least
$\rho^{\star,\sigma}-O(\varepsilon)$. Because all the lower-confidence bounds
are valid and the algorithm selects the largest one,
\[
    \begin{aligned}
        \rho^{\widehat\pi,\sigma}
        &\ge
        \max_{\gamma\in\Gamma_N}
        \left\{
        \operatorname{LCB}_{\gamma}^{\mathrm{nom}}
        (\widehat\pi_\gamma^0),
        \operatorname{LCB}_{\gamma}^{\mathrm{rob}}
        (\widehat\pi_\gamma)
        \right\} 
        \ge
        \rho^{\star,\sigma}-O(\varepsilon).
    \end{aligned}
\]
The complete proof is given in
Appendix~\suppSpanAgnosticHorizonCalibration{}.

\begin{figure}[t]
    \centering
    \includegraphics[width=\linewidth]{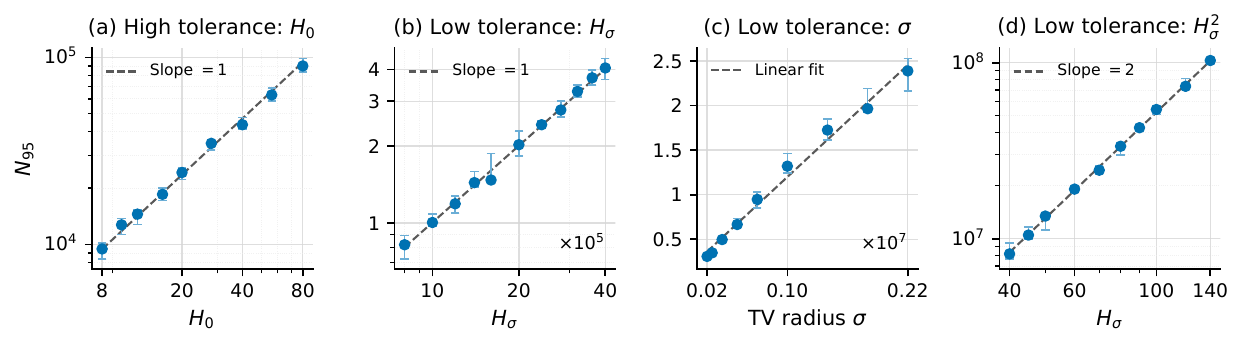}
    \caption{High- and low-tolerance sample-complexity checks.
        All four panels report $N_{95}$.
        (a) High tolerance, with $H_0 < H_\sigma$: $N_{95}$ against $H_0$.
        (b) Low tolerance, with $H_0 > H_\sigma$: the minimum-span component.
        (c)--(d) Low tolerance: dependence on $\sigma$ and
        $H_\sigma^2$ in the robustness-specific component.
        Bars are $95\%$ bootstrap intervals for $N_{95}$.
        Panels (a), (b), and (d) use log-log axes; the dashed lines have
        the indicated slopes. The tolerance $\varepsilon$ is held fixed
        within each panel.}
    \label{fig:component-rate-checks}
\end{figure}

\begin{figure}[t]
    \centering
    \includegraphics[width=\linewidth]{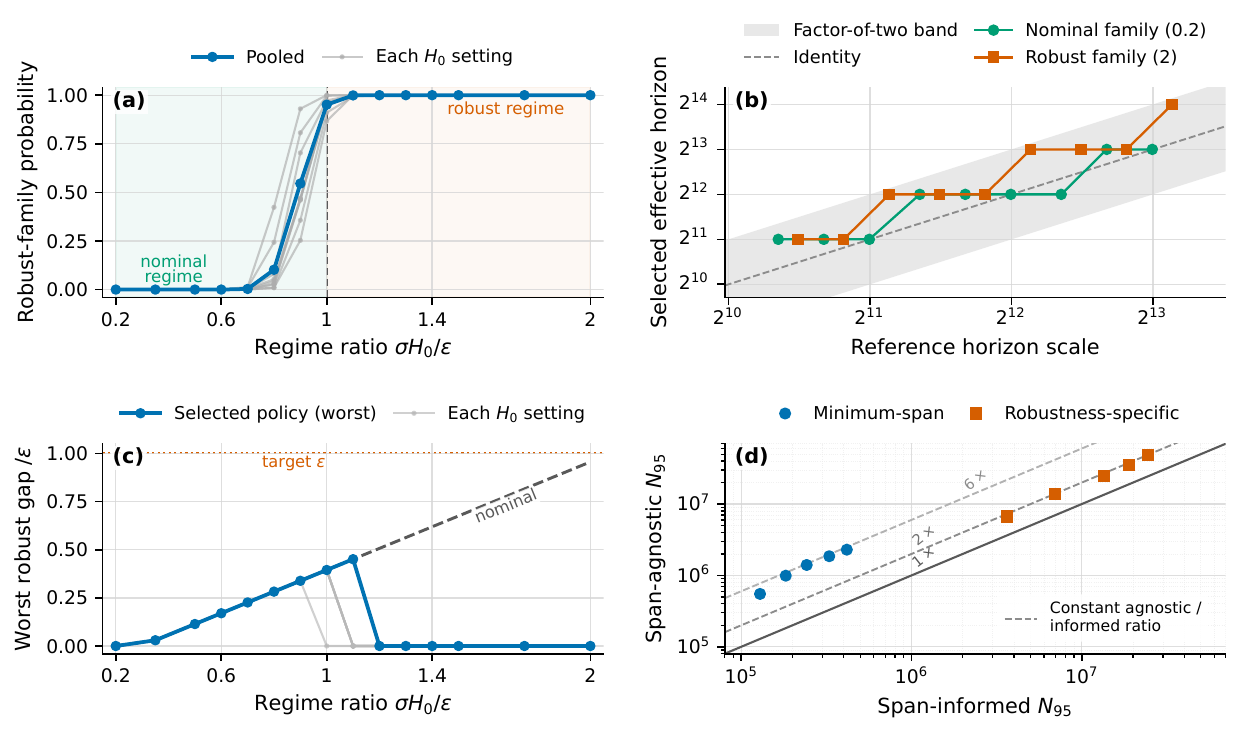}
    \caption{Span-agnostic adaptation.
        (a) Policy-family selection across the transition from the nominal
        to the robust family.
        (b) Selected effective horizon $(1-\gamma)^{-1}$ at the nominal-family
        and robust-family endpoints, compared with $H_0/\varepsilon$ and
        $H_\sigma/\varepsilon$, respectively; the shaded region is the
        factor-of-two band. (c) Robust performance of the selected and nominal policies.
        (d) Comparison of span-agnostic and span-informed $N_{95}$.}
    \label{fig:selector-adaptivity}
\end{figure}

\section{Experiments}
\label{sec:experiments}

We use controlled simulations to evaluate two parts of the theory. We first
verify the sample-complexity rates in
Theorem~\ref{thm:TV-upper-bound-average}, including the uniform linear dependence on
$\min\{H_0,H_\sigma\}$ and the robustness-specific dependence on
$\sigma H_\sigma^2$ in the low-tolerance regime. We then showcase the adaptivity of the span-agnostic
approach in Algorithm~\ref{alg:anchored-span-agnostic-robust-amdp}.

The experiments use finite AMDPs with various parameters. In each setting, the learner
receives $N$ nominal next-state samples per state-action pair.
We summarize sample cost by $N_{95}$, the sample size at which the estimated probability of returning a policy that is $\varepsilon$-optimal in robust average reward reaches $0.95$.
Because success should increase with $N$, we estimate $N_{95}$ by fitting a
nondecreasing success curve and interpolating its $0.95$ crossing.
Appendix~\suppExperimentDetails{} gives further experimental details and
additional figures.

\subsection{Empirical sample-complexity rates of Algorithm~\ref{alg:robust-amdp}}

For per-state-action sample complexity, Theorem~\ref{thm:TV-upper-bound-average} predicts the high-tolerance rate
\[
    \frac{\min\{H_0,H_\sigma\}}{\varepsilon^2}
\]
and the low-tolerance rate
\[
    \frac{\min\{H_0,H_\sigma\}+\sigma H_\sigma^2}{\varepsilon^2}.
\]

We first test the high-tolerance branch with $7\sigma H_0\le\varepsilon$ in the ordering $H_0<H_\sigma$.
We vary $H_0$ while holding $\varepsilon$ and $\sigma$ fixed. The
construction gives $H_\sigma=H_0/(1-\sigma H_0)>H_0$, so
$\min\{H_0,H_\sigma\}=H_0$. Figure~\hyperref[fig:component-rate-checks]
{\ref*{fig:component-rate-checks}a} shows the predicted linear increase of
$N_{95}$ with $H_0$.

For the low-tolerance branch, we isolate its two rate components when $H_\sigma\le H_0$.
Figure~\hyperref[fig:component-rate-checks]
{\ref*{fig:component-rate-checks}b} varies $H_\sigma$ below a fixed $H_0$
while keeping $\sigma H_\sigma^2$ small relative to the minimum-span component.
It shows the expected linear dependence on $H_\sigma$. The complementary
ordering $H_0<H_\sigma$ yields a similar result and is reported in
Appendix~\suppMinimumSpanComplementary{}.
Finally, Figures~\hyperref[fig:component-rate-checks]
{\ref*{fig:component-rate-checks}c} and~\hyperref[fig:component-rate-checks]
{\ref*{fig:component-rate-checks}d} isolate the robustness-specific term.
The first varies $\sigma$ at fixed $H_\sigma$ and the second varies
$H_\sigma$ at fixed $\sigma$; the resulting linear and quadratic trends are
consistent with the $\sigma H_\sigma^2$ component in the low-tolerance rate. 

Together, Figure~\ref{fig:component-rate-checks} shows that the sample-complexity rates in Theorem~\ref{thm:TV-upper-bound-average} are consistent with the empirical behavior of Algorithm~\ref{alg:robust-amdp}.

\subsection{Span-agnostic adaptation of Algorithm~\ref{alg:anchored-span-agnostic-robust-amdp}}
\label{sec:span-agnostic-adaptation}

We next evaluate the adaptive behavior of Algorithm~\ref{alg:anchored-span-agnostic-robust-amdp}.
We test whether the span-agnostic algorithm can choose between the nominal and
robust policy families and select the discounted horizon without knowing $H_0$
or $H_\sigma$. We vary the regime ratio $\sigma H_0/\varepsilon$ across the
nominal-to-robust transition and repeat the experiment over several $H_0$
settings.

Figures~\hyperref[fig:selector-adaptivity]
{\ref*{fig:selector-adaptivity}a}
and~\hyperref[fig:selector-adaptivity]
{\ref*{fig:selector-adaptivity}b} make this adaptation visible. As the regime
ratio crosses the transition, Figure~\hyperref[fig:selector-adaptivity]
{\ref*{fig:selector-adaptivity}a} shows that the algorithm shifts from the
nominal family to the robust family. The thin gray curves show the
robust-family selection probability separately for each $H_0$ setting, while
the blue curve pools trials across settings; their close agreement shows that
the transition is driven by the regime ratio rather than by a particular
$H_0$ setting. Figure~\hyperref[fig:selector-adaptivity]
{\ref*{fig:selector-adaptivity}b} examines horizon selection at the two
endpoints of this transition. At $\sigma H_0/\varepsilon=0.2$, the algorithm
selects the nominal family, and we compare its selected horizon with
$H_0/\varepsilon$. At $\sigma H_0/\varepsilon=2$, it selects the robust
family, and we compare its selected horizon with $H_\sigma/\varepsilon$. In
both cases, the selected effective horizon $(1-\gamma)^{-1}$ remains within a
factor of two of the reference scale for the corresponding policy family.

Figures~\hyperref[fig:selector-adaptivity]
{\ref*{fig:selector-adaptivity}c}
and~\hyperref[fig:selector-adaptivity]
{\ref*{fig:selector-adaptivity}d} connect these adaptive choices to performance.
Figure~\hyperref[fig:selector-adaptivity]
{\ref*{fig:selector-adaptivity}c} shows that the adaptively selected policies continue to meet the target robust accuracy as uncertainty grows, while the
nominal policy deteriorates. Figure~\hyperref[fig:selector-adaptivity]
{\ref*{fig:selector-adaptivity}d} further illustrates that the span-agnostic sample complexity thresholds remain
within constant factors of the span-informed benchmark. Together,
Figure~\ref{fig:selector-adaptivity} shows that
Algorithm~\ref{alg:anchored-span-agnostic-robust-amdp} adapts both its policy
family and effective horizon without span information, while maintaining
robust accuracy and sample complexity within constant factors of the
span-informed benchmark.

\section{Discussion}\label{sec:discussion}

In this paper, we develop a minimax theory for learning distributionally robust
average-reward MDPs from a generative model. We provide 
span-informed and span-agnostic reduction-based procedures that achieve minimax-optimal sample complexity. We identify the scale $\sigma H_0$ that separates
the high- and low-tolerance regimes.
In the high-tolerance regime, the optimal sample complexity is determined by
$SA\min\{H_0,H_\sigma\}\varepsilon^{-2}$, while an additional term $SA\sigma H_\sigma^2\varepsilon^{-2}$ appears in the low-tolerance regime. We supplement our theoretical results with numerical experiments that support our findings.

While this work provides a tight characterization of robust AMDPs in the generative model setting,
several avenues remain open for future work. We assume $(s, a)$-rectangular TV uncertainty sets. Exploring more coupled uncertainty structures or other divergences could expand our understanding
of the cost of distributional robustness. We also focus on the generative-model setting, where the agent has
access to a simulator to obtain samples for each state-action pair. A natural next step is to investigate regret
bounds or sample complexity in the online setting, where the agent must explore the environment without a
simulator.

\section*{Acknowledgments}

The work of Y. Yang and Y. Chi is supported in part by NSF under ECCS-2537078, ECCS-2537189, and CNS-2148212.

Y.~Chen is supported in part by the Alfred P. Sloan Research Fellowship, the NSF grants IIS-2218773 and CIF-2221009, the ONR grant N00014-25-1-2344, the AFOSR grant FA9550261B178, and the Wharton AI \& Analytics Initiative's AI Research Fund. This work is also supported in part by the NSF under Cooperative Agreement No. 2433450.

\bibliography{reference,bibfileRL,bibfileDRO}
\bibliographystyle{apalike}

\appendix

\section{Robust Bellman equations and optimal bias spans}
\label{app:technical-lemmas}
This appendix develops the Bellman theory underlying the nominal and robust
optimal bias spans. It establishes existence and minimum-span attainment of
optimal Bellman solutions, shows that $H_\sigma$ controls the span of the
robust discounted optimal value, and compares $H_\sigma$ with the
uniform-over-kernels span used in prior work.

\subsection{Fixed-policy robust average-reward verification}
\label{subsec:fixed-policy-robust-verification}

The following lemma is the fixed-policy robust analogue of the
average-reward verification argument in \citet[Section~8]{puterman2014markov}.
For a stationary policy $\pi$, write
$r^\pi(s)=\sum_{a\in\cA}\pi(a\mid s)r(s,a)$. Recall that $P_\pi$ denotes the
transition matrix induced by $\pi$ and a transition kernel $P$, so that
\[
    (P_\pi h)(s)
    =
    \sum_{a\in\cA}\pi(a\mid s)
    \sum_{s'\in\cS}P_{s,a}(s')h(s').
\]

\begin{lemma}[Fixed-policy robust average-reward verification]
    \label{lem:fixed-policy-robust-verification}
    Fix a stationary policy $\pi$ and suppose Assumption~\mainUnichainAssumption{}
    holds. Let $h:\cS\to\mathbb{R}$ be bounded and let $\rho\in\mathbb{R}$.
    If, for all $s\in\cS$,
    \[
        \rho+h(s)
        \ge
        r^\pi(s)+\inf_{P\in\cP}(P_\pi h)(s),
    \]
    then $\rho^{\pi,\sigma}\le\rho$. Conversely, if, for all $s\in\cS$,
    \[
        \rho+h(s)
        \le
        r^\pi(s)+\inf_{P\in\cP}(P_\pi h)(s),
    \]
    then $\rho^{\pi,\sigma}\ge\rho$.
\end{lemma}

\begin{proof}

We prove the two directions separately.

\paragraph*{Upper bound.}
We first prove $\rho^{\pi,\sigma}\le\rho$.
For every $(s,a)$, compactness of $\mathcal P_{s,a}$ allows us to choose
\[
    P^h_{s,a}
    \in
    \arg\min_{P_{s,a}\in\mathcal P_{s,a}}P_{s,a}^\top h.
\]
By rectangularity, these transition vectors define a kernel
$P^h\in\mathcal P$. Consequently,
\[
    (P^h_\pi h)(s)
    =
    \inf_{P\in\mathcal P}(P_\pi h)(s),
    \qquad s\in\mathcal S.
\]
The assumed inequality gives
\[
    \rho+h(s)\ge r^\pi(s)+(P^h_\pi h)(s),\qquad \forall s\in\cS.
\]
By Assumption~\mainUnichainAssumption{}, the Markov chain induced by
$(\pi,P^h)$ has a stationary distribution $\mu^h$ supported on its unique
recurrent class, and its average reward is state independent. Multiplying the
preceding inequality by $(\mu^h)^\top$ gives
\[
    \rho+(\mu^h)^\top h
    \ge
    (\mu^h)^\top r^\pi+(\mu^h)^\top P^h_\pi h.
\]
Since $(\mu^h)^\top P^h_\pi=(\mu^h)^\top$ and the stationary average reward
equals $\rho_{P^h}^\pi$, we have
\[
    (\mu^h)^\top P^h_\pi h=(\mu^h)^\top h,
    \qquad
    (\mu^h)^\top r^\pi=\rho_{P^h}^\pi.
\]
Consequently,
\[
    \rho+(\mu^h)^\top h
    \ge
    \rho_{P^h}^\pi+(\mu^h)^\top h.
\]
Thus $\rho_{P^h}^{\pi}\le\rho$. Since
$\rho^{\pi,\sigma}=\inf_{P\in\cP}\rho_P^\pi$, we obtain
$\rho^{\pi,\sigma}\le\rho$.

\paragraph*{Lower bound.} We now prove $\rho^{\pi,\sigma}\ge\rho$.
Fix an arbitrary $P\in\cP$. Since
$\inf_{P'\in\cP}(P'_\pi h)(s)\le(P_\pi h)(s)$, the assumed inequality
implies
\[
    \rho+h(s)\le r^\pi(s)+(P_\pi h)(s),\qquad \forall s\in\cS.
\]
Let $\mu$ be a stationary distribution supported on the unique recurrent
class induced by $(\pi,P)$. Multiplying the preceding inequality by
$\mu^\top$ gives
\[
    \rho+\mu^\top h
    \le
    \mu^\top r^\pi+\mu^\top P_\pi h.
\]
Since $\mu^\top P_\pi=\mu^\top$ and $\mu^\top r^\pi=\rho_P^\pi$, it follows
that
\[
    \rho+\mu^\top h
    \le
    \rho_P^\pi+\mu^\top h.
\]
Therefore $\rho_P^\pi\ge\rho$. Taking the infimum over $P\in\cP$ gives
$\rho^{\pi,\sigma}\ge\rho$.
\end{proof}

\subsection{Existence and attainment of robust optimal biases}
\label{subsec:robust-bellman-existence-attainment}
\label{subsec:robust-bellman-existence}

This subsection establishes the Bellman foundations needed for the robust
optimal bias span. It proves that the
robust average-reward Bellman equation has a solution and that a minimum-span
solution exists.
The same results hold for the nominal Bellman equation by applying the
same arguments to the singleton uncertainty set $\mathcal P=\{P^0\}$.

We begin by connecting the discounted value definition in
Section~\mainReductionUpperBoundSection{} with the robust Bellman equations used
below. This connection uses compactness and $(s,a)$-rectangularity, but not the
unichain assumption.

For $\gamma\in(0,1)$ and a stationary policy $\pi$, define the fixed-policy
robust discounted Bellman operator and the robust discounted Bellman
optimality operator by
\begin{subequations}
    \begin{align}
        (\mathcal T_{\gamma,\sigma}^{\pi}v)(s)
        &\coloneqq
        \sum_{a\in\mathcal A}\pi(a\mid s)
        \left\{
            r(s,a)
            +
            \gamma
            \min_{P_{s,a}\in\mathcal P_{s,a}}P_{s,a}^{\top}v
        \right\},
        \label{eq:fixed-policy-robust-discounted-bellman-operator} \\
        (\mathcal T_{\gamma,\sigma}v)(s)
        &\coloneqq
        \max_{a\in\mathcal A}
        \left\{
            r(s,a)
            +
            \gamma
            \min_{P_{s,a}\in\mathcal P_{s,a}}P_{s,a}^{\top}v
        \right\}.
        \label{eq:robust-discounted-bellman-optimality-operator}
    \end{align}
\end{subequations}

\begin{lemma}[Bellman characterization of robust discounted values]
    \label{lem:rectangular-discounted-bellman-representation}
    Fix $\gamma\in(0,1)$. For every stationary policy $\pi$, the unique fixed
    point of the operator in
    \eqref{eq:fixed-policy-robust-discounted-bellman-operator} equals the robust
    discounted value $V_\gamma^{\pi,\sigma}$ defined in
    Section~\mainReductionUpperBoundSection{}. The unique fixed point of the
    operator in \eqref{eq:robust-discounted-bellman-optimality-operator} is
    $V_\gamma^{\star,\sigma}$.
\end{lemma}

\begin{proof}
    These are the standard discounted Bellman characterizations for
    rectangular robust MDPs
    \citep{iyengar2005robust,nilim2005robust}. We give the short argument that
    also identifies their fixed points with the stationary-kernel definitions
    used in this paper. For any $f,g\in\mathbb R^{\mathcal S}$ and $(s,a)$,
    \[
        \left|
            \min_{P_{s,a}\in\mathcal P_{s,a}}P_{s,a}^{\top}f
            -
            \min_{P_{s,a}\in\mathcal P_{s,a}}P_{s,a}^{\top}g
        \right|
        \le \lVert f-g\rVert_\infty.
    \]
    Hence both operators are monotone $\gamma$-contractions and have unique
    fixed points.

    Let $v$ be the fixed point of
    $\mathcal T_{\gamma,\sigma}^{\pi}$. Compactness and rectangularity give a
    kernel $P^v\in\mathcal P$ whose rows attain the minima at $v$, so
    $v=r^\pi+\gamma P^v_\pi v=V_{\gamma,P^v}^{\pi}$. For every
    $P\in\mathcal P$,
    \[
        \mathcal T_{\gamma,\sigma}^{\pi}f
        \le r^\pi+\gamma P_\pi f.
    \]
    Iterating the right-hand side from $v$ gives
    $v\le V_{\gamma,P}^{\pi}$. Thus $P^v$ attains the componentwise infimum and
    $v=V_\gamma^{\pi,\sigma}$.

    Finally, let $V$ be the fixed point of
    $\mathcal T_{\gamma,\sigma}$ and choose a deterministic policy $\pi_V$
    greedy with respect to $V$. Then
    $\mathcal T_{\gamma,\sigma}^{\pi_V}V=V$, whereas
    $\mathcal T_{\gamma,\sigma}^{\pi}V\le V$ for every stationary policy
    $\pi$. The preceding fixed-policy characterization and monotone iteration
    therefore give
    \[
        V=V_\gamma^{\pi_V,\sigma}
        \quad\text{and}\quad
        V_\gamma^{\pi,\sigma}\le V
        \quad\text{for every }\pi.
    \]
    Hence $V=V_\gamma^{\star,\sigma}$.
\end{proof}

The next proposition proves directly that the robust average-reward Bellman
equation used in \mainRobustSpanDefinition{} has a solution for the paper's
uncertainty set. More generally, the proof applies whenever each local
uncertainty set is a nonempty compact subset of the probability simplex and
the resulting product set satisfies Assumption~\mainUnichainAssumption{}.

\begin{proposition}[Existence of a robust optimal Bellman solution]
    \label{prop:robust-bellman-existence}
    There exists $h\in\mathbb R^{\mathcal S}$ satisfying
    \[
        \rho^{\star,\sigma}\mathbf 1+h
        =
        \mathcal T_\sigma h.
    \]
\end{proposition}

\begin{proof}
    We first establish the uniform discounted span bound needed for a
    vanishing-discount argument. Fix a deterministic stationary policy $\pi$
    and $P\in\mathcal P$. By Assumption~\mainUnichainAssumption{}, $P_\pi$ has
    a unique stationary distribution $\mu_P^\pi$, and
    $\rho_P^\pi=(\mu_P^\pi)^\top r^\pi$. Select the normalized bias
    \begin{equation}
        h_P^\pi
        =
        \left(I-P_\pi+\mathbf 1(\mu_P^\pi)^\top\right)^{-1}
        \left(r^\pi-\rho_P^\pi\mathbf 1\right).
        \label{eq:policy-bias-continuous-selection}
    \end{equation}
    Equivalently,
    \begin{equation}
        \left(I-P_\pi+\mathbf 1(\mu_P^\pi)^\top\right)h_P^\pi
        =r^\pi-\rho_P^\pi\mathbf 1.
        \label{eq:policy-bias-linear-system}
    \end{equation}
    Standard finite-state unichain theory gives
    $(\mu_P^\pi)^\top h_P^\pi=0$ and the Poisson equation
    \begin{equation}
        \rho_P^\pi\mathbf 1+h_P^\pi
        =r^\pi+P_\pi h_P^\pi.
        \label{eq:policy-poisson-equation}
    \end{equation}
    The vector in \eqref{eq:policy-bias-continuous-selection} is the
    corresponding deviation matrix applied to $r^\pi$. That matrix is
    continuous and uniformly bounded over compact unichain families
    \citep[Appendix, Lemma~2]{Wang2023Robust}. Since rewards are bounded and
    the deterministic stationary policies form a finite set,
    \begin{equation}
        B\coloneqq
        \sup_{\substack{P\in\mathcal P,\pi}}
        \|h_P^\pi\|_{\mathrm{span}}
        <\infty,
        \label{eq:uniform-policy-bias-bound}
    \end{equation}
    where the supremum is over deterministic stationary policies.

    Fix $\gamma\in(0,1)$. Choose rowwise minimizers at
    $V_\gamma^{\star,\sigma}$ and assemble them into
    $P_\gamma\in\mathcal P$, and let $\pi_\gamma$ be a deterministic greedy
    policy. Lemma~\ref{lem:rectangular-discounted-bellman-representation}
    yields
    \begin{equation}
        V_\gamma^{\star,\sigma}
        =r^{\pi_\gamma}
        +\gamma(P_\gamma)_{\pi_\gamma}V_\gamma^{\star,\sigma}.
        \label{eq:robust-value-as-fixed-kernel-value}
    \end{equation}
    Write $Q=(P_\gamma)_{\pi_\gamma}$,
    $\rho=\rho_{P_\gamma}^{\pi_\gamma}$, and
    $h=h_{P_\gamma}^{\pi_\gamma}$. Combining
    \eqref{eq:robust-value-as-fixed-kernel-value} with the Poisson equation
    gives the resolvent identity
    \[
        V_\gamma^{\star,\sigma}
        =
        \frac{\rho}{1-\gamma}\mathbf 1+h-q_\gamma,
        \qquad
        q_\gamma
        =(1-\gamma)(I-\gamma Q)^{-1}Qh
        =(1-\gamma)\sum_{t=0}^\infty\gamma^tQ^{t+1}h.
    \]
    Each coordinate of $q_\gamma$ lies between the minimum and maximum
    coordinates of $h$. Hence
    $\|q_\gamma\|_{\mathrm{span}}\le\|h\|_{\mathrm{span}}$, and therefore
    \begin{equation}
        \sup_{\gamma\in(0,1)}
        \|V_\gamma^{\star,\sigma}\|_{\mathrm{span}}
        \le 2B<\infty.
        \label{eq:uniform-robust-discounted-span}
    \end{equation}

    We now apply the vanishing-discount argument. Choose
    $\gamma_n\to1$, fix $s_0\in\mathcal S$, and set
    \[
        u_n
        =V_{\gamma_n}^{\star,\sigma}
        -V_{\gamma_n}^{\star,\sigma}(s_0)\mathbf 1.
    \]
    By \eqref{eq:uniform-robust-discounted-span}, $(u_n)$ is bounded; moreover,
    rewards in $[0,1]$ give
    $0\le(1-\gamma_n)V_{\gamma_n}^{\star,\sigma}(s_0)\le1$.
    Passing to a subsequence, let
    \[
        u_n\to h,
        \qquad
        (1-\gamma_n)V_{\gamma_n}^{\star,\sigma}(s_0)\to\bar\rho.
    \]
    Translation of the discounted Bellman equation gives
    \begin{equation}
        (1-\gamma_n)V_{\gamma_n}^{\star,\sigma}(s_0)\mathbf 1+u_n
        =\mathcal T_{\gamma_n,\sigma}u_n.
        \label{eq:normalized-discounted-robust-bellman}
    \end{equation}
    Since
    \[
        \|\mathcal T_{\gamma_n,\sigma}u_n-\mathcal T_\sigma h\|_\infty
        \le
        \gamma_n\|u_n-h\|_\infty+(1-\gamma_n)\|h\|_\infty
        \longrightarrow0,
    \]
    taking limits yields
    \begin{equation}
        \bar\rho\mathbf 1+h=\mathcal T_\sigma h.
        \label{eq:limit-robust-bellman}
    \end{equation}

    For every stationary policy $\pi$, rectangularity and
    \eqref{eq:limit-robust-bellman} give
    \[
        \bar\rho\mathbf 1+h
        \ge r^\pi+\inf_{P\in\mathcal P}P_\pi h.
    \]
    Lemma~\ref{lem:fixed-policy-robust-verification} therefore implies
    $\rho^{\pi,\sigma}\le\bar\rho$, and hence
    $\rho^{\star,\sigma}\le\bar\rho$. Conversely, a deterministic policy
    $\bar\pi$ greedy with respect to $h$ makes the preceding inequality an
    equality. The reverse direction of the same lemma gives
    $\bar\rho\le\rho^{\bar\pi,\sigma}\le\rho^{\star,\sigma}$. Thus
    $\bar\rho=\rho^{\star,\sigma}$, and
    \eqref{eq:limit-robust-bellman} proves the claim.

\end{proof}

\begin{proposition}[Existence of a minimum-span robust Bellman solution]
    \label{prop:minimum-span-attainment}
    There exists $h^{\star,\sigma}\in\mathbb R^{\mathcal S}$ satisfying
    \[
        \rho^{\star,\sigma}\mathbf 1+h^{\star,\sigma}
        =
        \mathcal T_\sigma h^{\star,\sigma}
    \]
    and
    \[
        \|h^{\star,\sigma}\|_{\mathrm{span}}
        =
        \inf_{h:\,\rho^{\star,\sigma}\mathbf 1+h=\mathcal T_\sigma h}
        \|h\|_{\mathrm{span}}.
    \]
    Thus the inner infimum in \mainRobustSpanDefinition{} is attained.
\end{proposition}

\begin{proof}
    Proposition~\ref{prop:robust-bellman-existence} makes the Bellman-solution
    set nonempty. Fix $s_0\in\mathcal S$ and normalize its elements at $s_0$:
    \[
        \mathcal H_0
        \coloneqq
        \left\{
            h:h(s_0)=0,\quad
            \rho^{\star,\sigma}\mathbf 1+h=\mathcal T_\sigma h
        \right\}.
    \]
    This normalization does not change the span because
    $\mathcal T_\sigma(h+c\mathbf 1)=\mathcal T_\sigma h+c\mathbf 1$.
    The operator $\mathcal T_\sigma$ is $1$-Lipschitz in the supremum norm, so
    $\mathcal H_0$ is closed. Moreover, if $h\in\mathcal H_0$ and
    $\|h\|_{\mathrm{span}}\le C$, then
    $\|h\|_\infty\le C$. Hence every bounded-span sublevel set of
    $\mathcal H_0$ is compact.

    Let
    \[
        m=\inf_{h\in\mathcal H_0}\|h\|_{\mathrm{span}}
    \]
    and choose $h_n\in\mathcal H_0$ with
    $\|h_n\|_{\mathrm{span}}\le m+1/n$. The sequence lies in a compact
    bounded-span sublevel set, so a subsequence converges to some
    $h^{\star,\sigma}\in\mathcal H_0$. Continuity of the span seminorm gives
    $\|h^{\star,\sigma}\|_{\mathrm{span}}=m$, proving the claim.
\end{proof}

\subsection{Discounted-value control by the robust bias span}
\label{subsec:robust-discounted-span-control}
\label{subsec:bias-to-discounted-comparison}

The next lemma shows that $H_\sigma$ controls the span of the robust discounted
optimal value. 

\begin{lemma}[Robust discounted optimal-value span]
    \label{lem:robust-value-function-span-bound}
    For every $\gamma\in(0,1)$,
    \[
        \|V_\gamma^{\star,\sigma}\|_{\mathrm{span}}
        \le
        2H_\sigma.
    \]
\end{lemma}

\begin{proof}
    Fix $\gamma\in(0,1)$.
    We first state the comparison used below. Let $h\in\mathbb R^{\mathcal S}$
    satisfy
    \[
        \rho^{\star,\sigma}\mathbf 1+h
        =
        \mathcal T_\sigma h,
    \]
    and suppose that $\min_{s\in\mathcal S}h(s)=0$. Then
    \begin{equation}
        h+
        \left(
            \frac{\rho^{\star,\sigma}}{1-\gamma}
            -
            \|h\|_{\mathrm{span}}
        \right)\mathbf 1
        \le
        V_\gamma^{\star,\sigma}
        \le
        h+
        \frac{\rho^{\star,\sigma}}{1-\gamma}\mathbf 1.
        \label{eq:bias-discounted-comparison}
    \end{equation}

    By Proposition~\ref{prop:minimum-span-attainment}, the robust Bellman
    solution $h^{\star,\sigma}$ attaining the minimum in
    \mainRobustSpanDefinition{} exists. Shift it by a constant so that
    $\min_s h^{\star,\sigma}(s)=0$. 
    For any $s,t\in\mathcal S$, applying the upper comparison in
    \eqref{eq:bias-discounted-comparison} at $s$ and the lower comparison at
    $t$ gives
    \[
        V_\gamma^{\star,\sigma}(s)
        -V_\gamma^{\star,\sigma}(t)
        \le
        h^{\star,\sigma}(s)-h^{\star,\sigma}(t)
        +\|h^{\star,\sigma}\|_{\mathrm{span}}
        \le
        2H_\sigma.
    \]
    Taking the maximum over $s,t$ proves the claimed span bound.

    \paragraph*{Proof of \eqref{eq:bias-discounted-comparison}.}
    The condition $\min_{s\in\mathcal S}h(s)=0$ gives
    \[
        0
        \le
        h
        \le
        \|h\|_{\mathrm{span}}\mathbf 1.
    \]
    Hence, for every $(s,a)$,
    \[
        0
        \le
        \min_{P_{s,a}\in\mathcal P_{s,a}}P_{s,a}^{\top}h
        \le
        \|h\|_{\mathrm{span}}.
    \]
    Comparing the definition of $\mathcal T_\sigma$ in
    Section~\mainProblemSetupSection{} with the definition of
    $\mathcal T_{\gamma,\sigma}$ in
    \eqref{eq:robust-discounted-bellman-optimality-operator} gives
    \[
        \mathcal T_\sigma h
        -(1-\gamma)\|h\|_{\mathrm{span}}\mathbf 1
        \le
        \mathcal T_{\gamma,\sigma}h
        \le
        \mathcal T_\sigma h.
    \]
    Using the robust average-reward Bellman equation
    \mainRobustAverageRewardBellmanEquation{},
    \begin{equation}
        \rho^{\star,\sigma}\mathbf 1+h
        -(1-\gamma)\|h\|_{\mathrm{span}}\mathbf 1
        \le
        \mathcal T_{\gamma,\sigma}h
        \le
        \rho^{\star,\sigma}\mathbf 1+h.
        \label{eq:discounted-operator-at-bias}
    \end{equation}
    For every $v\in\mathbb R^{\mathcal S}$ and $c\in\mathbb R$, the discounted
    Bellman operator satisfies the shift identity
    \[
        \mathcal T_{\gamma,\sigma}(v+c\mathbf 1)
        =
        \mathcal T_{\gamma,\sigma}v+\gamma c\mathbf 1.
    \]
    Applying the lower bound in
    \eqref{eq:discounted-operator-at-bias} yields
    \begin{align*}
        &\mathcal T_{\gamma,\sigma}
        \left(
            h+
            \left(
                \frac{\rho^{\star,\sigma}}{1-\gamma}
                -
                \|h\|_{\mathrm{span}}
            \right)\mathbf 1
        \right)
        \\
        &\qquad\ge
        h+
        \left(
            \frac{\rho^{\star,\sigma}}{1-\gamma}
            -
            \|h\|_{\mathrm{span}}
        \right)\mathbf 1.
    \end{align*}
    Thus the lower comparison vector in
    \eqref{eq:bias-discounted-comparison} is a subsolution. Similarly, the
    upper bound in
    \eqref{eq:discounted-operator-at-bias} yields
    \[
        \mathcal T_{\gamma,\sigma}
        \left(
            h+
            \frac{\rho^{\star,\sigma}}{1-\gamma}\mathbf 1
        \right)
        \le
        h+
        \frac{\rho^{\star,\sigma}}{1-\gamma}\mathbf 1,
    \]
    so the upper comparison vector is a supersolution. By
    Lemma~\ref{lem:rectangular-discounted-bellman-representation}, the unique
    fixed point of $\mathcal T_{\gamma,\sigma}$ is
    $V_\gamma^{\star,\sigma}$. Iterating this monotone contraction from the
    subsolution and the supersolution proves
    \eqref{eq:bias-discounted-comparison}.
\end{proof}

\subsection{Comparison with other robust bias span parameters}
\label{subsec:comparison-roch-span}

Our bounds are expressed in terms of the minimum robust optimal bias span
$H_\sigma$, whereas the bounds of \citet{roch2025reduction} and
\citet{roch2025provably} use different robust span parameters. Comparing the
rates therefore requires comparing the parameters themselves. We show
below that $H_\sigma$ is no larger than either of the two parameters and then
explain why $H_\sigma$ suffices for our analysis. Throughout this subsection, we
truncate all span parameters from below at $1$.\footnote{\citet{roch2025reduction} define their parameter
without this truncation, although their reduction proof uses the corresponding
normalization truncated below at $1$; \citet{roch2025provably} assume that their
parameter is at least $1$.}

\paragraph*{Prior span parameters.}
We start with the span parameter used in \citet{roch2025reduction}.
Let
$(\pi^\star_{\mathrm{Roch}},h^\star_{\mathrm{Roch}})$ be a robust-optimal
policy--bias pair satisfying
\[
    \rho^{\star,\sigma}\bm{1}_S+h^\star_{\mathrm{Roch}}
    =
    \mathcal T_\sigma h^\star_{\mathrm{Roch}},
\]
with $\pi^\star_{\mathrm{Roch}}$ attaining the statewise maximum in
$\cT_\sigma$. Their span parameter, in the form used by their proof, is
\[
    H_{\mathrm{Roch}}
    \coloneqq
    \max\left\{
        1,
        \max_{P\in\mathcal P}
        \|h_P^{\pi^\star_{\mathrm{Roch}}}\|_{\mathrm{span}}
    \right\},
\]
where $h_P^{\pi^\star_{\mathrm{Roch}}}$ is the ordinary average-reward bias of
$\pi^\star_{\mathrm{Roch}}$ under kernel $P$. Thus,
$H_{\mathrm{Roch}}$ controls the bias of one robust-optimal policy uniformly
over every kernel in the uncertainty set.

The robust Halpern iteration result of \citet{roch2025provably} instead uses
\[
    H_{\mathrm{RHI}}
    \coloneqq
    \max\left\{
        1,\,
        \max_{\substack{
            \pi:\,\rho^{\pi,\sigma}=\rho^{\star,\sigma}\\
            P\in\mathcal P:\,\rho_P^\pi=\rho^{\pi,\sigma}
        }}
        \|h_P^\pi\|_{\mathrm{span}}
    \right\}.
\]
This parameter controls the ordinary biases associated with all
robust-optimal policies and their worst-case kernels.

\paragraph*{Comparison with $H_\sigma$.}
Our definition minimizes the span over robust Bellman solutions. Since
$h^\star_{\mathrm{Roch}}$ is one such solution,
\[
    H_\sigma
    \le
    \max\left\{
        1,\,
        \|h^\star_{\mathrm{Roch}}\|_{\mathrm{span}}
    \right\}.
\]
To relate $h^\star_{\mathrm{Roch}}$ to the ordinary bias functions appearing in
the two prior parameters, we associate it with a kernel in the uncertainty set.
For every state, choose transition rows attaining the minima in the
fixed-policy robust Bellman operator. Compactness guarantees that these
minimizers exist, and rectangularity allows them to be assembled into a single
kernel $P^\dagger\in\mathcal P$ satisfying
\[
    \bigl(P^\dagger_{\pi^\star_{\mathrm{Roch}}}
    h^\star_{\mathrm{Roch}}\bigr)(s)
    =
    \inf_{P\in\mathcal P}
    \bigl(P_{\pi^\star_{\mathrm{Roch}}}
    h^\star_{\mathrm{Roch}}\bigr)(s),
    \qquad s\in\mathcal S.
\]
Because $\pi^\star_{\mathrm{Roch}}$ attains the maximum in the robust Bellman
equation, we therefore have
\[
    \rho^{\star,\sigma}\bm{1}_S+h^\star_{\mathrm{Roch}}
    =
    r^{\pi^\star_{\mathrm{Roch}}}
    +\inf_{P\in\mathcal P}
    P_{\pi^\star_{\mathrm{Roch}}}h^\star_{\mathrm{Roch}}
    =
    r^{\pi^\star_{\mathrm{Roch}}}
    +P^\dagger_{\pi^\star_{\mathrm{Roch}}}
    h^\star_{\mathrm{Roch}}.
\]
Applying both directions of
Lemma~\ref{lem:fixed-policy-robust-verification} to the first equality gives
$\rho^{\pi^\star_{\mathrm{Roch}},\sigma}=\rho^{\star,\sigma}$. The second
equality, together with Assumption~\mainUnichainAssumption{}, identifies
$\rho^{\star,\sigma}$ as the average reward of
$\pi^\star_{\mathrm{Roch}}$ under $P^\dagger$. Hence
$\rho_{P^\dagger}^{\pi^\star_{\mathrm{Roch}}}=\rho^{\star,\sigma}$ and
$h^\star_{\mathrm{Roch}}$ is, up to an additive constant, the ordinary bias
$h_{P^\dagger}^{\pi^\star_{\mathrm{Roch}}}$. Thus $P^\dagger$ is a worst-case
kernel for $\pi^\star_{\mathrm{Roch}}$.

The definition of $H_{\mathrm{Roch}}$ maximizes over all kernels in
$\mathcal P$ for the fixed policy $\pi^\star_{\mathrm{Roch}}$, so its
maximization includes $P^\dagger$. Moreover,
$\pi^\star_{\mathrm{Roch}}$ is robust optimal and $P^\dagger$ is worst-case for
this policy, so the pair $(\pi^\star_{\mathrm{Roch}},P^\dagger)$ is included in
the maximization defining $H_{\mathrm{RHI}}$. Therefore,
\[
    \max\left\{
        1,
        \|h^\star_{\mathrm{Roch}}\|_{\mathrm{span}}
    \right\}
    \le
    H_{\mathrm{Roch}}
    \qquad\text{and}\qquad
    \max\left\{
        1,
        \|h^\star_{\mathrm{Roch}}\|_{\mathrm{span}}
    \right\}
    \le
    H_{\mathrm{RHI}}.
\]
Combining these inequalities with the preceding bound on $H_\sigma$ gives
\[
    H_\sigma
    \le
    \min\left\{
        H_{\mathrm{Roch}},
        H_{\mathrm{RHI}}
    \right\}.
\]

\paragraph*{Why $H_\sigma$ suffices.}
Our reduction only requires uniform control of the span of the robust discounted
optimal value. Lemma~\ref{lem:robust-value-function-span-bound} provides exactly
this control: for every $\gamma\in(0,1)$,
\[
    \|V_\gamma^{\star,\sigma}\|_{\mathrm{span}}
    \le
    2H_\sigma.
\]
This bound is obtained directly from a minimum-span robust Bellman solution.
Consequently, our analysis does not require uniform control over all kernels or
over all robust-optimal bias functions.

\section{Proofs for Section~\mainProblemSetupSection{}}
\label{app:proofs-problem-setup}
This appendix contains the proofs of the propositions stated in
Section~\mainProblemSetupSection{}.

\subsection{Proof of Proposition~\mainSpanIndependenceProposition{}}
\label{subsec:proof-span-independence}

Fix $\sigma>0$ and $H_0,H_\sigma\ge1$. If $H_0\ne H_\sigma$, choose
\[
    0<g\le
    \min\left\{
        1,
        \frac{\sigma}{\left|H_0^{-1}-H_\sigma^{-1}\right|}
    \right\}.
\]
If $H_0=H_\sigma$, set $g=1$. Consider a two-state AMDP with one action and
rewards
\[
    r(1)=g,
    \qquad
    r(2)=0.
\]
Because there is only one action, we suppress its action index below.
Denote the local TV radii at states $1$ and $2$ by $\sigma_1$ and
$\sigma_2$, respectively.
In both cases below, let the nominal transition vectors be
\[
    P^0_1=(1,0),
    \qquad
    P^0_2=
    \left(\frac{g}{H_0},1-\frac{g}{H_0}\right).
\]
The nominal Bellman equation is
\[
    \rho+h(1)=g+h(1),
    \qquad
    \rho+h(2)
    =
    \frac{g}{H_0}h(1)
    +\left(1-\frac{g}{H_0}\right)h(2).
\]
Thus every nominal Bellman solution satisfies $\rho=g$ and
$h(1)-h(2)=H_0$. Its span is therefore exactly $H_0$.

We now consider the two cases $H_\sigma\ge H_0$ and $H_\sigma<H_0$ separately. In each case, we set the local TV radii so that the robust Bellman solution has span exactly $H_\sigma$.
\paragraph*{Case $H_\sigma\ge H_0$.}
Set the local TV radii to
\[
    \sigma_1=0,
    \qquad
    \sigma_2
    =g\left(\frac{1}{H_0}-\frac{1}{H_\sigma}\right).
\]
Since $H_\sigma\ge H_0$, $\sigma_2$ is nonnegative and at most $\sigma$ by
the choice of $g$. Every admissible transition vector at state $2$ has the
form
\[
    P_2=(p,1-p),
    \qquad
    \frac{g}{H_\sigma}
    =\frac{g}{H_0}-\sigma_2
    \le p\le
    \min\left\{1,\frac{g}{H_0}+\sigma_2\right\}.
\]
Because $\sigma_1=0$, state $1$ is absorbing. Moreover,
$p\ge g/H_\sigma>0$, so from state $2$ the chain reaches state $1$ with
positive probability at each step and state $2$ is transient. Hence every
admissible kernel has the unique recurrent class $\{1\}$. Since there is only
one action, this verifies Assumption~\mainUnichainAssumption{}.

For a robust Bellman solution, the equation at state $1$ is
\[
    \rho^{\star,\sigma}+h(1)=g+h(1),
\]
so $\rho^{\star,\sigma}=g$. The equation at state $2$ is
\[
    \rho^{\star,\sigma}+h(2)
    =
    \min_{P_2\in\mathcal P_2}
    \left\{P_2(1)h(1)+P_2(2)h(2)\right\}.
\]
Using $P_2(1)+P_2(2)=1$ and $\rho^{\star,\sigma}=g$, this equation becomes
\[
    g
    =
    \min_{P_2\in\mathcal P_2}
    P_2(1)\bigl(h(1)-h(2)\bigr).
\]
The right-hand side cannot be positive if $h(1)-h(2)\le0$. Hence
$h(1)-h(2)>0$, so the minimizing transition vector assigns the smallest possible probability
$g/H_\sigma$ to state $1$. It follows that every robust Bellman solution
satisfies
\[
    h(1)-h(2)=H_\sigma.
\]

\paragraph*{Case $H_\sigma<H_0$.}
Set
\[
    \sigma_1
    =g\left(\frac{1}{H_\sigma}-\frac{1}{H_0}\right),
    \qquad
    \sigma_2=0.
\]
Since $H_\sigma<H_0$, we have $0<\sigma_1\le\sigma$ and $\sigma_1<1$.
Every admissible kernel has transition vectors
\[
    P_1=(1-x,x),
    \qquad
    0\le x\le \sigma_1,
    \qquad
    P_2=\left(\frac{g}{H_0},1-\frac{g}{H_0}\right).
\]
When $x=0$, state $1$ is the unique recurrent class and state $2$ is
transient. When $x>0$, the chain is irreducible. Thus every admissible kernel
is unichain.

For a robust Bellman solution, the equations at states $1$ and $2$ are
\begin{align*}
    \rho^{\star,\sigma}+h(1)
    &=
    g+
    \min_{0\le x\le\sigma_1}
    \left\{(1-x)h(1)+xh(2)\right\}, \\
    \rho^{\star,\sigma}+h(2)
    &=
    \frac{g}{H_0}h(1)
    +\left(1-\frac{g}{H_0}\right)h(2).
\end{align*}
If $h(1)-h(2)\le0$, the minimum in the first equation equals $h(1)$.
The two equations would then give
\[
    \rho^{\star,\sigma}=g,
    \qquad
    \rho^{\star,\sigma}
    =\frac{g}{H_0}\bigl(h(1)-h(2)\bigr)\le0,
\]
which is impossible. Hence $h(1)-h(2)>0$, and the minimum in the first
equation is attained at $x=\sigma_1$. The two Bellman equations therefore
give
\[
    \rho^{\star,\sigma}
    =g-\sigma_1\bigl(h(1)-h(2)\bigr),
    \qquad
    \rho^{\star,\sigma}
    =\frac{g}{H_0}\bigl(h(1)-h(2)\bigr).
\]
Equating these two expressions yields
\[
    g
    =
    \left\{
        \frac{g}{H_0}
        +\sigma_1
    \right\}
    \bigl(h(1)-h(2)\bigr)
    =
    \frac{g}{H_\sigma}\bigl(h(1)-h(2)\bigr).
\]
Thus every robust Bellman solution again satisfies
$h(1)-h(2)=H_\sigma$.

In either ordering, the nominal and robust Bellman solution sets therefore
have spans exactly $H_0$ and $H_\sigma$, respectively. Since both prescribed
spans are at least one, the maxima with one in \mainNominalSpanDefinition{} and
\mainRobustSpanDefinition{} do not change these values. This proves the proposition.

\subsection{Proof of Proposition~\mainPerturbationBoundProposition{}}
\label{proof:perturbation-bound-average-reward}
We use a nominal optimal bias as a robust Bellman certificate and show that
transition perturbations reduce this certificate by at most $\sigma H_0$.

Let $h_0\coloneqq h_{P^{0}}^{\star}$. For each state $s$, choose
\[
    a_0(s)
    \in
    \operatorname*{argmax}_{a\in\mathcal A}
    \left\{
        r(s,a)
        +
        \sum_{s'\in\mathcal S}P^0_{s,a}(s')h_0(s')
    \right\},
    \qquad
    \pi_0(a\mid s)
    \coloneqq
    \mathds{1}_{\{a=a_0(s)\}}.
\]
Thus, $\pi_0$ is a deterministic policy attaining the maximum in the nominal
Bellman equation at every state and is therefore nominally optimal. Moreover,
for every stationary policy $\pi$, since
$P^{0}\in\mathcal{P}$,
\[
    \rho^{\pi,\sigma}
    =
    \inf_{P\in\mathcal P}\rho_P^\pi
    \le
    \rho_{P^0}^\pi.
\]
Taking the supremum over $\pi$ gives
\[
    \rho^{\star,\sigma}\le\rho^{\star}.
\]

It remains to lower bound $\rho^{\star,\sigma}$. By the construction of
$\pi_0$, for every state $s$,
\begin{equation}
    \rho^{\star}+h_0(s)
    =
    \sum_{a\in\mathcal{A}}\pi_0(a\mid s)
    \left\{
    r(s,a)+\sum_{s'\in\mathcal{S}}P_{s,a}^{0}(s')h_0(s')
    \right\}.
    \label{eq:nominal-bellman-equation}
\end{equation}
For any distributions $p,q$ and any vector $h$,
\[
    |(p-q)^\top h|
    \le
    \|p-q\|_{\mathrm{TV}}\|h\|_{\mathrm{span}}.
\]
By the definition of $H_0$,
$\|h_0\|_{\mathrm{span}}\le H_0$. Hence, for any $P\in\mathcal{P}$ and
any state-action pair $(s,a)$,
\[
    \sum_{s'\in\mathcal{S}}
    \left(P_{s,a}(s')-P_{s,a}^{0}(s')\right)h_0(s')
    \ge
    -\|P_{s,a}-P_{s,a}^{0}\|_{\mathrm{TV}}\|h_0\|_{\mathrm{span}}
    \ge
        -\sigma H_0.
\]
Because the policy weights sum to one, averaging this rowwise bound with
weights $\pi_0(a\mid s)$ does not enlarge the error. Taking the infimum over
$P\in\mathcal P$ therefore gives
\[
    \inf_{P\in\mathcal{P}}
    \sum_{a\in\mathcal{A}}\pi_0(a\mid s)
    \sum_{s'\in\mathcal{S}}P_{s,a}(s')h_0(s')
    \ge
    \sum_{a\in\mathcal{A}}\pi_0(a\mid s)
    \sum_{s'\in\mathcal{S}}P_{s,a}^{0}(s')h_0(s')
    -\sigma H_0.
\]
Substituting this into \eqref{eq:nominal-bellman-equation},
we have
\[
    \rho^{\star}-\sigma H_0+h_0(s)
    \le
    \sum_{a\in\mathcal{A}}\pi_0(a\mid s)r(s,a)
    +
    \inf_{P\in\mathcal{P}}
    \sum_{a\in\mathcal{A}}\pi_0(a\mid s)
    \sum_{s'\in\mathcal{S}}P_{s,a}(s')h_0(s').
\]
By Lemma~\ref{lem:fixed-policy-robust-verification}, applied to
$\pi_0$ with
$\rho=\rho^{\star}-\sigma H_{0}$ and $h=h_0$, this implies
\[
    \rho^{\pi_0,\sigma}
    \ge
    \rho^{\star}-\sigma H_{0}.
\]

Since $\rho^{\pi_0,\sigma}\le\rho^{\star,\sigma}$ and
$\rho^{\star,\sigma}\le\rho^\star$, we conclude that
\[
    \rho^\star-\sigma H_0
    \le
    \rho^{\pi_0,\sigma}
    \le
    \rho^{\star,\sigma}
    \le
    \rho^\star,
\]
which proves the proposition.

\section{Proofs for the lower bound}
\label{sec:proof_lower_bound}
This appendix proves Theorem~\mainMinimaxLowerBoundTheorem{}. We first define the
hard-instance template and compute the quantities used in the proof. We then
convert policy estimation into a two-point testing problem and verify the
required inequalities for each parameter regime.

\subsection{Instance template}
\label{subsec:Instance-construction}

We construct a pair of instances $\{\cM_{\phi}:\phi\in\{2,3\}\}$ with the
same states, actions, and rewards. The index $\phi$ determines which of the actions
$2$ and $3$ has the larger transition probability to its reward state.

\paragraph*{Definition of the instances.}

Recall that $S$ and $A$ denote the numbers of states and actions and that
$\sigma$ is the global uncertainty level. For each $\phi\in\{2,3\}$, define
$\cM_\phi$ as follows. We write its nominal kernel as $P^0$, suppressing its
dependence on $\phi$.

\begin{itemize}
    \item \textbf{Parameters.} We use
          \[
              p_0,p_1,\delta,r_0,r_1\in[0,1],
              \qquad
              q_0,q_1\in(0,1],
              \qquad
              \sigma_0,\sigma_1\in[0,\sigma].
          \]
          The three parameter choices below set these quantities
          differently. We
          always choose them so that
          \[
              \sigma_0+p_0\le 1,
              \qquad
              \sigma_1+p_1+\delta\le 1.
          \]
    \item \textbf{States and actions.} The state space is $\cS=\{1,\ldots,S\}$ and the
          action space is $\cA=\{1,\ldots,A\}$. State $S$ is the decision state, states
          $1,2,3$ are reward states, state $4$ is the delay state, and when $S\ge 5$, states
          $5,\ldots,S-1$ are padding states.
    \item \textbf{Rewards.} For any action $a$,
          \[
              r(s,a)=\begin{cases}
                  r_0 & \text{if }s=1,         \\
                  r_1 & \text{if }s\in\{2,3\}, \\
                  0   & \text{otherwise}
              \end{cases}.
          \]
    \item \textbf{Nominal transition kernel.}
          \begin{itemize}
              \item Delay and padding states: for $s\in\{4,\ldots,S-1\}$ and
                    $a\in\cA$, $P^0_{s,a}(S)=1$ and $P^0_{s,a}(s')=0$ for
                    every $s'\neq S$.
              \item Decision state: for $s=S$,
                    \begin{enumerate}
                        \item When $a=1$,
                              \[
                                  P_{s,a}^{0}(s')=\begin{cases}
                                      \sigma_0+p_{0}   & \text{if }s'=1   \\
                                      1-p_{0}-\sigma_0 & \text{if }s'=S   \\
                                      0                & \text{otherwise}
                                  \end{cases}.
                              \]
                        \item When $a\in\{2,3\}$,
                              \[
                                  P_{s,a}^{0}(s')=\begin{cases}
                                      \sigma_1+p_{1}+\delta\mathds{1}_{\{\phi=a\}}   & \text{if }s'=a   \\
                                      1-p_{1}-\sigma_1-\delta\mathds{1}_{\{\phi=a\}} & \text{if }s'=S   \\
                                      0                                              & \text{otherwise}
                                  \end{cases}.
                              \]
                        \item When $a\notin\{1,2,3\}$, $P^0_{s,a}(S)=1$.
                    \end{enumerate}
              \item Reward states:
                    \begin{enumerate}
                        \item At state $1$,
                              \[
                                  P^0_{1,1}(1)=1-q_0,
                                  \qquad
                                  P^0_{1,1}(S)=q_0,
                                  \qquad
                                  P^0_{1,a}(S)=1,
                                  \quad a\in\cA\setminus\{1\}.
                              \]
                        \item At each state $s\in\{2,3\}$,
                              \[
                                  P^0_{s,s}(s)=1-q_1,
                                  \qquad
                                  P^0_{s,s}(S)=q_1,
                                  \qquad
                                  P^0_{s,a}(S)=1,
                                  \quad a\in\cA\setminus\{s\}.
                              \]
                    \end{enumerate}
                    All unspecified entries in these rows are zero.
          \end{itemize}
    \item \textbf{Uncertainty set.} The set $\cP$ is the standard rectangular TV ball
          centered at the nominal kernel. For the decision-state pairs, set
          \[
              \sigma_{S,1}=\sigma_0,
              \qquad
              \sigma_{S,a}=\sigma_1,\quad a\in\{2,3\},
          \]
          and for all other state-action pairs set $\sigma_{s,a}=0$. The local
          uncertainty sets are
          \begin{equation}
              \cP_{s,a}
              =
              \left\{
              P'_{s,a}\in\Delta(\cS):
              \|P'_{s,a}-P^0_{s,a}\|_{\mathrm{TV}}\le \sigma_{s,a}
              \right\}.
              \label{eq:lower-bound-local-tv-set}
          \end{equation}
          Rectangularity means that
          \[
              \cP
              =
              \prod_{(s,a)\in\cS\times\cA}\cP_{s,a}.
          \]
          Since $\sigma_0,\sigma_1\le\sigma$, all local radii are bounded by the global
          uncertainty level~$\sigma$. For zero-radius pairs, the uncertainty set is a
          singleton.
\end{itemize}

The nominal transition structure is illustrated in
Figure~\ref{fig:hard-mdp-diagram}.
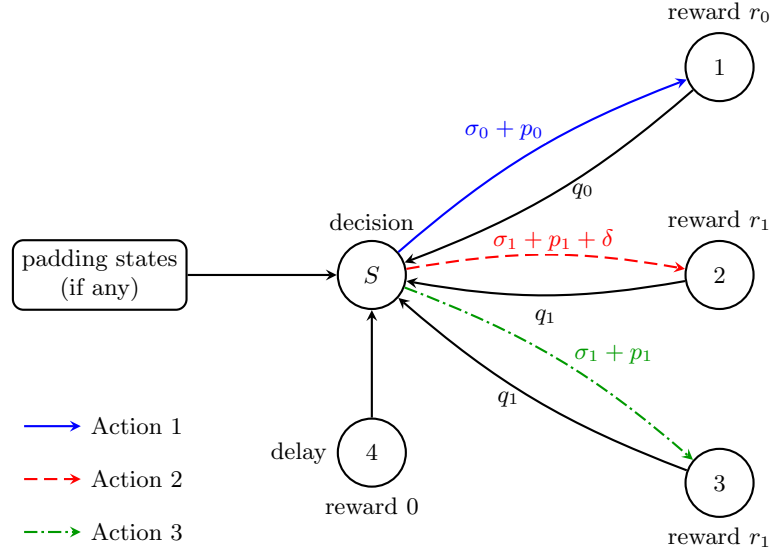
\begin{figure}[ht]
    \centering
    \begin{tikzpicture}[->, >=stealth, thick, every node/.style={font=\small},
            state/.style={circle, draw, minimum size=9mm},
            start/.style={rectangle, draw, rounded corners, minimum width=18mm, minimum height=8mm},
            actionone/.style={blue, solid},
            actiontwo/.style={red, dash pattern=on 5pt off 2.5pt},
            actionthree/.style={green!60!black, dash pattern=on 5pt off 1.5pt on 1pt off 1.5pt}]
        \node[start, align=center] (start) at (0,0) {padding states\\(if any)};
        \node[state, label=above:{decision}] (decision) at (3.6,0) {$S$};
        \node[state, label=left:{delay}, label=below:{reward $0$}] (delay) at (3.6,-2.35) {$4$};
        \node[state, label=above:{reward $r_0$}] (one) at (8.2,2.75) {$1$};
        \node[state, label=above:{reward $r_1$}] (two) at (8.2,0) {$2$};
        \node[state, label=below:{reward $r_1$}] (three) at (8.2,-2.75) {$3$};

        \draw (start) -- (decision);
        \draw (delay) -- (decision);
        \draw[actionone, bend left=10] (decision) to node[pos=0.55, above left, fill=white, inner sep=1pt] {$\sigma_0+p_{0}$} (one);
        \draw[actiontwo, bend left=10] (decision) to node[pos=0.53, above, fill=white, inner sep=1pt] {$\sigma_1+p_{1}+\delta$} (two);
        \draw[actionthree, bend left=10] (decision) to node[pos=0.55, above right, fill=white, inner sep=1pt] {$\sigma_1+p_{1}$} (three);

        \begin{scope}[shift={(-1.0,-2)}]
            \draw[actionone, thick] (0,0) -- (0.75,0) node[right, black] {Action 1};
            \draw[actiontwo, thick] (0,-0.7) -- (0.75,-0.7) node[right, black] {Action 2};
            \draw[actionthree, thick] (0,-1.4) -- (0.75,-1.4) node[right, black] {Action 3};
        \end{scope}

        \draw[bend left=10] (one) to node[pos=0.5, right, fill=none, inner sep=5pt] {$q_{0}$} (decision);
        \draw[bend left=10] (two) to node[pos=0.5, below, fill=none, inner sep=5pt] {$q_{1}$} (decision);
        \draw[bend left=10] (three) to node[pos=0.5, left, fill=none, inner sep=5pt] {$q_{1}$} (decision);
    \end{tikzpicture}
    \caption{The nominal transition structure for $\phi=2$. Self-loops and
    transitions that return directly to $S$ are omitted. For $\phi=3$, the
    roles of actions $2$ and $3$ are reversed.}
    \label{fig:hard-mdp-diagram}
\end{figure}

To analyze worst-case performance, we single out an admissible kernel $\bar P$ that
moves the available uncertainty mass from the reward states to the zero-reward
delay state $4$. We will show that this kernel attains the robust minimum for
the comparison policies and use it to bound the robust loss of an arbitrary
policy in the testing reduction. At the decision state, define
\begin{equation}
    \begin{aligned}
        \bar P_{S,1}^{\phi}(1)
        &=p_0,
        &
        \bar P_{S,1}^{\phi}(4)
        &=\sigma_0,
        &
        \bar P_{S,1}^{\phi}(S)
        &=1-p_0-\sigma_0,
        \\
        \bar P_{S,a}^{\phi}(a)
        &=p_1+\delta\mathds{1}_{\{\phi=a\}},
        &
        \bar P_{S,a}^{\phi}(4)
        &=\sigma_1,
        &
        \bar P_{S,a}^{\phi}(S)
        &=1-p_1-\sigma_1-\delta\mathds{1}_{\{\phi=a\}},
        \quad a\in\{2,3\}.
    \end{aligned}
    \label{eq:lower-bound-delay-kernel}
\end{equation}
All unspecified entries in these rows are zero, and
$\bar P_{s,a}^{\phi}=P^0_{s,a}$ at every other state-action pair. In particular,
\[
    \left\|\bar P_{S,1}^{\phi}-P^0_{S,1}\right\|_{\mathrm{TV}}
    =
    \sigma_0,
    \qquad
    \left\|\bar P_{S,a}^{\phi}-P^0_{S,a}\right\|_{\mathrm{TV}}
    =
    \sigma_1,
    \quad a\in\{2,3\},
\]
while every other row has TV distance zero from its nominal counterpart. Hence
$\bar P^\phi\in\cP$ by \eqref{eq:lower-bound-local-tv-set}.

\paragraph*{Action-wise comparison policies.}
To compare the actions available at $S$, we associate each action with a
reference policy. The robust average rewards and expected return times of these
policies will provide the action-wise benchmarks used in the testing reduction.
Specifically, for every $a\in\cA$, let $\pi_a$ choose action $a$ at every state:
\begin{equation}
    \pi_a(b\mid s)=\mathds{1}_{\{b=a\}},
    \qquad
    b\in\cA, s\in\cS.
    \label{eq:lower-bound-deterministic-policies}
\end{equation}
For $a\in\{1,2,3\}$, the policy $\pi_a$ selects action $a$ at $S$ and
continues with action $a$ at reward state $a$. At either of the other
reward states, it returns immediately to $S$. Its actions at the delay and
padding states are immaterial because these states return deterministically to
$S$.

\paragraph*{Unichain property.}
Fix a stationary policy and an admissible kernel. Every delay or padding state
returns directly to $S$. At reward state $1$, action $1$ returns to $S$ with
probability $q_0>0$, while every other action returns to $S$ immediately. At
each reward state $i\in\{2,3\}$, action $i$ returns to $S$ with probability
$q_1>0$, while every other action returns immediately. Consequently, every
state reaches $S$ almost surely, regardless of the policy's randomization.
Every recurrent class must therefore contain $S$, so the induced Markov chain
has exactly one recurrent class. Hence the constructed MDP satisfies
Assumption~\mainUnichainAssumption{}.

\paragraph*{Average rewards between successive returns to $S$.}
Because the process returns to the decision state $S$ under every stationary
policy and every $P\in\cP$, we can analyze its average reward through the reward
accumulated and time elapsed between successive visits to $S$. Starting from
$s_0=S$, let
\[
    \tau_S^+
    \coloneqq
    \inf\{t\ge1:s_t=S\}
\]
denote the first return time to $S$. Fix a stationary policy $\pi$ and a kernel
$P\in\cP$. The return time $\tau_S^+$ has finite expectation. Since successive
visits to $S$ divide the trajectory into identically distributed cycles, the
average reward equals the expected reward in one cycle divided by its expected
length:
\begin{equation}
    \rho_P^\pi
    =
    \frac{
    \mathbb E_P^\pi\!\left[
        \sum_{t=0}^{\tau_S^+-1}r(s_t,a_t)
        \,\middle|\,s_0=S
    \right]
    }{
    \mathbb E_P^\pi[\tau_S^+\mid s_0=S]
    }.
    \label{eq:lower-bound-return-reward-ratio}
\end{equation}
For each $a\in\cA$, define
\[
    \rho_a
    \coloneqq
    \rho_{P^0}^{\pi_a},
    \qquad
    \rho_a^\sigma
    \coloneqq
    \min_{P\in\cP}\rho_P^{\pi_a},
    \qquad
    T_a^\sigma
    \coloneqq
    \mathbb E_{\bar P^\phi}^{\pi_a}[\tau_S^+\mid s_0=S].
\]
These quantities depend on $\phi$, but we leave this dependence implicit in the
notation.
The following lemma computes these quantities and verifies that
$\bar P^\phi$ attains the robust minimum for the policies in
\eqref{eq:lower-bound-deterministic-policies}.
The proof is deferred to Appendix~\ref{subsec:Proof-lemma-avg-reward-bias-span}.

\begin{lemma}\label{lem:avg-reward-bias-span}
    For the constructed instances, $\bar P^\phi$ attains
    $\min_{P\in\cP}\rho_P^{\pi_a}$ for every $a\in\cA$. The nominal and robust
    average rewards and the expected return times for $\pi_1,\pi_\phi$, and
    $\pi_{5-\phi}$ are as follows.
    \begin{itemize}
        \item For $\pi_1$,
              \[
                  \rho_{1}=\frac{\sigma_0+p_{0}}{\sigma_0+p_{0}+q_{0}}\cdot r_0,
                  \qquad
                  \rho_{1}^{\sigma}=\frac{p_{0}}{p_{0}+q_{0}(1+\sigma_0)}\cdot r_0,
                  \qquad
                  T_1^\sigma=1+\sigma_0+\frac{p_0}{q_0}.
              \]

        \item For $\pi_\phi$,
              \[
                  \rho_{\phi}=\frac{\sigma_1+p_{1}+\delta}{\sigma_1+p_{1}+q_{1}+\delta}\cdot r_1,
                  \qquad
                  \rho_{\phi}^{\sigma}=\frac{p_{1}+\delta}{p_{1}+\delta+q_{1}(1+\sigma_1)}\cdot r_1,
                  \qquad
                  T_\phi^\sigma=1+\sigma_1+\frac{p_1+\delta}{q_1}.
              \]

        \item For $\pi_{5-\phi}$,
              \[
                  \rho_{5-\phi}=\frac{\sigma_1+p_{1}}{\sigma_1+p_{1}+q_{1}}\cdot r_1,
                  \qquad
                  \rho_{5-\phi}^{\sigma}=\frac{p_{1}}{p_{1}+q_{1}(1+\sigma_1)}\cdot r_1,
                  \qquad
                  T_{5-\phi}^\sigma=1+\sigma_1+\frac{p_1}{q_1}.
              \]
    \end{itemize}
    For $a\notin\{1,2,3\}$, the action returns immediately to $S$ with no
    reward, so $T_a^\sigma=1$ and $\rho_a=\rho_a^\sigma=0$.
\end{lemma}

\subsection{Proof of Theorem~\mainMinimaxLowerBoundTheorem{}}
\label{subsec:minimax-lower-bound}
It suffices to verify the two sample-size components in
Theorem~\mainMinimaxLowerBoundTheorem{} separately. We first establish a
common testing argument and then apply it to three constructions. The first
construction proves the common minimax linear minimum-span component
$\min\{H_0,H_\sigma\}$. 

Once all these components are established, we can combine them to obtain the desired lower bound.

For the $\sigma H_\sigma^2$ component, it suffices to assume
\begin{equation}
    \varepsilon
    \le
    0.01\sigma\min\{H_0,H_\sigma\}.
    \label{eq:standing-robust-construction-range}
\end{equation}
If this condition fails while condition~(b) of
Theorem~\mainMinimaxLowerBoundTheorem{} holds, namely
$\varepsilon\le0.01\sigma H_0$, then 
$H_\sigma<H_0$. In this case, 
$\sigma H_\sigma<100\varepsilon\le1$ and therefore
\begin{equation}
    \min\{H_0,H_\sigma\}+\sigma H_\sigma^2
    \le
    2H_\sigma
    =
    2\min\{H_0,H_\sigma\}.
    \label{eq:nominal-component-covers-complementary-range}
\end{equation}
Thus the $\min\{H_0,H_\sigma\}$ component already covers the desired lower bound. We proceed to prove the $\sigma H_\sigma^2$ component under this assumption with two different constructions, one for $\sigma < 4/\min\{H_0,H_\sigma\}$ and one for $\sigma \ge 4/\min\{H_0,H_\sigma\}$.

\subsubsection{A generic testing reduction}
\label{subsubsec:generic-testing-reduction}
We start with a generic testing argument. 
Suppose that a pair of instances
$\{\cM_\phi:\phi\in\{2,3\}\}$ satisfies
\begin{equation}
    \rho_\phi^\sigma\le r_0,
    \qquad
    \rho_\phi^\sigma\le r_1,
    \qquad
    \phi\in\{2,3\},
    \label{eq:generic-reward-dominance}
\end{equation}
and
\begin{equation}
    \frac{T_a^\sigma}{T_\phi^\sigma+T_a^\sigma}
    \bigl(\rho_\phi^\sigma-\rho_a^\sigma\bigr)
    >
    \varepsilon,
    \qquad
    a\in\cA,\ a\neq\phi,\ \phi\in\{2,3\}.
    \label{eq:generic-gap-condition}
\end{equation}
In particular, \eqref{eq:generic-gap-condition} implies
$\rho_\phi^\sigma>\rho_a^\sigma$ for every $a\neq\phi$.

Under \eqref{eq:generic-reward-dominance} and
\eqref{eq:generic-gap-condition}, we derive the minimax lower bound in three
steps. First, we convert any policy estimator into a test of $\phi$. Next, we
use Le Cam's method to lower-bound the resulting testing error in terms of KL
divergence. Finally, we compute the relevant KL divergence for the two
instances.

\paragraph*{Step 1: From policy error to testing error.}
We begin by comparing the return time and reward obtained after each possible
action at $S$ under an arbitrary stationary policy. Fix a stationary policy
$\pi$ and write $w_a=\pi(a\mid S)$. Recall from
\eqref{eq:lower-bound-delay-kernel} that $\bar P^\phi$ moves the available
uncertainty mass in the rows out of $S$ from the reward states to the
zero-reward delay state $4$, while leaving all other rows unchanged. By
Lemma~\ref{lem:avg-reward-bias-span}, it attains
$\min_{P\in\cP}\rho_P^{\pi_a}$ for every comparison policy $\pi_a$. For the
arbitrary policy $\pi$ considered here, we use only that $\bar P^\phi\in\cP$:
we evaluate $\pi$ under $\bar P^\phi$ and consider the path between successive
returns to $S$.

First fix $a\in\{1,2,3\}$. In the row corresponding to action $a$, let $p$
denote the probability of entering reward state $a$ under $\bar P^\phi$, let
$\sigma'$ denote the probability of entering state $4$, and let $q$ and $r$
denote the return probability and reward at state $a$. If $\beta$ is the
probability that $\pi$ chooses action $a$ whenever state $a$ is visited, then,
conditional on choosing action $a$ at $S$, the expected return time
$\widetilde T_a^\pi$ and reward $\widetilde R_a^\pi$ are
\[
    \widetilde T_a^\pi
    =
    1+\sigma'+\frac{p}{1-\beta(1-q)},
    \qquad
    \widetilde R_a^\pi
    =
    \frac{pr}{1-\beta(1-q)}.
\]
When $\beta=1$, these expressions become
\[
    T_a^\sigma
    =
    1+\sigma'+\frac{p}{q},
    \qquad
    T_a^\sigma\rho_a^\sigma
    =
    \frac{pr}{q}.
\]
Moreover,
\[
    1-\beta(1-q)\ge q,
    \qquad
    \rho_\phi^\sigma-r\le0,
\]
where the second inequality follows from
\eqref{eq:generic-reward-dominance}. Therefore,
\begin{align}
    \rho_\phi^\sigma\widetilde T_a^\pi-\widetilde R_a^\pi
    & =
    \rho_\phi^\sigma(1+\sigma')
    +
    \frac{p(\rho_\phi^\sigma-r)}{1-\beta(1-q)}
    \nonumber\\
    & \ge
    \rho_\phi^\sigma(1+\sigma')
    +
    \frac{p(\rho_\phi^\sigma-r)}{q}
    \nonumber\\
    & =
    T_a^\sigma
    \bigl(\rho_\phi^\sigma-\rho_a^\sigma\bigr).
    \label{eq:arbitrary-cycle-dominance}
\end{align}
We also have $\widetilde T_a^\pi\le T_a^\sigma$. For
$a\notin\{1,2,3\}$, the return time is one and the reward is zero, so both
conclusions remain valid.

We next average these action-wise comparisons according to the probabilities
$w_a$ to bound the robust loss of $\pi$. Let $\bar\rho^\pi$ be the average
reward of $\pi$ under $\bar P^\phi$.
Since a return cycle begins with action $a$ at $S$ with probability $w_a$,
\eqref{eq:lower-bound-return-reward-ratio} gives
\[
    \bar\rho^\pi
    =
    \frac{
        \sum_{a\in\cA}w_a\widetilde R_a^\pi
    }{
        \sum_{a\in\cA}w_a\widetilde T_a^\pi
    }.
\]
Consequently,
\[
    \rho_\phi^\sigma-\bar\rho^\pi
    =
    \frac{
        \sum_{a\in\cA}w_a
        \bigl(\rho_\phi^\sigma\widetilde T_a^\pi-\widetilde R_a^\pi\bigr)
    }{
        \sum_{a\in\cA}w_a\widetilde T_a^\pi
    }.
\]
The return-time comparison above implies
\[
    \sum_{a\in\cA}w_a\widetilde T_a^\pi
    \le
    \sum_{a\in\cA}w_aT_a^\sigma.
\]
On the other hand, applying \eqref{eq:arbitrary-cycle-dominance} to each
numerator term gives
\[
    \sum_{a\in\cA}w_a
    \bigl(\rho_\phi^\sigma\widetilde T_a^\pi-\widetilde R_a^\pi\bigr)
    \ge
    \sum_{a\in\cA}w_aT_a^\sigma
    \bigl(\rho_\phi^\sigma-\rho_a^\sigma\bigr)
    =
    \sum_{a\neq\phi}w_aT_a^\sigma
    \bigl(\rho_\phi^\sigma-\rho_a^\sigma\bigr).
\]
The last sum is nonnegative by \eqref{eq:generic-gap-condition}. We may
therefore combine the numerator lower bound with the denominator upper bound
to obtain
\begin{equation}
    \rho_\phi^\sigma-\bar\rho^\pi
    \ge
    \frac{
        \sum_{a\neq\phi}w_aT_a^\sigma
        \bigl(\rho_\phi^\sigma-\rho_a^\sigma\bigr)
    }{
        \sum_{a\in\cA}w_aT_a^\sigma
    }.
    \label{eq:generic-regenerative-policy-gap}
\end{equation}

By Lemma~\ref{lem:avg-reward-bias-span}, the robust average reward of
$\pi_\phi$ is $\rho_\phi^\sigma$. On the other hand,
\eqref{eq:generic-regenerative-policy-gap} shows that $\bar P^\phi$ gives every
stationary policy average reward at most $\rho_\phi^\sigma$. Therefore
\[
    \rho^{\star,\sigma}=\rho_\phi^\sigma.
\]
Because $\bar P^\phi\in\cP$, we also have
$\rho^{\pi,\sigma}\le\bar\rho^\pi$. Hence
\begin{equation}
    \rho^{\star,\sigma}-\rho^{\pi,\sigma}
    \ge
    \frac{
        \sum_{a\neq\phi}w_aT_a^\sigma
        \bigl(\rho_\phi^\sigma-\rho_a^\sigma\bigr)
    }{
        \sum_{a\in\cA}w_aT_a^\sigma
    }.
    \label{eq:generic-robust-policy-gap}
\end{equation}

Multiplying \eqref{eq:generic-gap-condition} by $w_a$ and summing over
$a\neq\phi$ gives
\[
    \sum_{a\neq\phi}w_aT_a^\sigma
    \bigl(\rho_\phi^\sigma-\rho_a^\sigma\bigr)
    >
    \varepsilon
    \left(
        (1-w_\phi)T_\phi^\sigma
        +
        \sum_{a\neq\phi}w_aT_a^\sigma
    \right).
\]
If $w_\phi\le1/2$, then $1-w_\phi\ge w_\phi$, so the
right-hand side is at least
\[
    \varepsilon\sum_{a\in\cA}w_aT_a^\sigma.
\]
Together with \eqref{eq:generic-robust-policy-gap}, this proves
\[
    w_\phi = \pi(\phi\mid S)\le\frac12
    \quad\Longrightarrow\quad
    \rho^{\star,\sigma}-\rho^{\pi,\sigma}>\varepsilon.
\]
Finally, we apply this pointwise implication to a policy estimator. Let
$\widehat\pi$ be any policy estimator and define its induced test by
\[
    \widehat\phi
    \in
    \arg\max_{i\in\{2,3\}}\widehat\pi(i\mid S),
\]
where ties are broken arbitrarily. If $\widehat\phi\neq\phi$, then
$\widehat\pi(\phi\mid S)\le1/2$. Applying the preceding implication to each
realization of $\widehat\pi$ gives
\begin{equation}
    \mathbb P_{\cM_\phi}\{\widehat\phi\neq\phi\}
    \le
    \mathbb P_{\cM_\phi}
    \left\{
        \rho^{\star,\sigma}-\rho^{\widehat\pi,\sigma}>\varepsilon
    \right\}.
    \label{eq:generic-policy-to-testing}
\end{equation}
Taking the maximum over $\phi\in\{2,3\}$ and then the infimum over policy
estimators gives
\begin{equation}
    \inf_{\widehat\pi}
    \max_{\phi\in\{2,3\}}
    \mathbb P_{\cM_\phi}
    \left\{
        \rho^{\star,\sigma}-\rho^{\widehat\pi,\sigma}>\varepsilon
    \right\}
    \ge
    \inf_{\widetilde\phi}
    \max_{\phi\in\{2,3\}}
    \mathbb P_{\cM_\phi}\{\widetilde\phi\neq\phi\}.
    \label{eq:generic-testing-lower}
\end{equation}
Thus it remains to lower-bound the testing error on the right-hand side of
\eqref{eq:generic-testing-lower}.

\paragraph*{Step 2: Le Cam's two-point bound.}
Let $P^{0,\phi}$ denote the nominal transition kernel of $\cM_\phi$, and let
\begin{equation}
    \mathbb Q_\phi^{(N)}
    \coloneqq
    \bigotimes_{(s,a)\in\cS\times\cA}
    \left(P_{s,a}^{0,\phi}\right)^{\otimes N}
    \label{eq:lower-bound-data-law}
\end{equation}
be the joint law of the $N$ transition samples from every state-action pair.
Additivity of KL divergence gives
\begin{equation}
    \mathrm{KL}\!\left(\mathbb Q_2^{(N)}\,\Vert\,\mathbb Q_3^{(N)}\right)
    =
    N\sum_{(s,a)\in\cS\times\cA}
    \mathrm{KL}\!\left(P_{s,a}^{0,2}\,\Vert\,P_{s,a}^{0,3}\right).
    \label{eq:data-law-kl-additivity}
\end{equation}
Le Cam's two-point method and Pinsker's inequality
\citep[see, e.g., (15.13) and Lemma 15.2]{Wainwright2019High}, followed by
\eqref{eq:data-law-kl-additivity}, yield
\begin{align}
    \inf_{\widetilde\phi}
    \max_{\phi\in\{2,3\}}
    \mathbb P_{\cM_\phi}\{\widetilde\phi\neq\phi\}
    & \ge
    \frac12
    -
    \frac12
    \sqrt{
        \frac{N}{2}
        \sum_{(s,a)\in\cS\times\cA}
        \mathrm{KL}\!\left(P_{s,a}^{0,2}\,\Vert\,P_{s,a}^{0,3}\right)
    }.
    \label{eq:generic-le-cam}
\end{align}
Consequently, if
\begin{equation}
    N
    \sum_{(s,a)\in\cS\times\cA}
    \mathrm{KL}\!\left(P_{s,a}^{0,2}\,\Vert\,P_{s,a}^{0,3}\right)
    \le\frac14,
    \label{eq:generic-kl-and-sample-bound}
\end{equation}
then the right-hand side of \eqref{eq:generic-le-cam} is at least
\[
    \frac12-\frac12\sqrt{\frac18}
    >
    \frac14.
\]
Combining this with
\eqref{eq:generic-testing-lower} gives
\begin{equation}
    \inf_{\widehat\pi}
    \max_{\phi\in\{2,3\}}
    \mathbb P_{\cM_\phi}
    \left\{
        \rho^{\star,\sigma}-\rho^{\widehat\pi,\sigma}>\varepsilon
    \right\}
    \ge\frac14.
    \label{eq:generic-testing-conclusion}
\end{equation}
To verify \eqref{eq:generic-kl-and-sample-bound} for each construction, we next
compute the KL divergence between the two nominal kernels.

\paragraph*{Step 3: KL divergence of the two changed rows.}
For each of the three parameter choices below, the nominal kernels differ only
at $(S,2)$ and $(S,3)$. On these rows, the two instances exchange Bernoulli
reward-transition probabilities $b$ and $b+\delta$. Therefore
\begin{align}
    &\sum_{(s,a)\in\cS\times\cA}
    \mathrm{KL}\!\left(P_{s,a}^{0,2}\,\Vert\,P_{s,a}^{0,3}\right)
    \nonumber\\
    &\quad=
    \mathrm{KL}\bigl(\mathrm{Bern}(b+\delta)\Vert\mathrm{Bern}(b)\bigr)
    +
    \mathrm{KL}\bigl(\mathrm{Bern}(b)\Vert\mathrm{Bern}(b+\delta)\bigr)
    \nonumber\\
    &\quad=
    \delta
    \log\left(
        \frac{(b+\delta)(1-b)}{b(1-b-\delta)}
    \right)
    =
    \delta\int_b^{b+\delta}\frac{1}{t(1-t)}\,\mathrm dt.
    \label{eq:symmetric-bernoulli-kl}
\end{align}
This completes the common testing argument. 

\subsubsection{$\min\{H_0,H_\sigma\}$ component}

To establish the $\min\{H_0,H_\sigma\}$ component, we specify an instance pair, verify the testing and span conditions, and then bound the KL divergence
between the two instances.

\paragraph*{Parameter choice and resulting quantities.}
We use a zero-radius subfamily of the instance template. Choose
\[
    p_0=\frac{1}{\min\{H_0,H_\sigma\}},
    \qquad
    q_0=\frac{2}{\min\{H_0,H_\sigma\}},
    \qquad
    \sigma_0=0,
\]
\[
    p_1=\frac{2}{\min\{H_0,H_\sigma\}},
    \qquad
    q_1=\frac{2}{\min\{H_0,H_\sigma\}},
    \qquad
    \sigma_1=0,
\]
and
\[
    \delta=\frac{40\varepsilon}{\min\{H_0,H_\sigma\}},
    \qquad
    r_0=r_1=1.
\]
These parameters are valid because $\min\{H_0,H_\sigma\}\ge4$ and
$\varepsilon\le1/100$. Indeed,
\[
    p_0\le\frac14,
    \qquad
    q_0=p_1=q_1\le\frac12,
    \qquad
    p_1+\delta\le\frac35.
\]
All local uncertainty radii are zero, so the nominal
and robust Bellman equations coincide. The unichain argument in
Appendix~\ref{subsec:Instance-construction} applies to every stationary policy.

Substituting these parameters into
Lemma~\ref{lem:avg-reward-bias-span} gives
\[
    \rho_1^\sigma=\frac13,
    \qquad
    \rho_{5-\phi}^\sigma=\frac12,
    \qquad
    \rho_\phi^\sigma
    =
    \frac{2+40\varepsilon}{4+40\varepsilon},
\]
and
\[
    T_1^\sigma=\frac32,
    \qquad
    T_{5-\phi}^\sigma=2,
    \qquad
    T_\phi^\sigma=2+20\varepsilon.
\]

\paragraph*{Verification of the testing conditions.}
We now check the conditions needed to apply the common testing argument.
Because $r_0=r_1=1$, the reward-dominance condition
\eqref{eq:generic-reward-dominance} holds. For the gap to action $5-\phi$,
\[
    \frac{T_{5-\phi}^\sigma}
    {T_\phi^\sigma+T_{5-\phi}^\sigma}
    \bigl(\rho_\phi^\sigma-\rho_{5-\phi}^\sigma\bigr)
    =
    \frac{5\varepsilon}
    {(1+10\varepsilon)(2+10\varepsilon)}
    >
    \varepsilon.
\]
For action $1$,
\[
    \frac{T_1^\sigma}{T_\phi^\sigma+T_1^\sigma}
    \bigl(\rho_\phi^\sigma-\rho_1^\sigma\bigr)
    \ge
    \frac{5}{74}
    >
    \varepsilon,
\]
and for $a\notin\{1,2,3\}$,
\[
    \frac{T_a^\sigma}{T_\phi^\sigma+T_a^\sigma}
    \bigl(\rho_\phi^\sigma-\rho_a^\sigma\bigr)
    \ge
    \frac{5}{32}
    >
    \varepsilon.
\]
Hence \eqref{eq:generic-gap-condition} also holds.

\paragraph*{Span constraints and class membership.}
We next verify that both instances belong to the class in the theorem. Set
$h(S)=0$ and let
\[
    h(1)=h(2)=h(3)
    =
    \frac{\min\{H_0,H_\sigma\}(1-\rho_\phi^\sigma)}{2},
    \qquad
    h(s)=-\rho_\phi^\sigma,
    \quad s\in\{4,\ldots,S-1\}.
\]
We now check that $h$ satisfies both Bellman equations and has the required
span.
At every reward state $i\in\{1,2,3\}$, action $a=i$ satisfies
\[
    \rho_\phi^\sigma+h(i)
    =
    1+\left(1-\frac{2}{\min\{H_0,H_\sigma\}}\right)h(i).
\]
This action weakly dominates every action that returns immediately to $S$
because $h(i)\ge0$. At the decision state,
\[
    (P^0_{S,\phi})^\top h
    =
    \frac{2+40\varepsilon}{\min\{H_0,H_\sigma\}}
    \cdot
    \frac{\min\{H_0,H_\sigma\}(1-\rho_\phi^\sigma)}{2}
    =
    \rho_\phi^\sigma.
\]
The transition probabilities multiplying the same positive reward-state
coordinate are smaller for actions $1$ and $5-\phi$, while every remaining
action has value zero. Thus action $\phi$ maximizes at $S$. Finally, for
$s\in\{4,\ldots,S-1\}$,
\[
    \rho_\phi^\sigma+h(s)=0=(P^0_{s,a})^\top h,
    \qquad a\in\cA.
\]
Consequently,
$\rho_\phi^\sigma\mathbf 1+h=\mathcal T_0h$. Since every uncertainty radius
is zero, the same equation holds for $\mathcal T_\sigma$. Moreover,
\begin{align*}
    \|h\|_{\mathrm{span}}
    & =
    \frac{\min\{H_0,H_\sigma\}(1-\rho_\phi^\sigma)}{2}
    +
    \rho_\phi^\sigma
    \\
    & =
    \frac{\min\{H_0,H_\sigma\}}{2}
    -
    \left(\frac{\min\{H_0,H_\sigma\}}{2}-1\right)
    \rho_\phi^\sigma
    \\
    & \le
    \frac{\min\{H_0,H_\sigma\}}{2}
    \le
    \min\{H_0,H_\sigma\}.
\end{align*}
Hence
\begin{equation}
    \cM_\phi\in\mathfrak{M}(H_0,H_\sigma,\sigma),
    \qquad \phi\in\{2,3\}.
    \label{eq:nominal-component-membership}
\end{equation}

\paragraph*{KL divergence and testing conclusion.}
It remains to bound the statistical distance between the two instances and
apply the common testing argument.
The two nominal kernels differ only at $(S,2)$ and $(S,3)$, where they
exchange Bernoulli parameters $2/\min\{H_0,H_\sigma\}$ and
$(2+40\varepsilon)/\min\{H_0,H_\sigma\}$. Throughout the integration interval in
\eqref{eq:symmetric-bernoulli-kl},
\[
    t\ge\frac{2}{\min\{H_0,H_\sigma\}},
    \qquad
    1-t\ge\frac25.
\]
Consequently, \eqref{eq:symmetric-bernoulli-kl} gives
\begin{equation}
    \sum_{(s,a)\in\cS\times\cA}
    \mathrm{KL}\!\left(P_{s,a}^{0,2}\,\Vert\,P_{s,a}^{0,3}\right)
    \le
    \frac{2000\varepsilon^2}{\min\{H_0,H_\sigma\}}.
    \label{eq:nominal-component-kl}
\end{equation}
If
\[
    N
    \le
    \frac{\min\{H_0,H_\sigma\}}{8000\varepsilon^2},
\]
then \eqref{eq:nominal-component-kl} implies
\eqref{eq:generic-kl-and-sample-bound}, which then gives
\begin{equation}
    \inf_{\widehat\pi}
    \max_{\phi\in\{2,3\}}
    \mathbb{P}_{\cM_\phi}
    \left\{
    \rho^{\star,\sigma}-\rho^{\widehat\pi,\sigma}>\varepsilon
    \right\}
    \ge
    \frac14.
    \label{eq:nominal-component-lower}
\end{equation}
This proves the $\min\{H_0,H_\sigma\}$ component in Theorem~\mainMinimaxLowerBoundTheorem{}.

\subsubsection{Case $\sigma < 4/\min\{H_0,H_\sigma\}$ for the $\sigma H_\sigma^2$ component}

For the case $\sigma < 4/\min\{H_0,H_\sigma\}$, we employ the generic testing argument to establish the $\sigma H_\sigma^2$ component of the lower bound.
We specify an instance pair, verify the testing and span
conditions, and then bound the KL divergence between the two instances.

\paragraph*{Parameter choice and resulting quantities.}
Recall that we assume
\[
    \sigma
    \ge
    \frac{100\varepsilon}{\min\{H_0,H_\sigma\}}.
\]
Choose
\[
    p_0=\frac{\sigma}{5},
    \qquad
    q_0=\frac{2\sigma}{5},
    \qquad
    \sigma_0=\sigma,
\]
\[
    p_1=q_1
    =
    \frac{\sigma\min\{H_0,H_\sigma\}}{2H_\sigma},
    \qquad
    \sigma_1=\frac{\sigma}{4},
\]
and
\[
    \delta=\frac{100\varepsilon}{H_\sigma},
    \qquad
    r_0=\frac{\sigma\min\{H_0,H_\sigma\}}{4},
    \qquad
    r_1=\frac{3\sigma\min\{H_0,H_\sigma\}}{16}.
\]
Since $\min\{H_0,H_\sigma\}\le H_\sigma$ and
\eqref{eq:standing-robust-construction-range} holds,
$p_1\le\sigma/2$ and $\delta\le\sigma$. Thus
$\sigma_1+p_1+\delta\le\sigma/4+\sigma/2+\sigma\le7/8<1$; also
$\sigma_0+p_0=6\sigma/5\le3/5<1$. In addition,
$r_0<1$ and $r_1<3/4$ because
$\sigma\min\{H_0,H_\sigma\}<4$.

Substitution into Lemma~\ref{lem:avg-reward-bias-span} gives
\begin{equation}
    \rho_\phi^\sigma
    >
    \rho_{5-\phi}^\sigma
    =
    \frac{3\sigma\min\{H_0,H_\sigma\}}{16(2+\sigma/4)}
    >
    \frac{\sigma\min\{H_0,H_\sigma\}}{4(3+2\sigma)}
    =
    \rho_1^\sigma.
    \label{eq:part-2-robust-reward-order}
\end{equation}
After cross-multiplication, the first inequality reduces to
$\delta q_1(1+\sigma_1)r_1>0$, and the second reduces to
$4+20\sigma>0$. Moreover,
\begin{equation}
    \rho_\phi^\sigma<r_1<r_0.
    \label{eq:part-2-reward-dominance}
\end{equation}

\paragraph*{Verification of the testing conditions.}
Equation~\eqref{eq:part-2-reward-dominance} verifies
the reward-dominance condition \eqref{eq:generic-reward-dominance}. The
following lemma verifies the gap condition \eqref{eq:generic-gap-condition};
its proof is deferred to
Appendix~\ref{subsec:proof-cycle-length-checks}.
\begin{lemma}
    \label{lem:part-2-cycle-check}
    For the instances constructed in the case
    $\sigma<4/\min\{H_0,H_\sigma\}$,
    the gap condition \eqref{eq:generic-gap-condition} holds.
\end{lemma}

\paragraph*{Span constraints and class membership.}
The following lemma provides the nominal and robust span bounds; its proof is
deferred to Appendix~\ref{subsec:proof-span-checks}.
\begin{lemma}
    \label{lem:part-2-span-check}
    For the instances constructed in the case
    $\sigma<4/\min\{H_0,H_\sigma\}$,
    $
        \|h_{P^0}^\star\|_{\mathrm{span}}\le H_0
    $
    and
    $
        \|h^{\star,\sigma}\|_{\mathrm{span}}\le H_\sigma
    $.
\end{lemma}
The parameter checks give bounded rewards and local TV radii at most
$\sigma$, the unichain argument in
Appendix~\ref{subsec:Instance-construction} verifies
Assumption~\mainUnichainAssumption{}, and
Lemma~\ref{lem:part-2-span-check} gives the nominal and robust
span bounds. Hence
\begin{equation}
    \cM_\phi\in\mathfrak{M}(H_0,H_\sigma,\sigma),
    \qquad \phi\in\{2,3\}.
    \label{eq:part-2-membership}
\end{equation}

\paragraph*{KL divergence and testing conclusion.}
It remains to bound the statistical distance between the two instances and
apply the testing reduction in
Appendix~\ref{subsubsec:generic-testing-reduction}. Here the lower Bernoulli parameter in
\eqref{eq:symmetric-bernoulli-kl} is $\sigma/4+p_1$. The parameter bounds give
\[
    2\left(\frac{\sigma}{4}+p_1\right)+\delta
    \le
    \frac{3}{4}+\frac{1}{4}
    =1,
\]
where we used $p_1\le\sigma/2$, $\sigma\le1/2$, and
$\delta\le1/H_\sigma\le1/4$. Equivalently,
\[
    \frac{\sigma}{4}+p_1+\delta
    \le
    1-\frac{\sigma}{4}-p_1.
\]
Hence the integration interval in \eqref{eq:symmetric-bernoulli-kl} is
contained in
$[\sigma/4+p_1,1-\sigma/4-p_1]$. Therefore
$t(1-t)$ is at least
$(\sigma/4+p_1)(1-\sigma/4-p_1)$ throughout the integration interval in
\eqref{eq:symmetric-bernoulli-kl}. Consequently,
\begin{equation}
    \sum_{(s,a)\in\cS\times\cA}
    \mathrm{KL}\!\left(P_{s,a}^{0,2}\,\Vert\,P_{s,a}^{0,3}\right)
    \le
    \frac{\delta^2}{(\sigma/4+p_1)(1-\sigma/4-p_1)}
    \le
    \frac{64000\varepsilon^{2}}{\sigma H_{\sigma}^2}.
    \label{eq:KL-P0-P1-part-2}
\end{equation}
The last inequality uses $\sigma/4+p_1\ge\sigma/4$ and
$1-\sigma/4-p_1\ge1-1/8-1/4=5/8$.

By
\eqref{eq:KL-P0-P1-part-2}, the sample-size condition
\eqref{eq:generic-kl-and-sample-bound} holds whenever
\[
    N
    \le
    \frac{\sigma H_\sigma^2}{256000\varepsilon^2}.
\]
Under this sample-size bound, \eqref{eq:generic-testing-conclusion} gives
\begin{equation}
    \inf_{\widehat\pi}
    \max_{\phi\in\{2,3\}}
    \mathbb P_{\cM_\phi}
    \left\{
        \rho^{\star,\sigma}-\rho^{\widehat\pi,\sigma}>\varepsilon
    \right\}
    \ge
    \frac14.
    \label{eq:small-sigma-robust-component-lower}
\end{equation}
This proves the $\sigma H_\sigma^2$ component when
$\sigma<4/\min\{H_0,H_\sigma\}$.

\subsubsection{Case $\sigma \ge 4/\min\{H_0,H_\sigma\}$ for the $\sigma H_\sigma^2$ component}

For the case $\sigma \ge 4/\min\{H_0,H_\sigma\}$, we employ the generic testing argument to establish the $\sigma H_\sigma^2$ component of
the lower bound. We specify an instance pair, verify the testing and span
conditions, and then bound the KL divergence between the two instances.

\paragraph*{Parameter choice and resulting quantities.}
Choose
\[
    p_0=\frac{1}{\min\{H_0,H_\sigma\}},
    \qquad
    q_0=\frac{3}{\min\{H_0,H_\sigma\}},
    \qquad
    \sigma_0=0,
\]
\[
    p_1=q_1=\frac{2}{H_\sigma},
    \qquad
    \sigma_1=\sigma,
\]
and
\[
    \delta=\frac{40\varepsilon}{H_\sigma},
    \qquad
    r_0=r_1=1.
\]
These parameters are valid. Indeed, $\min\{H_0,H_\sigma\}\ge4$, the case
condition, and \eqref{eq:standing-robust-construction-range} give
$1/\min\{H_0,H_\sigma\}\le1/4$,
$3/\min\{H_0,H_\sigma\}\le3/4$,
$2/H_\sigma\le\sigma/2$, and
$40\varepsilon/H_\sigma\le0.4\sigma$, so
$\sigma_1+p_1+\delta\le1.9\sigma\le0.95<1$.

Substituting these parameters into Lemma~\ref{lem:avg-reward-bias-span} gives
\begin{equation}
    \rho_\phi^\sigma
    =
    \frac{2+40\varepsilon}{4+2\sigma+40\varepsilon}
    >
    \frac{1}{2+\sigma}
    =
    \rho_{5-\phi}^\sigma
    \ge
    \frac25
    >
    \frac14
    =
    \rho_1^\sigma.
    \label{eq:part-3-robust-reward-order}
\end{equation}
The strict inequality follows from $\delta>0$, and the remaining inequalities
use $\sigma\le1/2$. In addition,
\begin{equation}
    \rho_\phi^\sigma<1=r_0=r_1.
    \label{eq:part-3-reward-dominance}
\end{equation}

\paragraph*{Verification of the testing conditions.}
Equation~\eqref{eq:part-3-reward-dominance} verifies
the reward-dominance condition \eqref{eq:generic-reward-dominance}. The
following lemma verifies the gap condition \eqref{eq:generic-gap-condition};
its proof is deferred to
Appendix~\ref{subsec:proof-cycle-length-checks}.
\begin{lemma}
    \label{lem:part-3-cycle-check}
    For the instances constructed in the case
    $\sigma\ge4/\min\{H_0,H_\sigma\}$,
    the gap condition \eqref{eq:generic-gap-condition} holds.
\end{lemma}

\paragraph*{Span constraints and class membership.}
The following lemma provides the nominal and robust span bounds; its proof is
deferred to Appendix~\ref{subsec:proof-span-checks}.
\begin{lemma}
    \label{lem:part-3-span-check}
    For the instances constructed in the case
    $\sigma\ge4/\min\{H_0,H_\sigma\}$,
    $
        \|h_{P^0}^\star\|_{\mathrm{span}}\le H_0
    $
    and
    $
        \|h^{\star,\sigma}\|_{\mathrm{span}}\le H_\sigma
    $.
\end{lemma}
The parameter checks give bounded rewards and local TV radii at most
$\sigma$, the unichain argument in
Appendix~\ref{subsec:Instance-construction} verifies
Assumption~\mainUnichainAssumption{}, and
Lemma~\ref{lem:part-3-span-check} gives the nominal and robust
span bounds. Hence
\begin{equation}
    \cM_\phi\in\mathfrak{M}(H_0,H_\sigma,\sigma),
    \qquad \phi\in\{2,3\}.
    \label{eq:part-3-membership}
\end{equation}

\paragraph*{KL divergence and testing conclusion.}
It remains to bound the statistical distance between the two instances and
apply the testing reduction in
Appendix~\ref{subsubsec:generic-testing-reduction}. Here the integration interval in
\eqref{eq:symmetric-bernoulli-kl} runs from $\sigma+p_1$ to
$\sigma+p_1+\delta$.
By the case condition, \eqref{eq:standing-robust-construction-range}, and
$\sigma\le1/2$, we have
\[
    \frac{2}{H_{\sigma}}\le \frac{\sigma}{2},
    \qquad
    \frac{40\varepsilon}{H_{\sigma}}\le \frac{2\sigma}{5}.
\]
Therefore
\[
    t\ge\sigma,
    \qquad
    1-t
    \ge
    1-\sigma-p_1-\delta
    \ge
    1-\sigma-\frac{\sigma}{2}-\frac{2\sigma}{5}
    \ge
    \frac{1}{20}
\]
throughout this interval. Applying \eqref{eq:symmetric-bernoulli-kl} directly
therefore gives
\begin{equation}
    \sum_{(s,a)\in\cS\times\cA}
    \mathrm{KL}\!\left(P_{s,a}^{0,2}\,\Vert\,P_{s,a}^{0,3}\right)
    \le
    \frac{20\delta^2}{\sigma}
    \le
    \frac{32000\varepsilon^{2}}{\sigma H_{\sigma}^2}.
    \label{eq:KL-P0-P1-part-3}
\end{equation}

By
\eqref{eq:KL-P0-P1-part-3}, the sample-size condition
\eqref{eq:generic-kl-and-sample-bound} holds whenever
\[
    N
    \le
    \frac{\sigma H_\sigma^2}{128000\varepsilon^2}.
\]
Under this sample-size bound, \eqref{eq:generic-testing-conclusion} gives
\begin{equation}
    \inf_{\widehat\pi}
    \max_{\phi\in\{2,3\}}
    \mathbb P_{\cM_\phi}
    \left\{
        \rho^{\star,\sigma}-\rho^{\widehat\pi,\sigma}>\varepsilon
    \right\}
    \ge
    \frac14.
    \label{eq:large-sigma-robust-component-lower}
\end{equation}
This proves the $\sigma H_\sigma^2$ component when
$\sigma\ge4/\min\{H_0,H_\sigma\}$.

\subsection{Proofs of auxiliary lemmas}
\label{subsec:proof-auxiliary-lemmas-lower-bound}

\subsubsection{Proof of Lemma~\ref{lem:avg-reward-bias-span}}
\label{subsec:Proof-lemma-avg-reward-bias-span}

\paragraph*{Policy $\pi_1$.}
We first analyze the policy $\pi_1$.
We calculate its nominal and robust average rewards and then its expected time
to return to $S$ under $\bar P^\phi$.

\subparagraph*{Nominal average reward.}
Under the nominal kernel, the
stationary distribution $\mu^0$ is supported on $\{S,1\}$ and satisfies
\[
    (\sigma_0+p_0)\mu^0(S)=q_0\mu^0(1),
    \qquad
    \mu^0(S)+\mu^0(1)=1.
\]
Solving these equations gives
\[
    \rho_1
    =
    r_0\mu^0(1)
    =
    \frac{\sigma_0+p_0}{\sigma_0+p_0+q_0}\,r_0.
\]

\subparagraph*{Robust average reward.}
We next verify the robust formula. We claim that
\[
    \rho_1^\sigma
    =
    \frac{p_0}{p_0+q_0(1+\sigma_0)}\,r_0.
\]
To verify this claim, we construct a fixed-policy bias for which the robust
Bellman equation holds and $\bar P^\phi$ attains the minimizing transition
rows. Normalize this bias by $h(S)=0$ and define
\[
    h(1)
    =
    \frac{r_0-\rho_1^\sigma}{q_0},
    \qquad
    h(2)=h(3)=r_1-\rho_1^\sigma,
\]
and
\[
    h(s)=-\rho_1^\sigma,
    \qquad
    s\in\{4,\ldots,S-1\}.
\]
Since $\rho_1^\sigma\le r_0$ and all rewards are nonnegative,
\[
    h(1)\ge h(S)=0,
    \qquad
    \min_{s\in\cS}h(s)
    =
    h(4)
    =
    -\rho_1^\sigma.
\]
The nominal row $P^0_{S,1}$ assigns mass only to states $1$ and $S$, with
$P^0_{S,1}(1)=\sigma_0+p_0\ge\sigma_0$. Among these two states,
$h(1)\ge h(S)$, and state $4$ has a minimum coordinate and zero
nominal mass. It is therefore feasible to transfer mass $\sigma_0$ from state
$1$ to state $4$, and no perturbation of TV distance at most $\sigma_0$ can
decrease the expectation further. Thus
\[
    \min_{Q\in\cP_{S,1}}Q^\top h
    =
    (P^0_{S,1})^\top h
    -
    \sigma_0\bigl(h(1)-h(4)\bigr)
    =
    (\bar P_{S,1}^\phi)^\top h.
\]
For this minimizing row,
\[
    (\bar P_{S,1}^\phi)^\top h
    =
    p_0\frac{r_0-\rho_1^\sigma}{q_0}
    -
    \sigma_0\rho_1^\sigma
    =
    \frac{
        p_0r_0-(p_0+q_0\sigma_0)\rho_1^\sigma
    }{q_0}
    =
    \rho_1^\sigma,
\]
where the last equality uses
$p_0r_0=\rho_1^\sigma[p_0+q_0(1+\sigma_0)]$. At state $1$,
\[
    r_0+(1-q_0)h(1)
    =
    \rho_1^\sigma+h(1).
\]
At $j\in\{2,3\}$, policy $\pi_1$ returns directly to $S$, and
\[
    r_1+h(S)
    =
    \rho_1^\sigma+h(j).
\]
At each delay or padding state,
\[
    0+h(S)
    =
    \rho_1^\sigma+h(s).
\]
Thus the fixed-policy robust Bellman equation holds at every state, and the
minimizing rows are those of $\bar P^\phi$:
\[
    \rho_1^\sigma\mathbf 1+h
    =
    r^{\pi_1}+\bar P^\phi_{\pi_1}h.
\]
Multiplying this equality by a stationary distribution of the Markov chain
induced by $(\pi_1,\bar P^\phi)$ and canceling the bias terms gives
\[
    \rho_{\bar P^\phi}^{\pi_1}=\rho_1^\sigma.
\]
Both directions of Lemma~\ref{lem:fixed-policy-robust-verification} also give
$\min_{P\in\cP}\rho_P^{\pi_1}=\rho_1^\sigma$. Hence $\bar P^\phi$ attains
this minimum.

\subparagraph*{Expected return time.}
Under $\bar P^\phi$, the time between successive returns to $S$ includes one
step from $S$, one additional step with probability $\sigma_0$, and an
expected $1/q_0$ steps at state $1$ with probability $p_0$. Hence
\[
    T_1^\sigma
    =
    1+\sigma_0+\frac{p_0}{q_0}.
\]

\paragraph*{Remaining policies.}
First fix $i\in\{2,3\}$. The same argument applies with
\[
    (p_0,q_0,\sigma_0,r_0)
    \quad\text{replaced by}\quad
    \left(
        p_1+\delta\mathds{1}_{\{\phi=i\}},q_1,\sigma_1,r_1
    \right).
\]
At each other reward state, set the bias coordinate equal to its immediate
reward minus $\rho_i^\sigma$. Because $\pi_i$ returns immediately to $S$ there,
the corresponding fixed-policy Bellman equation holds. Consequently,
$\bar P^\phi$ attains the robust minimum. The resulting rewards are
\[
    \rho_i
    =
    \frac{
        \sigma_1+p_1+\delta\mathds{1}_{\{\phi=i\}}
    }{
        \sigma_1+p_1+q_1+\delta\mathds{1}_{\{\phi=i\}}
    }r_1,
    \qquad
    \rho_i^\sigma
    =
    \frac{
        p_1+\delta\mathds{1}_{\{\phi=i\}}
    }{
        p_1+\delta\mathds{1}_{\{\phi=i\}}+q_1(1+\sigma_1)
    }r_1,
\]
and the expected return time is
\[
    T_i^\sigma
    =
    1+\sigma_1
    +
    \frac{p_1+\delta\mathds{1}_{\{\phi=i\}}}{q_1}.
\]
Taking $i=\phi$ and $i=5-\phi$ proves the two corresponding cases of the
lemma.

Finally, if $a\notin\{1,2,3\}$, the uncertainty radius at $(S,a)$ is zero.
Hence every $P\in\cP$ returns immediately from $S$ to $S$ under $\pi_a$, with
zero reward. Therefore
\[
    \rho_a=\rho_a^\sigma=0,
    \qquad
    T_a^\sigma=1,
\]
which completes the proof of Lemma~\ref{lem:avg-reward-bias-span}.

\subsubsection{Proof of Lemmas~\ref{lem:part-2-cycle-check} and~\ref{lem:part-3-cycle-check}}
\label{subsec:proof-cycle-length-checks}

Write $\bar H:=\min\{H_0,H_\sigma\}$. We verify the return-time and gap
conditions in the two parameter regimes. Throughout, the formulas for
$T_a^\sigma$ and $\rho_a^\sigma$ come from
Lemma~\ref{lem:avg-reward-bias-span}.

\paragraph*{Case $\sigma<4/\bar H$.}
Direct substitution, together with
$\varepsilon\le0.01\sigma\bar H$ and $\sigma\le1/2$, gives
\[
    T_\phi^\sigma
    =2+\frac{\sigma}{4}
      +\frac{200\varepsilon}{\sigma\bar H}
    \le\frac{33}{8},
    \quad
    T_{5-\phi}^\sigma=2+\frac{\sigma}{4}\in
    \left[2,\frac{17}{8}\right],
    \quad
    T_1^\sigma=\frac32+\sigma\in\left[\frac32,2\right],
\]
while $T_a^\sigma=1$ for $a\notin\{1,2,3\}$. For the only delicate
reward-gap calculation, subtracting the two robust reward formulas yields
\[
    \rho_\phi^\sigma-\rho_{5-\phi}^\sigma
    =
    \frac{3\sigma\bar H}{16}
    \frac{\frac{100\varepsilon}{H_\sigma}
          (1+\sigma/4)}
    {\left(\frac{\sigma\bar H}{2H_\sigma}(2+\sigma/4)
          +\frac{100\varepsilon}{H_\sigma}\right)(2+\sigma/4)}.
\]
Since the first factor in parentheses in the denominator is at most
$\frac{\sigma\bar H}{2H_\sigma}(4+\sigma/4)$,
\[
    \rho_\phi^\sigma-\rho_{5-\phi}^\sigma
    \ge
    \frac{600\varepsilon(1+\sigma/4)}
    {16(4+\sigma/4)(2+\sigma/4)}
    >4\varepsilon;
\]
the last inequality is equivalent to
$88+54\sigma-4\sigma^2>0$. Moreover,
\[
    \rho_{5-\phi}^\sigma-\rho_1^\sigma
    =\sigma\bar H
    \left[\frac{3}{16(2+\sigma/4)}
          -\frac{1}{4(3+2\sigma)}\right]
    \ge\frac{\sigma\bar H}{96}>\varepsilon,
\]
because the bracket is at least $1/96$. Thus
$\rho_\phi^\sigma-\rho_1^\sigma>5\varepsilon$. Finally,
$\rho_a^\sigma=0$ for $a\notin\{1,2,3\}$ and
$\rho_\phi^\sigma\ge\rho_{5-\phi}^\sigma
=3\sigma\bar H/[16(2+\sigma/4)]$. The required weighted gaps are therefore
summarized by
\[
\begin{array}{c|c}
 a &
 \displaystyle
 \frac{T_a^\sigma}{T_\phi^\sigma+T_a^\sigma}
 (\rho_\phi^\sigma-\rho_a^\sigma)
 \text{ is bounded below by}
 \\ \hline
 5-\phi &
 \displaystyle \frac{2}{33/8+17/8}\,4\varepsilon>\varepsilon
 \\
 1 &
 \displaystyle \frac{3/2}{33/8+2}\,5\varepsilon>\varepsilon
 \\
 a\notin\{1,2,3\} &
 \displaystyle \frac{12}{697}\sigma\bar H>\varepsilon .
\end{array}
\]
This proves \eqref{eq:generic-gap-condition} in the first regime.

\paragraph*{Case $\sigma\ge4/\bar H$.}
Here
\[
    T_\phi^\sigma=2+\sigma+20\varepsilon\le\frac{27}{10},
    \quad
    T_{5-\phi}^\sigma=2+\sigma\in\left[2,\frac52\right],
    \quad
    T_1^\sigma=\frac43,
\]
and again $T_a^\sigma=1$ for $a\notin\{1,2,3\}$. In this regime,
\[
    \rho_\phi^\sigma-\rho_{5-\phi}^\sigma
    =\frac{80\varepsilon(1+\sigma)}
    {(4+2\sigma+40\varepsilon)(4+2\sigma)}
    >4\varepsilon.
\]
Indeed, $40\varepsilon\le0.4$ and
\[
    20(1+\sigma)-(4+2\sigma+40\varepsilon)(4+2\sigma)
    \ge 2.4+3.2\sigma-4\sigma^2>0.
\]
The robust reward ordering established above also gives
$\rho_\phi^\sigma-\rho_1^\sigma>2/5-1/4=3/20$ and
$\rho_\phi^\sigma-\rho_a^\sigma>2/5$ for
$a\notin\{1,2,3\}$. Hence
\[
\begin{array}{c|c}
 a &
 \displaystyle
 \frac{T_a^\sigma}{T_\phi^\sigma+T_a^\sigma}
 (\rho_\phi^\sigma-\rho_a^\sigma)
 \text{ is bounded below by}
 \\ \hline
 5-\phi &
 \displaystyle \frac{2}{27/10+5/2}\,4\varepsilon>\varepsilon
 \\
 1 &
 \displaystyle \frac{4/3}{27/10+4/3}\,\frac{3}{20}
 =\frac{6}{121}>\varepsilon
 \\
 a\notin\{1,2,3\} &
 \displaystyle \frac{1}{27/10+1}\,\frac25
 =\frac{4}{37}>\varepsilon .
\end{array}
\]
This proves \eqref{eq:generic-gap-condition} in the second regime.

\subsubsection{Proof of Lemmas~\ref{lem:part-2-span-check} and~\ref{lem:part-3-span-check}}
\label{subsec:proof-span-checks}

We first give a robust Bellman certificate common to both parameter regimes.
Normalize $h(S)=0$ and set
\[
    h(1)=\frac{r_0-\rho_\phi^\sigma}{q_0},
    \qquad
    h(2)=h(3)=\frac{r_1-\rho_\phi^\sigma}{q_1},
    \qquad
    h(s)=-\rho_\phi^\sigma,
    \quad s\in\{4,\ldots,S-1\}.
\]
The reward comparisons established in the corresponding regimes imply
$h(1),h(2),h(3)\ge0$. At the reward states, the definition of $h$ gives
\[
    r_0+(1-q_0)h(1)=\rho_\phi^\sigma+h(1),
    \qquad
    r_1+(1-q_1)h(i)=\rho_\phi^\sigma+h(i),
    \quad i\in\{2,3\};
\]
the selected action therefore dominates an immediate return to $S$. The delay and padding states
satisfy $\rho_\phi^\sigma+h(s)=0$. At $S$, a minimizing row transfers its
full uncertainty budget from the reward state to a delay state; this is
feasible because the nominal row assigns at least the uncertainty radius to
the former and no mass to the latter. The expression for action $\phi$
then equals
\[
    (p_1+\delta)h(\phi)-\sigma_1\rho_\phi^\sigma
    =\rho_\phi^\sigma.
\]
For actions $1$ and $5-\phi$, the same substitution shows that their
expressions are at most $\rho_\phi^\sigma$ precisely when
$\rho_1^\sigma\le\rho_\phi^\sigma$ and
$\rho_{5-\phi}^\sigma\le\rho_\phi^\sigma$, respectively; these are the
robust reward orderings already proved. All remaining actions contribute
zero. Consequently,
\[
    \rho_\phi^\sigma\mathbf 1+h=\mathcal T_\sigma h.
\]
Lemma~\ref{lem:fixed-policy-robust-verification}, applied first to an
arbitrary policy and then to a maximizing selector of
$\mathcal T_\sigma h$, shows that
$\rho^{\star,\sigma}=\rho_\phi^\sigma$ and that $h$ is a robust optimal
Bellman solution.

It remains to construct the nominal certificates and bound their spans.
When $\sigma<4/\bar H$, Lemma~\ref{lem:avg-reward-bias-span} gives
$\rho_1=r_1$. Define
\[
    g(S)=0,
    \qquad
    g(1)=\frac{r_0-r_1}{q_0},
    \qquad
    g(2)=g(3)=0,
    \qquad
    g(s)=-r_1,
    \quad s\in\{4,\ldots,S-1\}.
\]
The reward-state and delay-state equations follow directly from this
definition, while at $S$,
\[
    (\mathcal T_0g)(S)
    =(\sigma_0+p_0)g(1)
    =\frac{6\sigma}{5}\frac{5\bar H}{32}
    =r_1.
\]
Thus $r_1\mathbf 1+g=\mathcal T_0g$.

When $\sigma\ge4/\bar H$, we have $r_0=r_1=1$,
$q_0=3/\bar H$, and $q_1=2/H_\sigma$. Moreover,
\[
    \rho_1=\frac14,
    \qquad
    \rho_\phi
    =\frac{\sigma+p_1+\delta}{\sigma+p_1+q_1+\delta}>\frac12,
    \qquad
    \rho_{5-\phi}<\rho_\phi,
\]
where $q_1\le\sigma/2<\sigma+p_1+\delta$. Define $g(S)=0$,
$g(1)=(1-\rho_\phi)/q_0$, $g(2)=g(3)=(1-\rho_\phi)/q_1$, and
$g(s)=-\rho_\phi$ for $s\in\{4,\ldots,S-1\}$.
Again the Bellman equations away from $S$ follow directly. At $S$, action
$\phi$ has value
$(\sigma+p_1+\delta)g(\phi)=\rho_\phi$; action $5-\phi$ has smaller value
because $\delta>0$, and action $1$ has value at most $\rho_\phi$ because
$\rho_1\le\rho_\phi$. Hence
$\rho_\phi\mathbf 1+g=\mathcal T_0g$.

For completeness, the span calculations for these certificates are collected
below. In the second nominal case, the parameter bounds give
\[
    g(2)+\rho_\phi
    =\frac{1+\sigma+p_1+\delta}{\sigma+p_1+q_1+\delta}
    \le\frac{2}{\sigma}\le\frac{\bar H}{2},
    \qquad
    g(1)+\rho_\phi
    \le\frac{1}{q_0}+1\le\frac{7\bar H}{12}.
\]
The remaining bounds follow immediately from the displayed definitions:
\[
\begin{array}{c@{\qquad}c@{\qquad}l}
 \text{regime} & \text{certificate} & \text{span calculation} \\ \hline
 \sigma<4/\bar H
 & \text{nominal }g
 & \max g=5\bar H/32,\quad -\min g=3\sigma\bar H/16,
   \quad \|g\|_{\mathrm{span}}\le\bar H/4
 \\[0.6ex]
 \sigma<4/\bar H
 & \text{robust }h
 & \max h\le5H_\sigma/8,\quad -\min h\le H_\sigma/8,
   \quad \|h\|_{\mathrm{span}}\le3H_\sigma/4
 \\[0.6ex]
 \sigma\ge4/\bar H
 & \text{nominal }g
 & \|g\|_{\mathrm{span}}
   =\max\{g(1)+\rho_\phi,g(2)+\rho_\phi\}\le7\bar H/12
 \\[0.6ex]
 \sigma\ge4/\bar H
 & \text{robust }h
 & \max h\le H_\sigma/2,\quad -\min h\le H_\sigma/4,
   \quad \|h\|_{\mathrm{span}}\le3H_\sigma/4.
\end{array}
\]
Since $h_{P^0}^\star$ and $h^{\star,\sigma}$ are minimum-span Bellman
solutions and $\bar H\le H_0$, the first two rows prove
Lemma~\ref{lem:part-2-span-check}, and the last two prove
Lemma~\ref{lem:part-3-span-check}.

\section{Proofs for the reduction-based upper bounds}\label{app:reduction-upper-bounds}

This appendix contains the proofs for the reduction-based upper bounds in
Section~\mainReductionUpperBoundSection{}.
Appendix~\ref{subsec:discounted-to-average-reduction} establishes the
discounted-to-average reduction, while
Appendix~\ref{subsec:discounted-guarantees} states the anchored robust DMDP
theorem and its corollaries.
Appendix~\ref{subsec:Proof-Upper-average} states the nominal-branch guarantee
and combines it with the robust reduction to prove
Theorem~\mainSpanInformedUpperBoundTheorem{}; the proof of the nominal-branch
guarantee is deferred to
Appendix~\ref{subsec:proof-nominal-branch-guarantee}.
Appendix~\ref{subsec:Proof-Upper-discounted} proves
Theorem~\ref{thm:TV-upper-bound-discount}.
Finally, Appendix~\ref{sec:proof_auxiliary_lemmas} collects the proofs of the
technical lemmas used above.

Throughout this appendix, $C>0$ denotes a universal constant whose
value may change from line to line. Constants with subscripts, such as $C_0$
and $C_1$, remain fixed once introduced.

\subsection{Discounted-to-average reduction}
\label{subsec:discounted-to-average-reduction}

In this subsection, we present a few results that convert discounted guarantees into average-reward guarantees. 

The first lemma controls the difference between a robust discounted value and the corresponding robust average reward.

\begin{lemma}[Robust discounted-to-average comparison]
    \label{lem:robust-discounted-average-sandwich}
    Suppose Assumption~\mainUnichainAssumption{} holds. Fix a stationary policy
    $\pi$ and a discount factor $\gamma\in(0,1)$.
    \begin{equation}
        (1-\gamma)\min_{s}V_{\gamma}^{\pi,\sigma}(s)
        \le
        \rho^{\pi,\sigma}
        \le
        (1-\gamma)\max_{s}V_{\gamma}^{\pi,\sigma}(s).
        \label{eq:robust-discounted-average-sandwich-value}
    \end{equation}
    Equivalently, for all $s\in\cS$,
    \begin{equation}
        -(1-\gamma)\|V_{\gamma}^{\pi,\sigma}\|_{\mathrm{span}}
        \le
        \rho^{\pi,\sigma}-(1-\gamma)V_{\gamma}^{\pi,\sigma}(s)
        \le
        (1-\gamma)\|V_{\gamma}^{\pi,\sigma}\|_{\mathrm{span}}.
        \label{eq:robust-discounted-average-sandwich-span}
    \end{equation}
    Moreover,
    \begin{equation}
        \rho^{\star,\sigma}
        \le
        (1-\gamma)\min_{s}V_{\gamma}^{\star,\sigma}(s)
        +
        2(1-\gamma)H_{\sigma}.
        \label{eq:robust-discounted-average-sandwich-optimal}
    \end{equation}
\end{lemma}

The next lemma turns a discounted-policy error bound into an
average-reward error bound.

\begin{lemma}[Direct robust discounted-to-average reduction]
    \label{lem:reduction-framework}
    Let $\pi$ be a stationary policy satisfying, componentwise,
    \[
        V_\gamma^{\star,\sigma}
        -
        V_\gamma^{\pi,\sigma}
        \le
        \varepsilon_\gamma\mathbf 1.
    \]
    Then
    \[
        \rho^{\star,\sigma}-\rho^{\pi,\sigma}
        \le
        (1-\gamma)
        \left(
            2H_\sigma+\varepsilon_\gamma
        \right).
    \]
\end{lemma}

\begin{proof}
    Equations~\eqref{eq:robust-discounted-average-sandwich-optimal} and
    \eqref{eq:robust-discounted-average-sandwich-value} give
    \begin{align*}
        \rho^{\star,\sigma}-\rho^{\pi,\sigma}
        &\le
        (1-\gamma)
        \left(
            \min_sV_\gamma^{\star,\sigma}(s)
            -
            \min_sV_\gamma^{\pi,\sigma}(s)
            +
            2H_\sigma
        \right) \\
        &\le
        (1-\gamma)
        \left(
            2H_\sigma+
            \varepsilon_\gamma
        \right).
    \end{align*}
    Here the second inequality follows from the componentwise assumption, which
    implies
    \[
        \min_sV_\gamma^{\pi,\sigma}(s)
        \ge
        \min_sV_\gamma^{\star,\sigma}(s)-\varepsilon_\gamma.
    \]
\end{proof}

Finally, the following corollary chooses the discount factor and discounted-policy error
bound so that the resulting policy is $\varepsilon$-optimal for the robust
average-reward problem.

\begin{corollary}[Parameters for the span-informed reduction]
    \label{cor:current-reduction-parameters}
    Suppose $0<\varepsilon\le1$ and
    \[
        \gamma
        =
        1-
        \frac{\varepsilon}{3H_\sigma}.
    \]
    If a stationary policy $\pi$ satisfies
    \[
        V_\gamma^{\star,\sigma}
        -
        V_\gamma^{\pi,\sigma}
        \le
        H_\sigma\mathbf 1,
    \]
    then
    \[
        \rho^{\star,\sigma}-\rho^{\pi,\sigma}
        \le
        \varepsilon.
    \]
\end{corollary}

\begin{proof}
    The definition in \mainRobustSpanDefinition{} gives $H_\sigma\ge1$, so the
    displayed choice of $\gamma$ belongs to $(0,1)$. Apply
    Lemma~\ref{lem:reduction-framework} with
    $\varepsilon_\gamma=H_\sigma$.
\end{proof}

\subsubsection{Proof of Lemma~\ref{lem:robust-discounted-average-sandwich}}
\label{subsec:proof-robust-discounted-average-sandwich}
We first prove
\eqref{eq:robust-discounted-average-sandwich-value} in two directions,
and then derive
\eqref{eq:robust-discounted-average-sandwich-span} and
\eqref{eq:robust-discounted-average-sandwich-optimal}.

\paragraph*{Lower bound in \eqref{eq:robust-discounted-average-sandwich-value}.}
Fix an arbitrary transition kernel $P\in\cP$, and let
$\rho_P^\pi$ and $V_{\gamma,P}^\pi$ denote the average reward and the
discounted value of $\pi$ under $P$, respectively. Under
Assumption~\mainUnichainAssumption{}, let $\mu$ be a stationary distribution
supported on the unique recurrent class induced by $(\pi,P)$. The discounted
Bellman equation
\[
    V_{\gamma,P}^{\pi}
    =
    r^{\pi}
    +
    \gamma P_{\pi}V_{\gamma,P}^{\pi}
\]
and $\mu^\top P_\pi=\mu^\top$ imply
\[
    \rho_P^\pi
    =
    \mu^\top r^\pi
    =
    (1-\gamma)\mu^\top V_{\gamma,P}^{\pi}.
\]
Therefore
\[
    (1-\gamma)\min_s V_{\gamma,P}^{\pi}(s)
    \le
    \rho_P^\pi
    \le
    (1-\gamma)\max_s V_{\gamma,P}^{\pi}(s).
\]
Since
$V_\gamma^{\pi,\sigma}(s)\le V_{\gamma,P}^{\pi}(s)$ for every
$s\in\cS$ and every $P\in\cP$,
\[
    (1-\gamma)\min_s V_{\gamma}^{\pi,\sigma}(s)
    \le
    \rho_P^\pi.
\]
Taking the infimum over $P\in\cP$ gives the lower bound in
\eqref{eq:robust-discounted-average-sandwich-value}.

\paragraph*{Upper bound in \eqref{eq:robust-discounted-average-sandwich-value}.}
For the fixed policy $\pi$, the discounted robust Bellman equation gives
\[
    V_{\gamma}^{\pi,\sigma}(s)
    =
    r^\pi(s)
    +
    \gamma\inf_{P\in\cP}
    (P_\pi V_{\gamma}^{\pi,\sigma})(s).
\]
Since
\[
    \inf_{P\in\cP}
    (P_\pi V_{\gamma}^{\pi,\sigma})(s)
    \le
    \max_{s'}V_{\gamma}^{\pi,\sigma}(s'),
\]
we have
\begin{align*}
    &(1-\gamma)\max_{s'}V_{\gamma}^{\pi,\sigma}(s')
    +V_{\gamma}^{\pi,\sigma}(s)
    \\
    &\qquad\ge
    V_{\gamma}^{\pi,\sigma}(s)
    +(1-\gamma)\inf_{P\in\cP}
    (P_\pi V_{\gamma}^{\pi,\sigma})(s)
    \\
    &\qquad=
    r^\pi(s)
    +
    \inf_{P\in\cP}
    (P_\pi V_{\gamma}^{\pi,\sigma})(s).
\end{align*}
Lemma~\ref{lem:fixed-policy-robust-verification}, applied with
$h=V_{\gamma}^{\pi,\sigma}$ and
$\rho=(1-\gamma)\max_{s'}V_{\gamma}^{\pi,\sigma}(s')$, gives
\[
    \rho^{\pi,\sigma}
    \le
    (1-\gamma)\max_{s'}V_{\gamma}^{\pi,\sigma}(s').
\]
This proves the upper bound in
\eqref{eq:robust-discounted-average-sandwich-value}.

\paragraph*{Proof of \eqref{eq:robust-discounted-average-sandwich-span}.}
Subtracting $(1-\gamma)V_{\gamma}^{\pi,\sigma}(s)$ from each side of
\eqref{eq:robust-discounted-average-sandwich-value} gives
\[
    \begin{aligned}
    &
    (1-\gamma)
    \left(
    \min_{s'}V_{\gamma}^{\pi,\sigma}(s')
    -V_{\gamma}^{\pi,\sigma}(s)
    \right)
    \\
    &\qquad
    \le
    \rho^{\pi,\sigma}-(1-\gamma)V_{\gamma}^{\pi,\sigma}(s)
    \\
    &\qquad
    \le
    (1-\gamma)
    \left(
    \max_{s'}V_{\gamma}^{\pi,\sigma}(s')
    -V_{\gamma}^{\pi,\sigma}(s)
    \right),
    \end{aligned}
\]
which implies \eqref{eq:robust-discounted-average-sandwich-span} by the
definition of the span seminorm.

\paragraph*{Proof of \eqref{eq:robust-discounted-average-sandwich-optimal}.}
For every $s\in\cS$ and $a\in\cA$,
\[
    r(s,a)
    +
    \inf_{P_{s,a}\in\cP_{s,a}}
    P_{s,a}V_{\gamma}^{\star,\sigma}
    \le
    r(s,a)
    +
    \gamma
    \inf_{P_{s,a}\in\cP_{s,a}}
    P_{s,a}V_{\gamma}^{\star,\sigma}
    +
    (1-\gamma)\max_{s'}V_{\gamma}^{\star,\sigma}(s').
\]
Taking the maximum over $a$ and using the discounted robust Bellman equation
for $V_{\gamma}^{\star,\sigma}$ gives
\[
    \max_{a\in\cA}
    \left\{
    r(s,a)
    +
    \inf_{P_{s,a}\in\cP_{s,a}}
    P_{s,a}V_{\gamma}^{\star,\sigma}
    \right\}
    \le
    V_{\gamma}^{\star,\sigma}(s)
    +
    (1-\gamma)\max_{s'}V_{\gamma}^{\star,\sigma}(s').
\]
By rectangularity of the uncertainty set, for every stationary policy $\pi$
and every state $s$,
\begin{align*}
    &r^\pi(s)
    +
    \inf_{P\in\cP}(P_\pi V_{\gamma}^{\star,\sigma})(s)
    \\
    &\qquad=
    \sum_{a\in\cA}\pi(a\mid s)
    \left\{
        r(s,a)
        +
        \inf_{P_{s,a}\in\cP_{s,a}}
        P_{s,a}V_{\gamma}^{\star,\sigma}
    \right\}
    \\
    &\qquad\le
    \max_{a\in\cA}
    \left\{
        r(s,a)
        +
        \inf_{P_{s,a}\in\cP_{s,a}}
        P_{s,a}V_{\gamma}^{\star,\sigma}
    \right\}.
\end{align*}
Combining this inequality with the preceding display gives
\[
    r^\pi(s)
    +
    \inf_{P\in\cP}(P_\pi V_{\gamma}^{\star,\sigma})(s)
    \le
    V_{\gamma}^{\star,\sigma}(s)
    +
    (1-\gamma)\max_{s'}V_{\gamma}^{\star,\sigma}(s'),
    \qquad s\in\cS.
\]
Lemma~\ref{lem:fixed-policy-robust-verification}, applied to each fixed
policy with $h=V_{\gamma}^{\star,\sigma}$ and
$\rho=(1-\gamma)\max_{s'}V_{\gamma}^{\star,\sigma}(s')$, gives
\[
    \rho^{\pi,\sigma}
    \le
    (1-\gamma)\max_{s'}V_{\gamma}^{\star,\sigma}(s')
\]
for every stationary policy $\pi$. Taking the supremum over $\pi$ yields
\[
    \rho^{\star,\sigma}
    \le
    (1-\gamma)\max_{s'}V_{\gamma}^{\star,\sigma}(s').
\]
Here the supremum is over stationary policies.
Finally,
\begin{align*}
    (1-\gamma)\max_{s'}V_{\gamma}^{\star,\sigma}(s')
    &\le
    (1-\gamma)\min_sV_{\gamma}^{\star,\sigma}(s)
    +
    (1-\gamma)\|V_{\gamma}^{\star,\sigma}\|_{\mathrm{span}} \\
    &\le
    (1-\gamma)\min_sV_{\gamma}^{\star,\sigma}(s)
    +
    2(1-\gamma)H_{\sigma},
\end{align*}
where the last inequality uses
Lemma~\ref{lem:robust-value-function-span-bound}.
Together with the preceding bound on $\rho^{\star,\sigma}$, this proves
\eqref{eq:robust-discounted-average-sandwich-optimal}.

\subsection{Anchored robust DMDP guarantees}
\label{subsec:discounted-guarantees}

Throughout the discounted analysis, fix a discount factor
$\gamma\in[1/2,1)$.

This subsection first states an anchored discounted MDP error bound with a corresponding target-accuracy guarantee. We also state an
anchor-free guarantee obtained from the trivial anchor.

Recall the nominal anchor condition
\mainNominalAnchorSupersolution{}. The anchor provides a reference
average-reward level whose associated span can sharpen the discounted bound.
For such an anchor pair, define its span scale
\begin{equation}
    \label{eq:anchor-span-scale}
    H_{\mathrm{anc}}
    \coloneqq
    \max\left\{
    1,
    \|\bar h\|_{\mathrm{span}}
    \right\}.
\end{equation}
We use $R_0$ as a common upper bound on the span of the anchor and
the true optimal discounted value. Specifically, let
\begin{equation}
    \label{eq:reference-radius}
    R_0
    \ge
    \max\left\{
    1,
    H_{\mathrm{anc}},
    \|V_{\gamma}^{\star,\sigma}\|_{\mathrm{span}}
    \right\}.
\end{equation}
The anchor defect measures how far the anchor reward level lies above the
discounted optimal baseline. Define
\begin{equation}
    \label{eq:true-anchor-defect}
    \beta_{\star}
    \coloneqq
    \left[
    \bar\rho-(1-\gamma)\min_s V_{\gamma}^{\star,\sigma}(s)
    \right]_+.
\end{equation}

For a confidence parameter $\delta\in(0,1)$ and an integer $N\ge1$, define
the logarithmic confidence factor
\[
    \iota
    \coloneqq
    \log\left(\frac{54SAN^2}{(1-\gamma)\delta}\right).
\]
With this notation in place, the following theorem gives the
finite-sample discounted plug-in bounds. Its proof is deferred to
Appendix~\ref{subsec:Proof-Upper-discounted}.
\begin{theorem}[Anchored robust DMDP plug-in theorem]\label{thm:TV-upper-bound-discount}
    Assume that $\cM$ is a robust discounted MDP described in
    Section~\mainProblemSetupSection{} with discount factor $\gamma\ge1/2$.
    Let the anchor quantities
    $(\bar\rho,\bar h,H_{\mathrm{anc}},R_0,\beta_\star)$ be fixed
    independently of the transition samples used to construct the empirical
    nominal kernel. The anchor may be constructed from an independent data
    batch, in which case the result applies conditionally on any realization
    satisfying the stated anchor conditions. Let a policy $\widehat\pi$ with
    solver tolerance $\varepsilon_{\mathrm{opt}}\ge0$ satisfy
    \mainRobustDMDPSolverTolerance{},
    \mainNominalAnchorSupersolution{},
    \eqref{eq:anchor-span-scale}, \eqref{eq:reference-radius}, and
    \eqref{eq:true-anchor-defect}.
    If, for a sufficiently large universal constant $C_{\mathrm{sam}}$ and a
    sufficiently small universal constant $c$,
    \begin{equation}
        \label{eq:anchored-dmdp-sample-condition}
        N
        \ge
        C_{\mathrm{sam}}\left[
        \frac{\iota}{1-\gamma}
        +\frac{\sigma\cdot\iota}{(1-\gamma)^2}
        +\frac{
        \left(
        H_{\mathrm{anc}}
        +R_0\beta_\star
        \right)\cdot\iota
        }{(1-\gamma)^2R_0^2}
        \right]
    \end{equation}
    and $\varepsilon_{\mathrm{opt}}\le c(1-\gamma)R_0$,
    then, with probability at least $1-O(\delta)$,
    \[
        \|V_{\gamma}^{\star,\sigma}
        -V_{\gamma}^{\widehat\pi,\sigma}\|_{\infty}
        \le
        C\sqrt{
        \frac{
        \left(
        H_{\mathrm{anc}}
        +R_0\beta_\star
        +\sigma R_0^2
        +(1-\gamma)R_0^2
        \right)\cdot\iota
        }{N(1-\gamma)^2}}
        +C\frac{R_0\cdot\iota}{N(1-\gamma)}
        +C\frac{\varepsilon_{\mathrm{opt}}}{1-\gamma}.
    \]
    \[
        \|\widehat V_{\gamma}^{\widehat\pi,\sigma}
        -V_{\gamma}^{\widehat\pi,\sigma}\|_{\infty}
        \le
        C\sqrt{
        \frac{
        \left(
        H_{\mathrm{anc}}
        +R_0\beta_\star
        +\sigma R_0^2
        +(1-\gamma)R_0^2
        \right)\cdot\iota
        }{N(1-\gamma)^2}}
        +C\frac{R_0\cdot\iota}{N(1-\gamma)}
        +C\frac{\varepsilon_{\mathrm{opt}}}{1-\gamma}.
    \]
\end{theorem}

The first corollary converts the theorem's error bound into a
sample-size condition for a prescribed discounted accuracy.

\begin{corollary}[Accuracy form of the anchored robust DMDP guarantee]
    \label{cor:anchored-robust-dmdp-accuracy}
    Under the model conditions and anchor notation of
    Theorem~\ref{thm:TV-upper-bound-discount}, let
    $0<\varepsilon_\gamma\le R_0$. For a sufficiently large
    universal constant $C_{\mathrm{sam}}$ and a sufficiently small universal
    constant $c$,
    if
    \begin{equation}
        \label{eq:anchored-dmdp-accuracy-sample-condition}
        N
        \ge
        C_{\mathrm{sam}}\frac{
        \left(
        H_{\mathrm{anc}}
        +R_0\beta_\star
        +\sigma R_0^2
        +(1-\gamma)R_0^2
        \right)\cdot\iota
        }{(1-\gamma)^2\varepsilon_\gamma^2}
        +C_{\mathrm{sam}}\frac{R_0\cdot\iota}
        {(1-\gamma)\varepsilon_\gamma}
    \end{equation}
    and $\varepsilon_{\mathrm{opt}}\le
    c(1-\gamma)\varepsilon_\gamma$, then, with probability at least
    $1-O(\delta)$,
    \[
        \|V_{\gamma}^{\star,\sigma}
        -V_{\gamma}^{\widehat\pi,\sigma}\|_\infty
        \le
        \varepsilon_\gamma.
    \]
\end{corollary}

\begin{proof}
    Since $\varepsilon_\gamma\le R_0$,
    \eqref{eq:anchored-dmdp-accuracy-sample-condition} implies
    the sample-size condition \eqref{eq:anchored-dmdp-sample-condition} and     \[
        \varepsilon_{\mathrm{opt}}
        \le
        c(1-\gamma)\varepsilon_\gamma
        \le
        c(1-\gamma)R_0,
    \]
    so the condition on $\varepsilon_{\mathrm{opt}}$ in
    Theorem~\ref{thm:TV-upper-bound-discount} is satisfied. 
    Substituting \eqref{eq:anchored-dmdp-accuracy-sample-condition} into the
    remaining terms in the theorem's error bound and then choosing
    $C_{\mathrm{sam}}$ sufficiently large and $c$ sufficiently small proves
    the claim.
\end{proof}

The following corollary yields a span-only guarantee when no informative
anchor is available.

\begin{corollary}[Span-only robust DMDP guarantee]
    \label{cor:span-only-robust-dmdp}
    Let $\varepsilon_\gamma>0$ be the target discounted accuracy. Assume that
    $\cM$ is a robust discounted MDP described in
    Section~\mainProblemSetupSection{} with discount factor $\gamma\ge1/2$, and let
    $\widehat\pi$ satisfy the solver guarantee
    \mainRobustDMDPSolverTolerance{} with
    solver tolerance $\varepsilon_{\mathrm{opt}}\ge0$. Choose
    $R_0\ge
    \max\{1,\|V_{\gamma}^{\star,\sigma}\|_{\mathrm{span}}\}$.
    If $\varepsilon_\gamma\le R_0$ and, for a sufficiently large universal
    constant $C_{\mathrm{sam}}$ and a sufficiently small universal constant
    $c$,
    \begin{equation*}
        N
        \ge
        C_{\mathrm{sam}}\frac{
        \left(
        R_0
        +\sigma R_0^2
        +(1-\gamma)R_0^2
        \right)\cdot\iota
        }{(1-\gamma)^2\varepsilon_\gamma^2}
        +C_{\mathrm{sam}}\frac{R_0\cdot\iota}{(1-\gamma)\varepsilon_\gamma},
    \end{equation*}
    and $\varepsilon_{\mathrm{opt}}\le
    c(1-\gamma)\varepsilon_\gamma$,
    then, with probability at least $1-O(\delta)$,
    \[
        \|V_{\gamma}^{\star,\sigma}
        -V_{\gamma}^{\widehat\pi,\sigma}\|_{\infty}
        \le
        \varepsilon_\gamma.
    \]
\end{corollary}

\begin{proof}
    Apply Corollary~\ref{cor:anchored-robust-dmdp-accuracy} with the trivial anchor
    $\bar\rho=1$ and $\bar h=0$. Since $r(s,a)\in[0,1]$, this pair satisfies the
    nominal anchor supersolution
    condition~\mainNominalAnchorSupersolution{} and has
    $H_{\mathrm{anc}}=1$.
    Moreover, nonnegativity of the rewards gives
    \[
        \beta_\star
        =
        \left[
        1-(1-\gamma)\min_sV_{\gamma}^{\star,\sigma}(s)
        \right]_+
        \le
        1.
    \]
    Hence
    \[
        H_{\mathrm{anc}}
        +R_0\beta_\star
        +\sigma R_0^2
        +(1-\gamma)R_0^2
        \le
        2R_0+\sigma R_0^2+(1-\gamma)R_0^2,
    \]
    where $R_0\ge1$. Absorbing the factor $2$ into $C_{\mathrm{sam}}$ in
    Corollary~\ref{cor:anchored-robust-dmdp-accuracy} proves the claim.
\end{proof}

\subsection{Proof of Theorem~\mainSpanInformedUpperBoundTheorem{}}\label{subsec:Proof-Upper-average}

The proof follows the two branches of
Algorithm~\mainSpanInformedAlgorithm{}. We first state a guarantee for the nominal
branch, whose proof is deferred to
Appendix~\ref{subsec:proof-nominal-branch-guarantee}. We then handle the robust
branch by combining the discounted guarantees in
Appendix~\ref{subsec:discounted-guarantees} with the discounted-to-average
reduction in Appendix~\ref{subsec:discounted-to-average-reduction}.

\paragraph*{Step 1: Nominal branch.}

Suppose that $H_0<H_\sigma$ and $7\sigma H_0\le\varepsilon$, so
Algorithm~\mainSpanInformedAlgorithm{} selects the nominal branch. In this case, we give the following guarantee for the nominal branch, whose proof is deferred to
Appendix~\ref{subsec:proof-nominal-branch-guarantee}.

\begin{lemma}[Nominal-branch guarantee]
\label{lem:nominal-branch-guarantee}
There exist a sufficiently large universal constant $C>0$ and a sufficiently
small universal constant $c_{\mathrm{opt}}>0$ such that the following
holds. Suppose Assumption~\mainUnichainAssumption{} holds and $H_0$ is known. Let
$\varepsilon\in(0,1]$ and $\delta\in(0,1/2]$, and assume that
$7\sigma H_0\le\varepsilon$ and
$\varepsilon_{\mathrm{opt}}\le c_{\mathrm{opt}}\varepsilon$.
If the total sample size obeys
\[
    NSA
    \ge
    C SA
    \frac{H_0}{\varepsilon^2}
    \log\!\left(\frac{SAH_0N}{\varepsilon\delta}\right),
\]
then, with probability at least $1-O(\delta)$, Algorithm~\mainSpanInformedAlgorithm{} returns a policy
$\widehat\pi$ that satisfies
\[
    \rho^{\star,\sigma}
    -
    \rho^{\widehat\pi,\sigma}
    \le
    \varepsilon.
\]
\end{lemma}

Since $\min\{H_0,H_\sigma\}=H_0$ in this case, condition~(a) in
Theorem~\mainSpanInformedUpperBoundTheorem{} implies the requirement of
Lemma~\ref{lem:nominal-branch-guarantee}.

\paragraph*{Step 2: Robust branch when $H_0<H_\sigma$ and $7\sigma H_0>\varepsilon$.}
We use Corollary~\ref{cor:anchored-robust-dmdp-accuracy} when
$H_0<H_\sigma$ and $7\sigma H_0>\varepsilon$, and Corollary~\ref{cor:span-only-robust-dmdp} when
$H_\sigma\le H_0$. We first verify the parameter choices shared by
Steps~2 and~3.
Recall that in the robust branch, Algorithm~\mainSpanInformedAlgorithm{} sets
\[
    \gamma=1-\frac{\varepsilon}{3H_{\sigma}}.
\]
Since $\varepsilon\le1$ and $H_\sigma\ge1$, this choice satisfies
$\gamma\ge2/3$.
For Steps~2 and~3, we use the target discounted accuracy
$\varepsilon_\gamma=H_{\sigma}$. The
solver tolerance satisfies
\[
    \varepsilon_{\mathrm{opt}}
    \le
    c_{\mathrm{opt}}\varepsilon
    =
    3c_{\mathrm{opt}}(1-\gamma)H_{\sigma}.
\]
By the choice of $c_{\mathrm{opt}}$, this satisfies the
$\varepsilon_{\mathrm{opt}}$ condition in both corollaries.
Moreover, since $1/(1-\gamma)=3H_\sigma/\varepsilon$ and
$\delta\le1/2$, the definition of $\iota$ gives
\[
    \iota
    \le
    C\log\left(
    \frac{H_\sigma SAN}{\varepsilon\delta}
    \right).
\]

We now assume $H_0<H_\sigma$ and $7\sigma H_0>\varepsilon$. In this case, we apply Corollary~\ref{cor:anchored-robust-dmdp-accuracy}.
Set the anchor pair to be the
nominal optimal average-reward pair
\[
    \bar\rho=\rho^{\star},
    \qquad
    \bar h=h_{P^0}^{\star}.
\]
By the nominal average-reward Bellman optimality equation,
\[
    \rho^{\star}+h_{P^0}^{\star}(s)
    =
    \max_{a\in\cA}
    \left\{
    r(s,a)+P^0_{s,a}h_{P^0}^{\star}
    \right\},
    \qquad s\in\cS.
\]
Thus the nominal anchor supersolution
condition~\mainNominalAnchorSupersolution{} holds.
Since $H_0\ge1$ by the
standing convention in Section~\mainProblemSetupSection{}, this anchor has
$H_{\mathrm{anc}}=H_0$.

We take the reference radius
\[
    R_0=2H_{\sigma}.
\]
By Lemma~\ref{lem:robust-value-function-span-bound},
\[
    \|V_{\gamma}^{\star,\sigma}\|_{\mathrm{span}}
    \le
    2H_{\sigma}.
\]
Since $H_\sigma>H_0\ge1$, this gives
\[
    R_0
    \ge
    \max\left\{
    1,
    H_{\mathrm{anc}},
    \|V_{\gamma}^{\star,\sigma}\|_{\mathrm{span}}
    \right\}.
\]
We next bound the anchor defect $\beta_\star$. Since
$\rho^\star-\rho^{\star,\sigma}\ge0$, the inequality
$[x+y]_+\le x+[y]_+$ for $x\ge0$ gives
\begin{align*}
    \beta_\star
     & =
    \left[
    \rho^\star-(1-\gamma)\min_s V_{\gamma}^{\star,\sigma}(s)
    \right]_+                                                                 \\
     & \le
    \rho^\star-\rho^{\star,\sigma}
    +
    \left[
    \rho^{\star,\sigma}
    -(1-\gamma)\min_s V_{\gamma}^{\star,\sigma}(s)
    \right]_+                                                                  \\
     & \le
    C\left[
    \sigma H_0
    +(1-\gamma)H_{\sigma}
    \right].
\end{align*}
Here the last line uses
Proposition~\mainPerturbationBoundProposition{} and the robust
discounted-to-average optimal comparison
\eqref{eq:robust-discounted-average-sandwich-optimal} in
Lemma~\ref{lem:robust-discounted-average-sandwich}.

Using $H_{\mathrm{anc}}=H_0$, $R_0=2H_\sigma$, the preceding bound on
$\beta_\star$, and $1-\gamma=\varepsilon/(3H_\sigma)$, we obtain
\begin{equation*}
    H_{\mathrm{anc}}
    +R_0\beta_\star
    +\sigma R_0^2
    +(1-\gamma)R_0^2
    \le
    C\left[
    H_0+\sigma H_\sigma^2+\varepsilon H_\sigma
    \right].
\end{equation*}

We now check the sample-size condition in
Corollary~\ref{cor:anchored-robust-dmdp-accuracy}. With
$\varepsilon_\gamma=H_\sigma$ and
$R_0=2H_\sigma$, it is enough to have
\[
    N
    \ge
    C\left[
    \frac{H_0+\sigma H_\sigma^2}{\varepsilon^2}
    +\frac{H_\sigma}{\varepsilon}
    \right]\cdot\iota.
\]
Since
$H_0<H_\sigma$ and $7\sigma H_0>\varepsilon$,
\[
    \frac{H_\sigma}{\varepsilon}
    \le
    7\frac{\sigma H_0H_\sigma}{\varepsilon^2}
    \le
    7\frac{\sigma H_\sigma^2}{\varepsilon^2}.
\]
Together with the preceding bound on $\iota$, condition~(b) in
Theorem~\mainSpanInformedUpperBoundTheorem{} therefore implies
the requirement of Corollary~\ref{cor:anchored-robust-dmdp-accuracy}.

\paragraph*{Step 3: Robust branch when $H_\sigma\le H_0$.}
In this case, we apply
Corollary~\ref{cor:span-only-robust-dmdp} with
\[
    R_0=2H_\sigma.
\]
By Lemma~\ref{lem:robust-value-function-span-bound}, this choice satisfies the
reference-radius requirement. Substituting
$\varepsilon_\gamma=H_\sigma$, $\min\{H_0,H_\sigma\}=H_\sigma$, and
$1-\gamma=\varepsilon/(3H_\sigma)$ into the sample-size condition of
Corollary~\ref{cor:span-only-robust-dmdp}, it is enough that
\[
    N
    \ge
    C\left[
    \frac{
    \left(
    H_\sigma+\sigma H_\sigma^2+\varepsilon H_\sigma
    \right)\cdot\iota
    }{\varepsilon^2}
    +\frac{H_\sigma\cdot\iota}{\varepsilon}
    \right].
\]
Since $\varepsilon\le1$, we can further simplify it to
\[
    N
    \ge
    C\left[
    \frac{H_\sigma+\sigma H_\sigma^2}{\varepsilon^2}
    \right]\cdot\iota.
\]
If $7\sigma H_0>\varepsilon$, condition~(b) in
Theorem~\mainSpanInformedUpperBoundTheorem{} implies this requirement. If
$7\sigma H_0\le\varepsilon$, then
\[
    \sigma H_\sigma^2
    \le
    \sigma H_0H_\sigma
    \le
    \frac{\varepsilon H_\sigma}{7}
    \le
    \frac{H_\sigma}{7}.
\]
Condition~(a) in Theorem~\mainSpanInformedUpperBoundTheorem{} therefore implies the
same requirement.

Thus, under either span ordering, the sample-size condition in
Theorem~\mainSpanInformedUpperBoundTheorem{} implies the requirement of the relevant
discounted corollary. We conclude that with probability at least
$1-O(\delta)$,
\[
    V_{\gamma}^{\star,\sigma}
    -
    V_{\gamma}^{\widehat\pi,\sigma}
    \le
    H_{\sigma}\cdot \bm{1}_S.
\]

On the same probability event, invoke
Corollary~\ref{cor:current-reduction-parameters}. Since the discounted
suboptimality level is $H_{\sigma}$, the corollary gives
\[
    \rho^{\star,\sigma}-\rho^{\widehat\pi,\sigma}
    \le
    \varepsilon.
\]

\subsection{Proof of Theorem~\ref{thm:TV-upper-bound-discount}}
\label{subsec:Proof-Upper-discounted}

We prove Theorem~\ref{thm:TV-upper-bound-discount} in this section. We first
collect the notation and introduce the localization idea used throughout the
proof.

\subsubsection{Preliminaries and notation}
\label{subsec:discounted-preliminaries}

We begin with the standard value, kernel, and variance notation. The
proof-specific localization radius and concentration budgets are introduced
afterward.

\paragraph*{Value functions and transition kernels.}
While some of the following objects are introduced elsewhere, we restate them
here for convenience. For a policy $\pi$, $V_{\gamma}^{\pi,\sigma}$
denotes its robust discounted value under the true uncertainty set
$\cU(P^0)$, and $\widehat V_{\gamma}^{\pi,\sigma}$ denotes the corresponding
robust discounted value under the empirical uncertainty set
$\cU(\widehat P^0)$. We write
\[
    V_{\gamma}^{\star,\sigma}
    =
    \sup_{\pi}V_{\gamma}^{\pi,\sigma},
    \qquad
    \widehat V_{\gamma}^{\star,\sigma}
    =
    \sup_{\pi}\widehat V_{\gamma}^{\pi,\sigma}.
\]
Let $\pi_\gamma^\star$ be a deterministic optimal policy for the true robust
discounted MDP\@. For the fixed-policy and learned-policy comparisons used below,
define
\[
    U\coloneqq V_\gamma^{\star,\sigma},
    \qquad
    \widehat U\coloneqq\widehat V_\gamma^{\pi_\gamma^\star,\sigma},
    \qquad
    \widehat V^\star\coloneqq\widehat V_\gamma^{\star,\sigma},
    \qquad
    W\coloneqq V_\gamma^{\widehat\pi,\sigma},
    \qquad
    \widehat W\coloneqq\widehat V_\gamma^{\widehat\pi,\sigma}.
\]
Thus, $U$ and $\widehat U$ are compared under the fixed true-optimal policy,
whereas $W$ and $\widehat W$ are compared under the learned policy. The value
$\widehat V^\star$ is the optimal value of the empirical robust MDP and links
these two comparisons through the $\varepsilon_{\mathrm{opt}}$ condition in \mainRobustDMDPSolverTolerance{}.
Recall that the empirical nominal kernel is denoted by $\widehat P^0$ and is
constructed from $N$ samples per state-action pair.
For a policy $\pi$, define the policy-induced reward vector
\[
    r^\pi(s)
    \coloneqq
    \sum_{a\in\cA}\pi(a\mid s)r(s,a).
\]
For vectors $x,y$, $x\circ y$ denotes entrywise multiplication, and $f(x)$ is applied entrywise for any scalar function $f$, e.g., $|x|$.

For a transition matrix $P$, $\Var_P(V)$ denotes the variance
vector
\[
    \Var_P(V)(s)
    \coloneqq
    P(V\circ V)(s)-(PV)\circ(PV)(s).
\]

For a value vector $V$, let $P_{s,a}^{V}$ and $\widehat P_{s,a}^{V}$ denote
worst-case transition distributions selected from the true and empirical
uncertainty sets, respectively:
\[
    P_{s,a}^{V}
    \in
    \arg\min_{Q\in\cU_{s,a}(P^0)}QV,
    \qquad
    \widehat P_{s,a}^{V}
    \in
    \arg\min_{Q\in\cU_{s,a}(\widehat P^0)}QV.
\]
For a policy $\pi$, we denote the induced robust transition matrices by
\[
    P^{\pi,V}(s,s')
    \coloneqq
    \sum_{a\in\cA}\pi(a\mid s)P_{s,a}^{V}(s'),
    \qquad
    \widehat P^{\pi,V}(s,s')
    \coloneqq
    \sum_{a\in\cA}\pi(a\mid s)\widehat P_{s,a}^{V}(s').
\]
Similarly, the nominal policy-induced transition matrices are
\[
    P^{0,\pi}(s,s')
    \coloneqq
    \sum_{a\in\cA}\pi(a\mid s)P^0_{s,a}(s'),
    \qquad
    \widehat P^{0,\pi}(s,s')
    \coloneqq
    \sum_{a\in\cA}\pi(a\mid s)\widehat P^0_{s,a}(s').
\]

\paragraph*{Localization radius and radius grid.}

The theorem aims for a bound in terms of the reference radius $R_0$, which
controls the true optimal span. The spans of the empirical value functions,
however, are random. We therefore treat $R$ as a candidate upper bound on all
relevant spans. At any fixed radius $R$, the
localized analysis assumes that all relevant spans are at most $R$ and
controls the corresponding estimation errors at a scale depending on $R$.

We do not know in advance which candidate radius is large enough. We therefore
consider the dyadic grid
\begin{equation}
    \label{eq:radius-grid}
    R_j\coloneqq 2^jR_0,
    \qquad
    j=0,\ldots,J,
    \qquad
    J\coloneqq\left\lceil\log_2\frac{8}{1-\gamma}\right\rceil,
    \qquad
    \mathcal R\coloneqq\{R_j:0\le j\le J\},
\end{equation}
where $R_0$ is the reference radius defined in
\eqref{eq:reference-radius}. The grid starts at $R_0$ and extends beyond
the worst-case discounted value-function span bound of order
$(1-\gamma)^{-1}$. At each grid point $R_j$, this localized
analysis provides error bounds whenever the relevant spans are at most
$R_j$. These error bounds then control the relevant spans in
turn. We take the first grid point
at which this cycle of bounds closes and yields valid error control.
Because neighboring grid points differ only by a factor of two, the selected
radius remains close to the smallest valid $R$.

\paragraph*{Localized concentration budgets.}

The quantities entering the localized bounds depend on this radius.
We lay out the definitions here for convenience in later presentations.

In addition to the true anchor defect $\beta_\star$ from
\eqref{eq:true-anchor-defect}, define the corresponding defect for the
empirical policy by
\begin{equation}
    \label{eq:empirical-anchor-defect}
    \beta_{\mathrm{emp}}
    \coloneqq
    \left[
    \bar\rho-(1-\gamma)\min_s \widehat W(s)
    \right]_+.
\end{equation}
The variance bounds repeatedly involve four contributions: the anchor scale
$H_{\mathrm{anc}}$, the anchor defect multiplied by the span radius, the
robustness contribution $\sigma R^2$, and the discounted contribution
$(1-\gamma)R^2$. For a generic defect $\beta$, collect them in
\begin{equation}
    \label{eq:generic-localized-budget}
    B(R,\beta)
    \coloneqq
    H_{\mathrm{anc}}
    +R\beta
    +\sigma R^2
    +(1-\gamma)R^2.
\end{equation}
The budgets for the true optimal value, the empirical policy, and their
combination are the specializations
\begin{equation}
    \label{eq:localized-budget-specializations}
    B_\star(R)\coloneqq B(R,\beta_\star),
    \qquad
    B_{\mathrm{emp}}(R)\coloneqq B(R,\beta_{\mathrm{emp}}),
    \qquad
    B_{\mathrm{com}}(R)
    \coloneqq
    B(R,\beta_\star+\beta_{\mathrm{emp}}).
\end{equation}

\paragraph*{High-probability event.}

For any deterministic radius $R\ge\max\{1,H_{\mathrm{anc}}\}$, let
$\mathcal E(R)$ be the intersection of the events in
Lemmas~\ref{lem:one-step-robust-concentration},
\ref{lem:nominal-fixed-vector-concentration},
\ref{lem:empirical-std-perturbation}, and~\ref{lem:shi-original-lemma-11},
where Lemma~\ref{lem:empirical-std-perturbation} is applied to both $U$ and
$\bar h$, and Lemma~\ref{lem:shi-original-lemma-11} is instantiated at radius
$R$. 
For the grid $\mathcal R$ in \eqref{eq:radius-grid}, we set the confidence parameter of each event to be $\delta/(J+1)$.

Since only a constant number of
concentration statements is used at each radius,
\[
    \mathbb{P}\!\left\{\mathcal E(R)^c\right\}
    \le
    \frac{C\delta}{J+1},
    \qquad R\in\mathcal R.
\]
Define the global event
\begin{equation}
    \label{eq:global-localized-event}
    \cE
    \coloneqq
    \bigcap_{R\in\mathcal R}\mathcal E(R).
\end{equation}
A union bound over the $J+1$ grid points gives
$\mathbb{P}\{\cE\}\ge1-O(\delta)$. Moreover, the definition of $J$ in
\eqref{eq:radius-grid} gives
\[
    \log\!\left(
        \frac{54SAN^2(J+1)}{(1-\gamma)\delta}
    \right)
    \le
    C\cdot\iota,
\]
so the additional factor $J+1$ caused by this confidence allocation is
absorbed into $\iota$ after adjusting the universal constants.

\subsubsection{Proof of the main error bound}

Recall that
\[
    U\coloneqq V_\gamma^{\star,\sigma},
    \qquad
    \widehat U\coloneqq\widehat V_\gamma^{\pi_\gamma^\star,\sigma},
    \qquad
    \widehat V^\star\coloneqq\widehat V_\gamma^{\star,\sigma},
    \qquad
    W\coloneqq V_\gamma^{\widehat\pi,\sigma},
    \qquad
    \widehat W\coloneqq\widehat V_\gamma^{\widehat\pi,\sigma}.
\]

We start the proof by decomposing the gap between $W$ and $U$.
Observe that for any $s\in\cS$,
\begin{align}
    U(s)-W(s)
    & =U(s)-\widehat U(s)
    +\widehat U(s)-\widehat V^\star(s)
    +\widehat V^\star(s)-\widehat W(s)
    +\widehat W(s)-W(s)\nonumber \\
    & \le
    \|\widehat U-U\|_{\infty}
    +\varepsilon_{\mathrm{opt}}
    +\|\widehat W-W\|_{\infty}.
    \label{eq:V-pi-approx-V-star-decomp}
\end{align}
Here, the last line uses the solver guarantee
\mainRobustDMDPSolverTolerance{},
and the fact that $\widehat U-\widehat V^\star\le0$
since $\widehat V^\star$ is the optimal value function
under the estimated kernel.

Define the actual comparison error
\begin{equation}
    \Delta
    \coloneqq
    \|\widehat W-W\|_{\infty}
    +
    \|\widehat U-U\|_{\infty}.
    \label{eq:Delta-defn}
\end{equation}
The decomposition \eqref{eq:V-pi-approx-V-star-decomp} and \eqref{eq:Delta-defn}
allow us to characterize the span of the relevant value functions as in the
following lemma. The proof is deferred to
Appendix~\ref{subsec:Proof-span-value-approx}.

\begin{lemma}\label{lem:span-value-approx}
    Under \eqref{eq:reference-radius} and the solver guarantee
    \mainRobustDMDPSolverTolerance{}, let $\Delta$ be defined by
    \eqref{eq:Delta-defn}. Then
    \[
        \|W\|_{\mathrm{span}},
        \quad
        \|\widehat W\|_{\mathrm{span}},
        \quad
        \|\widehat U\|_{\mathrm{span}},
        \quad
        \|\widehat V^\star\|_{\mathrm{span}}
        \le
        R_0+2\Delta+2\varepsilon_{\mathrm{opt}}.
    \]
    Moreover,
    \begin{equation}
        \label{eq:span-value-anchor-defect-bound}
        \beta_{\mathrm{emp}}
        \le
        \beta_\star
        +(1-\gamma)\Delta
        +(1-\gamma)\varepsilon_{\mathrm{opt}}.
    \end{equation}
    Consequently, for every $R\ge R_0$,
    \begin{equation}
        \label{eq:span-value-budget-comparison}
        B_{\mathrm{emp}}(R)-B_\star(R)
        \le
        R(1-\gamma)
        (\Delta+\varepsilon_{\mathrm{opt}}).
    \end{equation}
\end{lemma}

We proceed to bound the two components of $\Delta$. We package recursive
bounds for these components in two lemmas. 
Recall from \eqref{eq:localized-budget-specializations} the localized
concentration budgets
\[
    B_\star(R)=B(R,\beta_\star),
    \qquad
    B_{\mathrm{emp}}(R)=B(R,\beta_{\mathrm{emp}}).
\]
Both lemmas are stated under the shared localization condition
\begin{equation}
    \begin{aligned}
    \max\big\{&
    \|U\|_{\mathrm{span}},
    \|\widehat U\|_{\mathrm{span}},
    \|W\|_{\mathrm{span}},
    \|\widehat W\|_{\mathrm{span}},
    \|\widehat V^\star\|_{\mathrm{span}}
    \big\}
    \le R.
    \end{aligned}
    \label{eq:discounted-localized-span-condition}
\end{equation}
Moreover, they require the following two sample-size conditions:
\begin{equation}
    \label{eq:recursive-comparison-sample-conditions}
    N
    \ge
    C\frac{\iota}{1-\gamma},
    \qquad
    N
    \ge
    C\frac{\sigma\cdot\iota}{(1-\gamma)^2}.
\end{equation}
These are the two radius-independent sample-size requirements in
\eqref{eq:anchored-dmdp-sample-condition}.

We now state the two lemmas. 
We note that the second part of Lemma~\ref{lem:V-hat-V-part-1-decomp} is included for later use in the span-agnostic analysis. The proof of
Lemma~\ref{lem:V-hat-V-part-1-decomp} is deferred to
Appendix~\ref{subsec:Proof-lemma-V-hat-V-part-1-decomp}, and the proof of
Lemma~\ref{lem:V-hat-V-pi-hat-bound} is deferred to
Appendix~\ref{subsec:Proof-lemma-V-hat-V-pi-hat-bound}.
\begin{lemma}\label{lem:V-hat-V-part-1-decomp}
    Let $R\ge\max\{1,H_{\mathrm{anc}}\}$ be a deterministic radius and suppose
    \eqref{eq:discounted-localized-span-condition} and
    \eqref{eq:recursive-comparison-sample-conditions} hold. On the
    high-probability event $\mathcal E(R)$, for a universal constant $C$,
    \[
        \|\widehat U-U\|_{\infty}
        \le
        C\sqrt{
            \frac{B_\star(R)\cdot\iota}
            {N(1-\gamma)^2}}
        +
        C\frac{R\cdot\iota}{N(1-\gamma)}
        +
        \frac{1}{50}\Delta.
    \]
    More generally, let $\pi$ be any deterministic policy fixed independently
    of the empirical transition samples. Under
    \eqref{eq:recursive-comparison-sample-conditions}, if
    \[
        R
        \ge
        \max\left\{
        1,
        H_{\mathrm{anc}},
        \|V_\gamma^{\pi,\sigma}\|_{\mathrm{span}},
        \|\widehat V_\gamma^{\pi,\sigma}\|_{\mathrm{span}}
        \right\},
    \]
    then the same argument gives, with probability at least $1-O(\delta)$,
    \begin{align}
        &\left\|
        \widehat V_\gamma^{\pi,\sigma}
        -V_\gamma^{\pi,\sigma}
        \right\|_\infty
        \nonumber\\
        &\quad\le
        C\sqrt{
        \frac{
        \left(
        H_{\mathrm{anc}}
        +R\left[
        \bar\rho-(1-\gamma)
        \min_sV_\gamma^{\pi,\sigma}(s)
        \right]_+
        +\sigma R^2
        +(1-\gamma)R^2
        \right)\cdot\iota}
        {N(1-\gamma)^2}}
        \nonumber\\
        &\qquad
        +C\frac{R\cdot\iota}{N(1-\gamma)}
        +\frac{1}{50}
        \left\|
        \widehat V_\gamma^{\pi,\sigma}
        -V_\gamma^{\pi,\sigma}
        \right\|_\infty.
        \label{eq:generic-fixed-policy-localized-comparison}
    \end{align}
\end{lemma}
\begin{lemma}\label{lem:V-hat-V-pi-hat-bound}
    Let $R\ge\max\{1,H_{\mathrm{anc}}\}$ be deterministic and suppose
    \eqref{eq:discounted-localized-span-condition} and
    \eqref{eq:recursive-comparison-sample-conditions} hold. On the
    high-probability event $\mathcal E(R)$, for a universal constant $C$,
    \[
        \|\widehat W-W\|_{\infty}
        \le
        C\sqrt{
            \frac{B_{\mathrm{emp}}(R)\cdot\iota}{N(1-\gamma)^2}}
        +
        C\frac{R\cdot\iota}{N(1-\gamma)}
        +
        \frac{1}{20}\Delta
        +
        C\frac{\varepsilon_{\mathrm{opt}}}{1-\gamma}.
    \]
\end{lemma}

We now condition on the event $\cE$ from
\eqref{eq:global-localized-event}.
Because $\cE\subseteq\mathcal E(R)$ for every $R\in\mathcal R$, the localized
lemmas hold simultaneously at every grid radius.

Since $r(s,a)\in[0,1]$, for
every policy $\pi$,
\begin{equation}
    \label{eq:discounted-value-range}
    0
    \le
    V_\gamma^{\pi,\sigma}
    \le
    \frac{1}{1-\gamma}\cdot\bm{1}_S,
    \qquad
    0
    \le
    \widehat V_\gamma^{\pi,\sigma}
    \le
    \frac{1}{1-\gamma}\cdot\bm{1}_S.
\end{equation}
Hence $\Delta\le2/(1-\gamma)$. The assumption $\varepsilon_{\mathrm{opt}}\le c(1-\gamma)R_0$
therefore ensures that the grid contains a radius larger than
$R_0+2\Delta+2\varepsilon_{\mathrm{opt}}$. Let $j_\Delta$ be the smallest
index such that
\[
    R_{j_\Delta}
    \ge
    R_0+2\Delta+2\varepsilon_{\mathrm{opt}}.
\]
By minimality and the assumption
$\varepsilon_{\mathrm{opt}}\le c(1-\gamma)R_0\le cR_0$ in
Theorem~\ref{thm:TV-upper-bound-discount},
\begin{equation}
    R_{j_\Delta}
    \le
    2\left(R_0+2\Delta+2\varepsilon_{\mathrm{opt}}\right)
    \le
    C(R_0+\Delta).
    \label{eq:peeling-radius-upper-bound}
\end{equation}
Lemma~\ref{lem:span-value-approx}, together with the reference-radius bound
\eqref{eq:reference-radius}, shows that
\eqref{eq:discounted-localized-span-condition} holds at
$R=R_{j_\Delta}$.

At this radius, the budget comparison in
Lemma~\ref{lem:span-value-approx} gives
\begin{equation}
    \label{eq:empirical-budget-selected-radius}
    B_{\mathrm{emp}}(R_{j_\Delta})
    \le
    B_\star(R_{j_\Delta})
    +
    R_{j_\Delta}(1-\gamma)
    (\Delta+\varepsilon_{\mathrm{opt}}).
\end{equation}
Also, by \eqref{eq:peeling-radius-upper-bound},
\begin{equation}
    \label{eq:baseline-budget-selected-radius}
    B_\star(R_{j_\Delta})
    \le
    C\left[
    B_\star(R_0)
    +\beta_\star\Delta
    +\sigma\Delta^2
    +(1-\gamma)\Delta^2
    \right].
\end{equation}
Invoking 
Lemmas~\ref{lem:V-hat-V-part-1-decomp} and~\ref{lem:V-hat-V-pi-hat-bound} at $R = R_{j_\Delta}$ and substituting the two lemma bounds into the definition of $\Delta$ in
\eqref{eq:Delta-defn} gives
\begin{align*}
    \Delta
     & \le
    C\sqrt{
    \frac{B_\star(R_{j_\Delta})\cdot\iota}{N(1-\gamma)^2}}
    +C\sqrt{
    \frac{B_{\mathrm{emp}}(R_{j_\Delta})\cdot\iota}
    {N(1-\gamma)^2}}
    +C\frac{R_{j_\Delta}\cdot\iota}{N(1-\gamma)}
    \\
     & \quad
    +\frac{7}{100}\Delta
    +C\frac{\varepsilon_{\mathrm{opt}}}{1-\gamma}.
\end{align*}
Substituting \eqref{eq:empirical-budget-selected-radius} into the second
square-root term above and using $\sqrt{a+b}\le\sqrt a+\sqrt b$ produces the
additional term on the left below. By
\eqref{eq:peeling-radius-upper-bound} and
$\varepsilon_{\mathrm{opt}}\le c(1-\gamma)R_0\le cR_0$,
\begin{align*}
    C\sqrt{
        \frac{R_{j_\Delta}(\Delta+\varepsilon_{\mathrm{opt}})\cdot\iota}
        {N(1-\gamma)}}
    &\le
    C(R_0+\Delta)
    \sqrt{\frac{\iota}{N(1-\gamma)}} \\
    &\le
    C\sqrt{
        \frac{B_\star(R_0)\cdot\iota}{N(1-\gamma)^2}}
    +C\sqrt{\frac{\iota}{N(1-\gamma)}}\Delta,
\end{align*}
where the second inequality uses
$B_\star(R_0)\ge (1-\gamma)R_0^2$. Thus the first term is absorbed into the
leading $B_\star(R_0)$ square-root term, while the second contributes to the
coefficient of $\Delta$.
Combining this estimate with \eqref{eq:baseline-budget-selected-radius} and
\eqref{eq:peeling-radius-upper-bound} gives
\begin{equation}
\label{eq:recursive-comparison-before-young}
\begin{aligned}
    \Delta
     & \le
    C\sqrt{
    \frac{B_\star(R_0)\cdot\iota}{N(1-\gamma)^2}}
    +C\sqrt{
    \frac{\beta_\star\Delta\cdot\iota}
    {N(1-\gamma)^2}}
    \\
     & \quad
    +\left[
    \frac{7}{100}
    +C\sqrt{\frac{\sigma\cdot\iota}{N(1-\gamma)^2}}
    +C\sqrt{\frac{\iota}{N(1-\gamma)}}
    +C\frac{\iota}{N(1-\gamma)}
    \right]\Delta
    +C\frac{R_0\cdot\iota}{N(1-\gamma)}
    +C\frac{\varepsilon_{\mathrm{opt}}}{1-\gamma}.
\end{aligned}
\end{equation}
Young's inequality gives
\[
    C\sqrt{
    \frac{\beta_\star\Delta\cdot\iota}
    {N(1-\gamma)^2}}
    \le
    \frac{93}{400}\Delta
    +C\frac{\beta_\star\cdot\iota}{N(1-\gamma)^2}.
\]
The sample-size condition \eqref{eq:anchored-dmdp-sample-condition} implies
\[
    \frac{\beta_\star\cdot\iota}{N(1-\gamma)^2}
    \le
    C R_0.
\]
Since \eqref{eq:localized-budget-specializations} gives
$R_0\beta_\star\le B_\star(R_0)$, it follows that
\[
    \frac{\beta_\star\cdot\iota}{N(1-\gamma)^2}
    \le
    C\sqrt{
    \frac{B_\star(R_0)\cdot\iota}{N(1-\gamma)^2}},
\]
so the last term in the Young bound is absorbed into the first square-root
term. The two sample-size conditions in
\eqref{eq:recursive-comparison-sample-conditions} make the sum of the other
coefficients of $\Delta$ at most $93/400$. Combining this with the
$93\Delta/400$ term from Young's inequality gives
\begin{equation*}
    \Delta
    \le
    C\sqrt{
    \frac{B_\star(R_0)\cdot\iota}{N(1-\gamma)^2}}
    +\frac{93}{200}\Delta
    +C\frac{R_0\cdot\iota}{N(1-\gamma)}
    +C\frac{\varepsilon_{\mathrm{opt}}}{1-\gamma}.
\end{equation*}
Moving the $\Delta$ term on the right-hand side to the left-hand side and absorbing
the numerical factor into $C$, we obtain
\begin{equation*}
    \Delta
    \le
    C\sqrt{
    \frac{B_\star(R_0)\cdot\iota}{N(1-\gamma)^2}}
    +C\frac{R_0\cdot\iota}{N(1-\gamma)}
    +
    C\frac{\varepsilon_{\mathrm{opt}}}{1-\gamma}.
\end{equation*}
Thus, by the definition of $\Delta$ and \eqref{eq:V-pi-approx-V-star-decomp},
\[
    \|\widehat W-W\|_{\infty}
    \le
    C\sqrt{
    \frac{B_\star(R_0)\cdot\iota}{N(1-\gamma)^2}}
    +C\frac{R_0\cdot\iota}{N(1-\gamma)}
    +C\frac{\varepsilon_{\mathrm{opt}}}{1-\gamma},
\]
and
\[
    \|U-W\|_{\infty}
    \le
    C\sqrt{
    \frac{B_\star(R_0)\cdot\iota}{N(1-\gamma)^2}}
    +C\frac{R_0\cdot\iota}{N(1-\gamma)}
    +C\frac{\varepsilon_{\mathrm{opt}}}{1-\gamma}.
\]
Expanding $B_\star(R_0)$ using
\eqref{eq:localized-budget-specializations} shows that both bounds
in Theorem~\ref{thm:TV-upper-bound-discount} hold. Since
$\mathbb{P}\{\cE\}\ge1-O(\delta)$, this proves the theorem.

\subsection{Proof of auxiliary lemmas}
\label{sec:proof_auxiliary_lemmas}

Except for the nominal-branch proof in
Appendix~\ref{subsec:proof-nominal-branch-guarantee}, the auxiliary proofs
below inherit the notation from
Appendix~\ref{subsec:discounted-preliminaries}. In particular, recall that
\[
    U\coloneqq V_\gamma^{\star,\sigma},
    \qquad
    \widehat U\coloneqq\widehat V_\gamma^{\pi_\gamma^\star,\sigma},
    \qquad
    \widehat V^\star\coloneqq\widehat V_\gamma^{\star,\sigma},
    \qquad
    W\coloneqq V_\gamma^{\widehat\pi,\sigma},
    \qquad
    \widehat W\coloneqq\widehat V_\gamma^{\widehat\pi,\sigma}.
\]

\subsubsection{Proof of Lemma~\ref{lem:nominal-branch-guarantee}}
\label{subsec:proof-nominal-branch-guarantee}

The argument has three steps. We first obtain a discounted nominal guarantee,
then derive nominal average-reward and span bounds for the returned policy, and
finally transfer these bounds to the robust average reward.

\paragraph*{Step 1: Discounted nominal guarantee.}
We analyze the nominal branch of
Algorithm~\mainSpanInformedAlgorithm{}. This branch applies the discounted plug-in
reduction with span input $H_0$, the degenerate uncertainty rule
$\cU(P)=\{P\}$, and solver tolerance
$\varepsilon_{\mathrm{opt}}\le c_{\mathrm{opt}}\varepsilon$, where
$c_{\mathrm{opt}}>0$ is a sufficiently small universal constant. Under the
degenerate uncertainty rule, the empirical robust discounted problem is
exactly the empirical nominal discounted MDP with kernel $\widehat P^0$, and
Corollary~\ref{cor:anchored-robust-dmdp-accuracy} applies with $\sigma=0$.

We now check the conditions and relevant quantities in Corollary~\ref{cor:anchored-robust-dmdp-accuracy}. Set
\[
    \gamma
    =
    1-\frac{\varepsilon}{20H_0}.
\]
Since $\varepsilon\le1$ and $H_0\ge1$, this choice gives
$\gamma\ge19/20$.

We use the nominal optimal pair as the anchor:
\[
    \bar\rho=\rho^{\star,0},
    \qquad
    \bar h=h_{P^0}^{\star}.
\]
The nominal average-reward Bellman equation shows that this pair satisfies
the supersolution condition~\mainNominalAnchorSupersolution{}. It is
fixed independently of the transition samples, and
$H_{\mathrm{anc}}=H_0$. Moreover,
Lemma~\ref{lem:robust-value-function-span-bound}, applied to the degenerate
uncertainty rule, gives
\[
    \left\|V_{\gamma}^{\star,0}\right\|_{\mathrm{span}}
    \le
    2H_0.
\]
We may therefore take $R_0=2H_0$. The optimal discounted-to-average comparison \eqref{eq:robust-discounted-average-sandwich-optimal}
in Lemma~\ref{lem:robust-discounted-average-sandwich}, 
again applied to the
degenerate uncertainty rule, yields
\[
    \beta_\star
    =
    \left[
        \rho^{\star,0}
        -(1-\gamma)\min_s V_\gamma^{\star,0}(s)
    \right]_+
    \le
    2(1-\gamma)H_0.
\]
Consequently,
\begin{equation}
\label{eq:nominal-branch-anchor-budget}
    H_{\mathrm{anc}}
    +R_0\beta_\star
    +(1-\gamma)R_0^2
    \le
    H_0+8(1-\gamma)H_0^2
    \le
    C H_0.
\end{equation}

Apply Corollary~\ref{cor:anchored-robust-dmdp-accuracy} with discounted target
accuracy $\varepsilon_\gamma=H_0$. Since
$(1-\gamma)\varepsilon_\gamma=\varepsilon/20$, its solver-tolerance condition
holds as long as $c_{\mathrm{opt}}$ is sufficiently small. Using
\eqref{eq:nominal-branch-anchor-budget} and $R_0=2H_0$, the sample-size
requirement of this corollary reduces to
$NSA\ge CSA(H_0/\varepsilon^2+H_0/\varepsilon)\iota$. Since
$\varepsilon\le1$, the term $H_0/\varepsilon$ is absorbed by
$H_0/\varepsilon^2$. Moreover, the definition of $\iota$ and the choice of
$\gamma$ give
$\iota\le C\log(SAH_0N/(\varepsilon\delta))$. It is therefore sufficient that
\[
    NSA
    \ge
    C SA
    \frac{H_0}{\varepsilon^2}
    \log\!\left(\frac{SAH_0N}{\varepsilon\delta}\right).
\]
Thus, with probability at least $1-O(\delta)$,
\[
    \left\|
    V_{\gamma}^{\star,0}
    -
    V_{\gamma}^{\widehat\pi,0}
    \right\|_{\infty}
    \le
    H_0.
\]

\paragraph*{Step 2: Nominal average-reward and span bounds.}
Lemma~\ref{lem:reduction-framework}, applied to the degenerate uncertainty
rule, now gives
\begin{equation}
    \label{eq:nominal-policy-average-reward-accuracy}
    \rho^{\star,0}
    -
    \rho^{\widehat\pi,0}
    \le
    3(1-\gamma)H_0
    =
    \frac{3\varepsilon}{20}
    \le
    \frac{\varepsilon}{5}.
\end{equation}
The same discounted guarantee also controls the span of the returned policy:
\begin{equation}
    \label{eq:nominal-learned-policy-discounted-span}
    \left\|V_{\gamma}^{\widehat\pi,0}\right\|_{\mathrm{span}}
    \le
    \left\|V_{\gamma}^{\star,0}\right\|_{\mathrm{span}}
    +
    2
    \left\|
    V_{\gamma}^{\star,0}
    -
    V_{\gamma}^{\widehat\pi,0}
    \right\|_{\infty}
    \le
    4H_0.
\end{equation}

\paragraph*{Step 3: Transfer to the robust average reward.}
We now translate the nominal accuracy guarantee
\eqref{eq:nominal-policy-average-reward-accuracy} into the desired robust
accuracy bound
\[
    \rho^{\star,\sigma}
    -
    \rho^{\widehat\pi,\sigma}
    \le
    \varepsilon.
\]
The span-controlled nominal discounted value
$V_{\gamma}^{\widehat\pi,0}$ provides the link between the nominal and robust
average rewards.

Recall that the nominal discounted value function of $\widehat\pi$ satisfies
\[
    V_{\gamma}^{\widehat\pi,0}(s)
    =
    r^{\widehat\pi}(s)
    +
    \gamma P^{0,\widehat\pi}V_{\gamma}^{\widehat\pi,0}(s).
\]
By Lemma~\ref{lem:fixed-policy-robust-verification} applied to the nominal
kernel and $h=\gamma V_{\gamma}^{\widehat\pi,0}$, we have
\[
    \rho^{\widehat\pi,0}
    \le
    (1-\gamma)\min_s V_{\gamma}^{\widehat\pi,0}(s)
    +
    (1-\gamma)
    \left\|V_{\gamma}^{\widehat\pi,0}\right\|_{\mathrm{span}}.
\]
Consequently,
\begin{equation}
    \label{eq:nominal-average-reward-to-discounted-minimum}
    \rho^{\star,0}
    -
    (1-\gamma)\min_s V_{\gamma}^{\widehat\pi,0}(s)
    \le
    \frac{\varepsilon}{5}
    +
    4(1-\gamma)H_0.
\end{equation}

For any $P\in\cP$, the total-variation perturbation bound gives
\[
    \gamma P^{\widehat\pi}V_{\gamma}^{\widehat\pi,0}(s)
    \ge
    \gamma P^{0,\widehat\pi}V_{\gamma}^{\widehat\pi,0}(s)
    -
    \gamma\sigma
    \left\|V_{\gamma}^{\widehat\pi,0}\right\|_{\mathrm{span}}.
\]
Because the right-hand side does not depend on $P$, taking the infimum over
$P\in\cP$ and adding $r^{\widehat\pi}(s)$ gives
\[
\begin{aligned}
    r^{\widehat\pi}(s)
    +
    \inf_{P\in\cP}
    \gamma P^{\widehat\pi}V_{\gamma}^{\widehat\pi,0}(s)
    &\ge
    r^{\widehat\pi}(s)
    +
    \gamma P^{0,\widehat\pi}V_{\gamma}^{\widehat\pi,0}(s)
    -
    \gamma\sigma
    \left\|V_{\gamma}^{\widehat\pi,0}\right\|_{\mathrm{span}}
    \\
    &=
    V_{\gamma}^{\widehat\pi,0}(s)
    -
    \gamma\sigma
    \left\|V_{\gamma}^{\widehat\pi,0}\right\|_{\mathrm{span}},
\end{aligned}
\]
where the equality uses the nominal discounted Bellman equation.

Now set $h=\gamma V_{\gamma}^{\widehat\pi,0}$. Since
$(1-\gamma)\min_{x\in\cS}V_{\gamma}^{\widehat\pi,0}(x)
\le(1-\gamma)V_{\gamma}^{\widehat\pi,0}(s)$, the preceding display implies
that, for every state $s$,
\[
\begin{aligned}
    &\left[
        (1-\gamma)\min_{x\in\cS}V_{\gamma}^{\widehat\pi,0}(x)
        -
        \gamma\sigma
        \left\|V_{\gamma}^{\widehat\pi,0}\right\|_{\mathrm{span}}
    \right]
    +h(s)
    \\
    &\quad\le
    V_{\gamma}^{\widehat\pi,0}(s)
    -
    \gamma\sigma
    \left\|V_{\gamma}^{\widehat\pi,0}\right\|_{\mathrm{span}}
\le
    r^{\widehat\pi}(s)
    +
    \inf_{P\in\cP}(P_{\widehat\pi}h)(s).
\end{aligned}
\]
The lower-bound direction of
Lemma~\ref{lem:fixed-policy-robust-verification}, applied with this $h$ and
\[
    \rho
    =
    (1-\gamma)\min_{x\in\cS}V_{\gamma}^{\widehat\pi,0}(x)
    -
    \gamma\sigma
    \left\|V_{\gamma}^{\widehat\pi,0}\right\|_{\mathrm{span}},
\]
therefore gives
\[
    \rho^{\widehat\pi,\sigma}
    \ge
    (1-\gamma)\min_s V_{\gamma}^{\widehat\pi,0}(s)
    -
    \gamma\sigma
    \left\|V_{\gamma}^{\widehat\pi,0}\right\|_{\mathrm{span}}.
\]

Since $P^0\in\cP$, we have
$\rho^{\star,\sigma}\le\rho^{\star,0}$. Combining this with the preceding
lower bound on $\rho^{\widehat\pi,\sigma}$ gives
\[
\begin{aligned}
    \rho^{\star,\sigma}-\rho^{\widehat\pi,\sigma}
    &\le
    \rho^{\star,0}-\rho^{\widehat\pi,\sigma}
    \\
    &\le
    \rho^{\star,0}
    -
    (1-\gamma)\min_{s\in\cS}V_{\gamma}^{\widehat\pi,0}(s)
    +
    \gamma\sigma
    \left\|V_{\gamma}^{\widehat\pi,0}\right\|_{\mathrm{span}}
    \\
    &\le
    \frac{\varepsilon}{5}
    +
    4(1-\gamma)H_0
    +
    4\gamma\sigma H_0.
\end{aligned}
\]
Here, the last inequality uses
\eqref{eq:nominal-average-reward-to-discounted-minimum} and
\eqref{eq:nominal-learned-policy-discounted-span}.
Using $1-\gamma=\varepsilon/(20H_0)$, $\gamma\le1$, and
$7\sigma H_0\le\varepsilon$, we obtain
\[
    \rho^{\star,\sigma}
    -
    \rho^{\widehat\pi,\sigma}
    \le
    \left(
    \frac{1}{5}
    +\frac{1}{5}
    +\frac{4}{7}
    \right)\varepsilon
    =
    \frac{34}{35}\varepsilon
    \le
    \varepsilon,
\]
as desired.

\subsubsection{Proof of Lemma~\ref{lem:span-value-approx} }
\label{subsec:Proof-span-value-approx}
The proof has two steps. We first compare each relevant value function with
$U$ and derive the span bounds. We then compare the two anchor defects and
the corresponding concentration budgets.

\paragraph*{Step 1: Value comparisons and span bounds.}
Recall $\Delta$ from \eqref{eq:Delta-defn}. For any value functions $V_{1}$
and $V_{2}$,
\begin{equation}
    \label{eq:span-triangle-inequality}
    \|V_{1}\|_{\mathrm{span}}
    \le
    \|V_{2}\|_{\mathrm{span}}
    +2\|V_{1}-V_{2}\|_{\infty}.
\end{equation}
We will apply \eqref{eq:span-triangle-inequality} with $V_2=U$ after
establishing the required sup-norm comparisons.

By optimality in the true and empirical robust MDPs,
\[
    W\le U,
    \qquad
    \widehat U\le\widehat V^\star,
    \qquad
    \widehat W\le\widehat V^\star,
\]
where all inequalities are componentwise. Moreover,
optimality and the solver guarantee
\mainRobustDMDPSolverTolerance{} give
\[
    0
    \le
    \widehat V^\star-\widehat W
    \le
    \varepsilon_{\mathrm{opt}}\bm{1}_S.
\]
We first compare $\widehat V^\star$ with $U$. For every $s\in\cS$,
\begin{align*}
    \widehat V^\star(s)-U(s)
    & =
    \widehat V^\star(s)-\widehat W(s)
    +\widehat W(s)-W(s)
    +W(s)-U(s) \\
    & \le
    \varepsilon_{\mathrm{opt}}
    +\|\widehat W-W\|_\infty,
    \\
    U(s)-\widehat V^\star(s)
    & =
    U(s)-\widehat U(s)
    +\widehat U(s)-\widehat V^\star(s) \\
    & \le
    \|\widehat U-U\|_\infty.
\end{align*}
Hence, by \eqref{eq:Delta-defn},
\begin{equation}
    \label{eq:empirical-optimal-true-optimal-comparison}
    \|\widehat V^\star-U\|_\infty
    \le
    \Delta+\varepsilon_{\mathrm{opt}}.
\end{equation}
Similarly, for every $s\in\cS$,
\begin{align*}
    \widehat W(s)-U(s)
    & =
    \widehat W(s)-W(s)+W(s)-U(s)
    \le
    \|\widehat W-W\|_\infty,
    \\
    U(s)-\widehat W(s)
    & =
    U(s)-\widehat U(s)
    +\widehat U(s)-\widehat V^\star(s)
    +\widehat V^\star(s)-\widehat W(s) \\
    & \le
    \|\widehat U-U\|_\infty
    +\varepsilon_{\mathrm{opt}}.
\end{align*}
Therefore,
\begin{equation}
    \label{eq:empirical-policy-true-optimal-comparison}
    \|\widehat W-U\|_\infty
    \le
    \Delta+\varepsilon_{\mathrm{opt}}.
\end{equation}
Since $W\le U$, the decomposition
\eqref{eq:V-pi-approx-V-star-decomp} also gives
\begin{equation}
    \label{eq:true-policy-true-optimal-comparison}
    \|W-U\|_\infty
    \le
    \Delta+\varepsilon_{\mathrm{opt}}.
\end{equation}
Finally, \eqref{eq:Delta-defn} directly gives
\begin{equation}
    \label{eq:fixed-policy-true-optimal-comparison}
    \|\widehat U-U\|_\infty
    \le
    \Delta.
\end{equation}
Applying \eqref{eq:span-triangle-inequality} to
\eqref{eq:empirical-optimal-true-optimal-comparison}--\eqref{eq:fixed-policy-true-optimal-comparison} and using
\eqref{eq:reference-radius}, we conclude that
\[
    \|W\|_{\mathrm{span}},
    \quad
    \|\widehat W\|_{\mathrm{span}},
    \quad
    \|\widehat U\|_{\mathrm{span}},
    \quad
    \|\widehat V^\star\|_{\mathrm{span}}
    \le
    R_0+2\Delta+2\varepsilon_{\mathrm{opt}}.
\]

\paragraph*{Step 2: Anchor defects and concentration budgets.}
Empirical optimality and the solver guarantee
\mainRobustDMDPSolverTolerance{} imply
\[
    \widehat W
    \ge
    \widehat V^\star-\varepsilon_{\mathrm{opt}}\bm{1}_S
    \ge
    \widehat U-\varepsilon_{\mathrm{opt}}\bm{1}_S.
\]
Consequently,
\[
    \min_s\widehat W(s)
    \ge
    \min_sU(s)
    -
    \left\|
    \widehat U
    -
    U
    \right\|_\infty
    -\varepsilon_{\mathrm{opt}}.
\]
Substituting this bound into \eqref{eq:empirical-anchor-defect} and using
$[x+y]_+\le[x]_++y$ for $y\ge0$, we obtain
\begin{align*}
    \beta_{\mathrm{emp}}
    & =
    \left[
    \bar\rho-(1-\gamma)\min_s\widehat W(s)
    \right]_+ \\
    & \le
    \left[
    \bar\rho-(1-\gamma)\min_sU(s)
    +(1-\gamma)\|\widehat U-U\|_\infty
    +(1-\gamma)\varepsilon_{\mathrm{opt}}
    \right]_+ \\
    & \le
    \beta_\star
    +(1-\gamma)\|\widehat U-U\|_\infty
    +(1-\gamma)\varepsilon_{\mathrm{opt}} \\
    & \le
    \beta_\star
    +(1-\gamma)\Delta
    +(1-\gamma)\varepsilon_{\mathrm{opt}},
\end{align*}
which proves \eqref{eq:span-value-anchor-defect-bound}. Finally, by
\eqref{eq:localized-budget-specializations}, for every
$R\ge R_0$,
\[
    B_{\mathrm{emp}}(R)-B_\star(R)
    =R(\beta_{\mathrm{emp}}-\beta_\star)
    \le
    R(1-\gamma)(\Delta+\varepsilon_{\mathrm{opt}}),
\]
which proves \eqref{eq:span-value-budget-comparison}.

\subsubsection{Proof of Lemma~\ref{lem:V-hat-V-part-1-decomp} }
\label{subsec:Proof-lemma-V-hat-V-part-1-decomp}

The proof has three steps. We first decompose the error $\|\widehat U-U\|_\infty$
into two branches and then bound the terms in both branches. Finally, we
combine these bounds to obtain the claimed estimate.

\paragraph*{Step 1: Error decomposition.}
In this step, we derive an error decomposition that upper-bounds
$\|\widehat U-U\|_\infty$ by the maximum of two sums of more manageable
terms.

The following two lemmas control the fixed-vector residuals. These lemmas are used repeatedly for both the decomposition and the variance analysis. The
first treats the empirical robust kernel and is related to Lemma~8 in
\citet{shi2023curious}; our version explicitly incorporates the span to obtain
a tighter bound under a weaker sample-size condition. The second treats the
nominal kernel applied to the fixed anchor vector. For any stochastic matrix
$L$, define the normalized discounted resolvent
\begin{equation}
    \label{eq:normalized-discounted-resolvent}
    \mathcal G_L
    \coloneqq
    (1-\gamma)(I-\gamma L)^{-1}.
\end{equation}
Their proofs are deferred to
Appendices~\ref{subsec:Proof-lemma-one-step-robust-concentration} and~\ref{subsec:proof-nominal-fixed-vector-concentration}, respectively.

\begin{lemma}
    \label{lem:one-step-robust-concentration}
    Fix a value vector $V$ independently of the empirical transition kernels
    $\{\widehat P^0_{s,a}\}_{(s,a)\in\cS\times\cA}$. With probability at least
    $1-O(\delta)$, for every policy~$\pi$,
    \begin{equation}
        \label{eq:fixed-value-robust-kernel-pointwise}
        \left|
        \left(
        \widehat{P}^{\pi,V}
        -
        P^{\pi,V}
        \right)V
        \right|
        \le
        \sqrt{
            \frac{2\Var_{P^{0,\pi}}(V)\cdot\iota}{N}
        }
        +
        \frac{\|V\|_{\mathrm{span}}\cdot\iota}{N}\cdot \bm{1}_S.
    \end{equation}
    Moreover, for any possibly data-dependent stochastic matrix $L$,
    \begin{equation}
        \label{eq:fixed-value-robust-kernel-resolvent}
        \left\|
        \mathcal G_L
        \left|
        \left(\widehat P^{\pi,V}-P^{\pi,V}\right)V
        \right|
        \right\|_\infty
        \le
        C\sqrt{\frac{\iota}{N}}
        \left\|
        \mathcal G_L
        \sqrt{\Var_{P^{0,\pi}}(V)}
        \right\|_\infty
        +
        C\frac{\|V\|_{\mathrm{span}}\cdot\iota}{N}.
    \end{equation}
\end{lemma}

\begin{lemma}[Fixed-vector nominal-kernel concentration]
    \label{lem:nominal-fixed-vector-concentration}
    Fix a vector $g$ that is independent of the empirical transition kernels. With
    probability at least $1-O(\delta)$, for every policy~$\pi$,
    \begin{equation}
        \label{eq:fixed-vector-nominal-pointwise}
        \left|
        \left(
        \widehat P^{0,\pi}
        -
        P^{0,\pi}
        \right)g
        \right|
        \le
        C\sqrt{
            \frac{\Var_{P^{0,\pi}}(g)\cdot\iota}{N}
        }
        +
        C\frac{\|g\|_{\mathrm{span}}\cdot\iota}{N}\cdot \bm{1}_S.
    \end{equation}
    Moreover, for any possibly data-dependent stochastic matrix $L$,
    \begin{equation}
        \label{eq:fixed-vector-nominal-resolvent}
        \left\|
        \mathcal G_L
        \left|
        \left(\widehat P^{0,\pi}-P^{0,\pi}\right)g
        \right|
        \right\|_\infty
        \le
        C\sqrt{\frac{\iota}{N}}
        \left\|
        \mathcal G_L
        \sqrt{\Var_{P^{0,\pi}}(g)}
        \right\|_\infty
        +
        C\frac{\|g\|_{\mathrm{span}}\cdot\iota}{N}.
    \end{equation}
\end{lemma}

We now describe the decomposition of the fixed-policy error $\widehat U-U$.
The following lemma bounds this error by the maximum of two
resolvent-perturbation terms. The proof is deferred to
Appendix~\ref{subsec:Proof-lemma-two-sided-fixed-value-decomposition}.
\begin{lemma}
    \label{lem:two-sided-fixed-value-decomposition}
    Consider $\widehat{U}$ and $U$. We have that
    \begin{align}
        \label{eq:two-sided-fixed-value-decomposition}
        \|\widehat U-U\|_{\infty}
         & \le
        \max\Bigg\{
        \left\|
        \left(I-\gamma\widehat P^{\pi_\gamma^\star,U}\right)^{-1}
        \left|
        \left(
        \widehat P^{\pi_\gamma^\star,U}
        -
        P^{\pi_\gamma^\star,U}
        \right)
        U
        \right|
        \right\|_\infty,
        \\
         & \qquad\qquad
        \left\|
        \left(I-\gamma\widehat P^{\pi_\gamma^\star,\widehat U}\right)^{-1}
        \left|
        \left(
        \widehat P^{\pi_\gamma^\star,U}
        -
        P^{\pi_\gamma^\star,U}
        \right)
        U
        \right|
        \right\|_\infty
        \Bigg\}.\nonumber
    \end{align}
\end{lemma}

For simplicity, we first focus on the first term of
\eqref{eq:two-sided-fixed-value-decomposition}.
Invoking Lemma~\ref{lem:one-step-robust-concentration} with
$\pi=\pi_\gamma^\star$ and $V=U$ and using the
nonnegativity of the resolvent, we obtain
\begin{align}\label{eq:two-sided-fixed-value-decomposition-first-branch}
     & \left\|
    \left(I-\gamma\widehat P^{\pi_\gamma^\star,U}\right)^{-1}
    \left|
    \left(
    \widehat P^{\pi_\gamma^\star,U}
    -
    P^{\pi_\gamma^\star,U}
    \right)
    U
    \right|
    \right\|_\infty
    \\
     & \quad\le
    \left\|
    \left(I-\gamma\widehat P^{\pi_\gamma^\star,U}\right)^{-1}
    \frac{\|U\|_{\mathrm{span}}\cdot\iota}{N}\cdot \bm{1}_S
    +
    \sqrt{\frac{2\cdot\iota}{N}}\,
    \left(I-\gamma\widehat P^{\pi_\gamma^\star,U}\right)^{-1}
    \sqrt{\Var_{P^{0,\pi_\gamma^\star}}(U)}
    \right\|_\infty.
\end{align}
We further decompose the variance term. By elementary algebra, we have that 
\begin{align*}
    \sqrt{\Var_{P^{0,\pi_\gamma^\star}}(U)}
    \le{}&
    \sqrt{\Var_{\widehat P^{0,\pi_\gamma^\star}}(U)}
    +
    \left|
    \sqrt{\Var_{P^{0,\pi_\gamma^\star}}(U)}
    -
    \sqrt{\Var_{\widehat P^{0,\pi_\gamma^\star}}(U)}
    \right|,
\end{align*}
and
\begin{align*}
    \sqrt{\Var_{\widehat P^{0,\pi_\gamma^\star}}(U)}
    \le{}&
    \sqrt{\Var_{\widehat P^{\pi_\gamma^\star,U}}(U)}
    +
    \sqrt{
        \left|
        \Var_{\widehat P^{0,\pi_\gamma^\star}}(U)
        -
        \Var_{\widehat P^{\pi_\gamma^\star,U}}(U)
        \right|
    }.
\end{align*}
Substituting these into \eqref{eq:two-sided-fixed-value-decomposition-first-branch} gives
\[
    \left\|
    \left(I-\gamma\widehat P^{\pi_\gamma^\star,U}\right)^{-1}
    \left|
    \left(
    \widehat P^{\pi_\gamma^\star,U}
    -
    P^{\pi_\gamma^\star,U}
    \right)
    U
    \right|
    \right\|_\infty
    \le
    \left\|\sum_{i=1}^{4}T_i\right\|_\infty.
\] 
where $T_1,\ldots,T_4$ are defined by
\begin{equation}
    \label{eq:T1-T4-definitions}
    \begin{aligned}
        T_1 & \coloneqq \,
        \left(I-\gamma\widehat P^{\pi_\gamma^\star,U}\right)^{-1} \frac{\|U\|_{\mathrm{span}}\cdot\iota}{N}\cdot \bm{1}_S
        \\
        T_2
            & \coloneqq
        \sqrt{\frac{2\cdot\iota}{N}}\,
        \left(I-\gamma\widehat P^{\pi_\gamma^\star,U}\right)^{-1}
        \sqrt{\Var_{\widehat P^{\pi_\gamma^\star,U}}(U)},
        \\
        T_3
            & \coloneqq
        \sqrt{\frac{2\cdot\iota}{N}}\,
        \left(I-\gamma\widehat P^{\pi_\gamma^\star,U}\right)^{-1}
        \sqrt{
            \left|
            \Var_{\widehat P^{0,\pi_\gamma^\star}}(U)
            -
            \Var_{\widehat P^{\pi_\gamma^\star,U}}(U)
            \right|
        },
        \\
        T_4
            & \coloneqq
        \sqrt{\frac{2\cdot\iota}{N}}\,
        \left(I-\gamma\widehat P^{\pi_\gamma^\star,U}\right)^{-1}
        \left|
        \sqrt{\Var_{P^{0,\pi_\gamma^\star}}(U)}
        -
        \sqrt{\Var_{\widehat P^{0,\pi_\gamma^\star}}(U)}
        \right|.
    \end{aligned}
\end{equation}
Repeating the steps used for
\eqref{eq:two-sided-fixed-value-decomposition-first-branch} for the second term
in \eqref{eq:two-sided-fixed-value-decomposition}, we obtain
\[
    \left\|
    \left(I-\gamma\widehat P^{\pi_\gamma^\star,\widehat U}\right)^{-1}
    \left|
    \left(
    \widehat P^{\pi_\gamma^\star,U}
    -
    P^{\pi_\gamma^\star,U}
    \right)
    U
    \right|
    \right\|_\infty
    \le
    \left\|\sum_{i=5}^{8}T_i\right\|_\infty.
\]
where
$T_5,\ldots,T_8$ are defined by
\begin{equation}
    \label{eq:T5-T8-definitions}
    \begin{aligned}
        T_5
            & \coloneqq
        \,
        \left(I-\gamma\widehat P^{\pi_\gamma^\star,\widehat U}\right)^{-1} \frac{\|U\|_{\mathrm{span}}\cdot\iota}{N}\cdot \bm{1}_S
        \\
        T_6
            & \coloneqq
        \sqrt{\frac{2\cdot\iota}{N}}\,
        \left(I-\gamma\widehat P^{\pi_\gamma^\star,\widehat U}\right)^{-1}
        \sqrt{
            \Var_{\widehat P^{\pi_\gamma^\star,\widehat U}}
            (U)
        },
        \\
        T_7
            & \coloneqq
        \sqrt{\frac{2\cdot\iota}{N}}\,
        \left(I-\gamma\widehat P^{\pi_\gamma^\star,\widehat U}\right)^{-1}
        \sqrt{
            \left|
            \Var_{\widehat P^{0,\pi_\gamma^\star}}(U)
            -
            \Var_{\widehat P^{\pi_\gamma^\star,\widehat U}}
            (U)
            \right|
        },
        \\
        T_8
            & \coloneqq
        \sqrt{\frac{2\cdot\iota}{N}}\,
        \left(I-\gamma\widehat P^{\pi_\gamma^\star,\widehat U}\right)^{-1}
        \left|
        \sqrt{\Var_{P^{0,\pi_\gamma^\star}}(U)}
        -
        \sqrt{\Var_{\widehat P^{0,\pi_\gamma^\star}}(U)}
        \right|.
    \end{aligned}
\end{equation}
Putting the two parts together, we have that 
\begin{equation}\label{eq:V-hat-V-pi-star-C-decomp}
    \big\|\widehat U
    -U\big\|_{\infty}
    \le \max\left\{
    \left\|\sum_{i=1}^4 T_i\right\|_\infty,
    \left\|\sum_{i=5}^8 T_i\right\|_\infty
    \right\}.
\end{equation}
\begin{remark}
    This decomposition differs from \citet{shi2023curious} by merging their $\cC_4$ and $\cC_5$ terms into $T_6$. This avoids the linear self-perturbation term in their argument, which would impose a stronger sample-complexity condition than needed for our span-aware reduction. We instead control the variance term directly to obtain the required bound under the weaker condition. 
\end{remark}
It remains to bound the terms in these two branches.

\paragraph*{Step 2: Bounding the error terms $T_1,\ldots,T_8$.}
\paragraph*{Bound on $T_1$ and $T_5$.}
For any stochastic matrix $P$, expanding the matrix inverse yields
\[
    \left(I-\gamma P\right)^{-1} \bm{1}_S
    \le
    \sum_{t=0}^\infty \gamma^t P^t \bm{1}_S.
\]
Since $P$ is a stochastic matrix, $P^t \bm{1}_S = \bm{1}_S$ for all $t\ge 0$. Therefore,
\begin{equation}
    \left(I-\gamma P\right)^{-1} \bm{1}_S
    \le
    \frac{1}{1-\gamma}\cdot \bm{1}_S.
    \label{eq:resolvent-bound}
\end{equation}
Substituting this into the definition of $T_1$ in
\eqref{eq:T1-T4-definitions} and using
$\|U\|_{\mathrm{span}}\le R$, we have that
\begin{equation}
    T_1 \le
    \frac{R\cdot\iota}{N(1-\gamma)}\cdot \bm{1}_S.
    \label{eq:T1-bound}
\end{equation}
Applying \eqref{eq:resolvent-bound} to the definition of $T_5$ in
\eqref{eq:T5-T8-definitions} produces
\begin{equation}
    T_5 \le
    \frac{R\cdot\iota}{N(1-\gamma)}\cdot \bm{1}_S.
    \label{eq:T5-bound}
\end{equation}
\paragraph*{Bound on $T_2$ and $T_6$.}
The following lemma converts Bellman and anchor residuals into a
resolvent-variance bound. It is the common tool behind the matched terms of the form $(\bm{I}-\gamma P)^{-1}\sqrt{\Var_P(V)}$, where $V$ is a value function and $P$ is a stochastic kernel. Its proof is
deferred to Appendix~\ref{subsec:Proof-lemma-span-aware-resolvent-variance}.
\begin{lemma}[Anchored resolvent-variance bound]
    \label{lem:span-aware-resolvent-variance}
    Let $(\bar\rho,\bar h)$ satisfy
    \mainNominalAnchorSupersolution{}, and let $H_{\mathrm{anc}}$ be
    defined by \eqref{eq:anchor-span-scale}. Fix a radius
    $R\ge\max\{1,H_{\mathrm{anc}}\}$. Let $P$ be a stochastic kernel selected
    from the true or empirical robust uncertainty set, and let $b$ be a
    residual vector.
    Suppose that
    \begin{equation}
        \label{eq:anchored-resolvent-value-premises}
        \|V\|_{\mathrm{span}}\le R,
        \qquad
        (1-\gamma)\min_sV(s)\le1,
        \qquad
        V=r^\pi+\gamma PV+b.
    \end{equation}
    Define the anchor defect associated with $V$ by
    \begin{equation}
        \label{eq:generic-anchor-defect}
        \beta
        \coloneqq
        \left[\bar\rho-(1-\gamma)\min_sV(s)\right]_+.
    \end{equation}
    Let $\xi_V,\xi_h\in\mathbb R^S$ be residual envelopes such that, for a
    universal constant $C_{\mathrm{res}}$,
    \begin{equation}
        \label{eq:anchored-resolvent-residual-envelopes}
        |b|\le |\xi_V|,
        \qquad
        |(P-P^{0,\pi})\bar h|
        \le
        C_{\mathrm{res}}\sigma H_{\mathrm{anc}}\cdot \bm{1}_S+|\xi_h|.
    \end{equation}
    Then, for a universal constant $C>0$,
    \[
        \left\|
        (I-\gamma P)^{-1}\sqrt{\Var_P(V)}
        \right\|_\infty
        \le
        \frac{C}{1-\gamma}
        \sqrt{
            B(R,\beta)
            +R\|\mathcal G_P |\xi_V|\|_\infty
            +R\|\mathcal G_P |\xi_h|\|_\infty
        },
    \]
    where $B(R,\beta)$ and $\mathcal G_P$ are defined in
    \eqref{eq:generic-localized-budget} and
    \eqref{eq:normalized-discounted-resolvent}, respectively.
\end{lemma}

Recall from \eqref{eq:T1-T4-definitions} that
\[
    T_2 = \sqrt{\frac{2\cdot\iota}{N}} \left(I-\gamma\widehat{P}^{\pi_\gamma^\star,U}\right)^{-1} \sqrt{\Var_{\widehat{P}^{\pi_\gamma^\star,U}}(U)}.
\]

The following lemma is an application of Lemma~\ref{lem:span-aware-resolvent-variance} that bounds $T_2$. Its proof is deferred to
Appendix~\ref{subsec:proof-T2-anchored-variance-bound}.
\begin{lemma}[Matched anchored variance bound]
    \label{lem:T2-anchored-variance-bound}
    Let $R\ge\max\{1,H_{\mathrm{anc}}\}$ be a deterministic radius and suppose
    \eqref{eq:discounted-localized-span-condition} and
    \eqref{eq:recursive-comparison-sample-conditions} hold. On the event
    $\mathcal E(R)$,
    \begin{equation}
        \label{eq:T2-bound}
        T_2
        \le
        C\sqrt{
            \frac{B_\star(R)\cdot\iota}
            {N(1-\gamma)^2}}\cdot\bm{1}_S
        +C\frac{R\cdot\iota}{N(1-\gamma)}\cdot\bm{1}_S.
    \end{equation}
\end{lemma}

The resolvent kernel in $T_6$ is selected by $\widehat U$, whereas its
variance is evaluated at $U$. The following lemma is an application of Lemma~\ref{lem:span-aware-resolvent-variance} that bounds $T_6$;
its proof is deferred to
Appendix~\ref{subsec:proof-T6-anchored-variance-bound}.
\begin{lemma}[Mismatched anchored variance bound for $T_6$]
    \label{lem:T6-anchored-variance-bound}
    Let $R\ge\max\{1,H_{\mathrm{anc}}\}$ be a deterministic radius and suppose
    \eqref{eq:discounted-localized-span-condition} and
    \eqref{eq:recursive-comparison-sample-conditions} hold. On the event
    $\mathcal E(R)$,
    \begin{align}
        T_6
         & \le
        C\sqrt{
            \frac{B_\star(R)\cdot\iota}
            {N(1-\gamma)^2}}\cdot \bm{1}_S
        +C\frac{R\cdot\iota}{N(1-\gamma)}\cdot \bm{1}_S
        +\frac{1}{50}\Delta\cdot \bm{1}_S.
        \label{eq:T6-bound-young}
    \end{align}
\end{lemma}

\paragraph*{Bound on $T_3$ and $T_7$ (uncertainty-induced TV penalty).}
We first focus on $T_3$; the argument for $T_7$ is similar.
Recall from \eqref{eq:T1-T4-definitions} that
\[
    T_3
    =
    \sqrt{\frac{2\cdot\iota}{N}}\,
    \left(I-\gamma\widehat P^{\pi_\gamma^\star,U}\right)^{-1}
    \sqrt{
        \left\vert
        \Var_{\widehat P^{0,\pi_\gamma^\star}}(U)
        -
        \Var_{\widehat P^{\pi_\gamma^\star,U}}(U)
        \right\vert
    }.
\]
Let $U'\coloneqq U-\min_sU(s)\cdot \bm{1}_S$. Since
variance is invariant under shifts,
\[
    \left|
    \Var_{\widehat P^{0,\pi_\gamma^\star}}(U)-\Var_{\widehat P^{\pi_\gamma^\star,U}}(U)
    \right|
    =
    \left|
    \Var_{\widehat P^{0,\pi_\gamma^\star}}(U')-\Var_{\widehat P^{\pi_\gamma^\star,U}}(U')
    \right|.
\]
By H\"older's inequality,
\[
    \left|
    \Var_{\widehat P^{0,\pi_\gamma^\star}}(U')-\Var_{\widehat P^{\pi_\gamma^\star,U}}(U')
    \right|
    \le
    \|\widehat P^{0,\pi_\gamma^\star}-\widehat P^{\pi_\gamma^\star,U}\|_1 \|U'\|_\infty^2
    \le 2\sigma \|U'\|_\infty^2
    = 2\sigma \|U\|_{\mathrm{span}}^2.
\]
Substituting this into the definition of $T_3$ in
\eqref{eq:T1-T4-definitions}, using
\eqref{eq:resolvent-bound}, and using
$\|U\|_{\mathrm{span}}\le R$ from
\eqref{eq:discounted-localized-span-condition},
we have that
\begin{equation*}
    T_3
    \le
    \sqrt{\frac{2\cdot\iota}{N(1-\gamma)^{2}}}\cdot\sqrt{2\sigma}\cdot
    R\cdot\bm{1}_S.
\end{equation*}
Therefore,
\begin{equation}
    \label{eq:T3-bound-expanded}
    T_3
    \le
    C\sqrt{\frac{B_\star(R)\cdot\iota}{N(1-\gamma)^{2}}}
    \cdot\bm{1}_S.
\end{equation}
Repeating this calculation for $T_7$ in \eqref{eq:T5-T8-definitions} results in
\begin{equation}
    \label{eq:T7-bound}
    T_7
    \le
    C\sqrt{\frac{B_\star(R)\cdot\iota}{N(1-\gamma)^{2}}}
    \cdot\bm{1}_S.
\end{equation}

\paragraph*{Bound on $T_4$ and $T_8$.} The terms $T_4$ and $T_8$ come from replacing the true nominal variance
by the empirical nominal variance in the Bernstein part of
Lemma~\ref{lem:one-step-robust-concentration}. The relevant object is a
standard deviation, so the perturbation is of order
$\|U\|_{\mathrm{span}}\sqrt{\iota/N}$ rather
than a full variance perturbation.

\begin{lemma}[Empirical standard-deviation perturbation; Lemma~11 of \citet{panaganti2022sample}]
    \label{lem:empirical-std-perturbation}
    Fix a value vector $V$ that is independent of the empirical nominal
    transition kernel, and let $\pi$ be a fixed deterministic policy. With
    probability at least $1-\delta$, for every state $s$,
    \[
        \left|
        \sqrt{\Var_{P^{0,\pi}}(V)(s)}
        -
        \sqrt{\Var_{\widehat P^{0,\pi}}(V)(s)}
        \right|
        \le
        2\|V\|_{\mathrm{span}}
        \sqrt{2\frac{\iota}{N}}.
    \]
\end{lemma}

We apply Lemma~\ref{lem:empirical-std-perturbation} with
$V=U$ and $\pi=\pi_\gamma^\star$. Using
$\|U\|_{\mathrm{span}}\le R$ from
\eqref{eq:discounted-localized-span-condition} and
\eqref{eq:resolvent-bound}, we have that
\begin{equation}
    T_4
    \le
    C\frac{R\cdot\iota}{N(1-\gamma)}\cdot \bm{1}_S.
    \label{eq:T4-bound}
\end{equation}
Applying Lemma~\ref{lem:empirical-std-perturbation} and
\eqref{eq:resolvent-bound} to $T_8$ in \eqref{eq:T5-T8-definitions} yields
\begin{equation}\label{eq:T8-bound}
    T_8
    \le
    C\frac{R\cdot\iota}{N(1-\gamma)}\cdot \bm{1}_S.
\end{equation}
With all eight terms controlled, it remains to combine the two branches.

\paragraph*{Step 3: Putting the bounds together.}
Combining \eqref{eq:T1-bound}, \eqref{eq:T2-bound},
\eqref{eq:T3-bound-expanded}, and \eqref{eq:T4-bound}, we obtain
\begin{align*}
    \sum_{i=1}^{4}T_i
     & \le
    C\sqrt{\frac{B_\star(R)\cdot\iota}{N(1-\gamma)^2}}\cdot \bm{1}_S
    +
    C\frac{R\cdot\iota}{N(1-\gamma)}\cdot \bm{1}_S.
\end{align*}
Similarly, using \eqref{eq:T5-bound}, \eqref{eq:T6-bound-young},
\eqref{eq:T7-bound}, and \eqref{eq:T8-bound},
\begin{align*}
    \sum_{i=5}^{8}T_i
     & \le
    C\sqrt{\frac{B_\star(R)\cdot\iota}{N(1-\gamma)^2}}\cdot \bm{1}_S
    +
    C\frac{R\cdot\iota}{N(1-\gamma)}\cdot \bm{1}_S
    +
    \frac{1}{50}\Delta\cdot \bm{1}_S.
\end{align*}
Combining the two branches in \eqref{eq:V-hat-V-pi-star-C-decomp}, we have
that
\begin{align*}
    \big\|\widehat U
    -U\big\|_{\infty}
     & \le
    C\sqrt{
          \frac{B_\star(R)\cdot\iota}
          {N(1-\gamma)^2}}
    +
    C\frac{R\cdot\iota}{N(1-\gamma)}
    +
    \frac{1}{50}\Delta.
\end{align*}
This proves the first assertion in Lemma~\ref{lem:V-hat-V-part-1-decomp}. The argument within this proof uses
$\pi_\gamma^\star$ only as a deterministic policy that is independent of the
empirical transition samples. Replacing it throughout by any such policy
$\pi$, replacing $\Delta$ by
\[
    \left\|
    \widehat V_\gamma^{\pi,\sigma}
    -V_\gamma^{\pi,\sigma}
    \right\|_\infty,
\]
and repeating the argument gives \eqref{eq:generic-fixed-policy-localized-comparison}. This completes
the proof of Lemma~\ref{lem:V-hat-V-part-1-decomp}.

\subsubsection{Proof of Lemma~\ref{lem:V-hat-V-pi-hat-bound}}
\label{subsec:Proof-lemma-V-hat-V-pi-hat-bound}

Recall that
\[
    \widehat W
    \coloneqq
    \widehat V_\gamma^{\widehat\pi,\sigma},
    \qquad
    W
    \coloneqq
    V_\gamma^{\widehat\pi,\sigma},
    \qquad
    \widehat V^\star
    \coloneqq
    \widehat V_\gamma^{\star,\sigma}.
\]
The proof has three steps. We first decompose the error $\|\widehat W-W\|_\infty$
into two branches and then bound the terms in both branches. Finally, we
combine these bounds to obtain the claimed estimate.

\paragraph*{Step 1: Error decomposition.}
In this step, we derive an error decomposition that upper-bounds
$\|\widehat W-W\|_\infty$ by the maximum of two sums of more manageable
terms.

The following lemma controls the empirical-to-true robust Bellman difference
at $\widehat W$. The proof is deferred to
Appendix~\ref{sec:proof-of-shi-original-lemma-11}.

\begin{lemma}
    \label{lem:shi-original-lemma-11}
    Let $R\ge1$ be deterministic.
    For a sufficiently large numerical constant $C_0$, suppose
    $N\ge C_0\cdot\iota/(1-\gamma)$. Then, with probability at least
    $1-O(\delta)$, the following bound holds whenever the solver
    guarantee \mainRobustDMDPSolverTolerance{} holds and the span bounds
    \[
        \|\widehat W\|_{\mathrm{span}}\le R,
        \qquad
        \|\widehat V^\star\|_{\mathrm{span}}\le R
    \]
    are satisfied:
    \begin{equation}
        \label{eq:centered-empirical-planning-residual}
        \left|
        \left(\widehat P^{\widehat\pi,\widehat W}
        -P^{\widehat\pi,\widehat W}\right)\widehat W
        \right|
        \le
        C\sqrt{\frac{\iota}{N}}
        \sqrt{\Var_{P^{0,\widehat\pi}}(\widehat V^\star)}
        +C\frac{R\cdot\iota}{N}\cdot\bm{1}_S
        +C\varepsilon_{\mathrm{opt}}\cdot\bm{1}_S.
    \end{equation}
\end{lemma}

We now describe the decomposition of the learned-policy error
$\widehat W-W$. The following lemma bounds the error by
the maximum of two resolvent-perturbation terms. The proof is deferred to
Appendix~\ref{subsec:Proof-lemma-two-sided-pi-hat-value-decomposition}.
\begin{lemma}
    \label{lem:two-sided-pi-hat-value-decomposition}
    Consider $\widehat{W}$ and $W$. We have that
    \begin{align}
        \label{eq:pi-hat-two-sided-resolvent-decomp}
        \|\widehat W-W\|_{\infty}
         & \le
        \max\Bigg\{
        \left\|
        \left(I-\gamma P^{\widehat\pi,\widehat W}\right)^{-1}
        \left|
        \left(
        \widehat P^{\widehat\pi,\widehat W}
        -
        P^{\widehat\pi,\widehat W}
        \right)
        \widehat W
        \right|
        \right\|_\infty,
        \nonumber\\
         & \qquad\qquad
        \left\|
        \left(I-\gamma P^{\widehat\pi,W}\right)^{-1}
        \left|
        \left(
        \widehat P^{\widehat\pi,\widehat W}
        -
        P^{\widehat\pi,\widehat W}
        \right)
        \widehat W
        \right|
        \right\|_\infty
        \Bigg\}.
    \end{align}
\end{lemma}

We now follow the same termwise structure as in the proof of
Lemma~\ref{lem:V-hat-V-part-1-decomp}. The triangle inequality for standard
deviations implies
\[
    \begin{aligned}
    \sqrt{\Var_{P^{0,\widehat\pi}}(\widehat V^\star)}
    &\le
    \sqrt{\Var_{P^{\widehat\pi,\widehat W}}(\widehat W)}
    +\sqrt{\Var_{P^{\widehat\pi,\widehat W}}
    (\widehat V^\star-\widehat W)}
    \\
    &\quad
    +\left|
    \sqrt{\Var_{P^{0,\widehat\pi}}(\widehat V^\star)}
    -\sqrt{\Var_{P^{\widehat\pi,\widehat W}}(\widehat V^\star)}
    \right|.
    \end{aligned}
\]
Applying Lemma~\ref{lem:shi-original-lemma-11} to the first branch in
\eqref{eq:pi-hat-two-sided-resolvent-decomp} and using the preceding display,
we obtain
\[
    \left\|
    \left(I-\gamma P^{\widehat\pi,\widehat W}\right)^{-1}
    \left|
    \left(
    \widehat P^{\widehat\pi,\widehat W}
    -P^{\widehat\pi,\widehat W}
    \right)\widehat W
    \right|
    \right\|_\infty
    \le
    \left\|\sum_{i=1}^{4}T_i'\right\|_\infty,
\]
where
\begin{equation}
    \label{eq:T-prime-first-branch-definitions}
    \begin{aligned}
        T_1'
        &\coloneqq
        \left(I-\gamma P^{\widehat\pi,\widehat W}\right)^{-1}
        \left(
        C\frac{R\cdot\iota}{N}
        +C\varepsilon_{\mathrm{opt}}
        \right)\cdot\bm{1}_S,
        \\
        T_2'
        &\coloneqq
        C\sqrt{\frac{\iota}{N}}
        \left(I-\gamma P^{\widehat\pi,\widehat W}\right)^{-1}
        \sqrt{\Var_{P^{\widehat\pi,\widehat W}}(\widehat W)},
        \\
        T_3'
        &\coloneqq
        C\sqrt{\frac{\iota}{N}}
        \left(I-\gamma P^{\widehat\pi,\widehat W}\right)^{-1}
        \sqrt{\Var_{P^{\widehat\pi,\widehat W}}
        (\widehat V^\star-\widehat W)},
        \\
        T_4'
        &\coloneqq
        C\sqrt{\frac{\iota}{N}}
        \left(I-\gamma P^{\widehat\pi,\widehat W}\right)^{-1}
        \left|
        \sqrt{\Var_{P^{0,\widehat\pi}}(\widehat V^\star)}
        -\sqrt{\Var_{P^{\widehat\pi,\widehat W}}(\widehat V^\star)}
        \right|.
    \end{aligned}
\end{equation}
The same argument for the second branch yields
\[
    \left\|
    \left(I-\gamma P^{\widehat\pi,W}\right)^{-1}
    \left|
    \left(
    \widehat P^{\widehat\pi,\widehat W}
    -P^{\widehat\pi,\widehat W}
    \right)\widehat W
    \right|
    \right\|_\infty
    \le
    \left\|\sum_{i=5}^{8}T_i'\right\|_\infty,
\]
where
\begin{equation}
    \label{eq:T-prime-second-branch-definitions}
    \begin{aligned}
        T_5'
        &\coloneqq
        \left(I-\gamma P^{\widehat\pi,W}\right)^{-1}
        \left(
        C\frac{R\cdot\iota}{N}
        +C\varepsilon_{\mathrm{opt}}
        \right)\cdot\bm{1}_S,
        \\
        T_6'
        &\coloneqq
        C\sqrt{\frac{\iota}{N}}
        \left(I-\gamma P^{\widehat\pi,W}\right)^{-1}
        \sqrt{\Var_{P^{\widehat\pi,\widehat W}}(\widehat W)},
        \\
        T_7'
        &\coloneqq
        C\sqrt{\frac{\iota}{N}}
        \left(I-\gamma P^{\widehat\pi,W}\right)^{-1}
        \sqrt{\Var_{P^{\widehat\pi,\widehat W}}
        (\widehat V^\star-\widehat W)},
        \\
        T_8'
        &\coloneqq
        C\sqrt{\frac{\iota}{N}}
        \left(I-\gamma P^{\widehat\pi,W}\right)^{-1}
        \left|
        \sqrt{\Var_{P^{0,\widehat\pi}}(\widehat V^\star)}
        -\sqrt{\Var_{P^{\widehat\pi,\widehat W}}(\widehat V^\star)}
        \right|.
    \end{aligned}
\end{equation}
Combining the two branches in
\eqref{eq:pi-hat-two-sided-resolvent-decomp}, we obtain
\begin{equation}
    \label{eq:V-hat-V-pi-hat-T-decomp}
    \|\widehat W-W\|_\infty
    \le
    \max\left\{
    \left\|\sum_{i=1}^{4}T_i'\right\|_\infty,
    \left\|\sum_{i=5}^{8}T_i'\right\|_\infty
    \right\}.
\end{equation}
It remains to bound the terms in these two branches.

\paragraph*{Step 2: Bounding the error terms $T_1',\ldots,T_8'$.}
\paragraph*{Bound on $T_1'$ and $T_5'$.}
Applying \eqref{eq:resolvent-bound} to
\eqref{eq:T-prime-first-branch-definitions} and
\eqref{eq:T-prime-second-branch-definitions} yields
\begin{equation}
    \label{eq:T1-T5-prime-bound}
    \max\left\{
    \|T_1'\|_\infty,
    \|T_5'\|_\infty
    \right\}
    \le
    C\frac{R\cdot\iota}{N(1-\gamma)}
    +C\frac{\varepsilon_{\mathrm{opt}}}{1-\gamma}.
\end{equation}

\paragraph*{Bound on $T_2'$ and $T_6'$.}
These are the principal resolvent-variance terms. The first is matched because
its resolvent kernel is selected at $\widehat W$, while the second is
mismatched because its resolvent kernel is selected at $W$. The next two
lemmas provide the corresponding bounds; their proofs are deferred to
Appendices~\ref{sec:proof-of-span-aware-shi-lemma-12} and~\ref{sec:proof-of-mismatched-resolvent-variance-pi-hat}.

\begin{lemma}[Matched learned-policy variance bound for $T_2'$]
    \label{lem:span-aware-shi-lemma-12}
    Let $R\ge\max\{1,H_{\mathrm{anc}}\}$ be a deterministic radius and suppose
    \eqref{eq:discounted-localized-span-condition} and
    \eqref{eq:recursive-comparison-sample-conditions} hold. On the event
    $\mathcal E(R)$,
    \begin{equation}
        \label{eq:T2-prime-bound}
        \|T_2'\|_\infty
        \le
        C\sqrt{
        \frac{B_{\mathrm{emp}}(R)\cdot\iota}
        {N(1-\gamma)^2}}
        +C\frac{R\cdot\iota}{N(1-\gamma)}
        +C\frac{\varepsilon_{\mathrm{opt}}}{1-\gamma}.
    \end{equation}
\end{lemma}

\begin{lemma}[Mismatched learned-policy variance bound for $T_6'$]
    \label{lem:mismatched-resolvent-variance-pi-hat}
    Let $R\ge\max\{1,H_{\mathrm{anc}}\}$ be a deterministic radius and suppose
    \eqref{eq:discounted-localized-span-condition} and
    \eqref{eq:recursive-comparison-sample-conditions} hold. On the event
    $\mathcal E(R)$,
    \begin{equation}
        \label{eq:T6-prime-bound}
        \|T_6'\|_\infty
        \le
        C\sqrt{
        \frac{B_{\mathrm{emp}}(R)\cdot\iota}
        {N(1-\gamma)^2}}
        +C\frac{R\cdot\iota}{N(1-\gamma)}
        +\frac{1}{20}\|\widehat W-W\|_\infty
        +C\frac{\varepsilon_{\mathrm{opt}}}{1-\gamma}.
    \end{equation}
\end{lemma}

\paragraph*{Bound on $T_3'$ and $T_7'$.}
Optimality and the solver guarantee
\mainRobustDMDPSolverTolerance{} ensure that
\[
    0
    \le
    \widehat V^\star-\widehat W
    \le
    \varepsilon_{\mathrm{opt}}\cdot\bm{1}_S.
\]
Hence
\[
    \sqrt{\Var_{P^{\widehat\pi,\widehat W}}
    (\widehat V^\star-\widehat W)}
    \le
    \varepsilon_{\mathrm{opt}}\cdot\bm{1}_S.
\]
Using \eqref{eq:resolvent-bound}, $N\ge\iota$, and the definitions in
\eqref{eq:T-prime-first-branch-definitions} and
\eqref{eq:T-prime-second-branch-definitions}, we obtain
\begin{equation}
    \label{eq:T3-T7-prime-bound}
    \max\left\{
    \|T_3'\|_\infty,
    \|T_7'\|_\infty
    \right\}
    \le
    C\frac{\varepsilon_{\mathrm{opt}}}{1-\gamma}.
\end{equation}

\paragraph*{Bound on $T_4'$ and $T_8'$.}
We use the following deterministic comparison whenever two transition
distributions are close in total variation. Its proof is deferred to
Appendix~\ref{subsec:proof-standard-deviation-TV-perturbation}.
\begin{lemma}[Standard-deviation perturbation under total variation]
    \label{lem:standard-deviation-TV-perturbation}
    For any probability vectors $q,q'$ on the same finite space and any
    vector $f$ of the corresponding dimension,
    \begin{equation}
        \label{eq:standard-deviation-TV-perturbation}
        \left|
        \sqrt{\Var_q(f)}-\sqrt{\Var_{q'}(f)}
        \right|
        \le
        C\|f\|_{\mathrm{span}}
        \sqrt{\|q-q'\|_{\mathrm{TV}}},
    \end{equation}
    where $C>0$ is a universal constant.
\end{lemma}

Convexity of total variation and the robust-set constraint imply
\[
    \left\|
    P^{\widehat\pi,\widehat W}(s,\cdot)
    -P^{0,\widehat\pi}(s,\cdot)
    \right\|_{\mathrm{TV}}
    \le\sigma,
    \qquad s\in\cS.
\]
Lemma~\ref{lem:standard-deviation-TV-perturbation}, together with
$\|\widehat V^\star\|_{\mathrm{span}}\le R$ from
\eqref{eq:discounted-localized-span-condition}, therefore implies
\[
    \left|
    \sqrt{\Var_{P^{0,\widehat\pi}}(\widehat V^\star)}
    -\sqrt{\Var_{P^{\widehat\pi,\widehat W}}(\widehat V^\star)}
    \right|
    \le
    CR\sqrt{\sigma}\cdot\bm{1}_S.
\]
Using \eqref{eq:resolvent-bound} and
$B_{\mathrm{emp}}(R)\ge\sigma R^2$, we obtain
\begin{equation}
    \label{eq:T4-T8-prime-bound}
    \max\left\{
    \|T_4'\|_\infty,
    \|T_8'\|_\infty
    \right\}
    \le
    C\sqrt{
    \frac{B_{\mathrm{emp}}(R)\cdot\iota}
    {N(1-\gamma)^2}}.
\end{equation}
With all eight terms controlled, it remains to combine the two branches.

\paragraph*{Step 3: Putting the bounds together.}
Combining \eqref{eq:T1-T5-prime-bound},
\eqref{eq:T2-prime-bound}, \eqref{eq:T3-T7-prime-bound}, and
\eqref{eq:T4-T8-prime-bound} yields
\[
    \left\|\sum_{i=1}^{4}T_i'\right\|_\infty
    \le
    C\sqrt{
    \frac{B_{\mathrm{emp}}(R)\cdot\iota}
    {N(1-\gamma)^2}}
    +C\frac{R\cdot\iota}{N(1-\gamma)}
    +C\frac{\varepsilon_{\mathrm{opt}}}{1-\gamma},
\]
whereas \eqref{eq:T1-T5-prime-bound}, \eqref{eq:T6-prime-bound},
\eqref{eq:T3-T7-prime-bound}, and \eqref{eq:T4-T8-prime-bound} lead to
\[
    \left\|\sum_{i=5}^{8}T_i'\right\|_\infty
    \le
    C\sqrt{
    \frac{B_{\mathrm{emp}}(R)\cdot\iota}
    {N(1-\gamma)^2}}
    +C\frac{R\cdot\iota}{N(1-\gamma)}
    +\frac{1}{20}\|\widehat W-W\|_\infty
    +C\frac{\varepsilon_{\mathrm{opt}}}{1-\gamma}.
\]
Substituting these two estimates into
\eqref{eq:V-hat-V-pi-hat-T-decomp} and using
$\|\widehat W-W\|_\infty\le\Delta$ from \eqref{eq:Delta-defn} completes the
proof of Lemma~\ref{lem:V-hat-V-pi-hat-bound}.

\subsubsection{Proof of Lemma~\ref{lem:one-step-robust-concentration}}
\label{subsec:Proof-lemma-one-step-robust-concentration}
The proof has three steps. We first use TV duality and discretization to
reduce the robust-kernel error to a finite family of nominal empirical
processes. We then establish a uniform Bernstein bound for this family.
Finally, we average the rowwise estimate over an arbitrary policy and derive
the normalized-resolvent bound.

\paragraph*{Step 1: Reduction and discretization.}
In this step, we reduce the robust-kernel error for a fixed state-action pair
to a clipped nominal empirical process and discretize its clipping level.
Fix a value vector $V$ that is independent of the empirical transition
kernels, and set
\[
    V^\circ
    \coloneqq
    V-\min_sV(s)\cdot\bm{1}_S.
\]
The robust expectation and variance are invariant under adding constants to
$V$, and
\[
    0\le V^\circ\le\|V\|_{\mathrm{span}}\cdot\bm{1}_S,
    \qquad
    \|V^\circ\|_\infty=\|V\|_{\mathrm{span}}.
\]
Fix a state-action pair $(s,a)$. Applying the TV strong-duality formula in
\citet[Lemma~4]{shi2023curious} to $V^\circ$ yields
\begin{align*}
    P_{s,a}^{V}V
    & =
    \min_sV(s)
    +
    \max_{\alpha\in[0,\|V\|_{\mathrm{span}}]}
    \left\{
    P^0_{s,a}\min\{V^\circ,\alpha\}
    -\sigma_{s,a}\alpha
    \right\},
    \\
    \widehat P_{s,a}^{V}V
    & =
    \min_sV(s)
    +
    \max_{\alpha\in[0,\|V\|_{\mathrm{span}}]}
    \left\{
    \widehat P^0_{s,a}\min\{V^\circ,\alpha\}
    -\sigma_{s,a}\alpha
    \right\}.
\end{align*}
The terms $\min_sV(s)$ and $-\sigma_{s,a}\alpha$ are common to the two
expressions. Since the difference of two maxima is at most the maximum
pointwise difference, we obtain
\[
    \left|
    \widehat{P}_{s,a}^{V}V-P_{s,a}^{V}V
    \right|
    \le
    \sup_{\alpha\in [0,\|V\|_{\mathrm{span}}]}
    \left|
    (\widehat{P}^0_{s,a}-P^0_{s,a})\min\{V^\circ,\alpha\}
    \right|,
\]
where the minimum is taken coordinatewise.

If $\|V\|_{\mathrm{span}}=0$, the displayed supremum is zero. It therefore
remains to consider $\|V\|_{\mathrm{span}}>0$. Define the clipping grid
\[
    \mathcal N_\alpha
    \coloneqq
    \left\{
    \frac{k\|V\|_{\mathrm{span}}}{6N}
    :
    0\le k\le6N
    \right\}.
\]
Then $|\mathcal N_\alpha|=6N+1$. For every
$\alpha\in[0,\|V\|_{\mathrm{span}}]$, choose
$\alpha'\in\mathcal N_\alpha$ with
$|\alpha'-\alpha|\le\|V\|_{\mathrm{span}}/(6N)$. Moreover,
\[
    \left\|
    \min\{V^\circ,\alpha\}
    -\min\{V^\circ,\alpha'\}
    \right\|_\infty
    =
    \sup_{s'\in\cS}
    \left|\min\{V^\circ(s'),\alpha\}
    -\min\{V^\circ(s'),\alpha'\}\right|
    \le \frac{\|V\|_{\mathrm{span}}}{6N},
\]
and hence
\[
    \left|
    (\widehat P^0_{s,a}-P^0_{s,a})
    \left(
    \min\{V^\circ,\alpha\}
    -\min\{V^\circ,\alpha'\}
    \right)
    \right|
    \le
    \frac{\|V\|_{\mathrm{span}}}{3N}
\]
and
\begin{equation}
    \sup_{\alpha\in[0,\|V\|_{\mathrm{span}}]}
    \left|
    (\widehat P^0_{s,a}-P^0_{s,a})
    \min\{V^\circ,\alpha\}
    \right|
    \le
    \max_{\alpha\in\mathcal N_\alpha}
    \left|
    (\widehat P^0_{s,a}-P^0_{s,a})
    \min\{V^\circ,\alpha\}
    \right|
    +
    \frac{\|V\|_{\mathrm{span}}}{3N}.
    \label{eq:clip-grid-reduce}
\end{equation}
Therefore, up to the interpolation error in
\eqref{eq:clip-grid-reduce}, it remains to control the empirical process at
the finitely many levels in $\mathcal N_\alpha$.

\paragraph*{Step 2: Uniform rowwise concentration.}
In this step, we apply Bernstein's inequality on the clipping grid and obtain
a bound that holds simultaneously for all state-action pairs.
For a fixed grid level~$\alpha$, the vector
$\min\{V^\circ,\alpha\}$ is independent of
$\widehat P^0_{s,a}$ and is bounded in
$[0,\|V\|_{\mathrm{span}}]$. Bernstein's inequality shows that
with probability at least $1-\delta'$,
\[
    \left|
    (\widehat P^0_{s,a}-P^0_{s,a})
    \min\{V^\circ,\alpha\}
    \right|
    \le
    \sqrt{
        \frac{
        2\Var_{P^0_{s,a}}(\min\{V^\circ,\alpha\})
        \log(\frac{2}{\delta'})
        }{N}
    }
    +
    \frac{2\|V\|_{\mathrm{span}}\log(\frac{2}{\delta'})}{3N}.
\]
Since the scalar clipping map
$x\mapsto \min\{x,\alpha\}$ is 1-Lipschitz, the independent-copy
representation of variance implies
\[
    \Var_{P^0_{s,a}}(\min\{V^\circ,\alpha\})
    \le
    \Var_{P^0_{s,a}}(V^\circ)
    =
    \Var_{P^0_{s,a}}(V).
\]
Set
\[
    \delta'
    \coloneqq
    \frac{\delta}{SA(6N+1)}.
\]
A union bound over all $\alpha\in\mathcal N_\alpha$ and all state-action
pairs yields probability at least $1-\delta$. Moreover,
\[
    \log\left(\frac{2}{\delta'}\right)
    =
    \log\left(\frac{2SA(6N+1)}{\delta}\right)
    \le
    \iota.
\]
Consequently, on this event, simultaneously for every $(s,a)$,
\[
    \max_{\alpha\in\mathcal N_\alpha}
    \left|
    (\widehat{P}^0_{s,a}-P^0_{s,a})\min\{V^\circ,\alpha\}
    \right|
    \le
    \sqrt{
        \frac{2\Var_{P^0_{s,a}}(V)\cdot\iota}{N}
    }
    +
    \frac{2\|V\|_{\mathrm{span}}\cdot\iota}{3N}.
\]
Combining this with \eqref{eq:clip-grid-reduce} results in
\[
    \sup_{\alpha\in[0,\|V\|_{\mathrm{span}}]}
    \left|
    (\widehat{P}^0_{s,a}-P^0_{s,a})\min\{V^\circ,\alpha\}
    \right|
    \le
    \sqrt{
        \frac{2\Var_{P^0_{s,a}}(V)\cdot\iota}{N}
    }
    +
    \frac{\|V\|_{\mathrm{span}}\cdot\iota}{N}.
\]
Here we used $\iota\ge1$ to combine the two linear terms.
Thus, simultaneously for all $(s,a)$,
\begin{equation}
    \label{eq:local-fixed-value-robust-kernel-concentration}
    \left|
    \widehat{P}_{s,a}^{V}V-P_{s,a}^{V}V
    \right|
    \le
    \sqrt{
        \frac{2\Var_{P^0_{s,a}}(V)\cdot\iota}{N}
    }
    +
    \frac{\|V\|_{\mathrm{span}}\cdot\iota}{N}.
\end{equation}
We have therefore established the desired robust-kernel concentration bound
simultaneously for every state-action pair.

\paragraph*{Step 3: Policy averaging and the resolvent bound.}
In this final step, we average the rowwise estimate over an arbitrary policy,
convert the averaged conditional variances into
$\Var_{P^{0,\pi}}(V)$, and then pass to the normalized-resolvent form.
Because \eqref{eq:local-fixed-value-robust-kernel-concentration} holds
simultaneously for all $(s,a)$, we may now fix any policy $\pi$. By the
definitions of $P^{\pi,V}$ and
$\widehat P^{\pi,V}$, for every state~$s$,
\[
    \left|
    \left(
    \widehat P^{\pi,V}
    -
    P^{\pi,V}
    \right)V
    \right|(s)
    \le
    \sum_{a\in\cA}\pi(a\mid s)
    \left|
    \widehat P_{s,a}^{V}V
    -
    P_{s,a}^{V}V
    \right|.
\]
Combining \eqref{eq:local-fixed-value-robust-kernel-concentration} with
Jensen's inequality leads to
\[
    \left|
    \left(
    \widehat P^{\pi,V}
    -
    P^{\pi,V}
    \right)V
    \right|(s)
    \le
    \sqrt{
        \frac{
        2\sum_{a\in\cA}\pi(a\mid s)\Var_{P^0_{s,a}}(V)\cdot\iota
        }{N}
    }
    +
    \frac{\|V\|_{\mathrm{span}}\cdot\iota}{N}.
\]
The law of total variance implies
\begin{align*}
    \Var_{P^{0,\pi}}(V)(s)
    &=
    \sum_{a\in\cA}\pi(a\mid s)\Var_{P^0_{s,a}}(V)
+
    \sum_{a\in\cA}\pi(a\mid s)
    \left(
    P^0_{s,a}V-P^{0,\pi}V(s)
    \right)^2
    \\
    &\ge
    \sum_{a\in\cA}\pi(a\mid s)\Var_{P^0_{s,a}}(V).
\end{align*}
Substituting this inequality into the preceding policy-averaged estimate
proves \eqref{eq:fixed-value-robust-kernel-pointwise}.

For any stochastic matrix $L$, the matrix
$(I-\gamma L)^{-1}$ is nonnegative and
$\mathcal G_L\bm{1}_S=\bm{1}_S$. Multiplying
\eqref{eq:fixed-value-robust-kernel-pointwise} by the
nonnegative matrix $(1-\gamma)(I-\gamma L)^{-1}$ and using
\eqref{eq:normalized-discounted-resolvent} therefore proves
\eqref{eq:fixed-value-robust-kernel-resolvent}, completing the proof.

\subsubsection{Proof of Lemma~\ref{lem:nominal-fixed-vector-concentration}}
\label{subsec:proof-nominal-fixed-vector-concentration}

Consider a fixed vector $g$ that is independent of the empirical
transition kernels. Set
$g^\circ\coloneqq g-\min_s g(s)\cdot \bm{1}_S$. Then
$(\widehat P^{0,\pi}-P^{0,\pi})g=(\widehat P^{0,\pi}-P^{0,\pi})g^\circ$ and
$\Var_{P^{0,\pi}}(g)=\Var_{P^{0,\pi}}(g^\circ)$. Also,
$0\le g^\circ\le\|g\|_{\mathrm{span}}\cdot \bm{1}_S$.

Fix a state-action pair $(s,a)$. Bernstein's inequality applied to the fixed
bounded vector $g^\circ$ shows that, with probability at least $1-\delta'$,
\[
    \left|
    (\widehat P^0_{s,a}-P^0_{s,a})g
    \right|
    =
    \left|
    (\widehat P^0_{s,a}-P^0_{s,a})g^\circ
    \right|
    \le
    \sqrt{
        \frac{2\Var_{P^0_{s,a}}(g)\log(\frac{2}{\delta'})}{N}
    }
    +
    \frac{2\|g\|_{\mathrm{span}}\log(\frac{2}{\delta'})}{3N}.
\]
Set $\delta'=\delta/(SA)$ and take a union bound over all state-action
pairs. Then apply policy-averaging and total-variance arguments similar to those used in the proof of Lemma~\ref{lem:one-step-robust-concentration}. With probability at least
$1-\delta$, the following bound holds simultaneously for every policy $\pi$:
\[
    \left|
    \left(
    \widehat P^{0,\pi}
    -
    P^{0,\pi}
    \right)g
    \right|
    \le
    C\sqrt{
        \frac{\Var_{P^{0,\pi}}(g)\cdot\iota}{N}
    }
    +
    C\frac{\|g\|_{\mathrm{span}}\cdot\iota}{N}\cdot \bm{1}_S.
\]
Here $\iota$ dominates
$\log(2SA/\delta)$. This proves \eqref{eq:fixed-vector-nominal-pointwise}.  For any stochastic matrix $L$, the matrix
$(I-\gamma L)^{-1}$ is nonnegative and
$\mathcal G_L\bm{1}_S=\bm{1}_S$. Multiplying
\eqref{eq:fixed-vector-nominal-pointwise} by
$(1-\gamma)(I-\gamma L)^{-1}$ and using
\eqref{eq:normalized-discounted-resolvent} proves
\eqref{eq:fixed-vector-nominal-resolvent}.

\subsubsection{Proof of Lemma~\ref{lem:two-sided-fixed-value-decomposition}}
\label{subsec:Proof-lemma-two-sided-fixed-value-decomposition}

Recall the fixed-policy aliases $U$ and $\widehat U$ from
Appendix~\ref{subsec:discounted-preliminaries}.
For a generic value vector $V$, let $P^{\pi_\gamma^\star,V}$ and
$\widehat P^{\pi_\gamma^\star,V}$ denote minimizers in the true and empirical
robust Bellman operators for the fixed policy $\pi_\gamma^\star$ at $V$,
respectively. The true and empirical robust Bellman equations are
\[
    U
    =
    r^{\pi_\gamma^\star}
    +
    \gamma P^{\pi_\gamma^\star,U}
    U,
    \qquad
    \widehat U
    =
    r^{\pi_\gamma^\star}
    +
    \gamma \widehat P^{\pi_\gamma^\star,\widehat U}
    \widehat U.
\]

\paragraph*{Step 1: Bounding $\widehat U-U$.}
Because
$\widehat P^{\pi_\gamma^\star,\widehat U}$
minimizes the empirical
robust expectation of $\widehat U$ and
$\widehat P^{\pi_\gamma^\star,U}$ is feasible,
\[
    \widehat P^{\pi_\gamma^\star,\widehat U}
    \widehat U
    \le
    \widehat P^{\pi_\gamma^\star,U}
    \widehat U.
\]
Using this and comparing the two Bellman equations yields
\begin{align*}
    \widehat U
    -
    U
    &= \gamma \widehat{P}^{\pi_\gamma^\star,\widehat U} \widehat U
    - \gamma P^{\pi_\gamma^\star,U} U \\
    &\le \gamma \widehat{P}^{\pi_\gamma^\star,U} \widehat U
    - \gamma P^{\pi_\gamma^\star,U} U \\
     & \le
    \gamma\widehat P^{\pi_\gamma^\star,U}
    \left(
    \widehat U
    -
    U
    \right)
    +
    \gamma
    \left(
    \widehat P^{\pi_\gamma^\star,U}
    -
    P^{\pi_\gamma^\star,U}
    \right)
    U.
\end{align*}
The inverse
$(I-\gamma\widehat P^{\pi_\gamma^\star,U})^{-1}$ is
nonnegative.
Hence, after rearranging, multiplying by this inverse preserves the
coordinatewise inequality and yields
\begin{equation}
    \label{eq:fixed-value-first-resolvent-direction}
    \widehat U
    -
    U
     \le
    \gamma
    \left(I-\gamma\widehat P^{\pi_\gamma^\star,U}\right)^{-1}
    \left(
    \widehat P^{\pi_\gamma^\star,U}
    -
    P^{\pi_\gamma^\star,U}
    \right)
    U.
\end{equation}

\paragraph*{Step 2: Bounding $U-\widehat U$.}
We have
\begin{align*}
    U
    -
    \widehat U
     & =
    \gamma\widehat P^{\pi_\gamma^\star,\widehat U}
    \left(
    U
    -
    \widehat U
    \right)
    +
    \gamma
    \left(
    P^{\pi_\gamma^\star,U}
    -
    \widehat P^{\pi_\gamma^\star,\widehat U}
    \right)
    U.
\end{align*}
Since $\widehat P^{\pi_\gamma^\star,U}$ minimizes the
empirical robust
expectation of $U$,
$
    \widehat P^{\pi_\gamma^\star,U}
    U
    \le
    \widehat P^{\pi_\gamma^\star,\widehat U}
    U$.
Hence
\[
    \left(
    P^{\pi_\gamma^\star,U}
    -
    \widehat P^{\pi_\gamma^\star,\widehat U}
    \right)
    U
    \le
    \left(
    P^{\pi_\gamma^\star,U}
    -
    \widehat P^{\pi_\gamma^\star,U}
    \right)
    U.
\]
The inverse
$(I-\gamma\widehat P^{\pi_\gamma^\star,\widehat U})^{-1}$
is also nonnegative. Repeating the order-preserving rearrangement used for
\eqref{eq:fixed-value-first-resolvent-direction} yields
\begin{align*}
    U
    -
    \widehat U
     & \le
    \gamma
    \left(I-\gamma\widehat P^{\pi_\gamma^\star,\widehat U}\right)^{-1}
    \left(
    P^{\pi_\gamma^\star,U}
    -
    \widehat P^{\pi_\gamma^\star,U}
    \right)
    U.
\end{align*}

\paragraph*{Step 3: Combining the two one-sided bounds.}
Combining the two directions results in
\begin{align*}
     \left\|
    \widehat U
    -
    U
    \right\|_{\infty}
     \le
    \gamma
    \max\Bigg\{&
    \left\|
    \left(I-\gamma\widehat P^{\pi_\gamma^\star,U}\right)^{-1}
    \left(
    \widehat P^{\pi_\gamma^\star,U}
    -
    P^{\pi_\gamma^\star,U}
    \right)
    U
    \right\|_\infty,
    \\
     &
    \left\|
    \left(I-\gamma\widehat P^{\pi_\gamma^\star,\widehat U}\right)^{-1}
    \left(
    \widehat P^{\pi_\gamma^\star,U}
    -
    P^{\pi_\gamma^\star,U}
    \right)
    U
    \right\|_\infty
    \Bigg\}.
\end{align*}
Since each resolvent is nonnegative, it satisfies $|Ax|\le A|x|$.
Together with $\gamma\le1$, this inequality implies
\begin{align*}
    \big\|\widehat U
    -U\big\|_{\infty}
      \le
    \max\Bigg\{&
    \left\|
    \left(I-\gamma\widehat P^{\pi_\gamma^\star,U}\right)^{-1}
    \left|
    \left(
    \widehat P^{\pi_\gamma^\star,U}
    -
    P^{\pi_\gamma^\star,U}
    \right)
    U
    \right|
    \right\|_\infty,
    \\
    & 
    \left\|
    \left(I-\gamma\widehat P^{\pi_\gamma^\star,\widehat U}\right)^{-1}
    \left|
    \left(
    \widehat P^{\pi_\gamma^\star,U}
    -
    P^{\pi_\gamma^\star,U}
    \right)
    U
    \right|
    \right\|_\infty
    \Bigg\}.
\end{align*}
The last display is \eqref{eq:two-sided-fixed-value-decomposition},
completing the proof.

\subsubsection{Proof of Lemma~\ref{lem:span-aware-resolvent-variance}}
\label{subsec:Proof-lemma-span-aware-resolvent-variance}
The proof has three steps. We first prove a general resolvent-variance inequality. We then apply it twice to bound two different terms and combine the resulting bounds to obtain the
desired estimate.

\paragraph*{Step 1: A general resolvent-variance bound.}
Fix an initial state $s$, let $e_s$ be the corresponding standard basis
vector, and set
\begin{equation}
    \label{eq:discounted-state-distribution}
    \nu_s
    \coloneqq
    (1-\gamma)e_s^\top(I-\gamma P)^{-1}.
\end{equation}
Equivalently,
\[
    \nu_s
    =
    (1-\gamma)\sum_{t=0}^{\infty}\gamma^t e_s^\top P^t.
\]
The coefficients $(1-\gamma)\gamma^t$ sum to one, so $\nu_s$ is a probability
distribution. For every vector $f$, we claim that
\begin{equation}
    \label{eq:discounted-variance-drift}
    \nu_s\Var_P(f)
    \le
    \frac{1-\gamma}{\gamma}\|f\|_{\mathrm{span}}^2
    +\frac{2}{\gamma}\|f\|_{\mathrm{span}}\,
    \nu_s[f-Pf]_+.
\end{equation}
We now prove \eqref{eq:discounted-variance-drift}.
Since $\Var_P(f)$ and $f-Pf$ are translation invariant, we may shift $f$ so
that $0\le f\le\|f\|_{\mathrm{span}}\cdot\bm{1}_S$. Rearranging
\eqref{eq:discounted-state-distribution} yields
\[
    \nu_sP
    =
    \frac{1}{\gamma}\nu_s
    -
    \frac{1-\gamma}{\gamma}e_s^\top.
\]
Therefore,
\begin{align*}
    \nu_s\Var_P(f)
     & =
    \nu_sP(f\circ f)-\nu_s(Pf\circ Pf) \\
     & \le
    \frac{1}{\gamma}\nu_s(f\circ f)
    -\nu_s(Pf\circ Pf) \\
     & =
    \frac{1}{\gamma}\nu_s\!\left(f\circ f-\gamma Pf\circ Pf\right),
\end{align*}
where the inequality drops the nonpositive term
$-(1-\gamma)e_s^\top(f\circ f)/\gamma$. Since $P$ is stochastic,
$0\le Pf\le\|f\|_{\mathrm{span}}\cdot\bm{1}_S$. Hence, coordinatewise,
\[
    f(s)^2-(Pf)(s)^2
    =
    \bigl(f(s)-(Pf)(s)\bigr)
    \bigl(f(s)+(Pf)(s)\bigr)
    \le
    2\|f\|_{\mathrm{span}}[f(s)-(Pf)(s)]_+.
\]
Here $[\cdot]_+$ is applied coordinatewise. Consequently,
\begin{align*}
    f\circ f-\gamma Pf\circ Pf
     & =
    \bigl(f\circ f-Pf\circ Pf\bigr)
    +(1-\gamma)Pf\circ Pf \\
     & \le
    2\|f\|_{\mathrm{span}}[f-Pf]_+
    +(1-\gamma)\|f\|_{\mathrm{span}}^2\cdot\bm{1}_S.
\end{align*}
Combining this coordinatewise inequality with the preceding bound and
using $\nu_s\bm{1}_S=1$, we obtain
\begin{align*}
    \nu_s\Var_P(f)
    &\le
    \frac{1}{\gamma}\nu_s
    \left(f\circ f-\gamma Pf\circ Pf\right)
    \\
    &\le
    \frac{2}{\gamma}\|f\|_{\mathrm{span}}\,
    \nu_s[f-Pf]_+
    +\frac{1-\gamma}{\gamma}\|f\|_{\mathrm{span}}^2
    \nu_s\bm{1}_S
    \\
    &=
    \frac{2}{\gamma}\|f\|_{\mathrm{span}}\,
    \nu_s[f-Pf]_+
    +\frac{1-\gamma}{\gamma}\|f\|_{\mathrm{span}}^2.
\end{align*}
This is \eqref{eq:discounted-variance-drift}.

\paragraph*{Step 2: Bounding $\nu_s\Var_P(\bar h)$.}
Let
\[
    g
    \coloneqq
    \bar\rho+\bar h-r^\pi-P^{0,\pi}\bar h.
\]
By \mainNominalAnchorSupersolution{}, $g\ge0$. We also have
$\Var_P(\bar h)=\Var_P(-\bar h)$ and
\begin{equation}
    \label{eq:anchor-drift-decomposition}
    P\bar h-\bar h
    =
    \bar\rho-r^\pi-g+(P-P^{0,\pi})\bar h.
\end{equation}
This identity, the bounds in
\eqref{eq:anchored-resolvent-residual-envelopes}, and the inequalities
$r^\pi\ge0$ and $g\ge0$ imply
\[
    \nu_s[P\bar h-\bar h]_+
    \le
    \bar\rho+C\sigma H_{\mathrm{anc}}+\nu_s|\xi_h|.
\]
Equations \eqref{eq:generic-anchor-defect} and
\eqref{eq:anchored-resolvent-value-premises} imply
$\bar\rho\le\beta+1$.
By \eqref{eq:discounted-state-distribution} and
\eqref{eq:normalized-discounted-resolvent},
$\nu_s|\xi_h|\le\|\mathcal G_P |\xi_h|\|_\infty$.
Applying \eqref{eq:discounted-variance-drift} with $f=-\bar h$ and
using $\gamma\ge1/2$ now yields
\begin{align}
    \nu_s\Var_P(\bar h)
    &\le
    \frac{1-\gamma}{\gamma}H_{\mathrm{anc}}^2
    +\frac{2}{\gamma}H_{\mathrm{anc}}\,
    \nu_s[P\bar h-\bar h]_+
    \notag\\
    &\le
    C\left(
    (1-\gamma)H_{\mathrm{anc}}^2
    +H_{\mathrm{anc}}\beta
    +H_{\mathrm{anc}}
    +\sigma H_{\mathrm{anc}}^2
    +H_{\mathrm{anc}}\|\mathcal G_P|\xi_h|\|_\infty
    \right)
    \notag\\
    &\le
    C\left(
    B(R,\beta)
    +R\|\mathcal G_P |\xi_h|\|_\infty
    \right).
    \label{eq:anchor-variance-average}
\end{align}
For the last inequality, we used $H_{\mathrm{anc}}\le R$ and matched
the first four terms with the four terms in $B(R,\beta)$ from
\eqref{eq:generic-localized-budget}.

\paragraph*{Step 3: Bounding $\nu_s\Var_P(V-\bar h)$ and concluding.}
Set $u\coloneqq V-\bar h$. Then
\[
    \|u\|_{\mathrm{span}}
    \le
    \|V\|_{\mathrm{span}}+\|\bar h\|_{\mathrm{span}}
    \le 2R.
\]
Moreover, combining $V=r^\pi+\gamma PV+b$ from
\eqref{eq:anchored-resolvent-value-premises} with
\eqref{eq:anchor-drift-decomposition} results in
\begin{align*}
    u-Pu
     & =
    V-PV-\bar h+P\bar h \\
     & =
    \bar\rho-(1-\gamma)\min_sV(s)
    -g
    -(1-\gamma)P\!\left(V-\min_sV\cdot \bm{1}_S\right)
    +b
    +(P-P^{0,\pi})\bar h.
\end{align*}
Dropping the nonpositive terms $-g$ and
$-(1-\gamma)P(V-\min_sV\cdot\bm{1}_S)$, the assumed bounds in
\eqref{eq:anchored-resolvent-residual-envelopes} and the definition of
$\beta$ in \eqref{eq:generic-anchor-defect} then imply
\[
    \nu_s[u-Pu]_+
    \le
    \beta
    +C\sigma H_{\mathrm{anc}}
    +\|\mathcal G_P |\xi_V|\|_\infty
    +\|\mathcal G_P |\xi_h|\|_\infty.
\]
A second application of \eqref{eq:discounted-variance-drift}, now with
$f=u$, together with $\gamma\ge1/2$ and
$\|u\|_{\mathrm{span}}\le2R$, yields
\begin{align*}
    \nu_s\Var_P(u)
    &\le
    C(1-\gamma)R^2
    +CR\,\nu_s[u-Pu]_+
    \\
    &\le
    C\left(
    (1-\gamma)R^2
    +R\beta
    +\sigma R H_{\mathrm{anc}}
    +R\|\mathcal G_P|\xi_V|\|_\infty
    +R\|\mathcal G_P|\xi_h|\|_\infty
    \right).
\end{align*}
Using $H_{\mathrm{anc}}\le R$ and the definition of $B(R,\beta)$ in
\eqref{eq:generic-localized-budget}, we obtain
\begin{equation}
    \label{eq:u-variance-average}
    \nu_s\Var_P(u)
    \le
    C\left(
    B(R,\beta)
    +R\|\mathcal G_P |\xi_V|\|_\infty
    +R\|\mathcal G_P |\xi_h|\|_\infty
    \right).
\end{equation}

Since
\[
    \Var_P(V)
    =
    \Var_P(\bar h+u)
    \le
    2\Var_P(\bar h)+2\Var_P(u),
\]
\eqref{eq:anchor-variance-average} and
\eqref{eq:u-variance-average} imply
\[
    \nu_s\Var_P(V)
    \le
    C\left(
    B(R,\beta)
    +R\|\mathcal G_P |\xi_V|\|_\infty
    +R\|\mathcal G_P |\xi_h|\|_\infty
    \right).
\]
Finally, Jensen's inequality for the probability measure $\nu_s$ yields
\[
    e_s^\top(I-\gamma P)^{-1}\sqrt{\Var_P(V)}
    =
    \frac{1}{1-\gamma}\nu_s\sqrt{\Var_P(V)}
    \le
    \frac{1}{1-\gamma}\sqrt{\nu_s\Var_P(V)}.
\]
Since $s$ was arbitrary, this proves the lemma.

\subsubsection{Proof of Lemma~\ref{lem:T2-anchored-variance-bound}}
\label{subsec:proof-T2-anchored-variance-bound}

Recall the fixed-policy aliases from
Appendix~\ref{subsec:discounted-preliminaries}:
\[
    U\coloneqq V_\gamma^{\star,\sigma}
    =V_\gamma^{\pi_\gamma^\star,\sigma},
    \qquad
    \widehat U\coloneqq
    \widehat V_\gamma^{\pi_\gamma^\star,\sigma}.
\]
The proof has three steps. We first apply
Lemma~\ref{lem:span-aware-resolvent-variance} to reduce the bound on $T_2$ to
two fixed residuals. We then bound those residuals. Finally, we substitute the
residual bound into the reduction from the first step.

\paragraph*{Step 1: Reduction of $T_2$ to fixed residuals.}
In this step, we verify the premises of
Lemma~\ref{lem:span-aware-resolvent-variance} and apply it to isolate the two
fixed residuals that remain to be bounded.
For this application, set
\[
    \pi=\pi_\gamma^\star,
    \qquad
    V=U,
    \qquad
    P=\widehat P^{\pi_\gamma^\star,U},
    \qquad
    r^\pi=r^{\pi_\gamma^\star}.
\]
The span premise in \eqref{eq:anchored-resolvent-value-premises} follows from
\eqref{eq:discounted-localized-span-condition}, and the associated anchor
defect is $\beta=\beta_\star$ by \eqref{eq:true-anchor-defect}. Define the
value and anchor sampling residuals by
\[
    \xi_V
    =
    \left(
    \widehat P^{\pi_\gamma^\star,U}
    -P^{\pi_\gamma^\star,U}
    \right)
    U,
    \qquad
    \xi_h
    =
    \left(
    \widehat P^{0,\pi_\gamma^\star}
    -P^{0,\pi_\gamma^\star}
    \right)\bar h.
\]
The true Bellman equation can then be written as
\[
    U
    =
    r^{\pi_\gamma^\star}
    +\gamma \widehat P^{\pi_\gamma^\star,U} U
    -\gamma\xi_V.
\]
Since
$\widehat P^{\pi_\gamma^\star,U}$ belongs to the robust uncertainty set
centered at $\widehat P^{0,\pi_\gamma^\star}$, the definition of that set and
$\|\bar h\|_{\mathrm{span}}\le H_{\mathrm{anc}}$ imply
\[
    \left|
    \left(
    \widehat P^{\pi_\gamma^\star,U}
    -P^{0,\pi_\gamma^\star}
    \right)\bar h
    \right|
    \le
    C\sigma H_{\mathrm{anc}}\cdot\bm{1}_S
    +|\xi_h|.
\]
Because $\gamma\le1$, $|b|\le|\xi_V|$ for $b=-\gamma\xi_V$.
Applying Lemma~\ref{lem:span-aware-resolvent-variance} to the definition of
$T_2$ therefore yields
\begin{equation}
    \label{eq:T2-lemma18-intermediate}
    \|T_2\|_\infty
    \le
    C\sqrt{\frac{\iota}{N(1-\gamma)^2}}
    \sqrt{
        B_\star(R)
        +R\left\|
        \mathcal G_{\widehat P^{\pi_\gamma^\star,U}}|\xi_V|
        \right\|_\infty
        +R\left\|
        \mathcal G_{\widehat P^{\pi_\gamma^\star,U}}|\xi_h|
        \right\|_\infty
    }.
\end{equation}

\paragraph*{Step 2: Bounding the fixed residuals.}
In this step, we bound the two residuals
isolated in Step~1.
Lemma~\ref{lem:one-step-robust-concentration} applied to the fixed
value $U$ establishes
\begin{equation}
    \label{eq:T2-xiV-concentration}
    \left\|
    \mathcal G_{\widehat P^{\pi_\gamma^\star,U}} |\xi_V|
    \right\|_\infty
    \le
    C\sqrt{\frac{\iota}{N}}
    \left\|
    \mathcal G_{\widehat P^{\pi_\gamma^\star,U}}
    \sqrt{\Var_{P^{0,\pi_\gamma^\star}}
    (U)}
    \right\|_\infty
    +C\frac{R\cdot\iota}{N},
\end{equation}
where we used
$\|U\|_{\mathrm{span}}\le R$ from
\eqref{eq:discounted-localized-span-condition}.
Separately, Lemma~\ref{lem:nominal-fixed-vector-concentration} applied to the
fixed anchor vector $\bar h$ provides
\begin{equation}
    \label{eq:T2-xih-concentration}
    \left\|
    \mathcal G_{\widehat P^{\pi_\gamma^\star,U}} |\xi_h|
    \right\|_\infty
    \le
    C\sqrt{\frac{\iota}{N}}
    \left\|
    \mathcal G_{\widehat P^{\pi_\gamma^\star,U}}
    \sqrt{\Var_{P^{0,\pi_\gamma^\star}}(\bar h)}
    \right\|_\infty
    +C\frac{R\cdot\iota}{N},
\end{equation}
where \eqref{eq:anchor-span-scale} and $R\ge H_{\mathrm{anc}}$ give
$\|\bar h\|_{\mathrm{span}}\le H_{\mathrm{anc}}\le R$.
We first transfer the two nominal standard deviations to the empirical robust
kernel.
For each state, the robust-set constraint implies
\begin{align*}
    \left\|
    \widehat P^{\pi_\gamma^\star,U}(s,\cdot)
    -\widehat P^{0,\pi_\gamma^\star}(s,\cdot)
    \right\|_{\mathrm{TV}}
     & \le
    \sum_{a\in\cA}\pi_\gamma^\star(a\mid s)
    \left\|
    \widehat P_{s,a}^{U}-\widehat P^0_{s,a}
    \right\|_{\mathrm{TV}}
    \le \sigma,
\end{align*}
where the first inequality follows from the convexity of total variation.
The first transfer compares the empirical nominal kernel with the empirical
robust kernel. For $f\in\{U,\bar h\}$,
Lemma~\ref{lem:standard-deviation-TV-perturbation} and
$\|f\|_{\mathrm{span}}\le R$ imply
\[
    \sqrt{\Var_{\widehat P^{0,\pi_\gamma^\star}}(f)}
    \le
    \sqrt{\Var_{\widehat P^{\pi_\gamma^\star,U}}(f)}
    +CR\sqrt{\sigma}\cdot\bm{1}_S.
\]
The second transfer compares the true and empirical nominal kernels. For the
same two fixed vectors, Lemma~\ref{lem:empirical-std-perturbation} provides
\[
    \sqrt{\Var_{P^{0,\pi_\gamma^\star}}(f)}
    \le
    \sqrt{\Var_{\widehat P^{0,\pi_\gamma^\star}}(f)}
    +CR\sqrt{\frac{\iota}{N}}\cdot\bm{1}_S.
\]
Combining the two transfers results in
\[
    \sqrt{\Var_{P^{0,\pi_\gamma^\star}}(f)}
    \le
    \sqrt{\Var_{\widehat P^{\pi_\gamma^\star,U}}(f)}
    +CR\left(
    \sqrt{\sigma}+\sqrt{\frac{\iota}{N}}
    \right)\cdot\bm{1}_S,
    \qquad
    f\in\{U,\bar h\}.
\]
Because $\mathcal G_{\widehat P^{\pi_\gamma^\star,U}}$ is nonnegative and
maps $\bm{1}_S$ to itself, summing the two resulting resolvent bounds leads to
\begin{align}
    &\left\|
    \mathcal G_{\widehat P^{\pi_\gamma^\star,U}}
    \sqrt{\Var_{P^{0,\pi_\gamma^\star}}(U)}
    \right\|_\infty
    +
    \left\|
    \mathcal G_{\widehat P^{\pi_\gamma^\star,U}}
    \sqrt{\Var_{P^{0,\pi_\gamma^\star}}(\bar h)}
    \right\|_\infty
    \nonumber\\
    &\quad\le
    \left\|
    \mathcal G_{\widehat P^{\pi_\gamma^\star,U}}
    \sqrt{\Var_{\widehat P^{\pi_\gamma^\star,U}}
    (U)}
    \right\|_\infty
    +
    \left\|
    \mathcal G_{\widehat P^{\pi_\gamma^\star,U}}
    \sqrt{\Var_{\widehat P^{\pi_\gamma^\star,U}}(\bar h)}
    \right\|_\infty
    +CR\left(
    \sqrt{\sigma}+\sqrt{\frac{\iota}{N}}
    \right).
    \label{eq:T2-nominal-variance-transfer}
\end{align}

Lemma~\ref{lem:span-aware-resolvent-variance} bounds the first empirical
robust-kernel term on the right-hand side of
\eqref{eq:T2-nominal-variance-transfer}. For the second term, applying Jensen's
inequality to \eqref{eq:anchor-variance-average}, exactly as in the final step
of the proof of Lemma~\ref{lem:span-aware-resolvent-variance}, yields
\begin{equation}
    \label{eq:T2-anchor-variance-bound}
    \left\|
    \mathcal G_{\widehat P^{\pi_\gamma^\star,U}}
    \sqrt{\Var_{\widehat P^{\pi_\gamma^\star,U}}(\bar h)}
    \right\|_\infty
    \le
    C\sqrt{
    B_\star(R)
    +R\left\|
    \mathcal G_{\widehat P^{\pi_\gamma^\star,U}}|\xi_h|
    \right\|_\infty}.
\end{equation}
Consequently, the sum of the first two terms on the right-hand side of
\eqref{eq:T2-nominal-variance-transfer} is at most
\[
    C\sqrt{
    B_\star(R)
    +R\left\|
        \mathcal G_{\widehat P^{\pi_\gamma^\star,U}}|\xi_V|
    \right\|_\infty
    +R\left\|
    \mathcal G_{\widehat P^{\pi_\gamma^\star,U}}|\xi_h|
    \right\|_\infty}.
\]
Moreover,
\[
    R\sqrt{\sigma}
    \le \sqrt{B_\star(R)},
    \qquad
    R\sqrt{\frac{\iota}{N}}
    \le R\sqrt{1-\gamma}
    \le \sqrt{B_\star(R)},
\]
where the second inequality uses the first condition in
\eqref{eq:recursive-comparison-sample-conditions}. Substituting these estimates into
\eqref{eq:T2-nominal-variance-transfer}, we obtain
\begin{align}
    &\left\|
    \mathcal G_{\widehat P^{\pi_\gamma^\star,U}}
    \sqrt{\Var_{P^{0,\pi_\gamma^\star}}(U)}
    \right\|_\infty
    +
    \left\|
    \mathcal G_{\widehat P^{\pi_\gamma^\star,U}}
    \sqrt{\Var_{P^{0,\pi_\gamma^\star}}(\bar h)}
    \right\|_\infty
    \nonumber\\
    &\quad\le
    C\sqrt{
    B_\star(R)
    +R\left\|
        \mathcal G_{\widehat P^{\pi_\gamma^\star,U}}|\xi_V|
    \right\|_\infty
    +R\left\|
    \mathcal G_{\widehat P^{\pi_\gamma^\star,U}}|\xi_h|
    \right\|_\infty}.
    \label{eq:T2-nominal-variance-closure}
\end{align}

Adding \eqref{eq:T2-xiV-concentration} and
\eqref{eq:T2-xih-concentration} and then applying
\eqref{eq:T2-nominal-variance-closure} results in
\begin{align*}
    &\left\|
        \mathcal G_{\widehat P^{\pi_\gamma^\star,U}}|\xi_V|
    \right\|_\infty
    +
    \left\|
    \mathcal G_{\widehat P^{\pi_\gamma^\star,U}}|\xi_h|
    \right\|_\infty
    \\
    &\quad\le
    C\sqrt{\frac{\iota}{N}}
    \sqrt{
    B_\star(R)
    +R\left\|
        \mathcal G_{\widehat P^{\pi_\gamma^\star,U}}|\xi_V|
    \right\|_\infty
    +R\left\|
    \mathcal G_{\widehat P^{\pi_\gamma^\star,U}}|\xi_h|
    \right\|_\infty}
    +C\frac{R\cdot\iota}{N}.
\end{align*}
Using $\sqrt{x+y}\le\sqrt{x}+\sqrt{y}$ followed by Young's inequality,
\[
    \begin{aligned}
    C\sqrt{\frac{R\cdot\iota}{N}}
    \sqrt{
    \left\|
        \mathcal G_{\widehat P^{\pi_\gamma^\star,U}}|\xi_V|
    \right\|_\infty
    +
    \left\|
    \mathcal G_{\widehat P^{\pi_\gamma^\star,U}}|\xi_h|
    \right\|_\infty}
    &\le
    \frac12
    \left[
    \left\|
        \mathcal G_{\widehat P^{\pi_\gamma^\star,U}}|\xi_V|
    \right\|_\infty
    +
    \left\|
    \mathcal G_{\widehat P^{\pi_\gamma^\star,U}}|\xi_h|
    \right\|_\infty
    \right]
    \\
    &\quad
    +C\frac{R\cdot\iota}{N}.
    \end{aligned}
\]
Combining the preceding two displays leads to
\begin{align*}
    &\left\|
        \mathcal G_{\widehat P^{\pi_\gamma^\star,U}}|\xi_V|
    \right\|_\infty
    +
    \left\|
    \mathcal G_{\widehat P^{\pi_\gamma^\star,U}}|\xi_h|
    \right\|_\infty
    \\
    &\quad\le
    C\sqrt{\frac{B_\star(R)\cdot\iota}{N}}
    +\frac12
    \left[
    \left\|
        \mathcal G_{\widehat P^{\pi_\gamma^\star,U}}|\xi_V|
    \right\|_\infty
    +
    \left\|
    \mathcal G_{\widehat P^{\pi_\gamma^\star,U}}|\xi_h|
    \right\|_\infty
    \right]
    +C\frac{R\cdot\iota}{N}.
\end{align*}
Subtracting one half of the residual sum from both sides and adjusting the
universal constant $C$ gives
\begin{equation}
    \label{eq:T2-fixed-residual-bound}
    \left\|
        \mathcal G_{\widehat P^{\pi_\gamma^\star,U}}|\xi_V|
    \right\|_\infty
    +
    \left\|
    \mathcal G_{\widehat P^{\pi_\gamma^\star,U}}|\xi_h|
    \right\|_\infty
    \le
    C\sqrt{\frac{B_\star(R)\cdot\iota}{N}}
    +C\frac{R\cdot\iota}{N}.
\end{equation}
This is the second checkpoint: the two fixed residuals in Step~1 are
now bounded explicitly.

\paragraph*{Step 3: Final substitution.}
In this step, we first substitute \eqref{eq:T2-fixed-residual-bound}
into \eqref{eq:T2-lemma18-intermediate} and then simplify the resulting
expression. The substitution produces
\begin{equation}
    \label{eq:T2-after-residual-substitution}
    \|T_2\|_\infty
    \le
    C\sqrt{\frac{\iota}{N(1-\gamma)^2}}
    \sqrt{
    B_\star(R)
    +R\sqrt{\frac{B_\star(R)\cdot\iota}{N}}
    +\frac{R^2\cdot\iota}{N}}.
\end{equation}
Young's inequality bounds the middle term in
\eqref{eq:T2-after-residual-substitution} as
\[
    R\sqrt{\frac{B_\star(R)\cdot\iota}{N}}
    \le
    C B_\star(R)+C\frac{R^2\cdot\iota}{N}.
\]
Using this bound in \eqref{eq:T2-after-residual-substitution} and then
applying $\sqrt{x+y}\le\sqrt{x}+\sqrt{y}$, we obtain
\[
    \|T_2\|_\infty
    \le
    C\sqrt{
    \frac{B_\star(R)\cdot\iota}{N(1-\gamma)^2}}
    +C\frac{R\cdot\iota}{N(1-\gamma)}.
\]

\subsubsection{Proof of Lemma~\ref{lem:T6-anchored-variance-bound}}
\label{subsec:proof-T6-anchored-variance-bound}

Recall the fixed-policy aliases
\[
    U\coloneqq V_\gamma^{\star,\sigma}
    =V_\gamma^{\pi_\gamma^\star,\sigma},
    \qquad
    \widehat U\coloneqq
    \widehat V_\gamma^{\pi_\gamma^\star,\sigma}.
\]
Define the optimal-policy component of \eqref{eq:Delta-defn} by
\[
    \Delta_U
    \coloneqq
    \left\|
    U
    -\widehat U
    \right\|_\infty.
\]
The proof has three steps. We first control the selector mismatch and apply
Lemma~\ref{lem:span-aware-resolvent-variance} to reduce the bound on $T_6$ to
two fixed residuals. We then bound those residuals.
Finally, we substitute the residual bound and absorb the term involving
$\Delta_U$.

\paragraph*{Step 1: Reduction of $T_6$ to fixed residuals.}
In this step, we decompose $T_6$ into a form suitable for
Lemma~\ref{lem:span-aware-resolvent-variance} and apply the lemma to isolate
the two fixed empirical residuals and the selector-mismatch cost.

Set
\[
    \pi=\pi_\gamma^\star,
    \qquad
    V=U,
    \qquad
    P=\widehat P^{\pi_\gamma^\star,\widehat U},
    \qquad
    r^\pi=r^{\pi_\gamma^\star}.
\]
The span premise follows from \eqref{eq:discounted-localized-span-condition}, and
the associated anchor defect is again $\beta_\star$. Rewriting the true
Bellman equation under the mismatched empirical kernel yields
\[
    U
    =
    r^{\pi_\gamma^\star}
    +\gamma\widehat P^{\pi_\gamma^\star,\widehat U}
    U
    +\gamma
    \left(
    P^{\pi_\gamma^\star,U}
    -\widehat P^{\pi_\gamma^\star,\widehat U}
    \right)U.
\]
Thus the Bellman residual is
\[
    b
    =
    \gamma
    \left(
    P^{\pi_\gamma^\star,U}
    -\widehat P^{\pi_\gamma^\star,\widehat U}
    \right)U.
\]
As in the proof of Lemma~\ref{lem:T2-anchored-variance-bound}, define the
fixed-$U$ sampling residual and the anchor sampling residual by
\[
    \xi_V
    =
    \left(
    \widehat P^{\pi_\gamma^\star,U}
    -P^{\pi_\gamma^\star,U}
    \right)U,
    \qquad
    \xi_h
    =
    \left(
    \widehat P^{0,\pi_\gamma^\star}
    -P^{0,\pi_\gamma^\star}
    \right)\bar h.
\]

The only new ingredient is that the empirical robust kernel in $b$ is
selected by $\widehat U$ rather than by
$U$. For each $(s,a)$, optimality of the two selectors
implies
\[
    \left(
    \widehat P_{s,a}^{\widehat U}-\widehat P_{s,a}^{U}
    \right)U
    \ge0,
    \qquad
    \left(
    \widehat P_{s,a}^{\widehat U}-\widehat P_{s,a}^{U}
    \right)\widehat U
    \le0.
\]
Consequently,
\[
    0
    \le
    \left(
    \widehat P_{s,a}^{\widehat U}-\widehat P_{s,a}^{U}
    \right)U
    \le
    \Delta_U
    \left\|
    \widehat P_{s,a}^{\widehat U}-\widehat P_{s,a}^{U}
    \right\|_1.
\]
The two selected transition distributions
$\widehat P_{s,a}^{U}$ and $\widehat P_{s,a}^{\widehat U}$ both belong to
$\cU_{s,a}(\widehat P^0)$. Therefore,
\[
    \left\|
    \widehat P_{s,a}^{\widehat U}-\widehat P_{s,a}^{U}
    \right\|_1
    \le
    \left\|
    \widehat P_{s,a}^{\widehat U}-\widehat P^0_{s,a}
    \right\|_1
    +
    \left\|
    \widehat P_{s,a}^{U}-\widehat P^0_{s,a}
    \right\|_1
    \le 4\sigma.
\]
Averaging over
$a\sim\pi_\gamma^\star(\cdot\mid s)$ therefore implies
\begin{equation}
    \label{eq:policy-induced-tv-diff}
    \left\|
    \left(
    \widehat P^{\pi_\gamma^\star,\widehat U}
    -\widehat P^{\pi_\gamma^\star,U}
    \right)U
    \right\|_\infty
    \le
    4\sigma\Delta_U.
\end{equation}
By the definitions of $b$ and $\xi_V$,
\[
    b
    =
    -\gamma\xi_V
    +\gamma
    \left(
    \widehat P^{\pi_\gamma^\star,U}
    -\widehat P^{\pi_\gamma^\star,\widehat U}
    \right)U.
\]
Because the normalized resolvent
$\mathcal G_{\widehat P^{\pi_\gamma^\star,\widehat U}}$ is nonnegative, the
triangle inequality implies
\begin{align*}
    \mathcal G_{\widehat P^{\pi_\gamma^\star,\widehat U}}|b|
    \le{}&
    \gamma
    \mathcal G_{\widehat P^{\pi_\gamma^\star,\widehat U}}|\xi_V|
    +
    \gamma
    \mathcal G_{\widehat P^{\pi_\gamma^\star,\widehat U}}
    \left|
    \left(
    \widehat P^{\pi_\gamma^\star,U}
    -\widehat P^{\pi_\gamma^\star,\widehat U}
    \right)U
    \right|
\end{align*}
componentwise. Applying \eqref{eq:policy-induced-tv-diff}, using
$\mathcal G_{\widehat P^{\pi_\gamma^\star,\widehat U}}\bm{1}_S=\bm{1}_S$,
and then using $\gamma\le1$ yields
\begin{equation}
    \label{eq:T6-value-residual-decomposition}
    \left\|
    \mathcal G_{\widehat P^{\pi_\gamma^\star,\widehat U}}|b|
    \right\|_\infty
    \le
    \left\|
    \mathcal G_{\widehat P^{\pi_\gamma^\star,\widehat U}}|\xi_V|
    \right\|_\infty
    +C\sigma\Delta_U.
\end{equation}
For the anchor residual, decompose
\[
    \left(
    \widehat P^{\pi_\gamma^\star,\widehat U}
    -P^{0,\pi_\gamma^\star}
    \right)\bar h
    =
    \left(
    \widehat P^{\pi_\gamma^\star,\widehat U}
    -\widehat P^{0,\pi_\gamma^\star}
    \right)\bar h
    +\xi_h.
\]
The first term is the robust perturbation from the empirical nominal kernel.
The robust-set constraint and
$\|\bar h\|_{\mathrm{span}}\le H_{\mathrm{anc}}$ bound its absolute value by
$C\sigma H_{\mathrm{anc}}\cdot\bm{1}_S$. Therefore,
\[
    \left|
    \left(
    \widehat P^{\pi_\gamma^\star,\widehat U}
    -P^{0,\pi_\gamma^\star}
    \right)\bar h
    \right|
    \le
    C\sigma H_{\mathrm{anc}}\cdot\bm{1}_S
    +|\xi_h|.
\]
Applying Lemma~\ref{lem:span-aware-resolvent-variance} with $b$ itself as the
value residual envelope, with $\xi_h$ as defined above, and then using
\eqref{eq:T6-value-residual-decomposition}, we obtain
\begin{equation}
    \label{eq:T6-main-resolvent-reduction}
    \|T_6\|_\infty
    \le
    C\sqrt{\frac{\iota}{N(1-\gamma)^2}}
    \sqrt{
    B_\star(R)
    +R\left\|
    \mathcal G_{\widehat P^{\pi_\gamma^\star,\widehat U}}|\xi_V|
    \right\|_\infty
    +R\left\|
    \mathcal G_{\widehat P^{\pi_\gamma^\star,\widehat U}}|\xi_h|
    \right\|_\infty
    +\sigma R\Delta_U
    }.
\end{equation}
This is the first checkpoint: the direct analysis of $T_6$ is complete, and it
remains to control the same two fixed residuals as in the proof of
Lemma~\ref{lem:T2-anchored-variance-bound}, now under the mismatched
resolvent.

\paragraph*{Step 2: Bounding the fixed residuals.}
In this step, we obtain an explicit bound for the two residuals isolated in Step~1.
We claim that
\begin{equation}
    \label{eq:T6-fixed-residual-bound}
    \left\|
    \mathcal G_{\widehat P^{\pi_\gamma^\star,\widehat U}}|\xi_V|
    \right\|_\infty
    +
    \left\|
    \mathcal G_{\widehat P^{\pi_\gamma^\star,\widehat U}}|\xi_h|
    \right\|_\infty
    \le
    C\sqrt{
    \frac{
    \left(
    B_\star(R)+\sigma R\Delta_U
    \right)\cdot\iota}{N}}
    +C\frac{R\cdot\iota}{N}.
\end{equation}

The proof of \eqref{eq:T6-fixed-residual-bound} follows Step~2 of the proof of
Lemma~\ref{lem:T2-anchored-variance-bound}, with the additional term
$\sigma R\Delta_U$ arising from the selector-mismatch term
$C\sigma\Delta_U$ in \eqref{eq:T6-value-residual-decomposition}. We omit the
details for brevity.

\paragraph*{Step 3: Final substitution and absorption.}
In this step, we substitute the fixed-residual bound into the reduction from
Step~1 and bound the $\Delta_U$-related term.
To substitute \eqref{eq:T6-fixed-residual-bound} into the main reduction, first
note that Young's inequality shows that
\[
    R\sqrt{
    \frac{
    \left(
    B_\star(R)+\sigma R\Delta_U
    \right)\cdot\iota}{N}}
    \le
    C\left(
    B_\star(R)+\sigma R\Delta_U
    \right)
    +C\frac{R^2\cdot\iota}{N}.
\]
Substituting this inequality and \eqref{eq:T6-fixed-residual-bound} into
\eqref{eq:T6-main-resolvent-reduction} produces
\begin{equation}
    \label{eq:T6-bound-pre-young}
    \|T_6\|_\infty
    \le
    C\sqrt{
    \frac{B_\star(R)\cdot\iota}
    {N(1-\gamma)^2}}
    +C\frac{R\cdot\iota}{N(1-\gamma)}
    +C\sqrt{
    \frac{\sigma R\Delta_U\cdot\iota}
    {N(1-\gamma)^2}}.
\end{equation}
By Young's inequality and the second
condition in \eqref{eq:recursive-comparison-sample-conditions},
\[
    C\sqrt{
    \frac{\sigma R\Delta_U\cdot\iota}
    {N(1-\gamma)^2}}
    \le
    \frac{1}{50}\Delta_U
    +C\frac{\sigma R\cdot\iota}{N(1-\gamma)^2}
    \le
    \frac{1}{50}\Delta_U
    +C\sqrt{
    \frac{B_\star(R)\cdot\iota}
    {N(1-\gamma)^2}}.
\]
Since $\Delta_U\le\Delta$, the last display and
\eqref{eq:T6-bound-pre-young} imply the desired bound
\eqref{eq:T6-bound-young}.

\subsubsection{Proof of Lemma~\ref{lem:shi-original-lemma-11}}
\label{sec:proof-of-shi-original-lemma-11}

The proof has four steps. We first reduce the desired bound to a centered
empirical process involving $\widehat V^\star$. We then construct
leave-one-out value functions and discretize their auxiliary reward parameter.
Next, we establish uniform concentration and transfer it back to
$\widehat V^\star$. Finally, we average the resulting rowwise bounds over the
learned policy.

\paragraph*{Step 1: Reduction to a centered empirical-optimal-value process.}
In this step, we reduce the rowwise robust-kernel error at $\widehat W$ to a
centered empirical process involving $\widehat V^\star$.
Fix a state-action pair $(s,a)$ and a value vector $V$, and write
$V^\circ\coloneqq V-\min_{s'\in\cS}V(s')\cdot \bm{1}_S$. Then
$0\le V^\circ\le \|V\|_{\mathrm{span}}\cdot \bm{1}_S$. The TV
strong-duality formula in \citet[Lemma~4]{shi2023curious} yields
\begin{equation}
    \label{eq:centered-tv-duality-empirical-planning}
    \left|
    \widehat P_{s,a}^V V-P_{s,a}^V V
    \right|
    \le
    \sup_{\alpha\in[0,\|V\|_{\mathrm{span}}]}
    \left|
    (\widehat P^0_{s,a}-P^0_{s,a})[V^\circ]_\alpha
    \right|,
\end{equation}
where $[z]_\alpha\coloneqq(\min\{z(s'),\alpha\})_{s'\in\cS}$. Applying
\eqref{eq:centered-tv-duality-empirical-planning} with
$V=\widehat W$ and using $\|\widehat W\|_{\mathrm{span}}\le R$, we obtain
\[
    \left|
    \widehat P_{s,a}^{\widehat W}\widehat W
    -P_{s,a}^{\widehat W}\widehat W
    \right|
    \le
    \sup_{\alpha\in[0,R]}
    \left|
    (\widehat P^0_{s,a}-P^0_{s,a})
    [\widehat W^\circ]_\alpha
    \right|.
\]
On the other hand, the optimality of $\widehat V^\star$ and the
solver guarantee \mainRobustDMDPSolverTolerance{} imply
\[
    0
    \le
    \widehat V^\star-\widehat W
    \le
    \varepsilon_{\mathrm{opt}}\cdot\bm{1}_S.
\]
Since
\begin{align*}
    \left\|
    \widehat W^\circ-(\widehat V^\star)^\circ
    \right\|_\infty
    &\le
    \|\widehat W-\widehat V^\star\|_\infty
    +
    \left|
    \min_s\widehat W(s)-\min_s\widehat V^\star(s)
    \right| \\
    &\le
    2\|\widehat W-\widehat V^\star\|_\infty,
\end{align*}
we obtain
\[
    \left\|
    \widehat W^\circ-(\widehat V^\star)^\circ
    \right\|_\infty
    \le
    2\varepsilon_{\mathrm{opt}}.
\]
Since $z\mapsto[z]_\alpha$ is 1-Lipschitz and
$\|\widehat P^0_{s,a}-P^0_{s,a}\|_1\le2$, combining the preceding displays
results in
\begin{equation}
    \label{eq:reduce-V-to-Vstar-centered}
    \left|
    \widehat P_{s,a}^{\widehat W}\widehat W
    -P_{s,a}^{\widehat W}\widehat W
    \right|
    \le
    \sup_{\alpha\in[0,R]}
    \left|
    (\widehat P^0_{s,a}-P^0_{s,a})
    [(\widehat V^\star)^\circ]_\alpha
    \right|
    +
    C\varepsilon_{\mathrm{opt}}.
\end{equation}
Therefore, up to the additive $C\varepsilon_{\mathrm{opt}}$ term, our goal
reduces to controlling the centered empirical process given by the supremum
term in \eqref{eq:reduce-V-to-Vstar-centered} uniformly over all $(s,a)$.

\paragraph*{Step 2: Leave-one-out construction and discretization.}
The empirical process isolated in Step~1 is data-dependent, so in this step we
construct a discretized family of leave-one-out value functions that is
independent of each empirical transition row and approximates
$\widehat V^\star$.
We adapt the auxiliary-MDP leave-one-out construction of
\citet[Appendix~B.3.5]{shi2023curious} to the centered, span-localized
empirical process in \eqref{eq:reduce-V-to-Vstar-centered}. For each source
state $s$ and scalar $u\ge0$, let $\widehat{\mathcal M}^{s,u}$ have the same
uncertainty-set rule and discount factor as the empirical robust MDP, but with
nominal kernel and reward
\[
    \widehat P^{0,s,u}_{\widetilde s,a}
    \coloneqq
    \begin{cases}
        e_s, & \widetilde s=s, \\
        \widehat P^0_{\widetilde s,a}, & \widetilde s\ne s,
    \end{cases}
    \qquad
    r^{s,u}(\widetilde s,a)
    \coloneqq
    \begin{cases}
        u, & \widetilde s=s, \\
        r(\widetilde s,a), & \widetilde s\ne s,
    \end{cases}
\]
for every $(\widetilde s,a)\in\cS\times\cA$, where $e_s$ is the $s$-th
standard basis probability vector.
In words, the construction makes $s$ absorbing in the auxiliary nominal model
and assigns reward $u$ to every action at $s$, while leaving all other nominal
transition rows and rewards unchanged.
Let $\widehat V_{s,u}^\star$ be the optimal robust value of this auxiliary empirical
MDP\@. Because the nominal centers of the rows out of state $s$ are fixed at
$e_s$, $\widehat V_{s,u}^\star$ is independent of the samples used to form
$\widehat P^0_{s,a}$, for every action $a$ at that source state.

Define
\begin{equation*}
    u_\star
    \coloneqq
    \widehat V^\star(s)
    -\gamma\max_{b\in\cA}
    \inf_{Q\in\cU_{s,b}(e_s)}Q\widehat V^\star.
\end{equation*}
Since $e_s\in\cU_{s,b}(e_s)$ for every $b\in\cA$ and
\eqref{eq:discounted-value-range} applies to $\widehat V^\star$, we have
$0\le u_\star\le(1-\gamma)^{-1}$. With this choice, $\widehat V^\star$ is a
fixed point of the auxiliary robust Bellman operator: its update at state $s$
equals $\widehat V^\star(s)$, while all other states have the same updates as
in the original empirical robust MDP\@. Since this operator is a
$\gamma$-contraction, its fixed point is unique, and hence
\begin{equation*}
    \widehat V_{s,u_\star}^\star=\widehat V^\star.
\end{equation*}
Let $\mathcal N_u$ be a uniform grid of $[0,(1-\gamma)^{-1}]$ with spacing at most
\[
    \eta_u
    \coloneqq
    \frac{(1-\gamma)R\cdot\iota}{C_1N}
\]
for a sufficiently large numerical constant $C_1$. Choose
$u\in\mathcal N_u$ with $|u-u_\star|\le\eta_u$. The auxiliary empirical
robust MDPs $\widehat{\mathcal M}^{s,u}$ and
$\widehat{\mathcal M}^{s,u_\star}$ differ only in their rewards at state
$s$.

Consequently, for any value vector, their robust Bellman updates agree
at every state other than $s$, while the updates at $s$ differ by
$u-u_\star$. Thus, the two update vectors differ in sup-norm by
$|u-u_\star|\le\eta_u$. Since both robust Bellman operators are
$\gamma$-contractions, their fixed points satisfy
\begin{equation}
    \label{eq:loo-value-close-span-centered}
    \|\widehat V_{s,u}^\star-\widehat V^\star\|_\infty
    \le
    \frac{\eta_u}{1-\gamma}
    \le
    \frac{R\cdot\iota}{C_1N}.
\end{equation}
Using \eqref{eq:loo-value-close-span-centered}, we obtain
\begin{equation}
    \label{eq:centered-loo-close}
    \begin{aligned}
        \left\|
        (\widehat V_{s,u}^\star)^\circ-(\widehat V^\star)^\circ
        \right\|_\infty
        &\le
        \left\|\widehat V_{s,u}^\star-\widehat V^\star\right\|_\infty
        +
        \left|
        \min_y\widehat V_{s,u}^\star(y)-\min_y\widehat V^\star(y)
        \right|
        \\
        &\le
        2\left\|\widehat V_{s,u}^\star-\widehat V^\star\right\|_\infty
        \le
        \frac{2R\cdot\iota}{C_1N}.
    \end{aligned}
\end{equation}
Moreover, since $\|\widehat V^\star\|_{\mathrm{span}}\le R$,
\begin{equation}
    \label{eq:centered-loo-range}
    \begin{aligned}
        0\le(\widehat V_{s,u}^\star)^\circ
        &\le
        \|\widehat V_{s,u}^\star\|_{\mathrm{span}}\cdot\bm{1}_S
        \\
        &\le
        \left(
        \|\widehat V^\star\|_{\mathrm{span}}
        +2\|\widehat V_{s,u}^\star-\widehat V^\star\|_\infty
        \right)\bm{1}_S
        \le
        2R\cdot \bm{1}_S.
    \end{aligned}
\end{equation}
The final inequality uses the sample-size condition and sufficiently large
numerical constants $C_0$ and $C_1$.

\paragraph*{Step 3: Uniform concentration and transfer.}
In this step, we apply Bernstein's inequality uniformly to the leave-one-out
family from Step~2 and transfer the resulting bound to the empirical process
from Step~1.
Fix a triple $(s,a,u)$ and condition on all samples other than those forming
$\widehat P^0_{s,a}$. By the leave-one-out construction,
$(\widehat V_{s,u}^\star)^\circ$ is then fixed and independent of
$\widehat P^0_{s,a}$. If
$\|(\widehat V_{s,u}^\star)^\circ\|_\infty=0$, set
$\mathcal N_\alpha^{s,u}=\{0\}$. Otherwise, let
$\mathcal N_\alpha^{s,u}$ be a uniform grid of
$[0,\|(\widehat V_{s,u}^\star)^\circ\|_\infty]$ with spacing at most
\[
    \frac{
    \|(\widehat V_{s,u}^\star)^\circ\|_\infty\cdot\iota
    }{C_2N},
\]
where $C_2$ is a sufficiently large numerical constant. Both definitions
satisfy $|\mathcal N_\alpha^{s,u}|\le1+C_2N/\iota$.

For each $\alpha\in\mathcal N_\alpha^{s,u}$, the clipped vector is fixed and
independent of $\widehat P^0_{s,a}$. Moreover,
\[
    \left\|[(\widehat V_{s,u}^\star)^\circ]_\alpha\right\|_\infty
    \le
    \|(\widehat V_{s,u}^\star)^\circ\|_\infty,
    \qquad
    \Var_{P^0_{s,a}}
    ([(\widehat V_{s,u}^\star)^\circ]_\alpha)
    \le
    \Var_{P^0_{s,a}}((\widehat V_{s,u}^\star)^\circ).
\]
Therefore, under this conditioning, Bernstein's inequality and a union bound
over $\mathcal N_\alpha^{s,u}$ imply that, with probability at least
$1-2|\mathcal N_\alpha^{s,u}|e^{-3\iota}$,
\[
    \sup_{\alpha\in\mathcal N_\alpha^{s,u}}
    \left|
    (\widehat P^0_{s,a}-P^0_{s,a})
    [(\widehat V_{s,u}^\star)^\circ]_\alpha
    \right|
    \le
    C\sqrt{\frac{\iota}{N}}
    \sqrt{\Var_{P^0_{s,a}}((\widehat V_{s,u}^\star)^\circ)}
    +
    C\frac{\|(\widehat V_{s,u}^\star)^\circ\|_\infty\cdot\iota}{N}.
\]

It remains to extend the bound from the grid to every clipping level. Given
$\alpha\in[0,\|(\widehat V_{s,u}^\star)^\circ\|_\infty]$, let
$\alpha'\in\mathcal N_\alpha^{s,u}$ be a nearest grid point. The grid spacing
and the 1-Lipschitz property of clipping imply
\[
    \left\|
    [(\widehat V_{s,u}^\star)^\circ]_\alpha
    -[(\widehat V_{s,u}^\star)^\circ]_{\alpha'}
    \right\|_\infty
    \le
    |\alpha-\alpha'|
    \le
    \frac{
    \|(\widehat V_{s,u}^\star)^\circ\|_\infty\cdot\iota
    }{C_2N}.
\]
Together with $\|\widehat P^0_{s,a}-P^0_{s,a}\|_1\le2$, this shows that
replacing $\alpha$ by $\alpha'$ changes the empirical process by at most
$C\|(\widehat V_{s,u}^\star)^\circ\|_\infty\cdot\iota/N$. Thus, on the
same conditional event, the following bound holds for this fixed triple
$(s,a,u)$:
\begin{equation}
    \label{eq:fixed-loo-centered-concentration}
    \sup_{\alpha\in[0,\|(\widehat V_{s,u}^\star)^\circ\|_\infty]}
    \left|
    (\widehat P^0_{s,a}-P^0_{s,a})
    [(\widehat V_{s,u}^\star)^\circ]_\alpha
    \right|
    \le
    C\sqrt{\frac{\iota}{N}}
    \sqrt{\Var_{P^0_{s,a}}((\widehat V_{s,u}^\star)^\circ)}
    +
    C\frac{\|(\widehat V_{s,u}^\star)^\circ\|_\infty\cdot\iota}{N}.
\end{equation}

Finally, unconditioning and taking a union bound over all $(s,a)$ and
$u\in\mathcal N_u$ show that
\eqref{eq:fixed-loo-centered-concentration} holds simultaneously for all
$(s,a,u)$ with probability at least
\[
    1
    -2SA(1+CN^3)
    \left(1+\frac{C_2N}{\iota}\right)e^{-3\iota}
    \ge
    1-C SAN^4e^{-3\iota}
    =1-O(\delta).
\]
Here the last equality follows from
\[
    C SAN^4e^{-3\iota}
    =
    C SAN^4
    \left(
    \frac{(1-\gamma)\delta}{54SAN^2}
    \right)^3
    =O(\delta).
\]
For the grid point $u\in\mathcal N_u$ satisfying
$|u-u_\star|\le\eta_u$,
\eqref{eq:centered-loo-range} ensures that
$\|(\widehat V_{s,u}^\star)^\circ\|_\infty\le2R$, so the bounded-difference
term in \eqref{eq:fixed-loo-centered-concentration} is $O(R\cdot\iota/N)$
rather than $O(\iota/[N(1-\gamma)])$.

Combining \eqref{eq:centered-loo-close} with the Lipschitz property of
clipping yields
\begin{align}
    \sup_{\alpha\in[0,R]}
    \left|
    (\widehat P^0_{s,a}-P^0_{s,a})
    [(\widehat V^\star)^\circ]_\alpha
    \right|
    &\le
    \sup_{\alpha\in[0,2R]}
    \left|
    (\widehat P^0_{s,a}-P^0_{s,a})
    [(\widehat V_{s,u}^\star)^\circ]_\alpha
    \right|
    +
    C\frac{R\cdot\iota}{N}.        \label{eq:transfer-loo-to-target}
\end{align}
The concentration event \eqref{eq:fixed-loo-centered-concentration} applies
to the supremum over
$[0,2R]$: if $2R>\|(\widehat V_{s,u}^\star)^\circ\|_\infty$, then clipping
above $\|(\widehat V_{s,u}^\star)^\circ\|_\infty$ leaves
$(\widehat V_{s,u}^\star)^\circ$ unchanged, so the supremum over $[0,2R]$ is
the same as the supremum over
$[0,\|(\widehat V_{s,u}^\star)^\circ\|_\infty]$.
Moreover, applying Minkowski's inequality in $L_2(P^0_{s,a})$ to the
decomposition of $(\widehat V_{s,u}^\star)^\circ$ around
$(\widehat V^\star)^\circ$ and then using
\eqref{eq:centered-loo-close}, we obtain
\begin{align}
    \label{eq:variance-transfer-loo}
    \sqrt{\Var_{P^0_{s,a}}((\widehat V_{s,u}^\star)^\circ)}
    &\le
    \sqrt{\Var_{P^0_{s,a}}((\widehat V^\star)^\circ)}
    +
    \sqrt{\Var_{P^0_{s,a}}\left(
    (\widehat V_{s,u}^\star)^\circ-(\widehat V^\star)^\circ
    \right)}
    \notag\\
    &\le
    \sqrt{\Var_{P^0_{s,a}}(\widehat V^\star)}
    +
    \left\|
    (\widehat V_{s,u}^\star)^\circ-(\widehat V^\star)^\circ
    \right\|_\infty
    \notag\\
    &\le
    \sqrt{\Var_{P^0_{s,a}}(\widehat V^\star)}
    +
    C\frac{R\cdot\iota}{N}.
\end{align}
Substituting \eqref{eq:fixed-loo-centered-concentration} and
\eqref{eq:variance-transfer-loo} into \eqref{eq:transfer-loo-to-target}
yields
\begin{equation}
    \label{eq:centered-optimal-empirical-process}
    \sup_{\alpha\in[0,R]}
    \left|
    (\widehat P^0_{s,a}-P^0_{s,a})
    [(\widehat V^\star)^\circ]_\alpha
    \right|
    \le
    C\sqrt{\frac{\iota}{N}}
    \sqrt{\Var_{P^0_{s,a}}(\widehat V^\star)}
    +
    C\frac{R\cdot\iota}{N}.
\end{equation}
Combining \eqref{eq:reduce-V-to-Vstar-centered} and
\eqref{eq:centered-optimal-empirical-process} establishes the rowwise bound
\[
    \left|
    \widehat P_{s,a}^{\widehat W}\widehat W
    -P_{s,a}^{\widehat W}\widehat W
    \right|
    \le
    C\sqrt{\frac{\iota}{N}}
    \sqrt{\Var_{P^0_{s,a}}(\widehat V^\star)}
    +
    C\frac{R\cdot\iota}{N}
    +
    C\varepsilon_{\mathrm{opt}}.
\]

\paragraph*{Step 4: Averaging over the learned policy.}
In this final step, we average the rowwise estimate from Step~3 over the
learned policy and bound the resulting average of the state-action conditional
standard deviations by the standard deviation under $P^{0,\widehat\pi}$.
Because the rowwise estimate holds simultaneously for every $(s,a)$, it may
be averaged using the data-dependent policy $\widehat\pi$.

For every state $s$, the definitions of the policy-induced robust kernels and
the triangle inequality imply
\begin{align*}
    &\left|
    \left(
    (\widehat P^{\widehat\pi,\widehat W}
    -P^{\widehat\pi,\widehat W})\widehat W
    \right)(s)
    \right|
    \\
    &\qquad=
    \left|
    \sum_{a\in\cA}\widehat\pi(a\mid s)
    \left(
    \widehat P_{s,a}^{\widehat W}\widehat W
    -P_{s,a}^{\widehat W}\widehat W
    \right)
    \right|
    \\
    &\qquad\le
    \sum_{a\in\cA}\widehat\pi(a\mid s)
    \left|
    \widehat P_{s,a}^{\widehat W}\widehat W
    -P_{s,a}^{\widehat W}\widehat W
    \right|
    \\
    &\qquad\le
    C\sqrt{\frac{\iota}{N}}
    \sum_{a\in\cA}\widehat\pi(a\mid s)
    \sqrt{\Var_{P^0_{s,a}}(\widehat V^\star)}
    +
    C\frac{R\cdot\iota}{N}
    +
    C\varepsilon_{\mathrm{opt}},
\end{align*}
where the last two terms remain unchanged because
$\sum_{a\in\cA}\widehat\pi(a\mid s)=1$.

By Jensen's inequality,
\[
    \sum_{a\in\cA}\widehat\pi(a\mid s)
    \sqrt{\Var_{P^0_{s,a}}(\widehat V^\star)}
    \le
    \sqrt{
    \sum_{a\in\cA}\widehat\pi(a\mid s)
    \Var_{P^0_{s,a}}(\widehat V^\star)
    }.
\]
Moreover, the law of total variance and the definition of
$P^{0,\widehat\pi}$ yield
\begin{align*}
    \Var_{P^{0,\widehat\pi}}(\widehat V^\star)(s)
    &={}
    \sum_{a\in\cA}\widehat\pi(a\mid s)
    \Var_{P^0_{s,a}}(\widehat V^\star)
    \\
    &\quad+
    \sum_{a\in\cA}\widehat\pi(a\mid s)
    \left(
    P^0_{s,a}\widehat V^\star
    -P^{0,\widehat\pi}\widehat V^\star(s)
    \right)^2
    \\
    &\ge
    \sum_{a\in\cA}\widehat\pi(a\mid s)
    \Var_{P^0_{s,a}}(\widehat V^\star).
\end{align*}
Combining the preceding two displays leads to
\[
    \sum_{a\in\cA}\widehat\pi(a\mid s)
    \sqrt{\Var_{P^0_{s,a}}(\widehat V^\star)}
    \le
    \sqrt{
    \Var_{P^{0,\widehat\pi}}(\widehat V^\star)(s)
    }.
\]
Substituting this inequality into the averaged rowwise estimate produces
\[
    \left|
    \left(
    (\widehat P^{\widehat\pi,\widehat W}
    -P^{\widehat\pi,\widehat W})\widehat W
    \right)(s)
    \right|
    \le
    C\sqrt{\frac{\iota}{N}}
    \sqrt{\Var_{P^{0,\widehat\pi}}(\widehat V^\star)(s)}
    +
    C\frac{R\cdot\iota}{N}
    +
    C\varepsilon_{\mathrm{opt}}.
\]
This is \eqref{eq:centered-empirical-planning-residual}, completing the proof.

\subsubsection{Proof of Lemma~\ref{lem:two-sided-pi-hat-value-decomposition}}
\label{subsec:Proof-lemma-two-sided-pi-hat-value-decomposition}

The proof of \eqref{eq:pi-hat-two-sided-resolvent-decomp} follows the proof of
Lemma~\ref{lem:two-sided-fixed-value-decomposition} after interchanging the
true and empirical models and making the substitutions
$U\mapsto\widehat W$, $\widehat U\mapsto W$, and
$\pi_\gamma^\star\mapsto\widehat\pi$. The order-preserving resolvent
rearrangement is deterministic, so the data dependence of $\widehat\pi$ is
immaterial. We omit the details for brevity.

\subsubsection{Proof of Lemma~\ref{lem:span-aware-shi-lemma-12}}
\label{sec:proof-of-span-aware-shi-lemma-12}
The proof has three steps. We first apply
Lemma~\ref{lem:span-aware-resolvent-variance} to reduce the bound on $T_2'$ to
a transition-kernel error. We then bound this error using
Lemma~\ref{lem:shi-original-lemma-11} and solve the resulting self-bounding
inequality. Finally, we substitute the resulting bound into the reduction
from the first step.

\paragraph*{Step 1: Reduction of $T_2'$ to a transition-kernel error.}
In this step, we verify the premises of
Lemma~\ref{lem:span-aware-resolvent-variance} and apply it to isolate the
transition-kernel error that remains to be bounded.
For this application, set
\[
    \pi=\widehat\pi,
    \qquad
    V=\widehat W,
    \qquad
    P=P^{\widehat\pi,\widehat W},
    \qquad
    r^\pi=r^{\widehat\pi}.
\]
Recall from \eqref{eq:normalized-discounted-resolvent} that
$\mathcal G_Q\coloneqq(1-\gamma)(I-\gamma Q)^{-1}$ for any stochastic
kernel $Q$.
Define
\begin{equation}
    \label{eq:matched-learned-policy-residual-definitions}
    \begin{aligned}
    \xi_V
    &\coloneqq
    \left(\widehat P^{\widehat\pi,\widehat W}
    -P^{\widehat\pi,\widehat W}\right)\widehat W,
    \\
    \zeta_{\widehat W}
    &\coloneqq
    \left\|
    \mathcal G_{P^{\widehat\pi,\widehat W}}|\xi_V|
    \right\|_\infty.
    \end{aligned}
\end{equation}
Rewriting the empirical Bellman equation under the true robust kernel
$P^{\widehat\pi,\widehat W}$ yields
\[
    \widehat W
    =
    r^{\widehat\pi}
    +\gamma P^{\widehat\pi,\widehat W}\widehat W
    +\gamma\xi_V.
\]
The span premise of Lemma~\ref{lem:span-aware-resolvent-variance} follows
from \eqref{eq:discounted-localized-span-condition}, while
$(1-\gamma)\min_s\widehat W(s)\le1$ follows from
\eqref{eq:discounted-value-range}. The associated anchor defect is
$\beta_{\mathrm{emp}}$ by \eqref{eq:empirical-anchor-defect}.

Moreover, convexity of total variation and the true robust-set constraint
imply, for every state $s$,
\[
    \left\|
    P^{\widehat\pi,\widehat W}(s,\cdot)
    -P^{0,\widehat\pi}(s,\cdot)
    \right\|_{\mathrm{TV}}
    \le
    \sum_{a\in\cA}\widehat\pi(a\mid s)
    \left\|P_{s,a}^{\widehat W}-P^0_{s,a}\right\|_{\mathrm{TV}}
    \le\sigma.
\]
Hence
\[
    \left|
    \left(P^{\widehat\pi,\widehat W}
    -P^{0,\widehat\pi}\right)\bar h
    \right|
    \le
    C\sigma H_{\mathrm{anc}}\cdot\bm{1}_S,
\]
so \eqref{eq:anchored-resolvent-residual-envelopes} holds with $\xi_h=0$.
Because $\gamma\le1$, $|b|\le|\xi_V|$ for $b=\gamma\xi_V$.
Applying Lemma~\ref{lem:span-aware-resolvent-variance} and using the definition
of $\zeta_{\widehat W}$, we obtain
\begin{equation}
    \label{eq:matched-learned-policy-resolvent-reduction}
    \left\|
    \left(I-\gamma P^{\widehat\pi,\widehat W}\right)^{-1}
    \sqrt{\Var_{P^{\widehat\pi,\widehat W}}(\widehat W)}
    \right\|_\infty
    \le
    \frac{C}{1-\gamma}
    \sqrt{
    B_{\mathrm{emp}}(R)+R\zeta_{\widehat W}
    }.
\end{equation}
Combining this bound with the definition of $T_2'$ in
\eqref{eq:T-prime-first-branch-definitions} yields
\begin{equation}
    \label{eq:T2-prime-lemma23-intermediate}
    \|T_2'\|_\infty
    \le
    \frac{C}{1-\gamma}
    \sqrt{\frac{\iota}{N}}
    \sqrt{B_{\mathrm{emp}}(R)+R\zeta_{\widehat W}}.
\end{equation}
Thus, it remains to control $\zeta_{\widehat W}$.

\paragraph*{Step 2: Bounding $\zeta_{\widehat W}$.}
In this step, we use Lemma~\ref{lem:shi-original-lemma-11} and the bound from
Step~1 to derive and solve a self-bounding inequality for
$\zeta_{\widehat W}$.
Optimality and the solver guarantee
\mainRobustDMDPSolverTolerance{} ensure that
$0\le\widehat V^\star-\widehat W
\le\varepsilon_{\mathrm{opt}}\cdot\bm{1}_S$. Hence, for every state $s$,
Minkowski's inequality yields
\begin{equation}
    \label{eq:planning-standard-deviation-transfer}
    \begin{aligned}
    \sqrt{\Var_{P^{0,\widehat\pi}}(\widehat V^\star)(s)}
    &\le
    \sqrt{\Var_{P^{0,\widehat\pi}}(\widehat W)(s)}
    +
    \sqrt{\Var_{P^{0,\widehat\pi}}
    (\widehat V^\star-\widehat W)(s)}
    \\
    &\le
    \sqrt{\Var_{P^{0,\widehat\pi}}(\widehat W)(s)}
    +\varepsilon_{\mathrm{opt}}.
    \end{aligned}
\end{equation}
Apply the nonnegative matrix
$\mathcal G_{P^{\widehat\pi,\widehat W}}$ to
\eqref{eq:centered-empirical-planning-residual}. Using the preceding display,
$\mathcal G_{P^{\widehat\pi,\widehat W}}\bm{1}_S=\bm{1}_S$,
and $N\ge\iota$, which follows from the first condition in
\eqref{eq:recursive-comparison-sample-conditions}, we obtain
\begin{equation}
    \label{eq:empirical-residual-reduction-proof}
    \zeta_{\widehat W}
    \le
    C\sqrt{\frac{\iota}{N}}
    \left\|
    \mathcal G_{P^{\widehat\pi,\widehat W}}
    \sqrt{\Var_{P^{0,\widehat\pi}}(\widehat W)}
    \right\|_\infty
    +C\frac{R\cdot\iota}{N}
    +C\varepsilon_{\mathrm{opt}}.
\end{equation}

Because $P^{\widehat\pi,\widehat W}$ lies in the true robust uncertainty
set, Lemma~\ref{lem:standard-deviation-TV-perturbation} and
\eqref{eq:discounted-localized-span-condition} imply
\begin{equation}
    \label{eq:learned-policy-nominal-variance-transfer}
    \sqrt{\Var_{P^{0,\widehat\pi}}(\widehat W)}
    \le
    \sqrt{\Var_{P^{\widehat\pi,\widehat W}}(\widehat W)}
    +CR\sqrt{\sigma}\cdot\bm{1}_S.
\end{equation}
This deterministic inequality applies even though $\widehat W$ is
data-dependent. Applying
$\mathcal G_{P^{\widehat\pi,\widehat W}}$ to
\eqref{eq:learned-policy-nominal-variance-transfer} and using
\eqref{eq:matched-learned-policy-resolvent-reduction}, we obtain
\[
    \begin{aligned}
    &
    \left\|
    \mathcal G_{P^{\widehat\pi,\widehat W}}
    \sqrt{\Var_{P^{0,\widehat\pi}}(\widehat W)}
    \right\|_\infty
\le
    C\sqrt{B_{\mathrm{emp}}(R)+R\zeta_{\widehat W}}
    +CR\sqrt{\sigma}.
    \end{aligned}
\]
Since $B_{\mathrm{emp}}(R)\ge\sigma R^2$, the final term can be absorbed into
the square root. Therefore,
\begin{equation}
    \label{eq:empirical-variance-closure-proof}
    \left\|
    \mathcal G_{P^{\widehat\pi,\widehat W}}
    \sqrt{\Var_{P^{0,\widehat\pi}}(\widehat W)}
    \right\|_\infty
    \le
    C\sqrt{B_{\mathrm{emp}}(R)+R\zeta_{\widehat W}}.
\end{equation}
Together, \eqref{eq:empirical-residual-reduction-proof} and
\eqref{eq:empirical-variance-closure-proof} imply
\begin{equation}
    \label{eq:empirical-residual-self-bound-proof}
    \zeta_{\widehat W}
    \le
    C\sqrt{\frac{\iota}{N}}
    \sqrt{B_{\mathrm{emp}}(R)+R\zeta_{\widehat W}}
    +C\frac{R\cdot\iota}{N}
    +C\varepsilon_{\mathrm{opt}}.
\end{equation}
To solve \eqref{eq:empirical-residual-self-bound-proof},
$\sqrt{x+y}\le\sqrt{x}+\sqrt{y}$ and Young's
inequality yield
\begin{equation}
    \label{eq:empirical-residual-young-proof}
    \begin{aligned}
    C\sqrt{\frac{\iota}{N}}
    \sqrt{B_{\mathrm{emp}}(R)+R\zeta_{\widehat W}}
    &\le
    C\sqrt{\frac{B_{\mathrm{emp}}(R)\cdot\iota}{N}}
    +
    C\sqrt{\frac{R\cdot\iota}{N}}\sqrt{\zeta_{\widehat W}}
    \\
    &\le
    C\sqrt{\frac{B_{\mathrm{emp}}(R)\cdot\iota}{N}}
    +\frac12\zeta_{\widehat W}
    +C\frac{R\cdot\iota}{N}.
    \end{aligned}
\end{equation}
Substituting \eqref{eq:empirical-residual-young-proof} into
\eqref{eq:empirical-residual-self-bound-proof} results in
\[
    \zeta_{\widehat W}
    \le
    C\sqrt{\frac{B_{\mathrm{emp}}(R)\cdot\iota}{N}}
    +\frac12\zeta_{\widehat W}
    +C\frac{R\cdot\iota}{N}
    +C\varepsilon_{\mathrm{opt}}.
\]
Rearranging establishes
\begin{equation}
    \label{eq:empirical-total-residual-proof}
    \zeta_{\widehat W}
    \le
    C\sqrt{\frac{B_{\mathrm{emp}}(R)\cdot\iota}{N}}
    +C\frac{R\cdot\iota}{N}
    +C\varepsilon_{\mathrm{opt}}.
\end{equation}

\paragraph*{Step 3: Final substitution.}
Finally, we use the bound from Step~2 to simplify the right-hand side of
\eqref{eq:T2-prime-lemma23-intermediate}. Young's inequality and
\eqref{eq:empirical-total-residual-proof} show that
\[
    \begin{aligned}
    \sqrt{\frac{\iota}{N}}
    \sqrt{B_{\mathrm{emp}}(R)+R\zeta_{\widehat W}}
    &\le
    \sqrt{\frac{B_{\mathrm{emp}}(R)\cdot\iota}{N}}
    +\frac12\zeta_{\widehat W}
    +C\frac{R\cdot\iota}{N}
    \\
    &\le
    C\sqrt{\frac{B_{\mathrm{emp}}(R)\cdot\iota}{N}}
    +C\frac{R\cdot\iota}{N}
    +C\varepsilon_{\mathrm{opt}}.
    \end{aligned}
\]
Substituting the last display into
\eqref{eq:T2-prime-lemma23-intermediate} establishes
\eqref{eq:T2-prime-bound}, completing the proof.

\subsubsection{Proof of Lemma~\ref{lem:mismatched-resolvent-variance-pi-hat}}
\label{sec:proof-of-mismatched-resolvent-variance-pi-hat}
The proof again has three steps. We first apply the anchored
resolvent-variance bound under the mismatched kernel
$P^{\widehat\pi,W}$ and isolate the selector-mismatch cost. We then close the
empirical Bellman residual under this resolvent. Finally, we transfer the
variance from the kernel selected at $W$ to the kernel selected at
$\widehat W$ and absorb the resulting mismatch term.

\paragraph*{Step 1: Reduction under the mismatched resolvent.}
In this step, we rewrite the empirical Bellman equation under
$P^{\widehat\pi,W}$ and quantify the cost of replacing the selector
$P^{\widehat\pi,\widehat W}$ by $P^{\widehat\pi,W}$. Define
\[
    \zeta_{\mathrm{mis}}
    \coloneqq
    \left\|
    \mathcal G_{P^{\widehat\pi,W}}
    \left|
    \left(\widehat P^{\widehat\pi,\widehat W}
    -P^{\widehat\pi,\widehat W}\right)\widehat W
    \right|
    \right\|_\infty.
\]
For each $(s,a)$, optimality of the two true robust selectors implies
\[
    \left(P_{s,a}^{\widehat W}-P_{s,a}^{W}\right)\widehat W
    \le0,
    \qquad
    \left(P_{s,a}^{\widehat W}-P_{s,a}^{W}\right)W
    \ge0.
\]
Consequently,
\[
    \begin{aligned}
    0
    &\le
    -\left(P_{s,a}^{\widehat W}-P_{s,a}^{W}\right)\widehat W
    \\
    &\le
    \left|
    \left(P_{s,a}^{\widehat W}-P_{s,a}^{W}\right)
    (\widehat W-W)
    \right|
    \\
    &\le
    \left\|P_{s,a}^{\widehat W}-P_{s,a}^{W}\right\|_1
    \|\widehat W-W\|_\infty.
    \end{aligned}
\]
Both selected rows belong to $\cU_{s,a}(P^0)$, so their
$\ell_1$-distance is at most $4\sigma$. Averaging over
$a\sim\widehat\pi(\cdot\mid s)$ therefore leads to
\begin{equation}
    \label{eq:learned-policy-selector-mismatch}
    \left\|
    \left(P^{\widehat\pi,\widehat W}
    -P^{\widehat\pi,W}\right)\widehat W
    \right\|_\infty
    \le
    C\sigma\|\widehat W-W\|_\infty.
\end{equation}

The empirical Bellman equation can now be written as
\[
    \begin{aligned}
    \widehat W
    &=
    r^{\widehat\pi}
    +\gamma P^{\widehat\pi,W}\widehat W
    \\
    &\quad+
    \gamma
    \left(\widehat P^{\widehat\pi,\widehat W}
    -P^{\widehat\pi,\widehat W}\right)\widehat W
    +
    \gamma
    \left(P^{\widehat\pi,\widehat W}
    -P^{\widehat\pi,W}\right)\widehat W.
    \end{aligned}
\]
As in Step~1 of the proof of
Lemma~\ref{lem:span-aware-shi-lemma-12}, the span, minimum-value, and anchor
defect premises of Lemma~\ref{lem:span-aware-resolvent-variance} follow from
\eqref{eq:discounted-localized-span-condition},
\eqref{eq:discounted-value-range}, and
\eqref{eq:empirical-anchor-defect}, respectively. Since
$P^{\widehat\pi,W}$ belongs to the true robust uncertainty set,
\eqref{eq:anchored-resolvent-residual-envelopes} holds with $\xi_h=0$.
For the value residual in the displayed Bellman equation, $\gamma\le1$ allows
us to take
\[
    \xi_V
    =
    \left(\widehat P^{\widehat\pi,\widehat W}
    -P^{\widehat\pi,\widehat W}\right)\widehat W
    +
    \left(P^{\widehat\pi,\widehat W}
    -P^{\widehat\pi,W}\right)\widehat W
\]
in \eqref{eq:anchored-resolvent-residual-envelopes}. The triangle inequality,
\eqref{eq:learned-policy-selector-mismatch}, and
$\mathcal G_{P^{\widehat\pi,W}}\bm{1}_S=\bm{1}_S$ imply
\[
    \left\|
    \mathcal G_{P^{\widehat\pi,W}}|\xi_V|
    \right\|_\infty
    \le
    \zeta_{\mathrm{mis}}
    +C\sigma\|\widehat W-W\|_\infty.
\]
Lemma~\ref{lem:span-aware-resolvent-variance} therefore yields
\begin{equation}
    \label{eq:mismatched-learned-policy-resolvent-reduction}
    \begin{aligned}
    &
    \left\|
    \left(I-\gamma P^{\widehat\pi,W}\right)^{-1}
    \sqrt{\Var_{P^{\widehat\pi,W}}(\widehat W)}
    \right\|_\infty
    \\
    &\qquad\le
    \frac{C}{1-\gamma}
    \sqrt{
    B_{\mathrm{emp}}(R)
    +R\zeta_{\mathrm{mis}}
    +\sigma R\|\widehat W-W\|_\infty
    }.
    \end{aligned}
\end{equation}
Thus, relative to the matched reduction
\eqref{eq:matched-learned-policy-resolvent-reduction}, the only additional
term is $\sigma R\|\widehat W-W\|_\infty$.

\paragraph*{Step 2: Closing the empirical Bellman residual.}
We have
\begin{equation}
    \label{eq:mismatched-policy-residual-proof}
    \zeta_{\mathrm{mis}}
    \le
    C\sqrt{
    \frac{
    \left(
    B_{\mathrm{emp}}(R)
    +\sigma R\|\widehat W-W\|_\infty
    \right)\cdot\iota}{N}}
    +C\frac{R\cdot\iota}{N}
    +C\varepsilon_{\mathrm{opt}}.
\end{equation}
The proof of \eqref{eq:mismatched-policy-residual-proof} follows Step~2 of the
proof of Lemma~\ref{lem:span-aware-shi-lemma-12}, with the additional term
$\sigma R\|\widehat W-W\|_\infty$ arising from the selector-mismatch term in
\eqref{eq:mismatched-learned-policy-resolvent-reduction}. We omit the details
for brevity.

\paragraph*{Step 3: Substitution and absorption.}
In this step, we transfer the variance in $T_6'$ to the kernel analyzed in
Step~1, substitute the residual bound from Step~2, and absorb the remaining
selector-mismatch term. Both
$P^{\widehat\pi,\widehat W}$ and $P^{\widehat\pi,W}$ are within total
variation distance $\sigma$ of $P^{0,\widehat\pi}$, and hence
\[
    \left\|
    P^{\widehat\pi,\widehat W}(s,\cdot)
    -P^{\widehat\pi,W}(s,\cdot)
    \right\|_{\mathrm{TV}}
    \le2\sigma,
    \qquad s\in\cS.
\]
Therefore, Lemma~\ref{lem:standard-deviation-TV-perturbation} and
\eqref{eq:discounted-localized-span-condition} imply
\begin{equation}
    \label{eq:mismatched-policy-standard-deviation-transfer}
    \sqrt{\Var_{P^{\widehat\pi,\widehat W}}(\widehat W)}
    \le
    \sqrt{\Var_{P^{\widehat\pi,W}}(\widehat W)}
    +CR\sqrt{\sigma}\cdot\bm{1}_S.
\end{equation}
Using \eqref{eq:mismatched-policy-standard-deviation-transfer} in the
definition of $T_6'$ in
\eqref{eq:T-prime-second-branch-definitions} and then applying
\eqref{eq:mismatched-learned-policy-resolvent-reduction}, we obtain
\begin{equation}
    \label{eq:T6-prime-mismatched-intermediate}
    \begin{aligned}
    \|T_6'\|_\infty
    &\le
    \frac{C}{1-\gamma}\sqrt{\frac{\iota}{N}}
    \sqrt{
    B_{\mathrm{emp}}(R)
    +R\zeta_{\mathrm{mis}}
    +\sigma R\|\widehat W-W\|_\infty
    }
    \\
    &\quad+
    C\frac{R\sqrt{\sigma\cdot\iota/N}}{1-\gamma}.
    \end{aligned}
\end{equation}
Since $B_{\mathrm{emp}}(R)\ge\sigma R^2$, the additive variance-transfer term
in \eqref{eq:T6-prime-mismatched-intermediate} is absorbed into its
$B_{\mathrm{emp}}(R)$ contribution. Splitting the remaining square root,
applying Young's inequality to its $R\zeta_{\mathrm{mis}}$ contribution, and
then using \eqref{eq:mismatched-policy-residual-proof}, we obtain
\begin{equation}
    \label{eq:T6-prime-after-residual-substitution}
    \begin{aligned}
    \|T_6'\|_\infty
    &\le
    C\sqrt{
    \frac{B_{\mathrm{emp}}(R)\cdot\iota}
    {N(1-\gamma)^2}}
    +C\frac{R\cdot\iota}{N(1-\gamma)}
    \\
    &\quad+
    C\sqrt{
    \frac{\sigma R\|\widehat W-W\|_\infty\cdot\iota}
    {N(1-\gamma)^2}}
    +C\frac{\varepsilon_{\mathrm{opt}}}{1-\gamma}.
    \end{aligned}
\end{equation}
Young's inequality bounds the remaining selector-mismatch contribution by
\begin{equation}
    \label{eq:T6-prime-selector-mismatch-young}
    C\sqrt{
    \frac{\sigma R\|\widehat W-W\|_\infty\cdot\iota}
    {N(1-\gamma)^2}}
    \le
    \frac{1}{20}\|\widehat W-W\|_\infty
    +
    C\frac{\sigma R\cdot\iota}{N(1-\gamma)^2}.
\end{equation}
Moreover,
\begin{equation}
    \label{eq:T6-prime-selector-mismatch-budget}
    \frac{\sigma R\cdot\iota}{N(1-\gamma)^2}
    =
    \frac{\sqrt{\sigma\cdot\iota/N}}{1-\gamma}
    \sqrt{
    \frac{\sigma R^2\cdot\iota}{N(1-\gamma)^2}}
    \le
    C\sqrt{
    \frac{B_{\mathrm{emp}}(R)\cdot\iota}
    {N(1-\gamma)^2}},
\end{equation}
where the last inequality uses the second condition in
\eqref{eq:recursive-comparison-sample-conditions} and
$B_{\mathrm{emp}}(R)\ge\sigma R^2$. Substituting
\eqref{eq:T6-prime-selector-mismatch-young} and
\eqref{eq:T6-prime-selector-mismatch-budget} into
\eqref{eq:T6-prime-after-residual-substitution} establishes
\eqref{eq:T6-prime-bound}, completing the proof.

\subsubsection{Proof of
Lemma~\ref{lem:standard-deviation-TV-perturbation}}
\label{subsec:proof-standard-deviation-TV-perturbation}

Fix probability vectors $q,q'$ and a vector $f$. Since the variance is
the minimum mean squared deviation from a constant, evaluating the variance
under $q$ around the mean $q'f$ gives
\[
    \begin{aligned}
    \Var_q(f)
    &=
    \min_{c\in\mathbb R}q\!\left[(f-c)^2\right]
    \\
    &\le
    q\!\left[(f-q'f)^2\right]
    \\
    &=
    q'\!\left[(f-q'f)^2\right]
    +(q-q')\!\left[(f-q'f)^2\right]
    \\
    &=
    \Var_{q'}(f)
    +(q-q')\!\left[(f-q'f)^2\right].
    \end{aligned}
\]
Because $q'f$ is a convex combination of the coordinates of $f$, it
lies between their minimum and maximum. Hence
\[
    0
    \le
    (f-q'f)^2
    \le
    \|f\|_{\mathrm{span}}^2\cdot\bm{1}_S.
\]
Applying this bound to the previous inequality yields
\begin{equation}
    \label{eq:one-sided-variance-TV-perturbation}
    \Var_q(f)
    \le
    \Var_{q'}(f)
    +C\|f\|_{\mathrm{span}}^2
    \|q-q'\|_{\mathrm{TV}}.
\end{equation}
Finally, applying $\sqrt{x+y}\le\sqrt{x}+\sqrt{y}$ to
\eqref{eq:one-sided-variance-TV-perturbation} gives
\[
    \begin{aligned}
    \sqrt{\Var_q(f)}-
    \sqrt{\Var_{q'}(f)}
    &\le
    C\|f\|_{\mathrm{span}}
    \sqrt{\|q-q'\|_{\mathrm{TV}}}.
    \end{aligned}
\]
Interchanging $q$ and $q'$ gives the reverse comparison, thus establishing
\eqref{eq:standard-deviation-TV-perturbation}.

\section{Proofs for span-agnostic horizon calibration}\label{app:span-agnostic-horizon-calibration}

This appendix gives the full specification of
Algorithm~\mainAnchoredSpanAgnosticAlgorithm{} and proves
Theorem~\mainSpanAgnosticUpperBoundTheorem{}.
Appendix~\ref{subsec:anchored-calibration-computable-quantities} supplies the
implementation details omitted from the main text, including the horizon
grids, solver guarantees, anchor-calibration procedure, and
lower-confidence penalties.
Appendix~\ref{subsec:anchored-calibration-lemmas} states the supporting
lemmas for the anchor certificate and the nominal and robust policy
candidates.

Appendix~\ref{subsec:proof-anchored-calibrated-policy-upper-bound} combines
these lemmas to prove the theorem. It shows that every reported
lower-confidence bound is valid for its policy's robust average reward and
that, under either condition in the theorem, at least one candidate has a
lower-confidence bound of at least
$\rho^{\star,\sigma}-O(\varepsilon)$. The policy with the largest bound is
therefore $O(\varepsilon)$-optimal. Finally,
Appendix~\ref{subsec:proof-anchored-calibration-auxiliary-lemmas} proves the
supporting lemmas.

\subsection{Full specification of Algorithm~\mainAnchoredSpanAgnosticAlgorithm{}}
\label{subsec:anchored-calibration-computable-quantities}

This subsection gives the implementable specification summarized by
Algorithm~\mainAnchoredSpanAgnosticAlgorithm{}. We first define the
policy grid, which consists of the discount factors used to generate the
candidate policies, and the required solver guarantees. We then
specify the anchor calibration, the robust planning and evaluation calls, and
the lower-confidence penalties. The final displays introduce the expanded
penalty notation used in the proofs.

Recall that the independent batches
$\mathcal D_{\mathrm{nom}}$ and $\mathcal D_{\mathrm{rob}}$ contain
$N_{\mathrm{nom}}$ and $N_{\mathrm{rob}}$ samples per state-action pair,
respectively. Their empirical nominal kernels are
$\widehat P_{\mathrm{nom}}^0$ and $\widehat P_{\mathrm{rob}}^0$.
All nominal and robust discounted solver calls below are governed by
the common solver tolerance
$\varepsilon_{\mathrm{opt}}\le c_{\mathrm{opt}}\varepsilon$, where
$c_{\mathrm{opt}}$ is the universal constant in
Theorem~\mainSpanAgnosticUpperBoundTheorem{}. Their call-specific
discounted-value guarantees are stated below.

\paragraph*{Logarithmic factor and policy grid.}
For the remainder of this appendix, define
\begin{equation}
    \iota_N
    \coloneqq
    \log\!\left(\frac{SAN}{\delta}\right).
    \label{eq:anchored-iota-N}
\end{equation}
The policy grid $\Gamma_N$ consists of discount factors $\gamma$ with dyadic
effective horizons $(1-\gamma)^{-1}=2^k$; each of these discount factors is
used to generate candidate policies. The grid excludes horizons that are too
long for the robust batch.
For a sufficiently large
universal constant $C_\Gamma$, define
\begin{equation}
    \Gamma_N
    \coloneqq
    \left\{
    1-2^{-k}: k\in\{1,2,\ldots\},\
    2^k\le
    \frac{1}{C_\Gamma}
    \min\left\{
    \sqrt{\frac{N_{\mathrm{rob}}}{\sigma\cdot\iota_N}},
    \frac{N_{\mathrm{rob}}}{\iota_N}
    \right\}
    \right\},
    \label{eq:anchored-Gamma-N}
\end{equation}
When $\sigma=0$, the first cutoff in $\Gamma_N$ is interpreted as $+\infty$.

We use $\gamma$ for discount factors that generate candidate policies and
$\lambda$ for discount factors that generate anchor certificates. When a
definition applies to a discount factor in either role, we denote it by
$\eta$.

\paragraph*{Nominal solver.}
For each $\eta\in(0,1)$ at which we call the nominal solver, let
$\widehat V_\eta^{\star,0}$ denote the
exact optimal discounted value of the empirical nominal MDP with kernel
$\widehat P_{\mathrm{nom}}^0$. The solver returns a deterministic policy
$\widehat\pi_\eta^0$, which we evaluate exactly under
$\widehat P_{\mathrm{nom}}^0$. We denote its discounted value by
$\widehat V_\eta^{\widehat\pi_\eta^0,0}$ and require
\begin{equation}
    \label{eq:anchored-nominal-solver-tolerance}
    \left\|
    \widehat V_\eta^{\star,0}
    -\widehat V_\eta^{\widehat\pi_\eta^0,0}
    \right\|_\infty
    \le
    \min\left\{
    \frac{\varepsilon_{\mathrm{opt}}}{1-\eta},
    \frac{1}{N_{\mathrm{nom}}}
    \right\}.
\end{equation}
The $1/N_{\mathrm{nom}}$ requirement is a computational accuracy condition and
does not require additional samples. The solver is applied at each
$\gamma\in\Gamma_N$ to obtain the nominal policy candidates. It is also applied
at the discount factors in the anchor grid defined below. When a discount
factor belongs to both grids, we reuse the same nominal solve.

\paragraph*{Anchor calibration.}
The following procedure uses only $\mathcal D_{\mathrm{nom}}$. It constructs a
certificate at each anchor discount factor and selects one certificate to use
for every robust policy candidate.
\begin{algorithm}
    \caption{Nominal anchor calibration}
    \label{alg:nominal-anchor-calibration}
    \noindent\textbf{Input.}
    The nominal-batch empirical kernel $\widehat P_{\mathrm{nom}}^0$;
    $N$, $N_{\mathrm{nom}}$, $\varepsilon$, $\delta$, and
    the solver tolerance $\varepsilon_{\mathrm{opt}}$.
    \begin{enumerate}[leftmargin=*,label=\textbf{\arabic*.},itemsep=0.1em,topsep=0.5em]
        \item Construct the anchor grid
        \begin{equation}
            \Lambda_N
            \coloneqq
            \left\{
            1-2^{-j}: j\in\{1,2,\ldots\},\
            2^j\le N_{\mathrm{nom}}
            \right\}.
            \label{eq:anchored-Lambda-N}
        \end{equation}
        \item For each $\lambda\in\Lambda_N$, solve the empirical nominal
        discounted problem subject to
        \eqref{eq:anchored-nominal-solver-tolerance}. For a sufficiently large
        universal constant $C_{\mathrm{anc}}$, define
        \begin{subequations}
        \label{eq:anchored-nominal-certificate}
        \begin{align}
            \widehat H_\lambda^0
            &\coloneqq
            \max\left\{1,
            \left\|
            \widehat V_\lambda^{\widehat\pi_\lambda^0,0}
            \right\|_{\mathrm{span}}
            \right\},
            \label{eq:anchored-nominal-anchor-values}
            \\
            \widehat\rho_\lambda^+
            &\coloneqq
            (1-\lambda)\max_s
            \widehat V_\lambda^{\widehat\pi_\lambda^0,0}(s)
            +C_{\mathrm{anc}}\left(
            \sqrt{
            \frac{\widehat H_\lambda^0\cdot\iota_N}
            {N_{\mathrm{nom}}}}
            +\frac{\iota_N}{N_{\mathrm{nom}}(1-\lambda)}
            +\varepsilon_{\mathrm{opt}}
            \right),
            \label{eq:anchored-rho-lambda-plus}
            \\
            \widehat H_\lambda^+
            &\coloneqq
            \widehat H_\lambda^0
            +C_{\mathrm{anc}}\left(
            \frac{1}{1-\lambda}
            \sqrt{
            \frac{\widehat H_\lambda^0\cdot\iota_N}
            {N_{\mathrm{nom}}}}
            +\frac{\iota_N}
            {N_{\mathrm{nom}}(1-\lambda)^2}
            +\frac{\varepsilon_{\mathrm{opt}}}{1-\lambda}
            \right).
            \label{eq:anchored-H-lambda-plus}
        \end{align}
        \end{subequations}
        The pair
        $(\widehat\rho_\lambda^+,\widehat H_\lambda^+)$ is the candidate
        anchor certificate at $\lambda$.
        \item For each $\lambda\in\Lambda_N$, compute
        \begin{equation}
            \widehat Q_\lambda
            \coloneqq
            (1-\lambda)\widehat H_\lambda^+
            +2\left[
            \widehat\rho_\lambda^+
            -(1-\lambda)\max_s
            \widehat V_\lambda^{\widehat\pi_\lambda^0,0}(s)
            \right].
            \label{eq:anchored-anchor-quality}
        \end{equation}
        Select, using arbitrary deterministic tie-breaking,
        \begin{equation}
            \widehat\lambda_{\mathrm{anc}}
            \in
            \arg\min_{\substack{\lambda\in\Lambda_N:\\
            \widehat Q_\lambda
            \le
            \min_{\mu\in\Lambda_N}\widehat Q_\mu
            +\varepsilon}}
            \widehat H_\lambda^+.
            \label{eq:anchored-selected-anchor}
        \end{equation}
        \item Return the anchor certificate
        \begin{equation}
            \widehat\rho_{\mathrm{anc}}^+
            \coloneqq
            \widehat\rho_{\widehat\lambda_{\mathrm{anc}}}^+,
            \qquad
            \widehat H_{\mathrm{anc}}^+
            \coloneqq
            \widehat H_{\widehat\lambda_{\mathrm{anc}}}^+.
            \label{eq:anchored-selected-certificate}
        \end{equation}
    \end{enumerate}
\end{algorithm}

The rule \eqref{eq:anchored-selected-anchor} first restricts attention to the
discount factors whose $\widehat Q_\lambda$ values are within $\varepsilon$ of
the minimum and then selects one with the smallest certified span
$\widehat H_\lambda^+$. The resulting certificate
\eqref{eq:anchored-selected-certificate} supplies the anchor reward-level and
span bounds used in the confidence penalty for every robust policy candidate.
Because Algorithm~\ref{alg:nominal-anchor-calibration} uses only
$\mathcal D_{\mathrm{nom}}$, the selected discount factor and its certificate
are independent of $\mathcal D_{\mathrm{rob}}$.

\paragraph*{Robust planning and evaluation.}
For each $\gamma\in\Gamma_N$, we use
$\widehat P_{\mathrm{rob}}^0$ in two ways. First, we evaluate the nominal
candidate $\widehat\pi_\gamma^0$ exactly in the empirical robust MDP and denote
its value by $\widehat V_\gamma^{\widehat\pi_\gamma^0,\sigma}$. Second, we run
the robust discounted solver to obtain a policy $\widehat\pi_\gamma$ satisfying
\begin{equation}
    \label{eq:anchored-robust-solver-tolerance}
    \left\|
    \widehat V_\gamma^{\star,\sigma}
    -\widehat V_\gamma^{\widehat\pi_\gamma,\sigma}
    \right\|_\infty
    \le
    \varepsilon_{\mathrm{opt}}.
\end{equation}
All empirical robust values below are computed in the uncertainty set centered
at $\widehat P_{\mathrm{rob}}^0$. We assume exact fixed-policy evaluation, so
the nominal-candidate penalty has no optimization-error term; an approximate
evaluation can instead be accommodated by adding its tolerance to that
penalty.

\paragraph*{Penalties.}
Fix a discount factor $\eta$ and an empirical robust value function $V$.
Suppose an anchor $(\rho,h)$ is associated with computable bounds
$\bar\rho$ and $\bar H$ satisfying
\[
    \rho
    \le
    \bar\rho,
    \qquad
    \max\{1,\|h\|_{\mathrm{span}}\}
    \le
    \bar H.
\]
From these quantities, we construct the penalty in four stages.
First, the following quantity bounds both the span of the empirical value and
the certified span of the anchor:
\begin{equation}
    \label{eq:anchored-common-radius}
    \mathcal R_\eta(V,\bar H)
    \coloneqq
    \max\{1,\|V\|_{\mathrm{span}},\bar H\}.
\end{equation}
Here $[x]_+\coloneqq\max\{x,0\}$. Next,
\begin{equation}
    \label{eq:anchored-common-defect}
    \beta_\eta(V,\bar\rho)
    \coloneqq
    \left[
    \bar\rho-(1-\eta)\min_s V(s)
    \right]_+
\end{equation}
measures how far the anchor's reward upper bound lies above the empirical
lower estimate $(1-\eta)\min_s V(s)$. These two quantities determine
\begin{equation}
    \label{eq:anchored-common-budget}
    \mathcal B_\eta(V,\bar\rho,\bar H)
    \coloneqq
    \bar H
    +2\mathcal R_\eta(V,\bar H)\beta_\eta(V,\bar\rho)
    +\sigma\mathcal R_\eta(V,\bar H)^2
    +(1-\eta)\mathcal R_\eta(V,\bar H)^2,
\end{equation}
which determines the variance-dependent part of the penalty. Finally, define
\begin{equation}
    \label{eq:anchored-common-penalty}
    \operatorname{pen}_\eta
    (V;\bar\rho,\bar H)
    \coloneqq
    C_{\mathrm{pen}}\left(
    \sqrt{
    \frac{\mathcal B_\eta(V,\bar\rho,\bar H)\cdot\iota_N}
    {N_{\mathrm{rob}}}
    }
    +\frac{\mathcal R_\eta(V,\bar H)\cdot\iota_N}{N_{\mathrm{rob}}}
    +\frac{\beta_\eta(V,\bar\rho)\cdot\iota_N}
    {N_{\mathrm{rob}}(1-\eta)}
    \right).
\end{equation}
The constant $C_{\mathrm{pen}}$ is a sufficiently large universal constant.
The square-root term gives the main statistical error. The term
involving $\mathcal R_\eta$ scales with the common span bound, whereas the term
involving $\beta_\eta/(1-\eta)$ scales the anchor reward-level gap by the
effective horizon.

For a nominal candidate, we use the trivial anchor $(\rho,h)=(1,0)$ and hence
the bounds $(\bar\rho,\bar H)=(1,1)$. Its empirical robust fixed-policy
evaluation is exact, so its penalty contains only the statistical terms above.
For a robust candidate, the solver guarantee
\eqref{eq:anchored-robust-solver-tolerance} contributes the additional term
$C_{\mathrm{pen}}\varepsilon_{\mathrm{opt}}$. We compute one robust-candidate
penalty from the selected anchor certificate and another from the trivial
anchor. Accordingly, for each $\gamma\in\Gamma_N$, define
\begin{subequations}
\begin{align}
    \operatorname{pen}_\gamma^{\mathrm{nom}}(\widehat\pi_\gamma^0)
    &\coloneqq
    \operatorname{pen}_\gamma
    \left(
    \widehat V_\gamma^{\widehat\pi_\gamma^0,\sigma};1,1
    \right),
    \label{eq:anchored-pen-nom-gamma}
    \\
    \operatorname{pen}_\gamma^{\mathrm{anc}}(\widehat\pi_\gamma)
    &\coloneqq
    \operatorname{pen}_\gamma
    \left(
    \widehat V_\gamma^{\widehat\pi_\gamma,\sigma};
    \widehat\rho_{\mathrm{anc}}^+,
    \widehat H_{\mathrm{anc}}^+
    \right)
    +C_{\mathrm{pen}}\varepsilon_{\mathrm{opt}},
    \label{eq:anchored-pen-gamma-anc}
    \\
    \operatorname{pen}_\gamma^{\mathrm{triv}}(\widehat\pi_\gamma)
    &\coloneqq
    \operatorname{pen}_\gamma
    \left(
    \widehat V_\gamma^{\widehat\pi_\gamma,\sigma};
    1,1
    \right)
    +C_{\mathrm{pen}}\varepsilon_{\mathrm{opt}}.
    \label{eq:anchored-pen-gamma-triv}
\end{align}
\end{subequations}
The resulting nominal- and robust-policy lower-confidence bounds are defined in
\mainAnchoredNominalCandidateScore{} and
\mainAnchoredRobustCandidateScore{}, respectively. The latter uses the
smaller of the selected-anchor and trivial-anchor penalties.

\paragraph*{Notation shorthand.}
For each $\gamma\in\Gamma_N$, collect the three statistical components of a
penalty as
\begin{equation}
    \label{eq:anchored-component-map}
    \mathcal C_\gamma(V;\bar\rho,\bar H)
    \coloneqq
    \bigl(
    \mathcal R_\gamma(V,\bar H),
    \beta_\gamma(V,\bar\rho),
    \mathcal B_\gamma(V,\bar\rho,\bar H)
    \bigr).
\end{equation}
Abbreviating $\widehat\lambda_{\mathrm{anc}}$ by $\widehat\lambda$ in the
following display, the proof shorthand is
\begin{subequations}
\begin{equation}
    \bigl(
    \widehat R_\gamma^{\mathrm{nom}},
    \widehat\beta_\gamma^{\mathrm{nom}},
    \widehat B_\gamma^{\mathrm{nom}}
    \bigr)
    =
    \mathcal C_\gamma
    \left(\widehat V_\gamma^{\widehat\pi_\gamma^0,\sigma};1,1\right).
    \label{eq:anchored-R-nom-gamma}
\end{equation}
\begin{equation}
    \bigl(
    \widehat R_{\gamma,\widehat\lambda}^{\mathrm{anc}},
    \widehat\beta_{\gamma,\widehat\lambda},
    \widehat B_{\gamma,\widehat\lambda}
    \bigr)
    =
    \mathcal C_\gamma\!
    \left(
    \widehat V_\gamma^{\widehat\pi_\gamma,\sigma};
    \widehat\rho_{\widehat\lambda}^+,
    \widehat H_{\widehat\lambda}^+
    \right).
    \label{eq:anchored-R-gamma-lambda}
\end{equation}
\begin{equation}
    \bigl(
    \widehat R_\gamma,
    \widehat\beta_{\gamma,\mathrm{triv}},
    \widehat B_{\gamma,\mathrm{triv}}
    \bigr)
    =
    \mathcal C_\gamma
    \left(\widehat V_\gamma^{\widehat\pi_\gamma,\sigma};1,1\right).
    \label{eq:anchored-R-gamma}
\end{equation}
\end{subequations}
Thus $\widehat R$, $\widehat\beta$, and $\widehat B$ denote, respectively,
the common span bound, the anchor-level defect, and the budget under the leading
square root. The superscript $\mathrm{nom}$ identifies the nominal-policy
family, while the second subscript records the anchor used for a robust-policy
candidate. The associated penalties are those in
\eqref{eq:anchored-pen-nom-gamma}--\eqref{eq:anchored-pen-gamma-triv}; the
robust-policy lower-confidence bound uses the smaller of the selected-anchor
and trivial-anchor penalties.

\subsection{Supporting lemmas for the span-agnostic algorithm}
\label{subsec:anchored-calibration-lemmas}

We continue with the notation of
Appendix~\ref{subsec:anchored-calibration-computable-quantities}.
In particular, $\lambda\in\Lambda_N$ indexes anchor candidates constructed
from $\mathcal D_{\mathrm{nom}}$, whereas $\gamma\in\Gamma_N$ indexes policy
candidates evaluated using $\mathcal D_{\mathrm{rob}}$. Accordingly,
$\widehat V_\lambda^{\widehat\pi_\lambda^0,0}$ uses
$\widehat P_{\mathrm{nom}}^0$, while every hatted robust value uses
$\widehat P_{\mathrm{rob}}^0$.

The supporting lemmas are organized in four stages. First,
Lemma~\ref{lem:fixed-discount-nominal-plugin} bounds the errors of the
empirical nominal optimal value and the returned nominal policy at every
discount factor in the grid. Lemmas~\ref{lem:discounted-nominal-anchors}
and~\ref{lem:anchor-calibration-guarantees} then construct and select an
anchor with level within $O(\varepsilon)$ of $\rho^\star$ and span $O(H_0)$.
Second, Lemma~\ref{lem:raw-localized-robust-comparison} makes the fixed-policy
and learned-policy discounted value-error bounds hold simultaneously over the
dyadic span-radius grid. Third,
Lemma~\ref{lem:nominal-policy-candidate-guarantees} locates a nominal
candidate at a horizon of order $H_0/\varepsilon$, validates every nominal
lower-confidence bound, and identifies a near-optimal candidate in the
high-tolerance regime. Finally,
Lemma~\ref{lem:robust-policy-candidate-guarantees} gives the corresponding
validity and existence guarantees for the robust candidates. This yields the
high-tolerance guarantee when $H_\sigma\le H_0$ and the low-tolerance
guarantee for either span ordering.
Appendix~\ref{subsec:proof-anchored-calibrated-policy-upper-bound} combines
these ingredients to analyze the final selection.

\paragraph*{Fixed-discount nominal comparisons.}
The next lemma gives simultaneous nominal discounted-value estimation bounds
over a finite set of discount factors. Importantly, the lemma includes bounds
expressed in terms of observable empirical value spans. We use them both to certify the anchor
candidates over $\Lambda_N$ and to analyze a nominal policy candidate from
$\Gamma_N$. Its proof, given in
Appendix~\ref{subsec:proof-fixed-discount-nominal-plugin}, follows by
specializing the argument for Theorem~\ref{thm:TV-upper-bound-discount} to
$\sigma=0$.

\begin{lemma}[Fixed-discount nominal plug-in bounds]
    \label{lem:fixed-discount-nominal-plugin}
    There exists a universal constant $C>0$ such that the following holds.
    Let $N\ge16$ and let $\mathcal G_N\subset[1/2,1)$ be a deterministic set
    of $O(\log N)$ discount factors such that
    \[
        \frac{1}{1-\eta}
        \le
        N_{\mathrm{nom}}
        \qquad
        \text{for every }\eta\in\mathcal G_N.
    \]
    For every $\eta\in\mathcal G_N$, let
    $\widehat\pi_\eta^0$ satisfy
    \[
        \left\|
        \widehat V_\eta^{\star,0}
        -\widehat V_\eta^{\widehat\pi_\eta^0,0}
        \right\|_\infty
        \le
        \frac{1}{N_{\mathrm{nom}}}.
    \]
    Then, with probability at least $1-O(\delta)$, the following bounds hold
    simultaneously for every $\eta\in\mathcal G_N$:
    \begin{subequations}
    \label{eq:fixed-discount-nominal-plugin-bounds}
    \begin{align}
        \left\|
        V_\eta^{\star,0}-\widehat V_\eta^{\star,0}
        \right\|_\infty
        &\le
        \frac{C}{1-\eta}
        \sqrt{
        \frac{(\|\widehat V_\eta^{\star,0}\|_{\mathrm{span}}+1)
        \cdot\iota_N}{N_{\mathrm{nom}}}}
        +C\frac{\iota_N}{N_{\mathrm{nom}}(1-\eta)^2},
        \label{eq:fixed-discount-nominal-empirical-span}
        \\
        \left\|
        \widehat V_\eta^{\star,0}-V_\eta^{\star,0}
        \right\|_\infty
        &\le
        C\frac{\iota_N}{N_{\mathrm{nom}}(1-\eta)^2}
        +\frac{C}{1-\eta}
        \sqrt{
        \frac{(\|V_\eta^{\star,0}\|_{\mathrm{span}}+1)
        \cdot\iota_N}{N_{\mathrm{nom}}}}
        +\frac{C}{N_{\mathrm{nom}}},
        \label{eq:fixed-discount-nominal-population-span}
        \\
        \left\|
        \widehat V_\eta^{\widehat\pi_\eta^0,0}
        -V_\eta^{\widehat\pi_\eta^0,0}
        \right\|_\infty
        &\le
        \frac{C}{1-\eta}
        \sqrt{
        \frac{(\|\widehat V_\eta^{\widehat\pi_\eta^0,0}
        \|_{\mathrm{span}}+1)\cdot\iota_N}
        {N_{\mathrm{nom}}}}
        +C\frac{\iota_N}{N_{\mathrm{nom}}(1-\eta)^2}.
        \label{eq:fixed-discount-nominal-policy-evaluation}
    \end{align}
    \end{subequations}
\end{lemma}

\paragraph*{Anchor certificates.}
The next two lemmas connect the population anchors used in the analysis with
the certificates computed from $\mathcal D_{\mathrm{nom}}$. The first
associates each discount factor with a population anchor. The second collects
the validity, oracle, and selection properties of the empirical certificates
in a single calibration result.

We begin with the population anchors. The following lemma identifies a valid
anchor for each discount factor $\lambda\in(0,1)$. Its proof is deferred to
Appendix~\ref{subsec:proof-discounted-nominal-anchors}.
\begin{lemma}[Population nominal anchors]
    \label{lem:discounted-nominal-anchors}
    Fix $\lambda\in(0,1)$ and let
    \[
        h_\lambda^0
        =
        V_\lambda^{\star,0}-\min_s V_\lambda^{\star,0}(s),
    \]
    and
    \[
        \bar\rho_\lambda^0
        =
        (1-\lambda)\max_s V_\lambda^{\star,0}(s).
    \]
    Then, for every $s\in\cS$,
    \[
        \bar\rho_\lambda^0+h_\lambda^0(s)
        \ge
        \max_{a\in\cA}
        \left\{
        r(s,a)+P^0_{s,a}h_\lambda^0
        \right\}.
    \]
\end{lemma}

Recall from \eqref{eq:anchored-nominal-certificate} that
$(\widehat\rho_\lambda^+,\widehat H_\lambda^+)$ is the candidate certificate
computed at $\lambda$. The following lemma gives all properties of these
certificates used below. Its proof is deferred to
Appendix~\ref{subsec:proof-anchor-calibration-guarantees}.
\begin{lemma}[Anchor calibration guarantees]
    \label{lem:anchor-calibration-guarantees}
    \label{lem:nominal-anchor-certification}
    \label{lem:efficient-nominal-anchor}
    \label{lem:selected-nominal-anchor}
    Let $N\ge16$ and, for each $\lambda\in\Lambda_N$, let
    \[
        H_\lambda^0
        =
        \max\{1,\|h_\lambda^0\|_{\mathrm{span}}\}.
    \]
    If $C_{\mathrm{anc}}$ in
    \eqref{eq:anchored-nominal-certificate} is sufficiently large, then,
    with probability at least $1-O(\delta)$, simultaneously for every
    $\lambda\in\Lambda_N$,
    \begin{subequations}
    \label{eq:nominal-anchor-certificate-validity}
    \begin{align}
        \left|
        \bar\rho_\lambda^0
        -(1-\lambda)\max_s
        \widehat V_\lambda^{\widehat\pi_\lambda^0,0}(s)
        \right|
        &\le
        \widehat\rho_\lambda^+
        -(1-\lambda)\max_s
        \widehat V_\lambda^{\widehat\pi_\lambda^0,0}(s),
        \label{eq:nominal-anchor-baseline-certificate}
        \\
        \bar\rho_\lambda^0
        &\le
        \widehat\rho_\lambda^+,
        \label{eq:nominal-anchor-level-certificate}
        \\
        H_\lambda^0
        &\le
        \widehat H_\lambda^+.
        \label{eq:nominal-anchor-span-certificate}
    \end{align}
    \end{subequations}
    Moreover, there exist a sufficiently large universal constant $C>0$ and
    a sufficiently small universal constant $c_{\mathrm{opt}}>0$ such that,
    if
    \begin{equation}
        \label{eq:efficient-nominal-anchor-sample-condition}
        N_{\mathrm{nom}}
        \ge
        C\frac{H_0\cdot\iota_N}{\varepsilon^2},
        \qquad
        \varepsilon_{\mathrm{opt}}
        \le c_{\mathrm{opt}}\varepsilon,
    \end{equation}
    then the following conclusions hold on the same event:
    \begin{enumerate}[label=\textup{(\roman*)},leftmargin=2.5em]
        \item there exists $\lambda^\dagger\in\Lambda_N$ satisfying
        \[
            \frac{1}{1-\lambda^\dagger}
            \asymp
            \min\left\{
            \sqrt{\frac{N_{\mathrm{nom}}H_0}{\iota_N}},
            \frac{H_0}{\varepsilon_{\mathrm{opt}}}
            \right\},
        \]
        \[
            \widehat H_{\lambda^\dagger}^+
            \le C H_0,
            \qquad
            \widehat\rho_{\lambda^\dagger}^+
            \le
            \rho^\star
            +C\sqrt{\frac{H_0\cdot\iota_N}{N_{\mathrm{nom}}}}
            +C\varepsilon_{\mathrm{opt}};
        \]
        \item the selected certificate satisfies
        \begin{subequations}
        \label{eq:selected-anchor-certificate-bounds}
        \begin{align}
            \widehat Q_{\widehat\lambda_{\mathrm{anc}}}
            &\le C\varepsilon,
            \label{eq:selected-anchor-quality-bound}
            \\
            \widehat H_{\widehat\lambda_{\mathrm{anc}}}^+
            &\le C H_0,
            \label{eq:selected-anchor-span-bound}
            \\
            \rho^\star
            \le
            \bar\rho_{\widehat\lambda_{\mathrm{anc}}}^0
            \le
            \widehat\rho_{\widehat\lambda_{\mathrm{anc}}}^+
            &\le
            \rho^\star+C\varepsilon.
            \label{eq:selected-anchor-level-bound}
        \end{align}
        \end{subequations}
    \end{enumerate}
    When $\varepsilon_{\mathrm{opt}}=0$, the second quantity in the minimum
    is interpreted as $+\infty$.
\end{lemma}

\paragraph*{Shared robust comparison.}
The nominal and robust policy candidates require different discounted-value
comparisons. Conditional on the nominal batch, each nominal candidate is fixed
independently of $\mathcal D_{\mathrm{rob}}$. The fixed-policy extension of
Lemma~\ref{lem:V-hat-V-part-1-decomp} therefore controls
\[
    \left\|
    \widehat V_\gamma^{\pi,\sigma}
    -
    V_\gamma^{\pi,\sigma}
    \right\|_\infty
\]
for these candidates. By contrast, $\widehat\pi_\gamma$ is computed from
$\mathcal D_{\mathrm{rob}}$ itself, so the learned-policy bound in
Lemma~\ref{lem:V-hat-V-pi-hat-bound} is needed to control
\[
    \left\|
    \widehat V_\gamma^{\widehat\pi_\gamma,\sigma}
    -
    V_\gamma^{\widehat\pi_\gamma,\sigma}
    \right\|_\infty.
\]
The former is used with the trivial anchor, whereas the latter is used with
both the population anchor indexed by $\widehat\lambda_{\mathrm{anc}}$ and the
trivial anchor.

Both bounds require a deterministic number $R$ that upper-bounds the relevant
value spans; see \eqref{eq:discounted-localized-span-condition}. The required
radius depends on empirical robust values and therefore cannot be selected
after observing $\mathcal D_{\mathrm{rob}}$. We instead fix a dyadic grid of
candidate radii independently of $\mathcal D_{\mathrm{rob}}$ and make both
bounds hold simultaneously at every grid point. The following lemma formalizes
this simultaneous comparison.
Its proof is deferred to
Appendix~\ref{subsec:proof-raw-localized-robust-comparison}.
\begin{lemma}[Simultaneous comparison over candidate span bounds]
    \label{lem:raw-localized-robust-comparison}
    Fix $\gamma\in\Gamma_N$. Independently of $\mathcal D_{\mathrm{rob}}$,
    choose a deterministic policy $\pi$, an anchor pair
    $(\bar\rho,\bar h)$ satisfying
    \mainNominalAnchorSupersolution{}, and a dyadic grid
    \[
        R_j=2^jR_{\min},
        \qquad
        j=0,\ldots,J,
    \]
    where
    \[
        R_{\min}
        \ge
        \max\{1,\|\bar h\|_{\mathrm{span}}\},
        \qquad
        J=O(\log N).
    \]
    Then, with probability at least $1-O(\delta/N^3)$, the following
    conclusions hold simultaneously for every $j=0,\ldots,J$. Whenever
    \[
        R_j
        \ge
        \max\left\{
        1,
        \|\bar h\|_{\mathrm{span}},
        \|V_\gamma^{\pi,\sigma}\|_{\mathrm{span}},
        \|\widehat V_\gamma^{\pi,\sigma}\|_{\mathrm{span}}
        \right\},
    \]
    the fixed-policy bound
    \eqref{eq:generic-fixed-policy-localized-comparison} from
    Lemma~\ref{lem:V-hat-V-part-1-decomp} holds for $\pi$.
    If, in addition, $\pi=\pi_\gamma^\star$ and
    \[
        R_j
        \ge
        \max\left\{
        \|V_\gamma^{\widehat\pi_\gamma,\sigma}\|_{\mathrm{span}},
        \|\widehat V_\gamma^{\widehat\pi_\gamma,\sigma}\|_{\mathrm{span}},
        \|\widehat V_\gamma^{\star,\sigma}\|_{\mathrm{span}}
        \right\},
    \]
    then the learned-policy bound from
    Lemma~\ref{lem:V-hat-V-pi-hat-bound} also holds for the policy
    $\widehat\pi_\gamma$ returned by the empirical robust solver. In both
    bounds, $N_{\mathrm{rob}}$ and $\iota_N$ take the places of $N$ and $\iota$.
\end{lemma}

\paragraph*{Nominal policy candidates.}
The following lemma collects the validity and existence properties of the
nominal policy candidates. Its proof is deferred to
Appendix~\ref{subsec:proof-nominal-policy-candidate-guarantees}.
\begin{lemma}[Nominal-policy candidate guarantees]
    \label{lem:nominal-policy-candidate-guarantees}
    \label{lem:nominal-policy-suitable-grid-point}
    \label{lem:observable-fixed-policy-robust-evaluation}
    \label{lem:existence-good-nominal-policy-candidate}
    With probability at least $1-O(\delta)$, simultaneously for every
    $\gamma\in\Gamma_N$,
    \begin{align}
        (1-\gamma)
        \left\|
        \widehat V_\gamma^{\widehat\pi_\gamma^0,\sigma}
        -V_\gamma^{\widehat\pi_\gamma^0,\sigma}
        \right\|_\infty
        &\le
        \operatorname{pen}_\gamma^{\mathrm{nom}}(\widehat\pi_\gamma^0),
        \label{eq:observable-fixed-policy-robust-evaluation}
        \\
        \rho^{\widehat\pi_\gamma^0,\sigma}
        &\ge
        \operatorname{LCB}_{\gamma}^{\mathrm{nom}}(\widehat\pi_\gamma^0).
        \label{eq:nominal-lower-certificate-validity}
    \end{align}
    Moreover, there exist sufficiently large universal constants
    $C,C_{\mathrm{nom}}>0$ and a sufficiently small universal constant $c>0$
    such that, if
    \begin{equation}
        \label{eq:nominal-policy-grid-sample-condition}
        \sigma H_0\le c\varepsilon,
        \qquad
        N\ge C\frac{H_0\cdot\iota_N}{\varepsilon^2},
    \end{equation}
    and $N_{\mathrm{nom}}$ and $N_{\mathrm{rob}}$ are constant fractions of
    $N$, then $\Gamma_N$ contains a discount factor
    $\gamma_{\mathrm{nom}}^\dagger$ satisfying
    \begin{equation}
        \label{eq:nominal-grid-horizon}
        \frac{C_{\mathrm{nom}}H_0}{\varepsilon}
        \le
        \frac{1}{1-\gamma_{\mathrm{nom}}^\dagger}
        \le
        \frac{2C_{\mathrm{nom}}H_0}{\varepsilon}.
    \end{equation}
    On the same event,
    \begin{align}
        (1-\gamma_{\mathrm{nom}}^\dagger)
        \min_s V_{\gamma_{\mathrm{nom}}^\dagger}^{
        \widehat\pi_{\gamma_{\mathrm{nom}}^\dagger}^0,0}(s)
        &\ge \rho^\star-c\varepsilon,
        \label{eq:nominal-grid-policy-discounted-reward}
        \\
        \left\|V_{\gamma_{\mathrm{nom}}^\dagger}^{
        \widehat\pi_{\gamma_{\mathrm{nom}}^\dagger}^0,0}
        \right\|_{\mathrm{span}}
        &\le C H_0,
        \label{eq:nominal-grid-policy-span}
        \\
        \operatorname{LCB}_{\gamma_{\mathrm{nom}}^\dagger}^{\mathrm{nom}}
        (\widehat\pi_{\gamma_{\mathrm{nom}}^\dagger}^0)
        &\ge \rho^{\star,\sigma}-\varepsilon.
        \label{eq:good-nominal-policy-certificate}
    \end{align}
\end{lemma}

\paragraph*{Robust policy candidates.}
The following lemma collects the validity and existence properties of the
policies learned from the robust batch. Its proof is deferred to
Appendix~\ref{subsec:proof-robust-policy-candidate-guarantees}.
\begin{lemma}[Robust-policy candidate guarantees]
    \label{lem:robust-policy-candidate-guarantees}
    \label{lem:observable-budget-anchored-dmdp-comparison}
    \label{lem:anchored-lower-certificate-validity}
    \label{lem:existence-good-anchored-pair}
    There exists a sufficiently large universal constant $C>0$ such that the
    following holds.
    For each $\gamma\in\Gamma_N$, let $\pi_\gamma^\star$ be a deterministic
    discounted-optimal policy for the true robust MDP and define
    \begin{equation}
        \label{eq:anchored-Delta-gamma}
        \Delta_\gamma
        =
        \left\|
        \widehat V_\gamma^{\widehat\pi_\gamma,\sigma}
        -V_\gamma^{\widehat\pi_\gamma,\sigma}
        \right\|_\infty
        +
        \|\widehat V_\gamma^{\pi_\gamma^\star,\sigma}
        -V_\gamma^{\star,\sigma}\|_\infty.
    \end{equation}
    Assume $\varepsilon_{\mathrm{opt}}\le1$. Then, with probability at least
    $1-O(\delta)$, the following two bounds hold simultaneously for every
    $\gamma\in\Gamma_N$. The selected anchor certificate gives
    \begin{equation}
        \label{eq:observable-budget-selected-anchor-bound}
        (1-\gamma)\Delta_\gamma
        \le
        C\sqrt{
        \frac{\widehat B_{\gamma,\widehat\lambda_{\mathrm{anc}}}
        \cdot\iota_N}{N_{\mathrm{rob}}}}
        +C\frac{
        \widehat R_{\gamma,\widehat\lambda_{\mathrm{anc}}}^{\mathrm{anc}}
        \cdot\iota_N}
        {N_{\mathrm{rob}}}
        +C\frac{
        \widehat\beta_{\gamma,\widehat\lambda_{\mathrm{anc}}}\cdot\iota_N}
        {N_{\mathrm{rob}}(1-\gamma)}
        +C\varepsilon_{\mathrm{opt}}.
    \end{equation}
    The trivial anchor $(\rho,h)=(1,0)$ gives
    \begin{equation}
        \label{eq:observable-budget-trivial-anchor-bound}
        (1-\gamma)\Delta_\gamma
        \le
        C\sqrt{
        \frac{\widehat B_{\gamma,\mathrm{triv}}\cdot\iota_N}{N_{\mathrm{rob}}}}
        +C\frac{\widehat R_\gamma\cdot\iota_N}{N_{\mathrm{rob}}}
        +C\frac{\widehat\beta_{\gamma,\mathrm{triv}}\cdot\iota_N}
        {N_{\mathrm{rob}}(1-\gamma)}
        +C\varepsilon_{\mathrm{opt}}.
    \end{equation}
    On the same event, the following holds for every
    $\gamma\in\Gamma_N$:
    \[
        \rho^{\widehat\pi_\gamma,\sigma}
        \ge
        \operatorname{LCB}_{\gamma}^{\mathrm{rob}}(\widehat\pi_\gamma).
    \]
    Moreover, there exist sufficiently large universal constants
    $C,C_{\mathrm{rob}}>0$ and a sufficiently small universal constant
    $c_{\mathrm{opt}}>0$ such that the following holds.
    Let $N\ge16$ and $\varepsilon\in(0,1]$. Suppose that
    \begin{equation}
        \label{eq:good-robust-candidate-sample-condition}
        N
        \ge
        C
        \left[
        \frac{\min\{H_0,H_\sigma\}+\sigma H_\sigma^2}{\varepsilon^2}
        +\frac{H_\sigma}{\varepsilon}
        \right]\cdot\iota_N
    \end{equation}
    and that
    $N_{\mathrm{nom}}$ and $N_{\mathrm{rob}}$ are constant
    fractions of $N$. Assume also
    $\varepsilon_{\mathrm{opt}}\le
    c_{\mathrm{opt}}\varepsilon$. Then the policy grid
    $\Gamma_N$ contains a discount factor $\gamma^\dagger$ such that
    \[
        \frac{C_{\mathrm{rob}} H_\sigma}{\varepsilon}
        \le
        \frac{1}{1-\gamma^\dagger}
        \le
        \frac{2C_{\mathrm{rob}} H_\sigma}{\varepsilon}.
    \]
    Moreover, with probability at least $1-O(\delta)$,
    \begin{equation}
        \label{eq:good-robust-policy-certificate}
        \operatorname{LCB}_{\gamma^\dagger}^{\mathrm{rob}}
        (\widehat\pi_{\gamma^\dagger})
        \ge
        \rho^{\star,\sigma}-\varepsilon.
    \end{equation}
\end{lemma}

\subsection{Proof of Theorem~\mainSpanAgnosticUpperBoundTheorem{}}
\label{subsec:proof-anchored-calibrated-policy-upper-bound}

We prove Theorem~\mainSpanAgnosticUpperBoundTheorem{} by combining
the candidate guarantees established in
Appendix~\ref{subsec:anchored-calibration-lemmas}.

\paragraph*{Sample-complexity check in different regimes.}
Recall that the theorem's sample bound is
\[
    NSA
    \ge
    CSA\frac{\min\{H_0,H_\sigma\}}{\varepsilon^2}\cdot\iota_N
\]
in the high-tolerance regime and
\[
    NSA
    \ge
    CSA\frac{\min\{H_0,H_\sigma\}+\sigma H_\sigma^2}
    {\varepsilon^2}\cdot\iota_N
\]
in the low-tolerance regime.

We first verify the sample-complexity conditions required by the
nominal- and robust-policy candidate guarantees
(Lemmas~\ref{lem:nominal-policy-candidate-guarantees} and
\ref{lem:robust-policy-candidate-guarantees}).

Consider the low-tolerance regime in which $\sigma H_0>c\varepsilon$. We verify the sample-complexity condition of
Lemma~\ref{lem:robust-policy-candidate-guarantees} in
\eqref{eq:good-robust-candidate-sample-condition}. After dividing the
theorem's total-sample bound by $SA$, its low-tolerance bound already controls
the first two terms in that requirement. It remains to control
$H_\sigma/\varepsilon$. If $H_\sigma\le H_0$, then
\[
    \frac{H_\sigma}{\varepsilon}
    \le
    \frac{\min\{H_0,H_\sigma\}}{\varepsilon^2},
\]
where we used $\varepsilon\le1$. If instead $H_0<H_\sigma$, the condition
$\sigma H_0>c\varepsilon$ gives
\[
    \frac{H_\sigma}{\varepsilon}
    \le
    \frac{1}{c}
    \frac{\sigma H_\sigma^2}{\varepsilon^2}.
\]
Thus the theorem's low-tolerance sample bound implies
\eqref{eq:good-robust-candidate-sample-condition}.

Consider the high-tolerance regime in which $\sigma H_0\le c\varepsilon$. We first look at the case with $H_0<H_\sigma$. We verify the sample-complexity condition of
Lemma~\ref{lem:nominal-policy-candidate-guarantees} in
\eqref{eq:nominal-policy-grid-sample-condition}. Its tolerance
condition holds by assumption, and
$\min\{H_0,H_\sigma\}=H_0$. Hence, after dividing the theorem's total-sample
bound by $SA$, its high-tolerance bound gives
\[
    N
    \ge
    C\frac{H_0\iota_N}{\varepsilon^2},
\]
which verifies \eqref{eq:nominal-policy-grid-sample-condition}.

We now look at the case with $H_\sigma\le H_0$. We verify the sample-complexity condition of
Lemma~\ref{lem:robust-policy-candidate-guarantees} in \eqref{eq:good-robust-candidate-sample-condition}. Since
$\min\{H_0,H_\sigma\}=H_\sigma$, the theorem's high-tolerance sample bound
controls $H_\sigma/\varepsilon^2$. Moreover,
\[
    \frac{H_\sigma}{\varepsilon}
    \le
    \frac{H_\sigma}{\varepsilon^2},
    \qquad
    \frac{\sigma H_\sigma^2}{\varepsilon^2}
    \le
    c\frac{H_\sigma}{\varepsilon}
    \le
    c\frac{H_\sigma}{\varepsilon^2},
\]
where the second inequality uses
$\sigma H_\sigma\le\sigma H_0\le c\varepsilon$. Thus the theorem's
high-tolerance sample bound implies
\eqref{eq:good-robust-candidate-sample-condition}.

In addition, $N_{\mathrm{nom}}$ and $N_{\mathrm{rob}}$ are constant fractions
of $N$, and
$\varepsilon_{\mathrm{opt}}\le c_{\mathrm{opt}}\varepsilon\le1$.
Therefore all supporting lemmas invoked below apply. Throughout the rest of the
proof, we work on the intersection of the high-probability events from
Lemmas~\ref{lem:nominal-policy-candidate-guarantees} and
\ref{lem:robust-policy-candidate-guarantees}, using the nominal existence
conclusion in the high-tolerance case $H_0<H_\sigma$ and the robust existence
conclusion otherwise. In every case, this intersection has probability at
least $1-O(\delta)$.

\paragraph*{Candidate validity and selection.}
On the event fixed above, the following two bounds hold for every
$\gamma\in\Gamma_N$. Lemma~\ref{lem:robust-policy-candidate-guarantees} gives
\begin{equation}
    \label{eq:all-robust-candidate-certificates-valid}
    \rho^{\widehat\pi_\gamma,\sigma}
    \ge
    \operatorname{LCB}_{\gamma}^{\mathrm{rob}}(\widehat\pi_\gamma),
\end{equation}
and Lemma~\ref{lem:nominal-policy-candidate-guarantees} gives
\begin{equation}
    \label{eq:all-nominal-candidate-certificates-valid}
    \rho^{\widehat\pi_\gamma^0,\sigma}
    \ge
    \operatorname{LCB}_{\gamma}^{\mathrm{nom}}(\widehat\pi_\gamma^0).
\end{equation}

The candidate-existence lemma selected above ensures that $\Gamma_N$ is
nonempty. 
The algorithm returns a candidate with the largest lower-confidence bound.
Therefore
\begin{equation}
    \label{eq:combined-certificate-selector-bound}
    \rho^{\widehat\pi,\sigma}
    \ge
    \max\left\{
    \max_{\gamma\in\Gamma_N}
    \operatorname{LCB}_{\gamma}^{\mathrm{rob}}(\widehat\pi_\gamma),
    \max_{\gamma\in\Gamma_N}
    \operatorname{LCB}_{\gamma}^{\mathrm{nom}}(\widehat\pi_\gamma^0)
    \right\}.
\end{equation}

To lower-bound the maximum in
\eqref{eq:combined-certificate-selector-bound}, suppose first that either the
low-tolerance condition holds or the
high-tolerance condition holds together with $H_\sigma\le H_0$.
Lemma~\ref{lem:robust-policy-candidate-guarantees} supplies
$\gamma^\dagger\in\Gamma_N$ satisfying
\eqref{eq:good-robust-policy-certificate}. Hence
\eqref{eq:combined-certificate-selector-bound} and
\eqref{eq:good-robust-policy-certificate} give
\[
    \rho^{\widehat\pi,\sigma}
    \ge
    \operatorname{LCB}_{\gamma^\dagger}^{\mathrm{rob}}
    (\widehat\pi_{\gamma^\dagger})
    \ge
    \rho^{\star,\sigma}-\varepsilon.
\]

Suppose instead that the high-tolerance condition holds and $H_0<H_\sigma$.
Lemma~\ref{lem:nominal-policy-candidate-guarantees} supplies
$\gamma_{\mathrm{nom}}^\dagger\in\Gamma_N$ satisfying
\eqref{eq:good-nominal-policy-certificate}. Hence
\eqref{eq:combined-certificate-selector-bound} and
\eqref{eq:good-nominal-policy-certificate} give
\[
    \rho^{\widehat\pi,\sigma}
    \ge
    \operatorname{LCB}_{\gamma_{\mathrm{nom}}^\dagger}^{\mathrm{nom}}
    (\widehat\pi_{\gamma_{\mathrm{nom}}^\dagger}^0)
    \ge
    \rho^{\star,\sigma}-\varepsilon.
\]
Thus, in every case,
\[
    \rho^{\star,\sigma}-\rho^{\widehat\pi,\sigma}
    \le
    \varepsilon.
\]
This completes the proof of
Theorem~\mainSpanAgnosticUpperBoundTheorem{}.

\subsection{Proof of auxiliary lemmas}
\label{subsec:proof-anchored-calibration-auxiliary-lemmas}

\subsubsection{Proof of Lemma~\ref{lem:fixed-discount-nominal-plugin}}
\label{subsec:proof-fixed-discount-nominal-plugin}

If $\mathcal G_N$ is empty, the conclusion is vacuous. Otherwise, fix
$\eta\in\mathcal G_N$. The proof has three steps. First, we construct a
population nominal anchor and simultaneously bound the returned policy's
plug-in evaluation error and nominal discounted suboptimality. Second, we use
these bounds to prove the population- and empirical-span forms
\eqref{eq:fixed-discount-nominal-population-span} and
\eqref{eq:fixed-discount-nominal-empirical-span}. Third, we prove the
fixed-policy bound \eqref{eq:fixed-discount-nominal-policy-evaluation} and take
a union bound over the discount-factor grid.

\paragraph*{Step 1: Nominal anchor and comparison bound.}
We specialize the proof of
Theorem~\ref{thm:TV-upper-bound-discount} to $\sigma=0$. Use the population
nominal anchor
\[
    \bar h
    =
    V_\eta^{\star,0}
    -\min_sV_\eta^{\star,0}(s)\bm 1_S,
    \qquad
    \bar\rho
    =
    (1-\eta)\max_sV_\eta^{\star,0}(s).
\]
For every $(s,a)\in\cS\times\cA$, the discounted Bellman equation gives
\begin{align*}
    r(s,a)+P^0_{s,a}\bar h
    &=
    r(s,a)+P^0_{s,a}V_\eta^{\star,0}
    -\min_xV_\eta^{\star,0}(x)                                      \\
    &\le
    V_\eta^{\star,0}(s)
    +(1-\eta)P^0_{s,a}V_\eta^{\star,0}
    -\min_xV_\eta^{\star,0}(x)                                      \\
    &\le
    \bar h(s)+(1-\eta)\max_xV_\eta^{\star,0}(x)                     \\
    &=
    \bar\rho+\bar h(s).
\end{align*}
Thus $(\bar\rho,\bar h)$ satisfies
\mainNominalAnchorSupersolution{}. In the notation of
Theorem~\ref{thm:TV-upper-bound-discount}, take
\[
    H_{\mathrm{anc}}
    =R_0
    =
    \max\{1,\|V_\eta^{\star,0}\|_{\mathrm{span}}\}.
\]
The anchor defect then satisfies
\begin{equation}
    \beta_\star
    =
    (1-\eta)\|V_\eta^{\star,0}\|_{\mathrm{span}}
    \le
    (1-\eta)R_0,
    \label{eq:nominal-fixed-discount-anchor-defect}
\end{equation}
and the discounted-value range gives
$R_0\le(1-\eta)^{-1}$.

Allocate failure probability $\delta/|\mathcal G_N|$ to this discount factor.
After replacing $N$ by $N_{\mathrm{nom}}$ and $\delta$ by
$\delta/|\mathcal G_N|$ in the proof of
Theorem~\ref{thm:TV-upper-bound-discount}, its confidence factor becomes
\[
    \log\!\left(
    \frac{54SA N_{\mathrm{nom}}^2|\mathcal G_N|}
    {(1-\eta)\delta}
    \right).
\]
This factor already accounts for the union bound over the radius grid in
\eqref{eq:global-localized-event}. Since $N_{\mathrm{nom}}\le N$,
$(1-\eta)^{-1}\le N_{\mathrm{nom}}$, and
$|\mathcal G_N|=O(\log N)$, it is at most $C\iota_N$.

Suppose first that
\[
    N_{\mathrm{nom}}
    \ge
    C\frac{\iota_N}{1-\eta}.
\]
Then the recursive comparison conditions
\eqref{eq:recursive-comparison-sample-conditions} hold with $\sigma=0$.
Moreover, the empirical solver error $1/N_{\mathrm{nom}}$ is at most
$c(1-\eta)R_0$. Apply
\eqref{eq:recursive-comparison-before-young} with $\sigma=0$, $\gamma=\eta$,
$N=N_{\mathrm{nom}}$, and
$\varepsilon_{\mathrm{opt}}=1/N_{\mathrm{nom}}$. Applying Young's inequality
to the square-root term involving the comparison error and then absorbing all
terms proportional to that error using
\eqref{eq:recursive-comparison-sample-conditions} gives
\[
\begin{aligned}
    &\max\Big\{
    \|\widehat V_\eta^{\widehat\pi_\eta^0,0}
      -V_\eta^{\widehat\pi_\eta^0,0}\|_\infty,
    \|V_\eta^{\star,0}
      -V_\eta^{\widehat\pi_\eta^0,0}\|_\infty
    \Big\}
    \\
    &\qquad\le
    C\sqrt{
    \frac{B_\star(R_0)\cdot\iota_N}
    {N_{\mathrm{nom}}(1-\eta)^2}}
    +C\frac{\beta_\star\cdot\iota_N}
    {N_{\mathrm{nom}}(1-\eta)^2}
    \\
    &\qquad\quad
    +C\frac{R_0\cdot\iota_N}
    {N_{\mathrm{nom}}(1-\eta)}
    +C\frac{1}{N_{\mathrm{nom}}(1-\eta)}.
\end{aligned}
\]
Here we retain the remainder
\[
    C\frac{\beta_\star\cdot\iota_N}
    {N_{\mathrm{nom}}(1-\eta)^2}
\]
from Young's inequality instead of invoking the additional sample condition
used there to absorb it into the leading square-root term. By
\eqref{eq:nominal-fixed-discount-anchor-defect} and
$R_0\le(1-\eta)^{-1}$, this remainder and the linear-radius term are both at
most
\[
    C\frac{\iota_N}{N_{\mathrm{nom}}(1-\eta)^2}.
\]
The final solver-error term is absorbed into the same bound because
$\iota_N\ge 1$ and $1-\eta\le 1$.
Also, the nominal specialization of the localized budget satisfies
\[
    B_\star(R_0)
    =
    R_0+R_0\beta_\star+(1-\eta)R_0^2
    \le
    3R_0.
\]
Consequently,
\begin{equation}
\label{eq:nominal-fixed-discount-population-comparison}
\begin{aligned}
    &\max\Big\{
    \|\widehat V_\eta^{\widehat\pi_\eta^0,0}
      -V_\eta^{\widehat\pi_\eta^0,0}\|_\infty,
    \|V_\eta^{\star,0}
      -V_\eta^{\widehat\pi_\eta^0,0}\|_\infty
    \Big\}
    \\
    &\qquad\le
    \frac{C}{1-\eta}
    \sqrt{
    \frac{R_0\cdot\iota_N}{N_{\mathrm{nom}}}}
    +C\frac{\iota_N}
    {N_{\mathrm{nom}}(1-\eta)^2}.
\end{aligned}
\end{equation}
If instead
$N_{\mathrm{nom}}<C\iota_N/(1-\eta)$, the second term on the
right-hand side of
\eqref{eq:nominal-fixed-discount-population-comparison} dominates
$(1-\eta)^{-1}$. Both norms on the left-hand
side are at most $(1-\eta)^{-1}$ by the discounted-value range. Thus
\eqref{eq:nominal-fixed-discount-population-comparison} holds in this case as
well.

\paragraph*{Step 2: Optimal-value error bounds.}
The assumed empirical near-optimality and the triangle inequality now give
\begin{align*}
    \|\widehat V_\eta^{\star,0}-V_\eta^{\star,0}\|_\infty
    &\le
    \|\widehat V_\eta^{\star,0}
      -\widehat V_\eta^{\widehat\pi_\eta^0,0}\|_\infty+
    \|\widehat V_\eta^{\widehat\pi_\eta^0,0}
      -V_\eta^{\widehat\pi_\eta^0,0}\|_\infty
    \\
    &\quad+
    \|V_\eta^{\widehat\pi_\eta^0,0}
      -V_\eta^{\star,0}\|_\infty.
\end{align*}
Substituting \eqref{eq:nominal-fixed-discount-population-comparison} proves
\eqref{eq:fixed-discount-nominal-population-span}; the additional
$1/N_{\mathrm{nom}}$ solver term is the final term displayed there.

For the empirical-span form, apply
\eqref{eq:span-triangle-inequality} with
$V_1=V_\eta^{\star,0}$ and $V_2=\widehat V_\eta^{\star,0}$ to obtain
\[
    R_0
    \le
    \|\widehat V_\eta^{\star,0}\|_{\mathrm{span}}
    +1
    +2\|\widehat V_\eta^{\star,0}-V_\eta^{\star,0}\|_\infty.
\]
Substitute this inequality into
\eqref{eq:fixed-discount-nominal-population-span},
use $\sqrt{x+y}\le\sqrt{x}+\sqrt y$, and apply Young's inequality to the
term containing the square root of
$\|\widehat V_\eta^{\star,0}-V_\eta^{\star,0}\|_\infty$.
After moving one half of this error to the left-hand side, we obtain
\eqref{eq:fixed-discount-nominal-empirical-span}.

\paragraph*{Step 3: Returned-policy evaluation.}
Finally, two applications of \eqref{eq:span-triangle-inequality}, together
with \eqref{eq:nominal-fixed-discount-population-comparison}, give
\[
    R_0
    \le
    \|\widehat V_\eta^{\widehat\pi_\eta^0,0}\|_{\mathrm{span}}
    +1
    +4\max\Big\{
    \|\widehat V_\eta^{\widehat\pi_\eta^0,0}
      -V_\eta^{\widehat\pi_\eta^0,0}\|_\infty,
    \|V_\eta^{\star,0}
      -V_\eta^{\widehat\pi_\eta^0,0}\|_\infty
    \Big\}.
\]
Substituting this inequality back into
\eqref{eq:nominal-fixed-discount-population-comparison} and applying the same
Young-inequality absorption proves
\eqref{eq:fixed-discount-nominal-policy-evaluation}. A union bound over
$\eta\in\mathcal G_N$ completes the proof.

\subsubsection{Proof of Lemma~\ref{lem:discounted-nominal-anchors}}
\label{subsec:proof-discounted-nominal-anchors}

By the discounted Bellman equation,
\[
    V_\lambda^{\star,0}(s)
    =
    \max_{a\in\cA}
    \left\{
    r(s,a)+\lambda P^0_{s,a}V_\lambda^{\star,0}
    \right\}.
\]
By the definition of $h_\lambda^0$,
$V_\lambda^{\star,0}
=h_\lambda^0+(\min_{s'}V_\lambda^{\star,0}(s'))\bm{1}_S$,
so substituting this identity into the preceding display and subtracting
$\lambda\min_{s'}V_\lambda^{\star,0}(s')$ from both sides gives
\[
    h_\lambda^0(s)+(1-\lambda)\min_{s'}V_\lambda^{\star,0}(s')
    =
    \max_{a\in\cA}
    \left\{
    r(s,a)+\lambda P^0_{s,a}h_\lambda^0
    \right\}.
\]
Moreover, the definition of $h_\lambda^0$ implies
\[
    0
    \le
    h_\lambda^0
    \le
    \|h_\lambda^0\|_{\mathrm{span}}\bm{1}_S,
    \qquad
    \max_{s'}V_\lambda^{\star,0}(s')
    =
    \min_{s'}V_\lambda^{\star,0}(s')
    +
    \|h_\lambda^0\|_{\mathrm{span}}.
\]
Thus, for every $s\in\cS$ and $a\in\cA$,
\[
    r(s,a)+P^0_{s,a}h_\lambda^0
    =
    r(s,a)+\lambda P^0_{s,a}h_\lambda^0
    +(1-\lambda)P^0_{s,a}h_\lambda^0
    \le
    r(s,a)+\lambda P^0_{s,a}h_\lambda^0
    +(1-\lambda)\|h_\lambda^0\|_{\mathrm{span}}.
\]
Taking the maximum over $a$ and using the preceding identities, we obtain
\begin{align*}
    \max_{a\in\cA}
    \left\{
    r(s,a)+P^0_{s,a}h_\lambda^0
    \right\}
    &\le
    \max_{a\in\cA}
    \left\{
    r(s,a)+\lambda P^0_{s,a}h_\lambda^0
    \right\}
    +(1-\lambda)\|h_\lambda^0\|_{\mathrm{span}}
    \\
    &=
    h_\lambda^0(s)
    +(1-\lambda)\min_{s'}V_\lambda^{\star,0}(s')
    +(1-\lambda)\|h_\lambda^0\|_{\mathrm{span}}
    \\
    &=
    h_\lambda^0(s)
    +(1-\lambda)\max_{s'}V_\lambda^{\star,0}(s')
    \\
    &=
    h_\lambda^0(s)+\bar\rho_\lambda^0.
\end{align*}
This is precisely the claimed nominal anchor supersolution inequality.

\subsubsection{Proof of Lemma~\ref{lem:anchor-calibration-guarantees}}
\label{subsec:proof-anchor-calibration-guarantees}
\label{subsec:proof-nominal-anchor-certification}
\label{subsec:proof-efficient-nominal-anchor}
\label{subsec:proof-selected-nominal-anchor}

Apply Lemma~\ref{lem:fixed-discount-nominal-plugin} with
$\mathcal G_N=\Lambda_N$. The grid and solver conditions required there
follow from \eqref{eq:anchored-Lambda-N} and
\eqref{eq:anchored-nominal-solver-tolerance}. Hence, simultaneously for
every $\lambda\in\Lambda_N$, the span triangle inequality and Young's
inequality give
\begin{align*}
    &\left|
    \bar\rho_\lambda^0
    -(1-\lambda)\max_s
    \widehat V_\lambda^{\widehat\pi_\lambda^0,0}(s)
    \right|
    \le
    C_{\mathrm{anc}}\left(
    \sqrt{\frac{\widehat H_\lambda^0\iota_N}{N_{\mathrm{nom}}}}
    +\frac{\iota_N}{N_{\mathrm{nom}}(1-\lambda)}
    +\varepsilon_{\mathrm{opt}}
    \right),                                                    \\
    &\left\|
    V_\lambda^{\star,0}
    -\widehat V_\lambda^{\widehat\pi_\lambda^0,0}
    \right\|_\infty
    \le
    \frac{C_{\mathrm{anc}}}{2}\left(
    \frac{1}{1-\lambda}
    \sqrt{\frac{\widehat H_\lambda^0\iota_N}{N_{\mathrm{nom}}}}
    +\frac{\iota_N}{N_{\mathrm{nom}}(1-\lambda)^2}
    +\frac{\varepsilon_{\mathrm{opt}}}{1-\lambda}
    \right).
\end{align*}
Indeed, the solver error transfers the exact empirical bounds in
Lemma~\ref{lem:fixed-discount-nominal-plugin} to the returned value, while
\[
    \|\widehat V_\lambda^{\star,0}\|_{\mathrm{span}}+1
    \le
    2\widehat H_\lambda^0
    +\frac{2\varepsilon_{\mathrm{opt}}}{1-\lambda};
\]
the resulting mixed square-root terms are absorbed by Young's inequality.
The first display is exactly the baseline certificate after substituting
\eqref{eq:anchored-rho-lambda-plus}. It also implies the level certificate.
The second display and the span triangle inequality give
\[
    H_\lambda^0
    \le
    \widehat H_\lambda^0
    +2\left\|
    V_\lambda^{\star,0}
    -\widehat V_\lambda^{\widehat\pi_\lambda^0,0}
    \right\|_\infty
    \le
    \widehat H_\lambda^+.
\]
This proves \eqref{eq:nominal-anchor-certificate-validity}.

We next construct an oracle certificate. When
$\varepsilon_{\mathrm{opt}}=0$, interpret
$H_0/\varepsilon_{\mathrm{opt}}=+\infty$. The sample condition and the
dyadic grid give $\lambda^\dagger\in\Lambda_N$ such that
\[
    \frac{1}{1-\lambda^\dagger}
    \asymp
    \min\left\{
    \sqrt{\frac{N_{\mathrm{nom}}H_0}{\iota_N}},
    \frac{H_0}{\varepsilon_{\mathrm{opt}}}
    \right\}.
\]
Consequently,
\begin{align*}
    (1-\lambda^\dagger)H_0
    &\le
    C\left(
    \sqrt{\frac{H_0\iota_N}{N_{\mathrm{nom}}}}
    +\varepsilon_{\mathrm{opt}}
    \right),                                                        \\
    \frac{1}{1-\lambda^\dagger}
    \sqrt{\frac{H_0\iota_N}{N_{\mathrm{nom}}}}
    +\frac{\iota_N}{N_{\mathrm{nom}}(1-\lambda^\dagger)^2}
    +\frac{\varepsilon_{\mathrm{opt}}}{1-\lambda^\dagger}
    &\le CH_0,                                                       \\
    \frac{\iota_N}{N_{\mathrm{nom}}(1-\lambda^\dagger)}
    &\le
    C\sqrt{\frac{H_0\iota_N}{N_{\mathrm{nom}}}}.
\end{align*}
Lemma~\ref{lem:robust-value-function-span-bound}, applied with $\sigma=0$,
gives $\|V_{\lambda^\dagger}^{\star,0}\|_{\mathrm{span}}\le2H_0$.
Substituting the preceding bounds into the value comparison above and
absorbing $C\sqrt{H_0\widehat H_{\lambda^\dagger}^0}$ yields
$\widehat H_{\lambda^\dagger}^0\le CH_0$, and hence
$\widehat H_{\lambda^\dagger}^+\le CH_0$. Moreover,
Lemma~\ref{lem:robust-discounted-average-sandwich} gives
\[
    \bar\rho_{\lambda^\dagger}^0
    \le
    \rho^\star+C(1-\lambda^\dagger)H_0.
\]
The baseline comparison and the last three displayed bounds therefore imply
\[
    \widehat\rho_{\lambda^\dagger}^+
    \le
    \rho^\star
    +C\sqrt{\frac{H_0\iota_N}{N_{\mathrm{nom}}}}
    +C\varepsilon_{\mathrm{opt}},
\]
which proves the oracle assertion.

It remains to apply the selection rule. For every $\lambda\in\Lambda_N$,
discounted optimality, Lemma~\ref{lem:robust-discounted-average-sandwich}, and
the three certificate inequalities give
\[
    \rho^\star
    \le
    \bar\rho_\lambda^0
    \le
    \widehat\rho_\lambda^+,
    \qquad
    0\le
    \widehat\rho_\lambda^+-\rho^\star
    \le
    \widehat Q_\lambda.
\]
At $\lambda^\dagger$, the oracle bounds and
\eqref{eq:anchored-anchor-quality} show that
\[
    \widehat Q_{\lambda^\dagger}
    \le
    C\left(
    \sqrt{\frac{H_0\iota_N}{N_{\mathrm{nom}}}}
    +\varepsilon_{\mathrm{opt}}
    \right)
    \le\varepsilon
\]
after enlarging the sample constant and reducing $c_{\mathrm{opt}}$.
Thus $\lambda^\dagger$ is admissible in
\eqref{eq:anchored-selected-anchor}. Minimality of the selected span and its
$\varepsilon$-approximate quality give
\[
    \widehat H_{\widehat\lambda_{\mathrm{anc}}}^+
    \le CH_0,
    \qquad
    \widehat Q_{\widehat\lambda_{\mathrm{anc}}}
    \le C\varepsilon.
\]
The preceding level and quality inequalities, evaluated at
$\widehat\lambda_{\mathrm{anc}}$, now give
\eqref{eq:selected-anchor-level-bound} and complete the proof.

\subsubsection{Proof of Lemma~\ref{lem:raw-localized-robust-comparison}}
\label{subsec:proof-raw-localized-robust-comparison}

Consider an arbitrary realization of the policy $\pi$, the anchor, and the
radius grid. Since all three objects are constructed independently of
$\mathcal D_{\mathrm{rob}}$, fixing their realization does not change the
distribution of the robust-batch samples. We first check the logarithmic and
sample-size requirements for the two comparison bounds.
Since $N_{\mathrm{rob}}\le N$ and $(1-\gamma)^{-1}\le N$ for
$\gamma\in\Gamma_N$, replacing their failure probability by $\delta/N^5$
changes the logarithmic factor to at most
\[
    \log\left(
    \frac{54SA N_{\mathrm{rob}}^2}
    {(1-\gamma)(\delta/N^5)}
    \right)
    \le
    \log\left(\frac{54SA N^8}{\delta}\right)
    \le
    C\cdot\iota_N.
\]
The policy-grid cutoff \eqref{eq:anchored-Gamma-N} also gives
\[
    N_{\mathrm{rob}}
    \ge
    C_\Gamma^2
    \frac{\sigma\cdot\iota_N}{(1-\gamma)^2},
    \qquad
    N_{\mathrm{rob}}
    \ge
    C_\Gamma
    \frac{\iota_N}{1-\gamma}.
\]
Since $C_\Gamma$ is sufficiently large, the two conditions in
\eqref{eq:recursive-comparison-sample-conditions} hold with
$N_{\mathrm{rob}}$ in place of $N$ and with the logarithmic factor associated
with failure probability $\delta/N^5$ in place of $\iota$.

Now fix a radius $R_j$ that upper-bounds
$1$, $\|\bar h\|_{\mathrm{span}}$,
$\|V_\gamma^{\pi,\sigma}\|_{\mathrm{span}}$, and
$\|\widehat V_\gamma^{\pi,\sigma}\|_{\mathrm{span}}$.
By \eqref{eq:anchor-span-scale},
$H_{\mathrm{anc}}=\max\{1,\|\bar h\|_{\mathrm{span}}\}$, so $R_j$ satisfies
the radius requirement for the fixed-policy bound
\eqref{eq:generic-fixed-policy-localized-comparison} from
Lemma~\ref{lem:V-hat-V-part-1-decomp}. Hence that bound holds for $\pi$ with
failure probability $O(\delta/N^5)$.

Suppose, in addition, that $\pi=\pi_\gamma^\star$ and that $R_j$ upper-bounds
$\|V_\gamma^{\widehat\pi_\gamma,\sigma}\|_{\mathrm{span}}$,
$\|\widehat V_\gamma^{\widehat\pi_\gamma,\sigma}\|_{\mathrm{span}}$, and
$\|\widehat V_\gamma^{\star,\sigma}\|_{\mathrm{span}}$. Then $\pi$ is the
fixed optimal policy used in the definitions of $U$ and $\widehat U$ under
which Lemma~\ref{lem:V-hat-V-pi-hat-bound} is stated. Together with the
fixed-policy radius condition, the displayed bounds ensure that all five spans
in \eqref{eq:discounted-localized-span-condition} are at most $R_j$.
Consequently, the learned-policy bound from
Lemma~\ref{lem:V-hat-V-pi-hat-bound} holds directly at this radius with failure
probability $O(\delta/N^5)$.

There are two comparison bounds at each of the $J+1=O(\log N)$ radii. A union
bound gives a total conditional failure probability of at most
\[
    C(J+1)\frac{\delta}{N^5}
    \le
    C\frac{\delta\log N}{N^5}
    =
    O\left(\frac{\delta}{N^3}\right).
\]
This estimate is uniform over the fixed policy, anchor, and radius grid.
Averaging over their distribution gives the same unconditional probability
bound and completes the proof.

\subsubsection{Proof of Lemma~\ref{lem:nominal-policy-candidate-guarantees}}
\label{subsec:proof-nominal-policy-candidate-guarantees}
\label{subsec:proof-nominal-policy-suitable-grid-point}
\label{subsec:proof-observable-fixed-policy-robust-evaluation}
\label{subsec:proof-existence-good-nominal-policy-candidate}

We first establish the simultaneous evaluation guarantee. Fix
$\gamma\in\Gamma_N$ and condition on the nominal batch. Then
$\widehat\pi_\gamma^0$ is independent of $\mathcal D_{\mathrm{rob}}$.
Write
\[
    D_\gamma
    =
    \left\|
    \widehat V_\gamma^{\widehat\pi_\gamma^0,\sigma}
    -V_\gamma^{\widehat\pi_\gamma^0,\sigma}
    \right\|_\infty.
\]
Apply Lemma~\ref{lem:raw-localized-robust-comparison} with the trivial anchor
$(\bar\rho,\bar h)=(1,0)$ and a dyadic radius grid starting at $1$ and
extending beyond $8/(1-\gamma)$. Let $R_\gamma$ be its first point satisfying
\[
    R_\gamma
    \ge
    \widehat R_\gamma^{\mathrm{nom}}+2D_\gamma.
\]
The span triangle inequality makes this a valid radius, and minimality gives
$R_\gamma\le2(\widehat R_\gamma^{\mathrm{nom}}+2D_\gamma)$. Moreover, the
1-Lipschitz property of $[\,\cdot\,]_+$ gives
\[
    \left[
    1-(1-\gamma)\min_s
    V_\gamma^{\widehat\pi_\gamma^0,\sigma}(s)
    \right]_+
    \le
    \widehat\beta_\gamma^{\mathrm{nom}}+(1-\gamma)D_\gamma.
\]
Substitution in the localized budget therefore yields
\[
    B(R_\gamma)
    \le
    C\left[
    \widehat B_\gamma^{\mathrm{nom}}
    +\widehat\beta_\gamma^{\mathrm{nom}}D_\gamma
    +(\sigma+1-\gamma)D_\gamma^2
    \right].
\]
After multiplying the localized comparison by $1-\gamma$, Young's inequality
and the two cutoffs in \eqref{eq:anchored-Gamma-N} absorb all terms involving
$D_\gamma$ on the right. Thus
\begin{equation*}
    (1-\gamma)D_\gamma
    \le
    \operatorname{pen}_\gamma^{\mathrm{nom}}
    (\widehat\pi_\gamma^0),
\end{equation*}
where the inequality follows by choosing $C_{\mathrm{pen}}$
sufficiently large. The fixed-policy discounted-to-average comparison then
gives
\begin{align*}
    \rho^{\widehat\pi_\gamma^0,\sigma}
    &\ge
    (1-\gamma)\min_s
    V_\gamma^{\widehat\pi_\gamma^0,\sigma}(s)                         \\
    &\ge
    (1-\gamma)\min_s
    \widehat V_\gamma^{\widehat\pi_\gamma^0,\sigma}(s)
    -\operatorname{pen}_\gamma^{\mathrm{nom}}
    (\widehat\pi_\gamma^0)
    =
    \operatorname{LCB}_\gamma^{\mathrm{nom}}
    (\widehat\pi_\gamma^0).
\end{align*}
A union bound over $|\Gamma_N|=O(\log N)$ proves the simultaneous assertions
\eqref{eq:observable-fixed-policy-robust-evaluation} and
\eqref{eq:nominal-lower-certificate-validity}.

We now prove the existence assertion. Under
\eqref{eq:nominal-policy-grid-sample-condition}, the two cutoffs defining
$\Gamma_N$ exceed $2C_{\mathrm{nom}}H_0/\varepsilon$. Hence the dyadic grid
contains $\gamma_{\mathrm{nom}}^\dagger$ satisfying
\eqref{eq:nominal-grid-horizon}. For brevity write
$\gamma^\dagger=\gamma_{\mathrm{nom}}^\dagger$ and
$\widehat\pi^\dagger=\widehat\pi_{\gamma^\dagger}^0$.
Lemma~\ref{lem:robust-value-function-span-bound}, with $\sigma=0$, and
Lemma~\ref{lem:fixed-discount-nominal-plugin} give
\[
    \|V_{\gamma^\dagger}^{\star,0}\|_{\mathrm{span}}\le2H_0,
    \qquad
    \left\|
    V_{\gamma^\dagger}^{\star,0}
    -V_{\gamma^\dagger}^{\widehat\pi^\dagger,0}
    \right\|_\infty
    \le CH_0.
\]
The second bound follows by inserting the horizon and sample conditions into
the population and fixed-policy comparisons of that lemma and adding the
solver error. The span triangle inequality and
Lemma~\ref{lem:robust-discounted-average-sandwich} now give, after choosing
$C_{\mathrm{nom}}$ sufficiently large,
\begin{align*}
    \left\|V_{\gamma^\dagger}^{\widehat\pi^\dagger,0}
    \right\|_{\mathrm{span}}
    &\le CH_0,                                                        \\
    (1-\gamma^\dagger)\min_s
    V_{\gamma^\dagger}^{\widehat\pi^\dagger,0}(s)
    &\ge \rho^\star-c\varepsilon.
\end{align*}
These are \eqref{eq:nominal-grid-policy-span} and
\eqref{eq:nominal-grid-policy-discounted-reward}.

The total-variation constraint, translation equivariance, and monotonicity of
the fixed-policy Bellman operator give
\[
    V_{\gamma^\dagger}^{\widehat\pi^\dagger,0}
    -\frac{\gamma^\dagger\sigma}{1-\gamma^\dagger}
    \left\|V_{\gamma^\dagger}^{\widehat\pi^\dagger,0}
    \right\|_{\mathrm{span}}\bm 1_S
    \le
    V_{\gamma^\dagger}^{\widehat\pi^\dagger,\sigma}
    \le
    V_{\gamma^\dagger}^{\widehat\pi^\dagger,0}.
\]
Together with $\sigma H_0\le c\varepsilon$ and
$(1-\gamma^\dagger)^{-1}\asymp H_0/\varepsilon$, this implies
\[
    \left\|V_{\gamma^\dagger}^{\widehat\pi^\dagger,\sigma}
    \right\|_{\mathrm{span}}
    \le CH_0.
\]
The lower comparison also gives
\[
\begin{aligned}
    (1-\gamma^\dagger)\min_s
    V_{\gamma^\dagger}^{\widehat\pi^\dagger,\sigma}(s)
    &\ge
    (1-\gamma^\dagger)\min_s
    V_{\gamma^\dagger}^{\widehat\pi^\dagger,0}(s)
    -\gamma^\dagger\sigma
    \left\|V_{\gamma^\dagger}^{\widehat\pi^\dagger,0}
    \right\|_{\mathrm{span}}\\
    &\ge \rho^\star-c\varepsilon-C\sigma H_0
    \ge \rho^{\star,\sigma}-c\varepsilon,
\end{aligned}
\]
where the final inequality uses $P^0\in\mathcal P$ and absorbs
$C\sigma H_0$ into the displayed $c\varepsilon$ loss.

Finally, set
$D^\dagger=\|\widehat V_{\gamma^\dagger}^{
\widehat\pi^\dagger,\sigma}-V_{\gamma^\dagger}^{
\widehat\pi^\dagger,\sigma}\|_\infty$. The span bound just proved and the
definitions of the observable radius, defect, and budget give
\begin{align*}
    \widehat R_{\gamma^\dagger}^{\mathrm{nom}}
    &\le CH_0+2D^\dagger,
    &
    \widehat\beta_{\gamma^\dagger}^{\mathrm{nom}}
    &\le1,                                                           \\
    \widehat B_{\gamma^\dagger}^{\mathrm{nom}}
    &\le C\left[H_0+D^\dagger
    +(\sigma+1-\gamma^\dagger)(D^\dagger)^2\right].
\end{align*}
Substituting these inequalities into the simultaneous evaluation bound and
using Young's inequality and \eqref{eq:anchored-Gamma-N} once more gives
\[
    (1-\gamma^\dagger)D^\dagger
    \le
    C\left(
    \sqrt{\frac{H_0\iota_N}{N_{\mathrm{rob}}}}
    +\frac{H_0\iota_N}{N_{\mathrm{rob}}}
    +\frac{\iota_N}{N_{\mathrm{rob}}(1-\gamma^\dagger)}
    \right)
    \le c\varepsilon.
\]
It follows in turn that
$\widehat R_{\gamma^\dagger}^{\mathrm{nom}}\le CH_0$,
$\widehat B_{\gamma^\dagger}^{\mathrm{nom}}\le CH_0$, and
\[
    \operatorname{pen}_{\gamma^\dagger}^{\mathrm{nom}}
    (\widehat\pi^\dagger)
    \le c\varepsilon.
\]
Combining the last two displays with the true robust discounted-reward bound
in the definition \mainAnchoredNominalCandidateScore{} proves
\eqref{eq:good-nominal-policy-certificate} and completes the proof.

\subsubsection{Proof of Lemma~\ref{lem:robust-policy-candidate-guarantees}}
\label{subsec:proof-robust-policy-candidate-guarantees}
\label{subsec:proof-observable-budget-anchored-dmdp-comparison}
\label{subsec:proof-anchored-lower-certificate-validity}
\label{subsec:proof-existence-good-anchored-pair}

We first prove both observable comparisons through one anchor-indexed
argument. Fix $\gamma\in\Gamma_N$ and condition on the nominal batch. Let
$a\in\{\mathrm{anc},\mathrm{triv}\}$ index the selected and trivial anchors,
and define
\[
\begin{aligned}
 a=\mathrm{anc}:\quad
 (H_a,\widehat R_{\gamma,a})
 &=\left(
 \widehat H_{\widehat\lambda_{\mathrm{anc}}}^+,
 \widehat R_{\gamma,\widehat\lambda_{\mathrm{anc}}}^{\mathrm{anc}}
 \right), \\
 (\widehat\beta_{\gamma,a},\widehat B_{\gamma,a})
 &=\left(
 \widehat\beta_{\gamma,\widehat\lambda_{\mathrm{anc}}},
 \widehat B_{\gamma,\widehat\lambda_{\mathrm{anc}}}
 \right); \\[2pt]
 a=\mathrm{triv}:\quad
 (H_a,\widehat R_{\gamma,a})
 &=\left(1,\widehat R_\gamma\right), \\
 (\widehat\beta_{\gamma,a},\widehat B_{\gamma,a})
 &=\left(
 \widehat\beta_{\gamma,\mathrm{triv}},
 \widehat B_{\gamma,\mathrm{triv}}
 \right).
\end{aligned}
\]
For $a=\mathrm{anc}$, use the population anchor
$(\widehat\rho_{\widehat\lambda_{\mathrm{anc}}}^+,
h_{\widehat\lambda_{\mathrm{anc}}}^0)$. It is valid by
Lemmas~\ref{lem:discounted-nominal-anchors} and
\ref{lem:anchor-calibration-guarantees}, and its span is at most $H_a$.
For $a=\mathrm{triv}$, use $(1,0)$, which is valid because $r\in[0,1]$.

For either $a$, take a dyadic grid starting at $H_a$ and let
$R_{\Delta,a}$ be the first point satisfying
\[
    R_{\Delta,a}
    \ge
    8\bigl(
    \widehat R_{\gamma,a}+\Delta_\gamma
    +\varepsilon_{\mathrm{opt}}
    \bigr).
\]
The discounted-value range ensures that this point exists on an
$O(\log N)$ grid, and minimality gives the reverse bound with the factor
$16$. The span triangle inequality, the solver guarantee, and
\eqref{eq:anchored-Delta-gamma} show that $R_{\Delta,a}$ dominates all value
spans required in Lemma~\ref{lem:raw-localized-robust-comparison}. The true
anchor defect also satisfies
\[
    \beta_{\star,a}
    \le
    \widehat\beta_{\gamma,a}+(1-\gamma)\Delta_\gamma.
\]
Substituting these two comparisons in the generic localized budgets gives,
for both anchors,
\[
    B_{\star,a}(R_{\Delta,a})+B_{\mathrm{com},a}(R_{\Delta,a})
    \le
    C\left[
    \widehat B_{\gamma,a}
    +\widehat\beta_{\gamma,a}\Delta_\gamma
    +(\sigma+1-\gamma)\Delta_\gamma^2
    \right].
\]
We may choose $R_{\Delta,a}$ after observing the robust batch because the
event in Lemma~\ref{lem:raw-localized-robust-comparison} is simultaneous over
the radius grid. Summing its fixed-policy and learned-policy comparisons and
multiplying by $1-\gamma$ therefore yields
\begin{align*}
    (1-\gamma)\Delta_\gamma
    &\le
    C\sqrt{\frac{\widehat B_{\gamma,a}\iota_N}{N_{\mathrm{rob}}}}
    +C\frac{\widehat R_{\gamma,a}\iota_N}{N_{\mathrm{rob}}}
    +C\varepsilon_{\mathrm{opt}}                                      \\
    &\quad
    +C\sqrt{\frac{\widehat\beta_{\gamma,a}\Delta_\gamma\iota_N}
    {N_{\mathrm{rob}}}}
    +C\Delta_\gamma
    \sqrt{\frac{(\sigma+1-\gamma)\iota_N}{N_{\mathrm{rob}}}}
    +C\frac{\Delta_\gamma\iota_N}{N_{\mathrm{rob}}}
    +\frac{7}{100}(1-\gamma)\Delta_\gamma.
\end{align*}
Young's inequality absorbs the first mixed term at the cost of
$C\widehat\beta_{\gamma,a}\iota_N/
[N_{\mathrm{rob}}(1-\gamma)]$. The two cutoffs in
\eqref{eq:anchored-Gamma-N} absorb the remaining copies of
$\Delta_\gamma$. Thus, simultaneously for both $a$ and every
$\gamma\in\Gamma_N$,
\[
    (1-\gamma)\Delta_\gamma
    \le
    C\sqrt{\frac{\widehat B_{\gamma,a}\iota_N}{N_{\mathrm{rob}}}}
    +C\frac{\widehat R_{\gamma,a}\iota_N}{N_{\mathrm{rob}}}
    +C\frac{\widehat\beta_{\gamma,a}\iota_N}
    {N_{\mathrm{rob}}(1-\gamma)}
    +C\varepsilon_{\mathrm{opt}}.
\]
Taking $a=\mathrm{anc}$ and $a=\mathrm{triv}$ proves
\eqref{eq:observable-budget-selected-anchor-bound} and
\eqref{eq:observable-budget-trivial-anchor-bound}, respectively.

Choose $C_{\mathrm{pen}}$ larger than the preceding universal constant. Both
penalties then dominate $(1-\gamma)\Delta_\gamma$. Hence
Lemma~\ref{lem:robust-discounted-average-sandwich} gives
\begin{align*}
    \rho^{\widehat\pi_\gamma,\sigma}
    &\ge
    (1-\gamma)\min_s
    \widehat V_\gamma^{\widehat\pi_\gamma,\sigma}(s)
    -(1-\gamma)\Delta_\gamma                                        \\
    &\ge
    (1-\gamma)\min_s
    \widehat V_\gamma^{\widehat\pi_\gamma,\sigma}(s)
    -\min\left\{
    \operatorname{pen}_\gamma^{\mathrm{anc}}(\widehat\pi_\gamma),
    \operatorname{pen}_\gamma^{\mathrm{triv}}(\widehat\pi_\gamma)
    \right\}                                                         \\
    &=
    \operatorname{LCB}_\gamma^{\mathrm{rob}}(\widehat\pi_\gamma).
\end{align*}

It remains to construct a good candidate. The sample condition makes both
cutoffs defining $\Gamma_N$ at least
$2C_{\mathrm{rob}}H_\sigma/\varepsilon$, so the grid contains
$\gamma^\dagger$ with
\[
    \frac{C_{\mathrm{rob}}H_\sigma}{\varepsilon}
    \le
    \frac{1}{1-\gamma^\dagger}
    \le
    \frac{2C_{\mathrm{rob}}H_\sigma}{\varepsilon}.
\]
Use the selected anchor when $H_0<H_\sigma$ and the trivial anchor when
$H_\sigma\le H_0$. In the first case,
Lemma~\ref{lem:anchor-calibration-guarantees} gives
\[
    H_{\mathrm{anc}}\le CH_0,
    \qquad
    \beta_\star\le C(\sigma H_0+\varepsilon),
\]
whereas in the second case $H_{\mathrm{anc}}=1$ and
$\beta_\star\le1$. In both cases, with $R_0=2H_\sigma$,
\[
    H_{\mathrm{anc}}+R_0\beta_\star
    +\sigma R_0^2+(1-\gamma^\dagger)R_0^2
    \le
    C\left[
    \min\{H_0,H_\sigma\}
    +\sigma H_\sigma^2+\varepsilon H_\sigma
    \right].
\]
The sample and optimization conditions therefore verify the hypotheses of
Theorem~\ref{thm:TV-upper-bound-discount}. Together with
Lemma~\ref{lem:robust-value-function-span-bound}, that theorem gives, on an
event of probability at least $1-O(\delta)$,
\[
    \left\|
    \widehat V_{\gamma^\dagger}^{\widehat\pi_{\gamma^\dagger},\sigma}
    -V_{\gamma^\dagger}^{\widehat\pi_{\gamma^\dagger},\sigma}
    \right\|_\infty
    +\left\|
    V_{\gamma^\dagger}^{\star,\sigma}
    -V_{\gamma^\dagger}^{\widehat\pi_{\gamma^\dagger},\sigma}
    \right\|_\infty
    \le cH_\sigma.
\]
The robust discounted-to-average comparison and the choice of
$\gamma^\dagger$ now imply
\[
    (1-\gamma^\dagger)\min_s
    \widehat V_{\gamma^\dagger}^{
    \widehat\pi_{\gamma^\dagger},\sigma}(s)
    \ge
    \rho^{\star,\sigma}-\frac{c\varepsilon}{2}.
\]
The same error bound and the span triangle inequality give the following
observable bounds. In the selected-anchor case,
\begin{align*}
    \widehat R_{\gamma^\dagger,
    \widehat\lambda_{\mathrm{anc}}}^{\mathrm{anc}}
    &\le CH_\sigma,
    &
    \widehat\beta_{\gamma^\dagger,
    \widehat\lambda_{\mathrm{anc}}}
    &\le C(\sigma H_0+\varepsilon),                                  \\
    \widehat B_{\gamma^\dagger,
    \widehat\lambda_{\mathrm{anc}}}
    &\le C(H_0+\sigma H_\sigma^2+\varepsilon H_\sigma),
\end{align*}
while in the trivial-anchor case,
\[
    \widehat R_{\gamma^\dagger}\le CH_\sigma,
    \qquad
    \widehat\beta_{\gamma^\dagger,\mathrm{triv}}\le1,
    \qquad
    \widehat B_{\gamma^\dagger,\mathrm{triv}}
    \le C(H_\sigma+\sigma H_\sigma^2).
\]
Substitution in the corresponding penalty, using
\eqref{eq:good-robust-candidate-sample-condition}, the constant-fraction
sample split, and
$\varepsilon_{\mathrm{opt}}\le c_{\mathrm{opt}}\varepsilon$, gives
\[
    \min\left\{
    \operatorname{pen}_{\gamma^\dagger}^{\mathrm{anc}}
    (\widehat\pi_{\gamma^\dagger}),
    \operatorname{pen}_{\gamma^\dagger}^{\mathrm{triv}}
    (\widehat\pi_{\gamma^\dagger})
    \right\}
    \le\frac{c\varepsilon}{2}.
\]
Combining the last two displays in
\mainAnchoredRobustCandidateScore{} proves
\eqref{eq:good-robust-policy-certificate} and completes the proof.

\section{Additional experiments and simulation details}
\label{app:experiment-details}

This appendix provides supplementary results and implementation details for the
experiments in Section~\mainExperimentsSection{}.
Section~\ref{app:rate-experiment-details} describes the instances used to test
the sample-complexity rates and reports an additional minimum-span experiment.
Section~\ref{app:selector-experiment-details} describes the span-agnostic
adaptation experiment.
The experiment repository documents the exact sample-size grids and
trial allocations, the code and data organization, and the reproduction
instructions.

For every instance, we compute the nominal and robust average rewards and
bias functions from the underlying AMDP to a tolerance of $10^{-10}$. The
resulting bias spans are used in all sample-size calculations,
normalizations, and rate fits. For readability, the figure labels report
these spans rounded to the nearest integer.

At each state-action pair, the learner receives $N$ independent next-state
samples. As in the main text, $N_{95}$ denotes the sample size at which the
estimated probability of returning a robustly $\varepsilon$-optimal policy
reaches $0.95$. We estimate this crossing by fitting a success curve constrained
to be nondecreasing in $N$. For the rate experiments, the $95\%$ intervals for
$N_{95}$ use $500$ parametric bootstrap repetitions. In each repetition, we
draw new success counts from the observed success rate at every tested sample
size, refit the curve, and recalculate its crossing.

\subsection{Sample-complexity experiments}
\label{app:rate-experiment-details}

Section~\mainExperimentsSection{} tests the high-tolerance rate
\[
    \frac{\min\{H_0,H_\sigma\}}{\varepsilon^2}
\]
and separately tests the two components of the low-tolerance rate,
\[
    \frac{\min\{H_0,H_\sigma\}}{\varepsilon^2}
    \qquad\text{and}\qquad
    \frac{\sigma H_\sigma^2}{\varepsilon^2}.
\]
Each tested sample size uses at least $1000$ independent trials, increased
to $2000$ near the estimated $N_{95}$ crossing.

\paragraph*{Instance constructions.}
For the high-tolerance experiment in
Figure~\mainComponentRatePanel{a}, we use a three-state AMDP consisting of
a decision state and two rewarding states. The rewarding states have reward
$1$ and return to the decision state with probability $1/(2H_0)$. Of the two
informative actions, the better action enters its rewarding state with
probability $1/(2H_0)$ and has TV radius $\sigma$ at the decision state. The
other action has radius zero and is calibrated so that its robust average
reward is $1.5\varepsilon$ below that of the better action. A third action
self-loops at the decision state and has reward zero. This gives
$H_0$ exactly and
\[
    H_\sigma=\frac{H_0}{1-\sigma H_0}>H_0.
\]
We fix $\varepsilon=0.02$ and $\sigma=2.5\times10^{-5}$ and vary
\[
    H_0\in\{8,10,12,16,20,28,40,56,80\}.
\]
The nominal and robust optimal actions agree in every setting, and
$7\sigma H_0/\varepsilon$ is at most $0.7$, so all settings lie within the
formal high-tolerance condition.

The low-tolerance minimum-span experiment in
Figure~\mainComponentRatePanel{b} uses the four-state layout shown in
Figure~\ref{fig:minimum-span-instance}. This is the core of the lower-bound
instance in Figure~\ref{fig:hard-mdp-diagram}: three actions at a decision
state lead to three different rewarding states. The padding states used to
extend the lower-bound construction to larger state spaces are unnecessary
for this comparison and are omitted. Unlike the lower-bound proof, where the
instances are constructed in pairs, each simulation uses a single instance
with a fixed optimal action. We use $\varepsilon = 0.001$ for this family.

\tikzset{
    app instance diagram/.style={
        ->,
        >=stealth,
        thick,
        every node/.style={font=\small}
    },
    app state/.style={circle, draw, minimum size=10mm},
    app uncertain state/.style={app state, densely dashed},
    app uncertain action/.style={densely dashed},
    app edge label/.style={fill=white, inner sep=1pt}
}

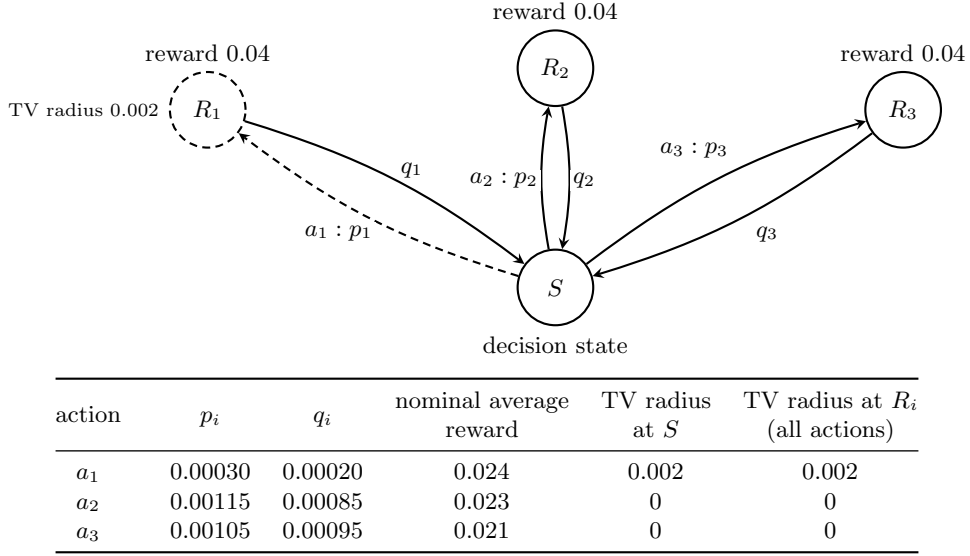
\begin{figure}
    \centering
    \begin{tikzpicture}[app instance diagram]
        \node[app state, label=below:{decision state}]
            (decision) at (0,0) {$S$};
        \node[app uncertain state, label=above:{reward $0.04$},
            label={[font=\scriptsize]left:{TV radius $0.002$}}]
            (one) at (-4.6,2.35) {$R_1$};
        \node[app state, label=above:{reward $0.04$}]
            (two) at (0,2.9) {$R_2$};
        \node[app state, label=above:{reward $0.04$}]
            (three) at (4.6,2.35) {$R_3$};

        \draw[app uncertain action, bend left=10] (decision) to
            node[app edge label, pos=0.48, below left]
            {$a_1:p_1$} (one);
        \draw[bend left=10] (one) to
            node[app edge label, pos=0.52, above right]
            {$q_1$} (decision);
        \draw[bend left=10] (decision) to
            node[app edge label, pos=0.48, left]
            {$a_2:p_2$} (two);
        \draw[bend left=10] (two) to
            node[app edge label, pos=0.52, right]
            {$q_2$} (decision);
        \draw[bend left=10] (decision) to
            node[app edge label, pos=0.55, above left, yshift=4pt]
            {$a_3:p_3$} (three);
        \draw[bend left=10] (three) to
            node[app edge label, pos=0.45, below right, yshift=-4pt]
            {$q_3$} (decision);
    \end{tikzpicture}

    \vspace{0.5em}
    \small
    \begin{tabular}{@{}c@{\qquad}ccccc@{}}
        \toprule
        action & $p_i$ & $q_i$ &
        \makecell{nominal average\\reward} &
        \makecell{TV radius\\at $S$} &
        \makecell{TV radius at $R_i$\\(all actions)} \\
        \midrule
        $a_1$ & $0.00030$ & $0.00020$ & $0.024$
            & $0.002$ & $0.002$ \\
        $a_2$ & $0.00115$ & $0.00085$ & $0.023$
            & $0$ & $0$ \\
        $a_3$ & $0.00105$ & $0.00095$ & $0.021$
            & $0$ & $0$ \\
        \bottomrule
    \end{tabular}
    \caption{Representative minimum-span instance for
        $H_\sigma\le H_0$, at $\varepsilon=0.001$ and $\sigma=0.002$.
        Action $a_i$ enters $R_i$ with probability $p_i$, and every action at
        $R_i$ returns to $S$ with probability $q_i$; remaining probability is
        assigned to a self-loop. Dashed elements mark positive TV radius. The
        spans are $H_0=80$ and $H_\sigma=20$.}
    \label{fig:minimum-span-instance}
\end{figure}
\FloatBarrier

For the other settings with $H_\sigma\le H_0$, $p_2$ and $q_2$ are
$0.023/H_\sigma$ and $0.017/H_\sigma$, while $p_3$ and $q_3$ are
$0.021/H_\sigma$ and $0.019/H_\sigma$; all other displayed quantities remain
fixed.

To study the minimum-span term $\min\{H_0,H_\sigma\}$ when $H_\sigma<H_0$, we fix $H_0=80$ and vary
\[
    H_\sigma
    \in
    \{8,10,12,14,16,20,24,28,32,36,40\}.
\]
Action $a_1$ is nominally best, whereas uncertainty makes $a_2$ robustly
optimal. The difference between $a_2$ and $a_3$ in robust average reward is
$2\varepsilon$. Across these settings,
$\sigma H_\sigma^2$ is at most $8.1\%$ of
$\min\{H_0,H_\sigma\}=H_\sigma$. 
Therefore the expected sample-complexity dependence is dominated by $H_\sigma$.

To study the robustness-specific term $\sigma H_\sigma^2$, we instead keep $H_0$ near $5$. We
vary $H_\sigma$ at fixed $\sigma=0.1$,
\[
    H_\sigma
    \in
    \{40,45,50,60,70,80,90,100,120,140\},
\]
and vary $\sigma$ at fixed $H_\sigma=50$,
\[
    \sigma
    \in
    \{0.02,0.025,0.035,0.05,0.07,0.1,0.14,0.18,0.22\}.
\]
Across both sets of settings, $\min\{H_0,H_\sigma\}$ is at most $10\%$ of
$\sigma H_\sigma^2$. The robustness-specific component therefore determines
most of the predicted sample complexity in these instances.

\paragraph*{Normalized success curves.}
The low-tolerance panels of Figure~\mainComponentRateFigure{} summarize
each setting by the single crossing $N_{95}$.
Figure~\ref{fig:app-component-collapses} shows the corresponding
normalized success curves. Figure~\hyperref[fig:app-component-collapses]
{\ref*{fig:app-component-collapses}a} plots success against
$N\varepsilon^2/H_\sigma$ for the minimum-span settings, while
Figure~\hyperref[fig:app-component-collapses]
{\ref*{fig:app-component-collapses}b} uses
$N\varepsilon^2/(\sigma H_\sigma^2)$ for the robustness-specific settings. The
near alignment within each panel shows that these normalizations capture the
main change in the required sample size, supporting the rate
\[\frac{\min \{H_0,H_\sigma\} + \sigma H_\sigma^2}{\varepsilon^2}\]
predicted by the theory.

\begin{figure}[H]
    \centering
    \includegraphics[width=\linewidth]{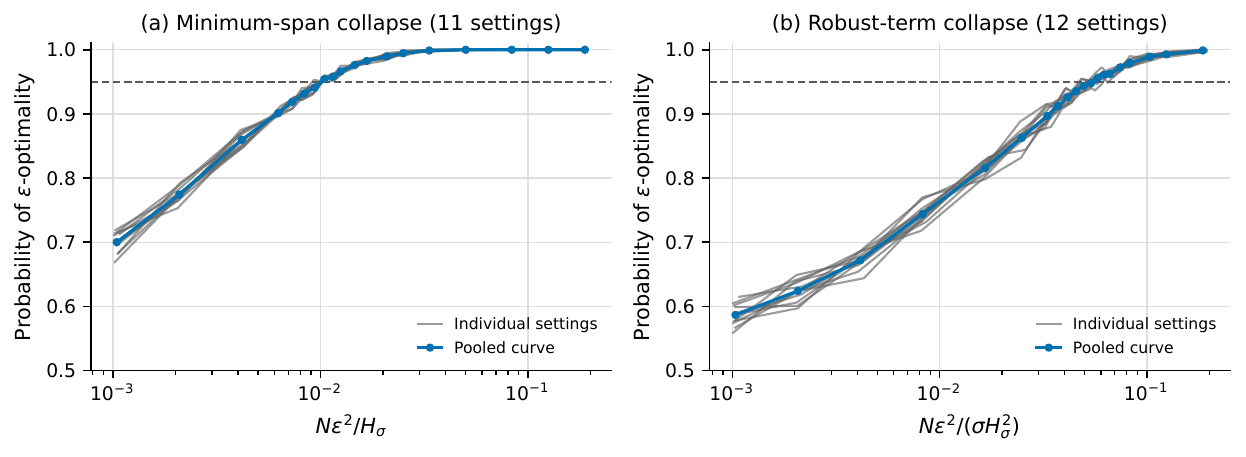}
    \caption{Normalized success curves for the two rate components.
        (a) The $11$ minimum-span settings with $H_\sigma\le H_0$, plotted
        against $N\varepsilon^2/H_\sigma$.
        (b) The $12$ robustness-specific settings, plotted against
        $N\varepsilon^2/(\sigma H_\sigma^2)$. The dashed line marks the
        $0.95$ target defining $N_{95}$.}
    \label{fig:app-component-collapses}
\end{figure}
\FloatBarrier

\phantomsection
\label{app:minimum-span-complementary}
\paragraph*{Additional experiment in the case $H_0<H_\sigma$.}
Figure~\mainComponentRatePanel{b} considers the case $H_\sigma \le H_0$, in which
$\min\{H_0,H_\sigma\}=H_\sigma$. To check this in the complementary scenario where $H_0<H_\sigma$, we fix $H_\sigma=50$ and vary
\[
    H_0\in\{20,23,26,29,32,35,38,41,44\}.
\]
This check uses a simpler three-state instance and
$\varepsilon=2\times10^{-5}$. Figure~\ref{fig:minimum-span-complementary-instance}
shows the setting with $H_0=32$.

\begin{figure}[H]
    \centering
    \begin{tikzpicture}[app instance diagram]
        \node[app state] (decision) at (0,0) {$S$};
        \node[app state, label=above:{reward $0.01$}]
            (reward) at (4.3,2.1) {$R$};
        \node[app uncertain state, label=above:{reward $0$},
            label={[font=\scriptsize]left:{TV radius $9/54400$}}]
            (transient) at (-4.3,2.1) {$T$};

        \draw[bend left=24] (decision) to
            node[app edge label, pos=0.55, above left]
            {$a_2:p_2$} (reward);
        \draw[bend right=24] (decision) to
            node[app edge label, pos=0.55, below right]
            {$a_3:p_3$} (reward);
        \draw (reward) to
            node[app edge label, pos=0.50, above right]
            {$q_R$} (decision);
        \draw (transient) to
            node[app edge label, pos=0.50, above left]
            {$q_T^0$} (decision);
        \path (decision) edge[loop below]
            node[app edge label, below] {$a_1:1$} (decision);
        \node[anchor=north east, inner sep=0pt]
            at ([xshift=-2mm,yshift=-1mm]decision.south west)
            {decision state};
    \end{tikzpicture}
    \caption{Representative complementary minimum-span instance
        ($H_0=32<H_\sigma=50$ and $\varepsilon=2\times10^{-5}$).
        Actions $a_2$ and $a_3$ enter the rewarding state $R$; the unreachable
        state $T$ changes the robust bias span without changing their
        comparison. Remaining probability is assigned to self-loops. The full
        transition table is recorded in the experiment documentation.}
    \label{fig:minimum-span-complementary-instance}
\end{figure}
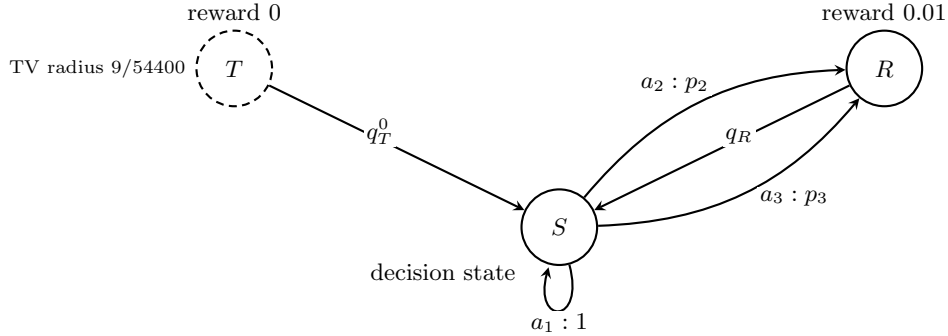
\FloatBarrier

In this setting, Figure~\hyperref[fig:app-minimum-span-complementary]{\ref*{fig:app-minimum-span-complementary}a} again shows an approximately
linear increase in $N_{95}$ as $H_0$ increases.
Figure~\hyperref[fig:app-minimum-span-complementary]
{\ref*{fig:app-minimum-span-complementary}b} gives the corresponding
curve-level check: after normalizing the sample size by $H_0 \varepsilon^{-2}$, the success-probability
transitions in different settings nearly align. Both the $H_\sigma \le H_0$
case and the $H_0<H_\sigma$ case are consistent with the predicted
$\min\{H_0,H_\sigma\}$ component in the sample complexity.

\begin{figure}[ht]
    \centering
    \includegraphics[width=.88\linewidth]{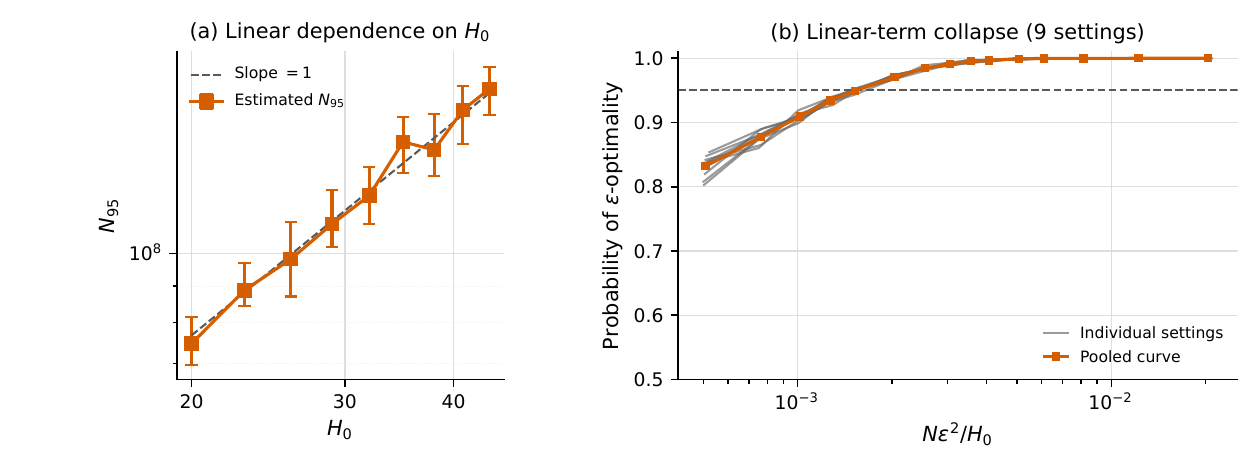}
    \caption{Complementary minimum-span experiment for nine settings
        with $H_0<H_\sigma$. (a) $N_{95}$ versus $H_0$, with the proportional
        reference $N_{95}\propto H_0$. (b) Success curves normalized by
        $N\varepsilon^2/H_0$; the dashed line marks the $0.95$ target.}
    \label{fig:app-minimum-span-complementary}
\end{figure}
\FloatBarrier

\subsection{Span-agnostic adaptation}
\label{app:selector-experiment-details}

Section~\mainSpanAgnosticAdaptationSection{} illustrates how
Algorithm~\mainAnchoredSpanAgnosticAlgorithm{} adapts its policy
family and effective horizon without knowing $H_0$ or $H_\sigma$. This
subsection describes the instance, implementation, and simulation design
behind Figure~\mainSelectorAdaptivityFigure{}.

\paragraph*{Instance for span-agnostic adaptation.}
We use a four-state instance in which action $a_i$ at the decision
state $S$ enters rewarding state $R_i$. With common scale $\tau$, the entry
probability, return probability, and reward for $a_1,a_2,a_3$ are respectively
$(0.2/\tau,0.04/\tau,0.992)$, $(2/\tau,2/\tau,1)$, and
$(0.1/\tau,0.02/\tau,0.9916)$. The TV radius is $\sigma$ at $(S,a_1)$ and
at every state-action pair in $R_1$, and is zero elsewhere; unshown
probability is assigned to a self-loop. Thus the three actions trade off
nominal value, robustness, and effective horizon.

\paragraph*{Span-agnostic algorithm implementation.}
We implement Algorithm~\mainAnchoredSpanAgnosticAlgorithm{} as stated
in Appendix~\ref{app:span-agnostic-horizon-calibration}. 
We divide the samples equally between the two policy families, setting
$N_\mathrm{nom}=N_\mathrm{rob}=N/2$.

The theorem leaves three universal numerical constants unspecified. We use
the fixed setting
\[
    C_{\mathrm{anc}}
    =
    2,
    \qquad
    C_{\Gamma}
    =
    1,
    \qquad
    C_{\mathrm{pen}}
    =
    6.
\]
This is an implementation convention, not an estimate of the best possible
constants.
The experiment documentation records the complete per-trial diagnostic
schema.

\paragraph*{Main adaptation experiment.}
We evaluate this instance and implementation in the experiment behind
Figure~\mainSelectorAdaptivityFigure{}. The experiment uses
$\varepsilon=0.002$ and nine $H_0$ settings obtained from
\[
    \tau\in\{16,20,25,32,40,50,64,80,100\}.
\]
For each setting, we test $15$ values of $\sigma H_0/\varepsilon$ between
$0.2$ and $2$. The sample budget per state-action pair is ten times
\[
    \frac{\min\{H_0,H_\sigma\}+\sigma H_\sigma^2}
    {\varepsilon^2}.
\]
We use $300$ trials at ratios between $0.7$ and $1.3$, where the
policy-family transition occurs, and $100$ trials elsewhere. Across
all $26{,}100$ trials, every numerical solver converges, every selected
policy is robustly $\varepsilon$-optimal, and every selected candidate has a
valid lower-confidence bound.

For panel~\mainSelectorAdaptivityPanel{d}, we use a separate $19$-point sample grid
on five minimum-span and five robustness-specific settings. Both methods
receive the same empirical transition counts in each trial. Each grid point
has at least $1000$ paired trials, with $2000$ near the two $N_{95}$
crossings. The uncertainty intervals use $1000$ paired bootstrap repetitions.
Each repetition resamples the joint outcome at every sample size: both methods
succeed, only the span-agnostic method succeeds, only the span-informed method
succeeds, or neither succeeds. We then recompute both crossings from the same
resampled trials. The comparison uses $270{,}000$ paired trials in total.

\end{document}